\documentclass[letterpaper, 10 pt, journal]{IEEEtran}
\usepackage{etex}
\IEEEoverridecommandlockouts  

\usepackage{comment}
\usepackage{siunitx}
\usepackage{relsize}
\usepackage{ifthen}
\usepackage[colorinlistoftodos]{todonotes}

\usepackage[caption=false]{subfig}

\usepackage[vlined,ruled,linesnumbered]{algorithm2e}
\usepackage{graphics} 
\usepackage{rotating}
\usepackage{color}
\usepackage{enumerate}
\usepackage[T1]{fontenc}
\usepackage{psfrag}
\usepackage{epsfig} 
\usepackage{booktabs}
\usepackage{graphicx,url}
\usepackage{multirow}
\usepackage{array}
\usepackage{latexsym}
\usepackage{amsfonts}
\usepackage{amsmath}
\usepackage{amssymb}
\usepackage{xstring}
\usepackage[noend]{algorithmic}
\usepackage{multirow}
\usepackage{xcolor}
\usepackage{prettyref}
\usepackage{flexisym}
\usepackage{bigdelim}
\usepackage{breqn} 
\usepackage{listings}
\let\labelindent\relax
\usepackage{enumitem}
\usepackage{xspace}
\usepackage{bm}
\graphicspath{{./figures/}}
\usepackage{tikz}
\usetikzlibrary{matrix,calc}

\usepackage{mdwlist}

\let\itemize\undefined
\makecompactlist{itemize}{stditemize}

\usepackage{amsthm}
\usepackage{multicol}
\usepackage{lipsum}  

\usepackage[bookmarks=true]{hyperref} 
 \usepackage{amsmath}
 \usepackage{nicefrac,xfrac}
\usepackage{capt-of}
\usepackage{etoolbox}
\usepackage{tikz}
\newcommand{\extraEdits}[1]{#1}

\newtheoremstyle{mystyle}
  {}
  {}
  {\itshape}
  {}
  {\bfseries}
  {.}
  { }
  {\thmname{#1}\thmnumber{ #2}\thmnote{ (#3)}}
\theoremstyle{mystyle}

\newrefformat{prob}{Problem\,\ref{#1}}
\newrefformat{def}{Definition\,\ref{#1}}
\newrefformat{sec}{Section\,\ref{#1}}
\newrefformat{sub}{Section\,\ref{#1}}
\newrefformat{prop}{Proposition\,\ref{#1}}
\newrefformat{app}{Appendix\,\ref{#1}}
\newrefformat{alg}{Algorithm\,\ref{#1}}
\newrefformat{cor}{Corollary\,\ref{#1}}
\newrefformat{thm}{Theorem\,\ref{#1}}
\newrefformat{lem}{Lemma\,\ref{#1}}
\newrefformat{fig}{Fig.\,\ref{#1}}
\newrefformat{tab}{Table\,\ref{#1}}

\newtheorem{theorem}{Theorem}

\newtheorem{corollary}[theorem]{Corollary}

\newtheorem{lemma}[theorem]{Lemma}

\newtheorem{definition}[theorem]{Definition}
\newtheorem{proposition}[theorem]{Proposition}
\newtheorem{remark}[theorem]{Remark}
\newtheorem{example}[theorem]{Example}

\newcommand{\cf}{\emph{cf.}\xspace}

\newcommand{\bdmath}{\begin{dmath}}
\newcommand{\edmath}{\end{dmath}}
\newcommand{\beq}{\begin{equation}}
\newcommand{\eeq}{\end{equation}}
\newcommand{\bdm}{\begin{displaymath}}
\newcommand{\edm}{\end{displaymath}}
\newcommand{\bea}{\begin{eqnarray}}
\newcommand{\eea}{\end{eqnarray}}
\newcommand{\beal}{\beq \begin{array}{ll}}
\newcommand{\eeal}{\end{array} \eeq}
\newcommand{\beas}{\begin{eqnarray*}}
\newcommand{\eeas}{\end{eqnarray*}}
\newcommand{\ba}{\begin{array}}
\newcommand{\ea}{\end{array}}
\newcommand{\bit}{\begin{itemize}}
\newcommand{\eit}{\end{itemize}}
\newcommand{\ben}{\begin{enumerate}}
\newcommand{\een}{\end{enumerate}}
\newcommand{\expl}[1]{&&\qquad\text{\color{gray}(#1)}\nonumber}

\newcommand{\calA}{{\cal A}}
\newcommand{\calB}{{\cal B}}
\newcommand{\calC}{{\cal C}}

\newcommand{\calE}{{\cal E}}
\newcommand{\calF}{{\cal F}}
\newcommand{\calG}{{\cal G}}
\newcommand{\calH}{{\cal H}}
\newcommand{\calI}{{\cal I}}
\newcommand{\calJ}{{\cal J}}

\newcommand{\calL}{{\cal L}}
\newcommand{\calM}{{\cal M}}
\newcommand{\calN}{{\cal N}}
\newcommand{\calO}{{\cal O}}

\newcommand{\calS}{{\cal S}}

\newcommand{\calU}{{\cal U}}
\newcommand{\calV}{{\cal V}}
\newcommand{\calW}{{\cal W}}
\newcommand{\calX}{{\cal X}}
\newcommand{\calZ}{{\cal Z}}

\newcommand{\etal}{\emph{et~al.}\xspace}
\newcommand{\setal}{~\emph{et~al.}\xspace}
\newcommand{\eg}{\emph{e.g.,}\xspace}
\newcommand{\ie}{\emph{i.e.,}\xspace}
\newcommand{\myParagraph}[1]{{\bf #1.}\xspace}

\newcommand{\M}[1]{{\bm #1}} 
\renewcommand{\boldsymbol}[1]{{\bm #1}}

\newcommand{\hide}[1]{}
\newcommand{\wrt}{w.r.t.\xspace}

\newcommand{\hiddenText}{{\color{gray} hidden text.}}
\newcommand{\hideWithText}[1]{\hiddenText}

\newcommand{\kron}{\otimes}

\newcommand{\subject}{\text{ subject to }}

\DeclareMathOperator*{\argmin}{arg\,min}

\newcommand{\normsq}[2]{\left\|#1\right\|^2_{#2}}

\newcommand{\tran}{^{\mathsf{T}}}

\newcommand{\diag}[1]{\mathrm{diag}\left(#1\right)}
\newcommand{\trace}[1]{\mathrm{tr}\left(#1\right)}

\newcommand{\rank}[1]{\mathrm{rank}\left(#1\right)}

\newcommand{\inv}{^{-1}}

\newcommand{\ones}{{\mathbf 1}}
\newcommand{\zero}{{\mathbf 0}}
\newcommand{\eye}{{\mathbf I}}

\newcommand{\Real}[1]{ { {\mathbb R}^{#1} } }

\newcommand{\setdef}[2]{ \{#1 \; {:} \; #2 \} }

\newcommand{\SOthree}{\ensuremath{\mathrm{SO}(3)}\xspace}

\newcommand{\MA}{\M{A}}
\newcommand{\MB}{\M{B}}

\newcommand{\MJ}{\M{J}}

\newcommand{\MM}{\M{M}}

\newcommand{\MP}{\M{P}}
\newcommand{\MQ}{\M{Q}}
\newcommand{\MU}{\M{U}}
\newcommand{\MR}{\M{R}}
\newcommand{\MS}{\M{S}}
\newcommand{\MI}{\M{I}}
\newcommand{\MV}{\M{V}}

\newcommand{\MH}{\M{H}}

\newcommand{\MX}{\M{X}}
\newcommand{\MY}{\M{Y}}
\newcommand{\MW}{\M{W}}
\newcommand{\MZ}{\M{Z}}
 
\newcommand{\MOmega}{\M{\Omega}}

\newcommand{\MDelta}{\M{\Delta}}
\newcommand{\MLambda}{\M{\Lambda}}

\newcommand{\va}{\boldsymbol{a}} 
 
\newcommand{\vb}{\boldsymbol{b}}

\newcommand{\ve}{\boldsymbol{e}}

\newcommand{\vo}{\boldsymbol{o}}

\newcommand{\vq}{\boldsymbol{q}}
\newcommand{\vr}{\boldsymbol{r}}

\newcommand{\vu}{\boldsymbol{u}}
\newcommand{\vv}{\boldsymbol{v}}
\newcommand{\vt}{\boldsymbol{t}}
\newcommand{\vxx}{\boldsymbol{x}}

\newcommand{\vepsilon}{\boldsymbol{\epsilon}}

\newcommand{\scenario}[1]{{\smaller \sf#1}\xspace}

\newcommand{\MOSEK}{\scenario{MOSEK}}

\newcommand{\cvx}{{\sf cvx}\xspace}

\newcommand{\blue}[1]{{\color{blue}#1}}
\newcommand{\green}[1]{{\color{green}#1}}
\newcommand{\red}[1]{{\color{red}#1}}

\newcommand{\linkToPdf}[1]{\href{#1}{\blue{(pdf)}}}
\newcommand{\linkToPpt}[1]{\href{#1}{\blue{(ppt)}}}
\newcommand{\linkToCode}[1]{\href{#1}{\blue{(code)}}}
\newcommand{\linkToWeb}[1]{\href{#1}{\blue{(web)}}}
\newcommand{\linkToVideo}[1]{\href{#1}{\blue{(video)}}}
\newcommand{\linkToMedia}[1]{\href{#1}{\blue{(media)}}}
\newcommand{\award}[1]{\xspace} 

\newcommand{\revone}[1]{#1\xspace}

\newcommand{\GNC}{\scenario{GNC}} 
\newcommand{\Eigen}{\scenario{Eigen3}}

\newcommand{\maxOutliers}{{99\%}\xspace}
\newcommand{\maxOutliersRot}{{90\%}\xspace}
\newcommand{\betaNoise}{\sigma}

\newcommand{\TWlong}{\TLS rotation estimation\xspace}

\newcommand{\psdCone}[1]{\calS^{#1}_{+}}

\newcommand{\QCQP}{QCQP\xspace}

\newcommand{\TLS}{\scenario{TLS}}

\newcommand{\TIM}{\scenario{TIM}}
\newcommand{\TIMs}{\scenario{TIMs}}
\newcommand{\TRIM}{\scenario{TRIM}}
\newcommand{\TRIMs}{\scenario{TRIMs}}

\newcommand{\bnb}{BnB\xspace}

\newcommand{\SPC}{SPC\xspace}

\newcommand{\TIMa}{\bar{\va}}
\newcommand{\TIMb}{\bar{\vb}}

\newcommand{\hats}{\hat{s}}
\newcommand{\hatMR}{\hat{\MR}}
\newcommand{\hatvt}{\hat{\vt}}
\newcommand{\hatvq}{\hat{\vq}}
\newcommand{\hattheta}{\hat{\theta}}

\newcommand{\barvxx}{\bar{\vxx}}
\newcommand{\qProd}{\circ}

\newcommand{\MZero}{\M{0}}

\newcommand{\barMQ}{\bar{\MQ}}
\newcommand{\eps}{{\epsilon}}
\newcommand{\veps}{\M{\epsilon}}

\newcommand{\nrPoints}{N}
\newcommand{\nrTIM}{K}
\newcommand{\sumAllPointsi}{\sum_{i=1}^{N}}
\newcommand{\sumAllPointsk}{\sum_{k=1}^{K}}

\newcommand{\sumAllIM}{\sum_{k=1}^{K}}
\newcommand{\barc}{\bar{c}}
\newcommand{\barcsq}{\barc^2}

\newcommand{\hatOmega}{\hat{\MOmega}}

\newcommand{\supp}{{Supplementary Material}\xspace}

\newcommand{\nameRotation}{\scenario{QUASAR}}

\newcommand{\name}{\scenario{TEASER}}
\newcommand{\namepp}{\scenario{TEASER++}}
\newcommand{\nameLong}{Truncated least squares Estimation And SEmidefinite Relaxation\xspace}

\newcommand{\mcis}{\scenario{MCIS}}

\newcommand{\FGR}{\scenario{FGR}}
\newcommand{\GORE}{\scenario{GORE}} 
\newcommand{\goICP}{\scenario{Go-ICP}} 
\newcommand{\ransac}{\scenario{RANSAC}}
\newcommand{\ransaconek}{\scenario{RANSAC1K}}
\newcommand{\ransaconemin}{\scenario{RANSAC1min}}

\newcommand{\ransactenk}{\scenario{RANSAC10K}}
\newcommand{\ICP}{\scenario{ICP}} 

\newcommand{\matchTD}{\scenario{3DMatch}}

\newcommand{\SIFT}{\scenario{SIFT}} 
\newcommand{\ORB}{\scenario{ORB}} 

\newcommand{\bunny}{\scenario{Bunny}} 
\newcommand{\armadillo}{\scenario{Armadillo}} 
\newcommand{\dragon}{\scenario{Dragon}}
\newcommand{\buddha}{\scenario{Buddha}}

\newcommand{\hatvb}{\hat{\vb}}
\newcommand{\hatva}{\hat{\va}}

\newcommand{\vSkew}[1]{[#1]_{\times}}
\newcommand{\bmat}{\left[ \begin{array}}
\newcommand{\emat}{\end{array} \right]}

\newcommand{\barva}{\bar{\va}}
\newcommand{\barvb}{\bar{\vb}}

\newcommand{\vxi}{\boldsymbol{\xi}}
\newcommand{\half}{\frac{1}{2}}
\newcommand{\quarter}{\frac{1}{4}}

\newcommand{\vphi}{\boldsymbol{\phi}}

\newcommand{\edit}[1]{#1\xspace}

\newcommand{\frob}{{\small \text{F}}}

\newcommand{\gt}{^\circ}
\newcommand{\sgt}{s^\circ}
\newcommand{\MRgt}{\MR^\circ}
\newcommand{\vtgt}{\vt^\circ}

\newcommand{\Fig}[1]{Fig.~\ref{#1}\xspace}

\newcommand{\resF}{\vr}
\newcommand{\conSet}{\calI}
\newcommand{\nrIn}{{N}_{\mathrm{in}}}
\newcommand{\nrOut}{{N}_{\mathrm{out}}}

\newcommand{\resInMC}{r^{\maxCon}_{\mathrm{in}}}

\newcommand{\nrOutMC}{\nrOut^\maxCon}

\newcommand{\nrInMC}{\nrIn^\maxCon}
\newcommand{\maxCon}{\scenario{MC}}
\newcommand{\maxConLong}{Consensus Maximization\xspace}

\newcommand{\ssym}{\mathfrak{S}}
\newcommand{\SOd}{\text{SO}(d)}
\newcommand{\usphere}{\mathbb{S}}
\newcommand{\TIMNoiseBound}{\delta} 
\newcommand{\hatvxx}{\hat{\vxx}}
\newcommand{\hatmu}{\hat{\mu}}
\renewcommand{\MDelta}{\boldsymbol{\Delta}}
\newcommand{\bareta}{\bar{\eta}}
\newcommand{\barcalH}{\bar{\calH}}
\newcommand{\blkdiag}{\text{blkdiag}}
\newcommand{\hatMZ}{\hat{\MZ}}

\newcommand{\barMM}{\bar{\MM}}

\newcommand{\barMDelta}{\bar{\MDelta}}
\newcommand{\hatxi}{\hat{\vxi}}

\newcommand{\barMH}{\bar{\MH}}

\newcommand{\vPPhi}{\boldsymbol{\Phi}}

\newcommand{\psdSub}{\psdCone{}}
\newcommand{\affineSub}{\bar{\calL}}
\newcommand{\affineSubNonRot}{\calL}

\newcommand{\urlTEASER}{{{\tt \smaller\url{https://github.com/MIT-SPARK/TEASER-plusplus}}}}
\newcommand{\urlVideo}{{\tt \smaller\url{https://youtu.be/xib1RSUoeeQ}}}
\newcommand{\sRt}{\substack{ s > 0,  \MR \in \SOthree, \vt \in \Real{3}}}

\newcommand{\setCon}{\bar{\calC}}
\newcommand{\sizeCon}{\bar{N}_{\mathrm{in}}}
\newcommand{\setTrueIn}{\calI}
\newcommand{\sizeTrueIn}{N_I}
\newcommand{\setFakeIn}{\calJ}
\newcommand{\sizeFakeIn}{N_J}
\newcommand{\tgt}{t^{\circ}}
\newcommand{\hatt}{\hat{t}}

\newcommand{\myProb}{\mathbb{P}}
\newcommand{\myData}{\mathbb{D}}
\newcommand{\myAlgo}{\mathbb{A}}
\newcommand{\myQuantile}{\gamma}

\newcommand{\arxivCite}{~in~\cite{Yang20arxiv-teaser}}

\newcommand{\myBlocks}{\calZ}
\newcommand{\gtadj}{ground-truth\xspace}

\newcommand{\bounds}{\zeta^s}
\newcommand{\boundR}{\zeta^R}
\newcommand{\boundt}{\zeta^t}

\newcommand{\vxxstar}{\vxx^\star}
\newcommand{\vqstar}{\vq^\star}

\newcommand{\DRS}{\scenario{DRS}}

\newcommand{\magenta}[1]{\textcolor{magenta}{#1}}

\newcommand{\tlssdp}{\scenario{TLS: SDP}}
\newcommand{\tlsgnc}{\scenario{TLS: GNC}}
\newcommand{\gnc}{\GNC}
\newcommand{\tls}{\TLS}
\newcommand{\teaser}{\name}

\newcommand{\spacebeforesection}{\vspace{-4mm}}

\newrobustcmd*{\mytriangle}[1]{\tikz{\filldraw[draw=#1,fill=#1] (0,0) --
(0.2cm,0) -- (0.1cm,0.2cm);}}
\renewcommand{\expl}[1]{\text{\small\color{gray}(#1)}\nonumber\\}
\renewcommand{\subject}{\text{ s.t. }}

\usepackage{float}

\usepackage{hyperref}
\usepackage{graphicx}
\usepackage{epstopdf}
\usepackage{epsfig}
\newcommand{\optional}[2]{#2\xspace} %
\newcommand{\spaceBeforeSection}{\vspace{-2mm}}

\newcommand{\myparagraph}[1]{{\bf #1.}}

\usepackage{xr}

\newcommand{\isExtended}[2]{{#1}\xspace} 

\newcommand{\isTwoCols}[2]{#1\xspace} 

\graphicspath{{./figures/}}

\title{\huge{\name: Fast and Certifiable Point Cloud Registration}}

\author{Heng Yang, Jingnan Shi, Luca Carlone \vspace{-7mm}
\thanks{H.\,Yang, J.\,Shi, and L.\,Carlone are with the Laboratory for 
Information \& Decision Systems (LIDS), Massachusetts Institute of Technology, Cambridge, MA 02139, USA, Email: {\sf \{hankyang,jnshi,lcarlone\}@mit.edu}}
\thanks{The authors would like to thank the associate editor and the anonymous reviewers for their constructive feedback, and \'Alvaro Parra and Russ Tedrake for providing comments on an early draft of this paper. This work was partially funded by ARL DCIST CRA W911NF-17-2-0181, ONR RAIDER
N00014-18-1-2828, Lincoln Laboratory ``Resilient Perception in Degraded Environments'', 
and the Google Daydream Research Program.}
}

\begin{document}

\maketitle

\begin{tikzpicture}[overlay, remember picture]
\path (current page.north east) ++(-4.4,-0.4) node[below left] {
This paper has been accepted for publication in the IEEE Transactions on Robotics.
};
\end{tikzpicture}
\begin{tikzpicture}[overlay, remember picture]
\path (current page.north east) ++(-6.5,-0.8) node[below left] {
Please cite the paper as: H. Yang, J. Shi, and L. Carlone,
};
\end{tikzpicture}
\begin{tikzpicture}[overlay, remember picture]
\path (current page.north east) ++(-3.1,-1.2) node[below left] {
``TEASER: Fast and Certifiable Point Cloud Registration'', \emph{IEEE Transactions on Robotics (T-RO)}, 2020.
};
\end{tikzpicture}
 


\begin{abstract}
We propose the first fast and certifiable algorithm  
for the registration of two sets of 3D points  in the presence of large amounts of outlier correspondences.
A \emph{certifiable algorithm} is one that attempts to solve an intractable optimization problem (\eg robust estimation with outliers) and provides readily checkable conditions to {verify} if the returned solution is optimal (\eg if the algorithm produced the most accurate estimate in the face of outliers) or bound its sub-optimality or accuracy.

Towards this goal, we first reformulate the registration problem using a \emph{Truncated Least Squares} (\TLS) cost that makes the estimation insensitive to a large fraction of spurious correspondences.  
Then, we provide a general graph-theoretic framework to decouple scale, rotation, and translation estimation,
which allows solving in cascade for the three transformations. 
Despite the fact that each subproblem (scale, rotation, and translation estimation) is still non-convex and combinatorial in nature, 
we show that (i) \TLS scale and (component-wise) translation estimation can be solved in polynomial time via an \emph{adaptive voting} scheme, (ii) \TLS rotation estimation can be relaxed to a semidefinite program (SDP) and the relaxation is tight, even in the presence of extreme outlier rates, \revone{and (iii) the graph-theoretic framework allows drastic pruning of outliers by finding the maximum clique.}
We name the resulting algorithm \name (\emph{\nameLong}). While solving large SDP relaxations is typically slow, we develop a 
second fast and certifiable algorithm, named \namepp, \revone{that uses \emph{graduated non-convexity} to solve the rotation subproblem and leverages \emph{Douglas-Rachford Splitting} to efficiently certify global optimality.}

For both algorithms, we provide theoretical bounds on the estimation errors, which are 
the first of their kind for robust registration problems. 
Moreover, we test their performance on standard benchmarks, object detection datasets, and 
the \emph{\matchTD} scan matching dataset, and show that (i) both algorithms dominate the state of the art (\eg{ }\ransac, branch-\&-bound, heuristics) and are robust to more than $99\%$ outliers \revone{when the scale is known}, (ii) \namepp can run in milliseconds and it is currently the fastest robust registration algorithm, and (iii) \namepp is so robust it can also solve problems without correspondences (\eg hypothesizing all-to-all correspondences) where it largely outperforms \ICP and \revone{it is more accurate than \goICP while being orders of magnitude faster.}
We release a fast open-source C++ implementation of \namepp. 

\end{abstract}

\vspace{-2mm}
\section*{Supplementary Material}
\vspace{-1mm}
\begin{itemize}
  \item Video: {\smaller \urlVideo}
  \item Code: {\smaller \urlTEASER}
\end{itemize}
\vspace{-2mm}



\section{Introduction}
\label{sec:intro}


\newcommand{\myhspacet}{\hspace{-4mm}}
\newcommand{\myhspace}{\hspace{-2mm}}
\newcommand{\myvspace}{\hspace{-4mm}}
\newcommand{\mpwt}{2.8cm}
\newcommand{\mpw}{4cm}

\begin{figure}[t!]
	\begin{center}
	\begin{minipage}{\textwidth}
	\hspace{-2mm}
	\begin{tabular}{p{0.2cm}p{\mpwt}p{\mpwt}p{\mpwt}}%
		\begin{minipage}{0.3cm}%
		\rotatebox{90}{Correspondence-based\hspace{-2cm}}
		\end{minipage}
		&
			\begin{minipage}[t][][t]{\mpwt}%
			\centering%
			\includegraphics[trim= 0mm 0mm 10mm 10mm, clip, width=1\columnwidth]{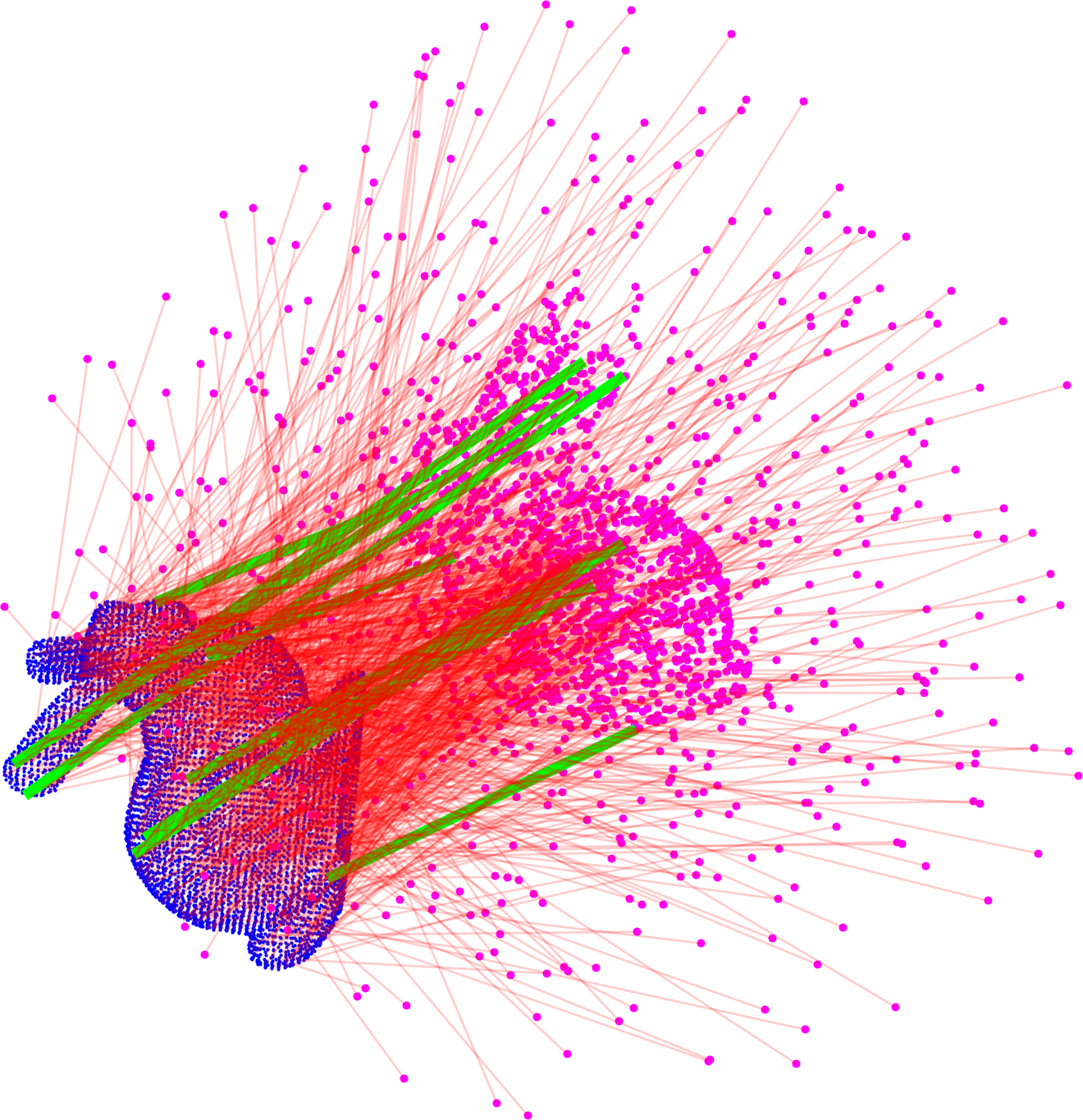} \myvspace\\
			(a) Input 
			\end{minipage}
		\myhspacet& \myhspacet
			\begin{minipage}[t][][t]{\mpwt}%
			\centering%
			\includegraphics[width=0.9\columnwidth]{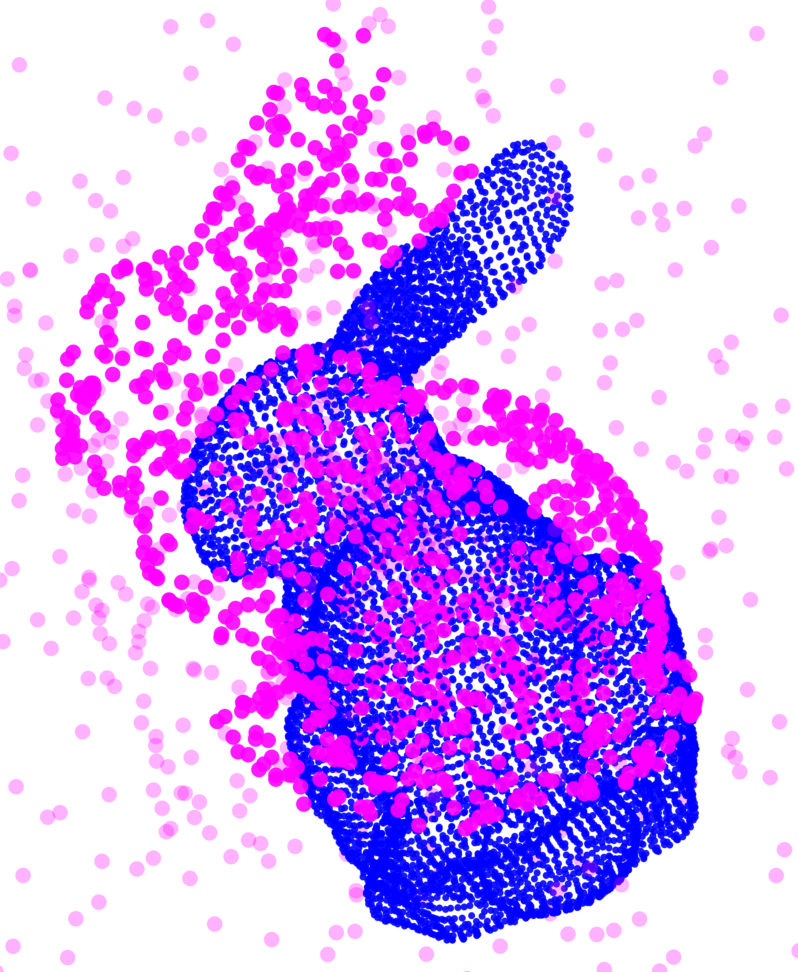} \myvspace\\
			(b) \ransac 
			\end{minipage}
		\myhspacet& \myhspacet\myhspacet
			\begin{minipage}[t][][t]{\mpwt}%
			\centering%
			\includegraphics[width=0.9\columnwidth]{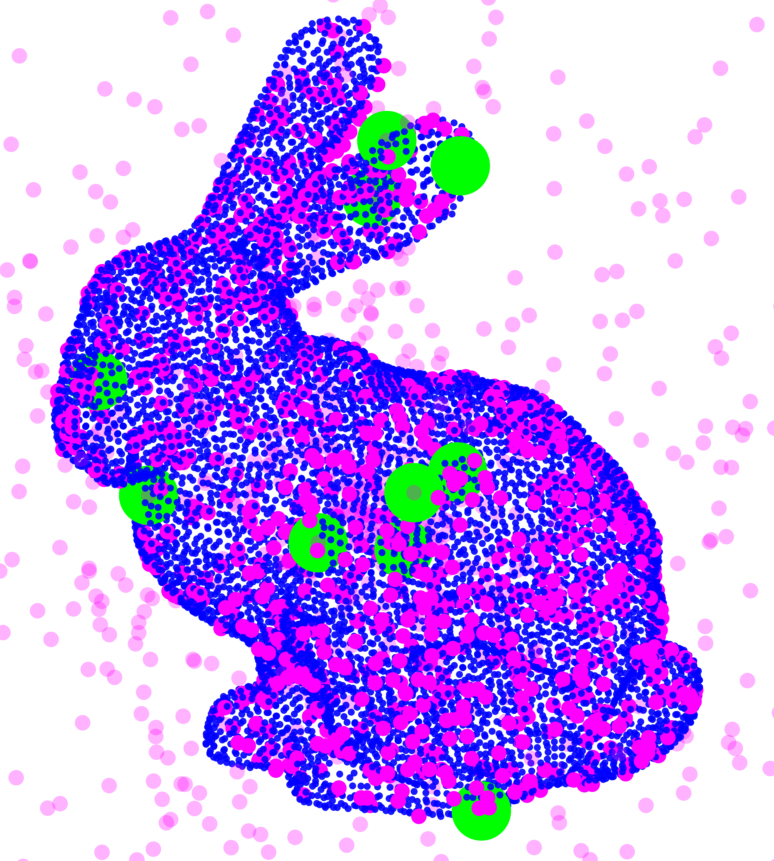} \myvspace\\
			(c) \namepp 
			\end{minipage}
		\vspace{2mm}\\
		\begin{minipage}{0.2cm}%
		\rotatebox{90}{Correspondence-free\hspace{-2.3cm}}
		\end{minipage}
		& 
			\begin{minipage}[t][][t]{\mpwt}%
			\centering%
			\includegraphics[width=0.7\columnwidth]{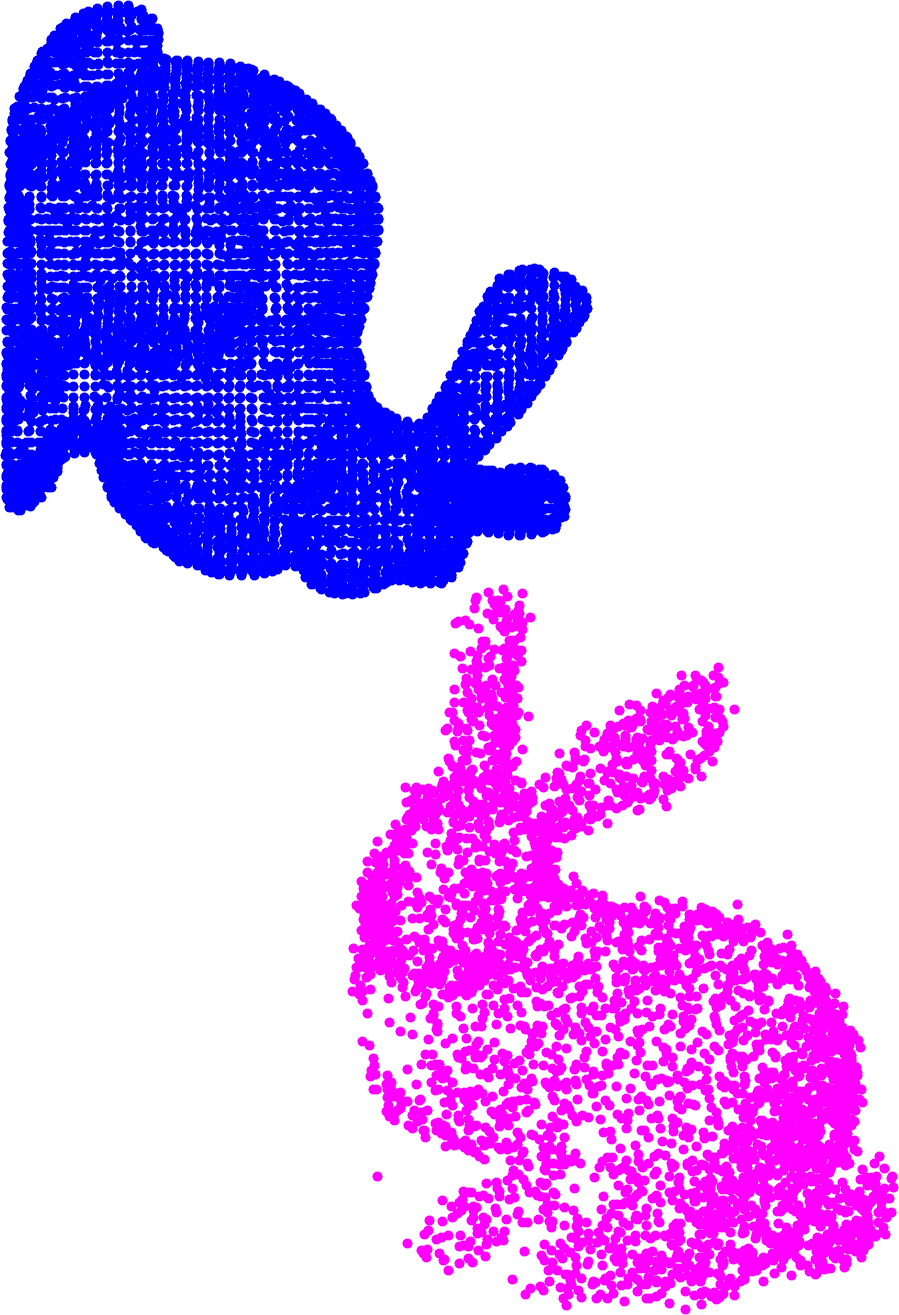} \myvspace\\
			(d) Input 
			\end{minipage}
		\myhspacet& \myhspacet
			\begin{minipage}[t][][t]{\mpwt}%
			\centering%
			\includegraphics[width=0.9\columnwidth]{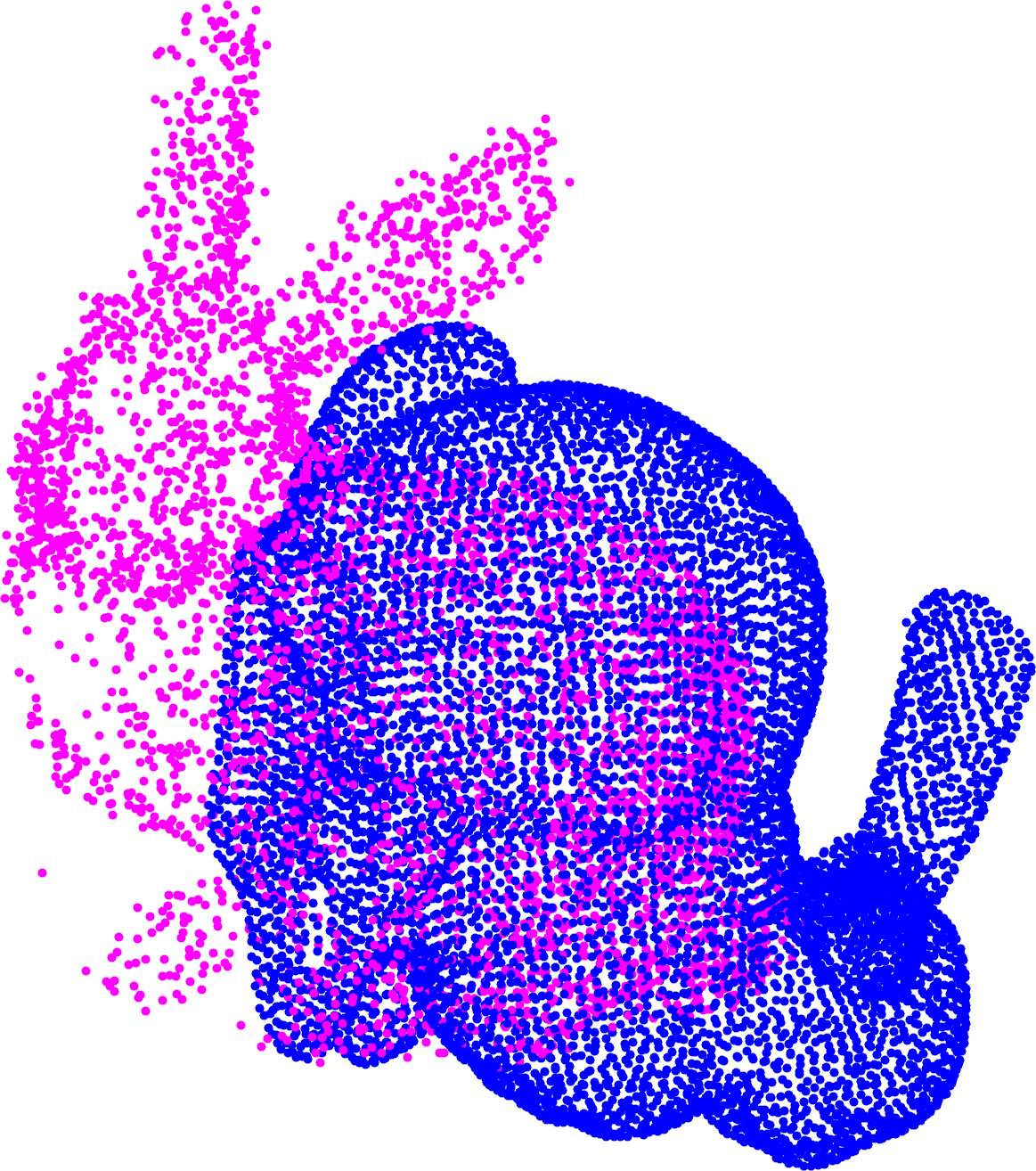} \myvspace\\
			(e) \ICP
			\end{minipage}
		\myhspacet& \myhspacet\myhspacet
			\begin{minipage}[t][][t]{\mpwt}%
			\centering%
			\includegraphics[width=0.77\columnwidth]{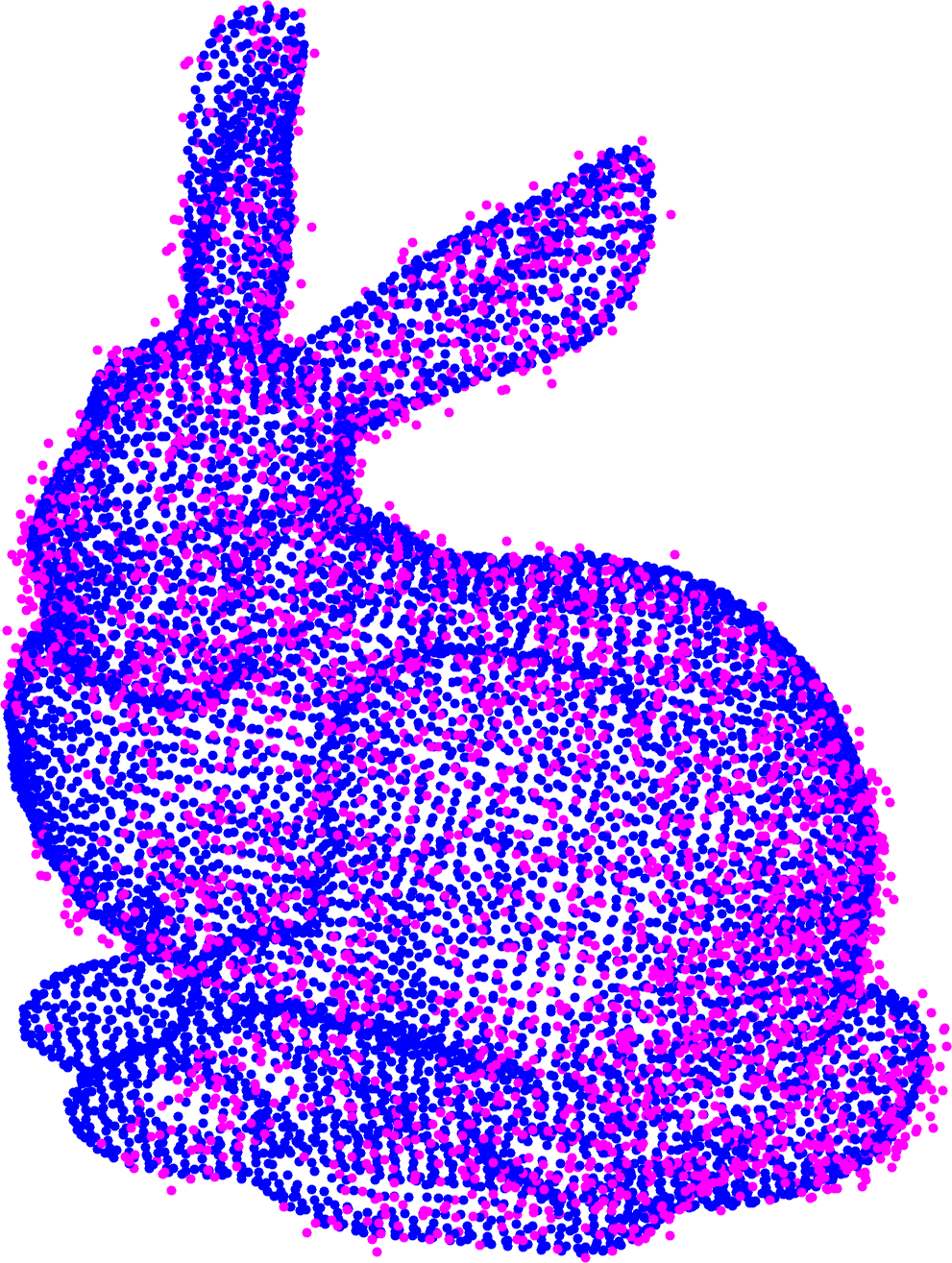} \myvspace\\
			(f) \namepp
			\end{minipage}
		\end{tabular}
		\vspace{2mm}\\
	\begin{tabular}{p{0.3cm}p{\mpw}p{\mpw}}%
		\begin{minipage}{0.3cm}%
		\rotatebox{90}{Object Localization}
		\end{minipage}
		& \hspace{-5mm}
			\begin{minipage}{\mpw}%
			\centering%
			\includegraphics[trim= 30mm 20mm 30mm 0mm, clip, width=1.0\columnwidth]{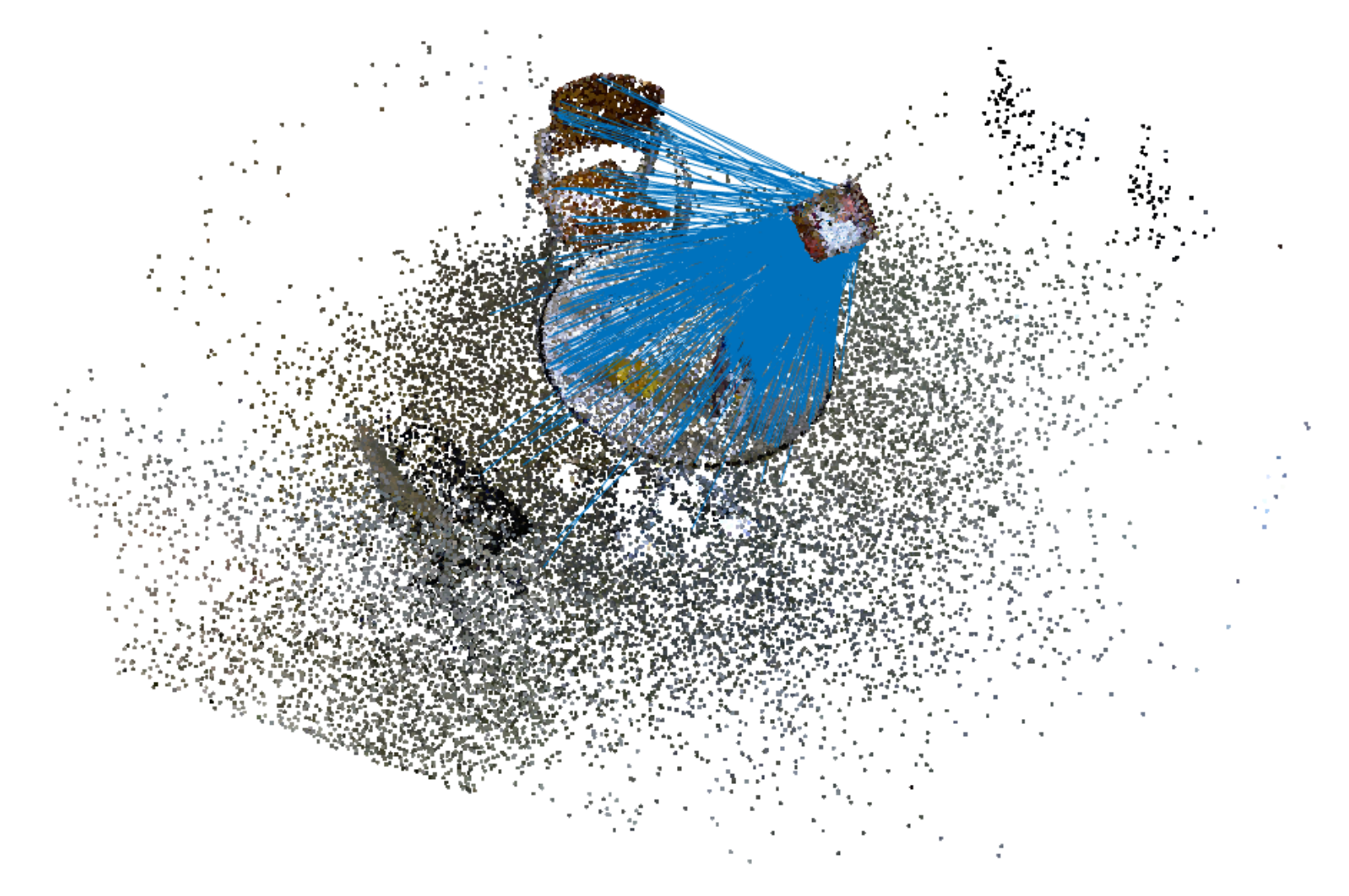} \myvspace
			\vspace{0mm}\\
			(g) Correspondences
			\end{minipage}
		& \hspace{-8mm}
			\begin{minipage}{\mpw}%
			\centering%
			\includegraphics[trim= 30mm 20mm 30mm 0mm, clip,width=1.0\columnwidth]{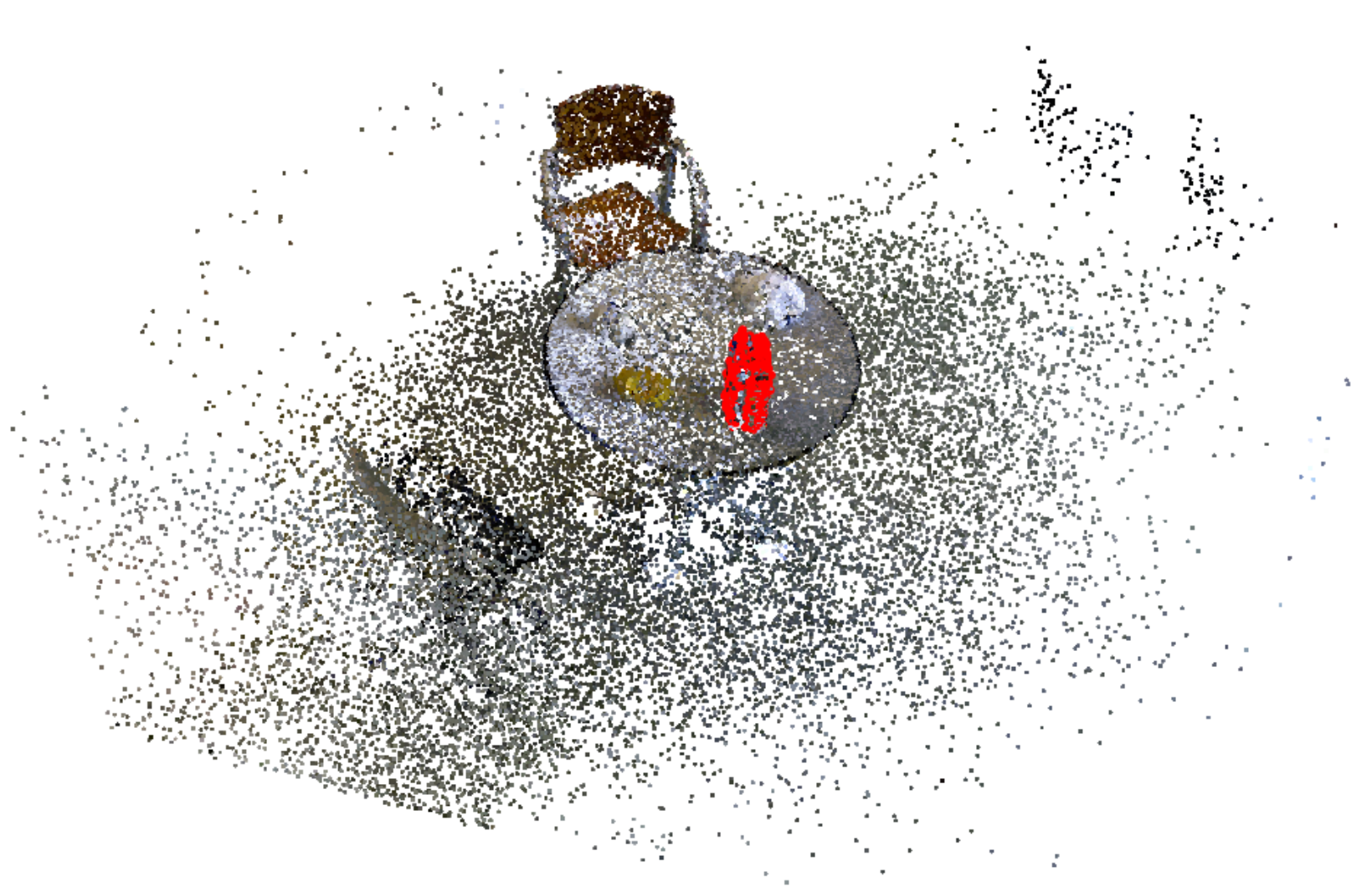} \myvspace
			\vspace{0mm}\\
			(h) \namepp
			\end{minipage}
			\vspace{2mm}\\ 
		\begin{minipage}{0.2cm}%
		\rotatebox{90}{Scan Matching}
		\end{minipage}
		& \hspace{-5mm}
			\begin{minipage}{\mpw}%
			\centering%
			\includegraphics[width=0.95\columnwidth]{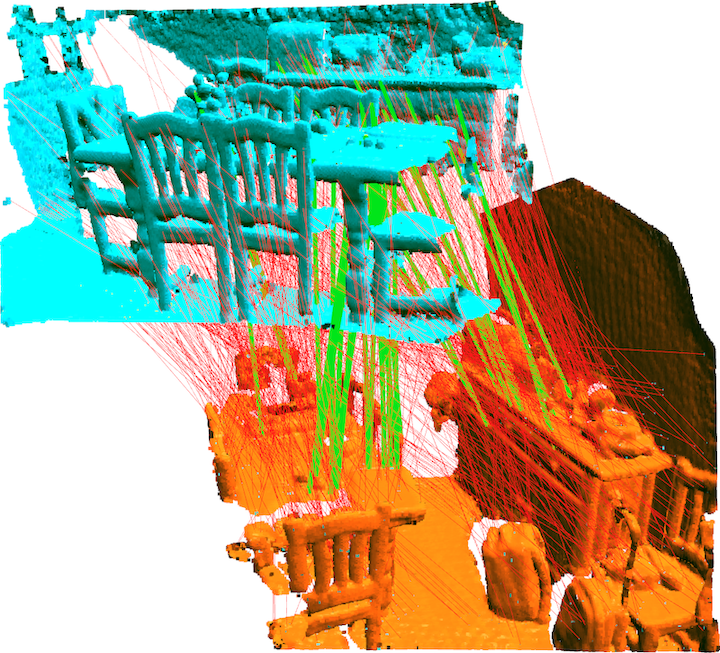} \myvspace\\
			(i) Correspondences
			\end{minipage}
		& \hspace{-8mm}
			\begin{minipage}{\mpw}%
			\centering%
			\includegraphics[width=1.0\columnwidth]{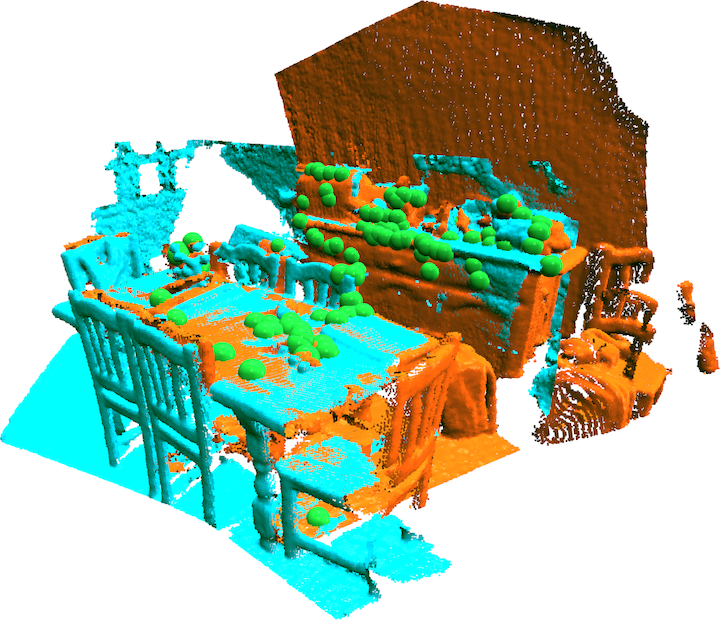} \myvspace\\
			(j) \namepp
			\end{minipage}
	\end{tabular}
	\end{minipage}
	\vspace{-3mm} 
	\caption{
	We address 3D registration in the realistic case where 
	many point-to-point correspondences are outliers due to incorrect matching.
	(a) \bunny dataset (source cloud in \blue{blue} and target cloud in \magenta{magenta}) with {95\% outliers (shown as \red{red} lines) and 5\% inliers (shown as \green{green} lines).}
	Existing algorithms, such as \ransac (b), can produce incorrect estimates without notice even after running for 10,000 iterations.
	Our certifiable algorithm, \name, largely outperforms the state of the art in 
	terms of robustness and accuracy, and a fast implementation, named \namepp (c), computes accurate estimates in milliseconds even with extreme outlier rates 
	and finds the small set of inliers (shown as \green{green} dots).
	The unprecedented robustness of \namepp enables the solution of correspondence-free registration problems (d), where \ICP (e) fails without a good initial guess, while \namepp~{(f)} succeeds without {requiring} an initial guess. We test our approach in challenging object localization {(g-h)} and scan matching {(i-j)} RGB-D datasets, using both traditional features (\ie FPFH~\cite{Rusu09icra-fast3Dkeypoints}) and deep-learned features (\ie~3DSmoothNet~\cite{gojcic19cvpr-3dsmoothnet}).
	 \label{fig:overview}\vspace{-1.0cm}}
	\end{center}
\end{figure}
 
\emph{Point cloud registration} (also known as \emph{scan matching} or \emph{point cloud alignment})
is a fundamental problem in robotics and computer vision and consists in finding 
the best transformation (rotation, translation, and potentially scale) that aligns two point clouds.
It finds applications in motion estimation and 
%
%
3D reconstruction~\cite{Henry12ijrr-rgbdMapping,Blais95pami-registration,Choi15cvpr-robustReconstruction,Zhang15icra-vloam},
object recognition and localization~\cite{Drost10cvpr,Wong17iros-segicp,Zeng17icra-amazonChallenge,Marion18icra-labelFusion}, 
panorama stitching~\cite{Bazin14eccv-robustRelRot},  
and medical imaging~\cite{Audette00mia-surveyMedical,Tam13tvcg-registrationSurvey}, to name a few.

When the ground-truth correspondences between the point clouds are known and the points are affected by zero-mean Gaussian noise, 
the registration problem 
can be readily solved, since elegant closed-form solutions~\cite{Horn87josa,Arun87pami} 
exist for the case of isotropic noise.
In practice, however, the correspondences are either unknown, or contain many outliers, 
leading these solvers to produce poor estimates. 
Large outlier rates are typical of 3D keypoint detection and matching~\cite{Tombari13ijcv-3DkeypointEvaluation}. 

Commonly used approaches for registration with unknown or uncertain correspondences 
either rely on the availability of an initial guess for the 
unknown transformation (\eg the Iterative Closest Point, \ICP~\cite{Besl92pami}), or 
implicitly assume the presence of a small set of outliers
(\eg~\ransac~\cite{Fischler81}\footnote{\ransac's runtime grows exponentially with the outlier ratio~\cite{Bustos18pami-GORE} and it
typically performs poorly with large outlier rates (see Sections~\ref{sec:relatedWork} and~\ref{sec:experiments}).}). 
These 
algorithms may fail without notice (\Fig{fig:overview}(b),(e)) and may return estimates that are arbitrarily far from the 
ground-truth transformation.   \
In general, the literature is divided between \emph{heuristics}, that are fast but brittle, and 
\emph{global methods} that are guaranteed to be robust, but run in worst-case exponential time 
(\eg branch-\&-bound methods, such as \goICP~\cite{Yang16pami-goicp}).

This paper is motivated by the goal of designing an approach that 
(i) can solve registration globally (without relying on an initial guess), 
(ii) can tolerate extreme amounts of outliers 
(e.g., when $\maxOutliers$ of the correspondences are outliers),  
(iii) runs in polynomial time and is fast in practice (\ie can operate in real-time on a robot), 
and (iv) provides formal performance guarantees. 
In particular, we look for \emph{a posteriori} guarantees, \eg conditions that one can check after 
executing the algorithm to assess the quality of the estimate. 
This leads to the notion of \emph{certifiable algorithms}, \ie an algorithm that attempts to solve an intractable problem and provides checkable conditions on whether it succeeded~\cite{Bandeira16crm,Carlone15icra-verification,Rosen18ijrr-sesync,RahmanIROS19-falseNegative,Seshia16arxiv-verifiedAutonomy,Desai17icrv-verification}. 
The interested reader can find a broader discussion in \isExtended{Appendix~\ref{sec:certifiableAlgorithms}.}{Appendix~\ref{sec:certifiableAlgorithms}\arxivCite.}
\myParagraph{Contribution}
%
This paper proposes the first certifiable algorithm for 3D registration with outliers. 
\revone{We reformulate the registration problem using a \emph{Truncated Least Squares} (\TLS) cost (presented in \prettyref{sec:TLSregistration}), which is insensitive to a large fraction of spurious correspondences but leads to a hard, combinatorial, and non-convex optimization.}

The \revone{first} contribution (\prettyref{sec:decoupling}) 
 is a general framework to decouple scale, rotation, and translation estimation. 
 The idea of decoupling rotation and translation has appeared in related work, \edit{e.g.,~\cite{Makadia06cvpr-registration,Liu18eccv-registration,Bustos18pami-GORE}.}
 The novelty of our proposal is \revone{fourfold}: 
 (i) we develop invariant measurements to estimate the scale (\cite{Liu18eccv-registration,Bustos18pami-GORE} 
 assume the scale is given), 
 (ii) we make the decoupling formal \extraEdits{under the assumption of} 
 \emph{unknown-but-bounded} noise~\cite{Milanese89chapter-ubb,Bertsekas71thesis}, 
 (iii) we provide a general graph-theoretic framework to derive these invariant measurements, and 
 (iv) \extraEdits{we show that this framework}  allows pruning a large amount of outliers by finding the \emph{maximum clique} of the graph defined by the invariant measurements (Section~\ref{sec:teaser}).

The decoupling allows solving in cascade for scale, rotation, and translation.
However, each subproblem is still combinatorial in nature. 
Our \revone{second} contribution is to show that (i) in the scalar case \TLS estimation 
can be solved exactly in polynomial time 
using an \emph{adaptive voting} scheme, and this enables efficient estimation of 
the scale and the (component-wise) translation (Section~\ref{sec:scaleTranslation_adaptiveVoting}); 
(ii) we  can formulate a tight \emph{semidefinite programming (SDP) relaxation} to estimate the rotation and establish {\emph{a posteriori}} conditions to check the quality of the relaxation (Section~\ref{sec:robustRotationSearch_SDPRelaxationVerification}). 
 We remark that  the rotation subproblem addressed in this paper is in itself a 
 foundational problem in vision (where it is known as \emph{rotation search}~\cite{Hartley09ijcv-globalRotationRegistration}) and aerospace (where it is known as the \emph{Wahba problem}~\cite{wahba1965siam-wahbaProblem}). Our SDP relaxation  is the first certifiable algorithm for robust rotation search.

Our \revone{third} contribution (Section~\ref{sec:guarantees}) is a set of theoretical results certifying the quality of the solution returned by our algorithm, named \emph{\nameLong} (\name). 
In the noiseless case, we provide easy-to-check conditions under which \name recovers the true transformation between the point clouds in the presence of outliers. 
In the noisy case, we provide bounds on the distance between the ground-truth transformation 
and \name's estimate.
To the best of our knowledge these are the first non-asymptotic error bounds for geometric estimation problems with outliers, while the literature on robust estimation in statistics (\eg~\cite{Diakonikolas16focs-robustEstimation}) typically studies simpler problems in Euclidean space and focuses on asymptotic bounds.

Our \revone{fourth} contribution (Section~\ref{sec:teaserpp}) is to implement a fast version of \name, named \namepp, \revone{that uses \emph{graduated non-convexity} (\GNC)~\cite{Yang20ral-GNC} to estimate the rotation without solving a large SDP. We show that \namepp is also \emph{certifiable}, and \extraEdits{in particular we leverage} \emph{Douglas-Rachford Splitting}~\cite{Yang20arXiv-certifiablePerception} to design a scalable \emph{optimality certifier} that can assert global optimality of the estimate returned by \GNC.}
We release a fast open-source C++ implementation of \namepp.\optional{\footnote{\urlTEASER}}{}

Our last contribution (Section~\ref{sec:experiments}) is an extensive evaluation in both standard benchmarks and on real datasets for object detection~\cite{Lai11icra-largeRGBD} and scan matching~\cite{Zeng17cvpr-3dmatch}. In particular, we show that (i) both \name and \namepp dominate the state of the art (\eg{ }\ransac, branch-\&-bound, heuristics) and are robust to more than $99\%$ outliers \revone{when the scale is known}, (ii) \namepp can run in milliseconds and it is currently the fastest robust registration algorithm, (iii) \namepp is so robust it can also solve problems without correspondences (\eg hypothesizing all-to-all correspondences) where it largely outperforms \ICP~\revone{and it is more accurate than \goICP while being orders of magnitude faster}, and 
(iv) \namepp can boost registration performance when combined with deep-learned keypoint detection and matching.

\isExtended{}{All the appendices with proofs, extra experimental results, and discussions are given in the \supp~\cite{Yang20arxiv-teaser}.}

\myParagraph{Novelty with respect to~\cite{Yang19rss-teaser,Yang19iccv-QUASAR}} 
In our previous works, we introduced \name~\cite{Yang19rss-teaser} and the quaternion-based relaxation of the rotation subproblem~\cite{Yang19iccv-QUASAR} (named \nameRotation).
 The present manuscript brings \name to maturity by (i) providing explicit theoretical results on 
 \name's performance (Section~\ref{sec:guarantees}), 
 (ii) providing a fast optimality certification method  (Section~\ref{sec:fastCertification}),
 (iii) developing a fast algorithm, \namepp, that \revone{uses \GNC to estimate the rotation without solving an SDP}, while still being certifiable (Section~\ref{sec:teaserpp}), and 
 (iv) reporting a more comprehensive experimental evaluation, 
  including real tests on the \matchTD dataset and examples of registration without correspondences (Sections~\ref{sec:bruteForce} and~\ref{sec:roboticsApplication2}).
 These are major improvements 
 both on the theoretical side (\cite{Yang19rss-teaser,Yang19iccv-QUASAR} only certified performance on each subproblem, rather than end-to-end, and required solving a large SDP) and on the practical side (\namepp is more than three orders of magnitude faster than our proposal in~\cite{Yang19rss-teaser}).


\section{Related Work}
\label{sec:relatedWork}
There are two popular paradigms for the registration of 3D point clouds: 
\emph{Correspondence-based} and 
\emph{Simultaneous Pose and Correspondence} (\ie \emph{correspondence-free}) methods. 

\subsection{Correspondence-based Methods}
Correspondence-based methods first detect and match 3D keypoints between point clouds using feature descriptors~\cite{Tombari13ijcv-3DkeypointEvaluation,Rusu09icra-fast3Dkeypoints,Drost10cvpr,Choy19iccv-FCGF}, 
and then 
 use an estimator to infer the transformation from these putative correspondences.
3D keypoint matching is 
known to be less accurate compared to 2D counterparts like \SIFT and \ORB,
thus causing much higher outlier rates, e.g., having 95\% spurious correspondences is considered common~\cite{Bustos18pami-GORE}. Therefore, a robust backend that can deal with extreme outlier rates is highly desirable.

\myParagraph{Registration without Outliers} Horn~\cite{Horn87josa} and Arun~\cite{Arun87pami} show that optimal solutions 
(in the maximum likelihood sense) for scale, rotation, and translation can be computed in closed form 
when the correspondences are known and the points are affected by isotropic zero-mean Gaussian noise.
 Olsson\setal~\cite{Olsson09pami-bnbRegistration} propose a method based on Branch-\&-bound (\bnb) that is globally optimal and allows point-to-point, point-to-line, and point-to-plane correspondences. Briales and Gonzalez-Jimenez~\cite{Briales17cvpr-registration} propose a \revone{tight semidefinite relaxation to solve the same registration problem as in~\cite{Olsson09pami-bnbRegistration}.} 

If the two sets of points only differ by an unknown rotation (i.e., we do not attempt to estimate the scale and translation), we obtain a simplified version of the registration problem that is known as 
\emph{rotation search} in computer vision~\cite{Hartley09ijcv-globalRotationRegistration}, or \emph{Wahba problem} in aerospace~\cite{wahba1965siam-wahbaProblem}. 
In aerospace, the vector observations are typically the directions to visible stars observed by sensors onboard the satellite.
Closed-form solutions to the Wahba problem are known using both quaternion~\cite{Horn87josa,Markley14book-fundamentalsAttitudeDetermine} and rotation matrix~\cite{markley1988jas-svdAttitudeDeter,Arun87pami} representations. 
The Wahba problem  is also related to the well-known 
\emph{Orthogonal Procrustes} problem~\cite{Gower05oss-procrustes} where one searches for 
orthogonal matrices (rather than rotations), for which a closed-form solution also exists~\cite{Schonemann66psycho}. 
%
The computer vision community has investigated the rotation search problem in the context of  
 point cloud registration~\cite{Besl92pami,Yang16pami-goicp}, image stitching~\cite{Bazin14eccv-robustRelRot}, motion estimation and 3D reconstruction~\cite{Blais95pami-registration,Choi15cvpr-robustReconstruction}. 
 In particular, the closed-form solutions from  
 Horn~\cite{Horn87josa} and Arun\setal~\cite{Arun87pami} can be used for (outlier-free) rotation search 
 with isotropic Gaussian noise.
 \revone{Ohta and Kanatani~\cite{Ohta98TIS-optimal} propose a quaternion-based optimal solution, and
 Cheng and Crassidis~\cite{cheng2019aiaaScitech-totalLeastSquares} develop a local optimization algorithm, for the case of anisotropic Gaussian noise.}
 Ahmed\setal~\cite{Ahmed12tsp-wahba} develop an SDP relaxation for the case with bounded noise and no outliers.

\myParagraph{Robust Registration} Probably the most widely used robust registration approach is based on \ransac~\cite{Fischler81}, which has enabled several early applications in vision and robotics~\cite{Hartley00,Meer91ijcv-robustVision}. Despite its efficiency in the low-noise and low-outlier regime, \ransac exhibits slow convergence and low accuracy 
 with large outlier rates~\cite{Bustos18pami-GORE}, where it becomes harder to sample a ``good'' consensus set.
 Other approaches resort to \emph{M-estimation}, which replaces the least squares objective function with robust costs that are less sensitive to outliers~\cite{MacTavish15crv-robustEstimation,Black96ijcv-unification,Lajoie19ral-DCGM}. Zhou\setal~\cite{Zhou16eccv-fastGlobalRegistration} propose \emph{Fast Global Registration} (\FGR) that uses the Geman-McClure cost function and leverages 
 graduated non-convexity to solve the resulting non-convex optimization.
 \revone{Despite its efficiency, \FGR offers no optimality guarantees.} Indeed, \FGR tends to fail when the outlier ratio is high (>80\%), as we show in~\prettyref{sec:experiments}.
 \revone{Enqvist~\etal~\cite{Enqvist09iccv} propose to optimally find correct correspondences by solving the vertex cover problem and demonstrate robustness against over $95\%$ outliers in 3D registration.}
 Parra and Chin~\cite{Bustos18pami-GORE} propose a \emph{Guaranteed Outlier REmoval} (\GORE) technique, that uses 
 geometric operations to significantly reduce the amount of outlier correspondences before passing them to the optimization backend. \GORE has been shown to be robust to 95\% spurious correspondences~\cite{Bustos18pami-GORE}.
 In an independent effort, Parra~\etal~\cite{Parra19arXiv-practicalMaxClique} find pairwise-consistent correspondences in 3D registration using a practical maximum clique (PMC) algorithm. 
 While similar in spirit to our original proposal~\cite{Yang19rss-teaser}, 
 both \GORE and PMC do not estimate the \emph{scale} of the registration and 
 are typically slower since they rely on \bnb (see Algorithm 2 in~\cite{Bustos18pami-GORE}).
 Yang and Carlone~\cite{Yang19rss-teaser} propose the first certifiable 
 algorithm for robust registration, which however requires solving a large-scale SDP hence being limited to 
 small problems.
 \revone{Recently, physics-based registration has been proposed in~\cite{Golyanik16CVPR-gravitationalRegistration,Jauer18PAMI-physicsBasedRegistration,Yang20arxiv-dynamical}.}

Similar to the registration problem, rotation search with outliers has also been  investigated in the computer vision community. Local techniques for robust rotation search are again based on \ransac or M-estimation, but 
are brittle and do not provide performance guarantees.
Global methods (that guarantee to compute globally optimal solutions) are based on
\emph{\maxConLong}~\cite{Chin17slcv-maximumConsensusAdvances,Wen19TRO-efficientMaxConsensus,Cai19ICCV-CMtreeSearch,Le19pami-deterministicApproxMaxConsensus,Tzoumas19iros-outliers,Chin18eccv-robustFitting} and \bnb~\cite{Campbell17iccv-2D3Dposeestimation}. Hartley and Kahl~\cite{Hartley09ijcv-globalRotationRegistration} first proposed using \bnb for rotation search, and Bazin \etal~\cite{Bazin12accv-globalRotSearch} adopted consensus maximization to extend their \bnb algorithm with a robust formulation. \bnb is guaranteed to return the globally optimal solution, but it runs in exponential time in the worst case. 
Another class of global methods for \maxConLong enumerates all possible subsets of measurements with size no larger than the problem dimension ({3 for rotation search}) to analytically compute candidate solutions, and then verify global optimality using computational geometry~\cite{Olsson08cvpr-polyRegOutlier,Enqvist12eccv-robustFitting}. \revone{Similarly, Ask~\etal~\cite{Ask13cvpr-optimalTLS} show the \TLS estimation can also be solved globally by enumerating subsets with size no larger than the model dimension.}
\revone{However, these methods 
require exhaustive enumeration and become intractable when the problem dimension is large (\eg~$6$ for 3D registration).}
In~\cite{Yang19iccv-QUASAR}, we propose the first quaternion-based certifiable algorithm for robust rotation search (reviewed in Section~\ref{sec:robustRotationSearch_SDPRelaxationVerification}).

\spaceBeforeSection
\subsection{Simultaneous Pose and Correspondence Methods}
\emph{Simultaneous Pose and Correspondence} (\SPC) methods 
 alternate between finding the correspondences and computing the best transformation given the correspondences.

\myParagraph{Local Methods} The \emph{Iterative Closest Point} (\ICP) algorithm~\cite{Besl92pami} is considered a milestone in point cloud registration. 
However, \ICP is prone to converge
to \emph{local minima} and only performs well given a good initial guess. 
Multiple variants of \ICP~\cite{Granger02eccv,
Maier12pami-convergentICP,Chetverikov05ivc-trimmedICP,Kaneko03pr-robustICP} \edit{have proposed to use}
 robust cost functions to improve convergence.
Probabilistic interpretations have also been proposed to improve \ICP as a minimization of the Kullback-Leibler divergence between two Mixture models~\cite{Myronenko10pami-coherentPointDrift,Jian11pami-registrationGMM}. 
Clark~\etal~\cite{Clark20arXiv-nonparametricRegistration} alignn point clouds as continuous functions using Riemannian optimization. \revone{Le~\etal~\cite{Le19cvpr-sdrsac} use semidefinite relaxation to generate hypotheses for randomized methods.}
\revone{All these methods do not provide global optimality guarantees.}

\myParagraph{Global Methods} Global \SPC approaches compute a globally optimal solution without initial guesses,
and are usually based on \bnb.
A series of geometric techniques have been proposed to improve the bounding tightness~\cite{Hartley09ijcv-globalRotationRegistration,
Breuel03cviu-BnBimplementation,Yang16pami-goicp,Parra14cvpr-fastRotationRegistration} and increase the search speed~\cite{Li07iccv-3DRegistration,Yang16pami-goicp}. However, the runtime of \bnb increases exponentially with the size of the point cloud and is made worse by the explosion of the number of local minima resulting from high outlier ratios~\cite{Bustos18pami-GORE}. \edit{Global SPC registration can  be also formulated as a mixed-integer program~\cite{Izatt17isrr-MIPregistration}, though the runtime remains exponential.} \revone{Maron~\etal~\cite{Maron16tog-PMSDP} design a tight semidefinite relaxation, which, however, only works for SPC registration with full overlap.}

\myParagraph{Deep Learning Methods} The success of deep learning on 3D point clouds (\eg PointNet~\cite{Qi17cvpr-pointnet} and DGCNN~\cite{Wang19tog-DGCNN}) opens new opportunities for learning point cloud registration from data. Deep learning methods first learn to embed points clouds in a high-dimensional feature space, then learn to match keypoints to generate correspondences, after which optimization over the space of rigid transformations is performed for the best alignment. PointNetLK~\cite{Aoki19cvpr-pointnetlk} uses PointNet to learn feature representations and then iteratively align the features representations, instead of the 3D coordinates. DCP~\cite{Wang19ICCV-DeepClosestPoint} uses DGCNN features for correspondence matching and Horn's method for registration in an end-to-end fashion (Horn's method is differentiable).
 PRNet~\cite{Wang19nips-prnet} extends DCP to aligning partially overlapping point clouds. Scan2CAD~\cite{Avetisyan19CVPR-scan2cad} and its improvement~\cite{Avetisyan19ICCV-e2eCADAlign} apply similar pipelines to align CAD models to RGB-D scans. 
3DSmoothNet~\cite{gojcic19cvpr-3dsmoothnet} uses a siamese deep learning architecture
to establish keypoint correspondences between two point clouds. \revone{FCGF~\cite{Choy19iccv-FCGF} leverages sparse high-dimensional convolutions to extract dense  feature descriptors from point clouds. Deep global registration~\cite{Choy20cvpr-deepGlobalRegistration} uses FCGF feature descriptors for point cloud registration.}
 {In Section~\ref{sec:roboticsApplication2} we show that (i) our approach provides a robust back-end for deep-learned keypoint matching algorithms (we use~\cite{gojcic19cvpr-3dsmoothnet}), and (ii) that current 
 deep learning approaches still struggle to produce acceptable \revone{inlier} rates in real problems.} 

\begin{remark}[Reconciling Correspondence-based and \SPC Methods]
\label{rmk:spc}
\optional{
While correspondence-based and \SPC methods seem to pursue different strategies they are tightly coupled. First of all, approaches like \ICP alternate between guessing the correspondences and solving a correspondence-based
registration problem. More importantly, one}{
Correspondence-based and \SPC methods are tightly coupled. First of all, approaches like \ICP alternate between finding the correspondences and solving a correspondence-base
 problem. More importantly, one} 
 can always reformulate an \SPC problem as a correspondence-based problem by hypothesizing all-to-all correspondences, \ie associating each point in the first point cloud to all the points in the second. \revone{To the best of our knowledge, only~\cite{Enqvist09iccv,Fredriksson16cvpr-optimalRelativePose} have pursued this formulation since it leads to an extreme number of outliers. In this paper, we show that our approach can indeed solve the \SPC problem thanks to its unprecedented robustness to outliers (\prettyref{sec:bruteForce}).}
\end{remark}

\spacebeforesection
\section{Notations and Preliminaries}
\label{sec:notationPreliminary}

{\bf Scalars, Vectors, Matrices.} We use lowercase characters (\eg~$s$) to denote real scalars, bold lowercase characters (\eg~$\vv$) for real vectors, and bold uppercase characters (\eg~$\MM$) for real matrices. 
$\MM_{ij}$ denotes the $i$-th row and $j$-th column \emph{scalar} entry of matrix $\MM \in \Real{m\times n}$, and $[\MM]_{ij,d}$ (or simply $[\MM]_{ij}$ when $d$ is clear from the context) denotes the $i$-th row and $j$-th column 
$d\!\times\!d$ \emph{block} of the matrix $\MM\!\in\!\Real{md \times nd}$. $\eye_d$ is the identity matrix of size $d$.
We use ``$\kron$'' to denote the Kronecker product. \revone{For a square matrix $\MM$, $\det(\MM)$ and $\trace{\MM}$ denote its determinant and trace.}
The 2-norm of a vector is denoted as $\|\cdot\|$. The Frobenious norm of a matrix is denoted as $\|\cdot\|_\frob$.
For a symmetric matrix $\MM$ of size $n \times n$, we use $\lambda_1 \leq \dots \leq \lambda_n$ to denote its  real eigenvalues. 

{\bf Sets.} 
We use calligraphic
fonts to denote sets (\eg~$\calS$).
We use $\calS^n$ (resp. $\ssym^n$) to denote the group of real symmetric (resp. skew-symmetric) matrices with size $n\!\times\!n$. $\psdCone{n} \doteq \{ \!\MM\! \in\! \calS^n: \!\MM\! \succeq 0\}$ denotes the set of $n\!\times\!n$ symmetric \emph{positive semidefinite} matrices. $\SOd \doteq \{ \MR \!\in\! \Real{d \times d}\!:\! \MR\tran\MR \!=\! \eye_d, \det(\MR)=+1 \}$ denotes the 
$d$-dimensional \emph{special orthogonal group}, while  
$\usphere^{d-1} \!\!=\!\! \{ \vu \!\!\in\!\! \Real{d}\!:\!\! \| \vu \| =\! 1 \}$ \mbox{denotes the $d$-dimensional unit sphere.}

{\bf Quaternions.} \revone{Unit quaternions are a representation for a 3D rotation $\MR \in \SOthree$.}
We denote a unit quaternion as a unit-norm column vector $\vq=[\vv\tran\ s]\tran \in \usphere^3$, where $\vv \in \Real{3}$ is the \emph{vector part} of the quaternion and the last element $s$ is the \emph{scalar part}. We  also use $\vq = [q_1\ q_2\ q_3\ q_4]\tran$ to denote the four entries of the quaternion.
Each quaternion represents a 3D rotation and the composition of two rotations $\vq_a$ and $\vq_b$ can be 
computed using the \emph{quaternion product} $\vq_c = \vq_a \qProd \vq_b$:
\bea \label{eq:quatProduct}
\vq_c = \vq_a \qProd \vq_b = \MOmega_1(\vq_a) \vq_b = \MOmega_2(\vq_b) \vq_a,
\eea
where $\MOmega_1(\vq)$ and $\MOmega_2(\vq)$ are defined as follows:
\begin{equation} \label{eq:Omega_1}
\arraycolsep=2pt\def\arraystretch{0.5}
\scriptscriptstyle \MOmega_1(\vq) \scriptscriptstyle = 
\scriptscriptstyle \left[ \scriptscriptstyle \begin{array}{cccc}
\scriptscriptstyle q_4 & \scriptscriptstyle -q_3 &  \scriptscriptstyle q_2 &  \scriptscriptstyle q_1 \\
\scriptscriptstyle q_3 & \scriptscriptstyle q_4 & \scriptscriptstyle -q_1 & \scriptscriptstyle q_2 \\
\scriptscriptstyle -q_2 & \scriptscriptstyle q_1 & \scriptscriptstyle q_4 & \scriptscriptstyle q_3 \\
\scriptscriptstyle -q_1 & \scriptscriptstyle -q_2 & \scriptscriptstyle -q_3 & \scriptscriptstyle q_4 
\scriptscriptstyle \end{array} \scriptscriptstyle \right],\quad 
\arraycolsep=2pt\def\arraystretch{0.5}
\scriptscriptstyle \MOmega_2(\vq) = \scriptscriptstyle
\bmat{cccc}
\scriptscriptstyle q_4 & \scriptscriptstyle q_3 & \scriptscriptstyle -q_2 & \scriptscriptstyle q_1 \\
\scriptscriptstyle -q_3 & \scriptscriptstyle q_4 & \scriptscriptstyle q_1 & \scriptscriptstyle q_2 \\
\scriptscriptstyle q_2 & \scriptscriptstyle -q_1 & \scriptscriptstyle q_4 & \scriptscriptstyle q_3 \\
\scriptscriptstyle -q_1 & \scriptscriptstyle -q_2 & \scriptscriptstyle -q_3 & \scriptscriptstyle q_4 
\emat.
\end{equation}
The inverse of a quaternion $\vq=[\vv\tran\ s]\tran$ is defined as $\vq\inv = [-\vv\tran\ s]\tran$,
where one simply reverses the sign of the vector part. 
The rotation of a vector $\va \in \Real{3}$ can be expressed in terms of quaternion product. 
Formally, if $\MR$ is the (unique)  rotation matrix corresponding to a unit quaternion $\vq$, then:
\bea
\label{eq:q_pointRot}
\bmat{c}
\MR \va \\
0 
\emat = \vq \qProd \hat{\va} \qProd \vq\inv,
\eea
where $\hat{\va}=[\va\tran\ 0]\tran$ is the homogenization of $\va$, obtained by augmenting $\va$ with an extra entry equal to zero.
The set of unit quaternions, \ie the $4$-dimensional unit sphere $\usphere^3$, is a \emph{double cover} of $\SOthree$ since $\vq$ and $-\vq$ represent the same rotation (this fact can be easily seen by examining eq.~\eqref{eq:q_pointRot}).
\spacebeforesection
\section{Robust Registration with \\ Truncated Least Squares Cost}
\label{sec:TLSregistration}

In the robust registration problem, we are given two 3D point clouds 
$\calA = \{\va_i\}_{i=1}^N$ 
and 
$\calB=\{\vb_i\}_{i=1}^N$, 
with $\va_i, \vb_i \in \Real{3}$. 
 We consider a \emph{correspondence-based} setup, where 
 we are given putative correspondences $(\va_i,\vb_i), i=1,\ldots,\nrPoints$, that 
 obey the following generative model:
 \bea
 \label{eq:robustGenModel}
 \vb_i = \sgt \MRgt \va_i + \vtgt + \vo_i + \veps_i,
 \eea
 where $\sgt > 0$,
$\MRgt \in \SOthree$,
and $\vtgt \in\Real{3}$ are the unknown (to-be-computed) scale, rotation, and translation,
 $\veps_i$ models the measurement noise, and $\vo_i$ is a vector of zeros if the pair $(\va_i, \vb_i)$ is an \emph{inlier}, or a vector of arbitrary numbers for \emph{outlier correspondences}. 
 In words, if the $i$-th correspondence $(\va_i,\vb_i)$ is an inlier correspondence, $\vb_i$ corresponds to a 3D transformation of $\va_i$ (plus noise $\veps_i$), while if 
 $(\va_i,\vb_i)$ is an outlier correspondence,  $\vb_i$ is just an arbitrary vector.

 \myParagraph{Registration without Outliers} 
 When $\veps_i$ is a zero-mean Gaussian noise with isotropic covariance $\sigma_i^2 \eye_3$, and all the correspondences are correct (i.e., $\vo_i=\zero,\forall i$), the Maximum Likelihood estimator of $(\sgt,\MRgt,\vtgt)$ can be computed 
 by solving the following nonlinear least squares problem:
 \bea
 \label{eq:standardRegistration}
 \min_{ \sRt } \sumAllPointsi \frac{1}{\sigma_i^2} \normsq{  \vb_i - s \MR \va_i - \vt }.
 \eea
 %
 Although~\eqref{eq:standardRegistration} is a non-convex problem, due to the non-convexity of the set $\SOthree$, 
 its optimal solution can be computed  
 in closed form by decoupling the estimation of the scale, rotation, and translation, using Horn's~\cite{Horn87josa} or Arun's method~\cite{Arun87pami}. 
 A key contribution of the present paper is to provide a way to decouple scale, rotation, and translation in 
 the more challenging case with outliers.

In practice, a large fraction of the correspondences are \emph{outliers}, due to 
incorrect keypoint matching.
Despite the elegance of the closed-form solutions~\cite{Horn87josa,Arun87pami}, they 
are not robust to outliers, and a single ``bad'' outlier can 
 compromise the correctness of the resulting estimate. 
 Hence, we propose a \emph{truncated least squares registration} formulation that can tolerate extreme amounts of spurious data.

\myParagraph{Truncated Least Squares Registration}  
We depart from the Gaussian noise model and assume the noise is \emph{unknown but bounded}~\cite{Milanese89chapter-ubb}.
Formally, we assume
the \emph{inlier} noise $\veps_i$ in~\eqref{eq:robustGenModel} is such that $\| \veps_i \| \leq \beta_i$, where 
$\beta_i$ is a given bound. 

Then we adopt the following \emph{Truncated Least Squares (\TLS) Registration} formulation:
%
 \beq
 \hspace{-2mm}
 \label{eq:TLSRegistration}
 \min_{ \sRt } 
 \sumAllPointsi \min \left( \frac{1}{\beta_i^2} \normsq{  \vb_i - s \MR \va_i - \vt }{} \!, \; \barcsq \right) ,
 \eeq
 which computes a least squares solution of measurements with small residuals ($\frac{1}{\beta_i^2} \normsq{  \vb_i - s \MR \va_i - \vt }{} \!\leq\! \barcsq$), while discarding measurements with large residuals (when $\frac{1}{\beta_i^2} \normsq{  \vb_i - s \MR \va_i - \vt }{} \!>\! \barcsq$ the $i$-th summand becomes a constant and does not influence the optimization). 
 Note that one can always divide each summand in~\eqref{eq:TLSRegistration}  by $\barcsq$: therefore, 
  one can safely assume $\barcsq$ to be $1$. 
  For the sake of generality, in the following we keep $\barcsq$ since it provides a more direct ``knob'' 
  to be stricter or more lenient towards potential outliers.

The noise bound $\beta_i$ is fairly easy to set in practice and can be understood as a ``3-sigma'' noise bound 
or as the maximum error we expect from an inlier. 
The interested reader can find a more formal discussion on how to set $\beta_i$ and $\barc$ 
\isExtended{in~Appendix~\ref{sec:choiceBeta}.}{in~Appendix~\ref{sec:choiceBeta}\arxivCite.}
We remark that while we assume to have a bound on the maximum error we expect from the inliers ($\beta_i$), we do not make assumptions on the generative model for the outliers, which is typically unknown in practice.

\begin{remark}[\TLS vs. \maxConLong]
\TLS estimation is related to \emph{\maxConLong}~\cite{Chin17slcv-maximumConsensusAdvances}, a popular robust estimation approach in computer vision. \maxConLong looks for an estimate that maximizes the number of inliers, while \TLS simultaneously computes a least squares estimate for the inliers. 
The two methods are not guaranteed to produce the same choice of inliers in general, 
since \TLS also penalizes inliers with large errors. 
 \isExtended{Appendix~\ref{sec:TLSvsMC}}{Appendix~\ref{sec:TLSvsMC}\arxivCite} provides a toy example to illustrate the potential mismatch between the two techniques and provides necessary conditions under which the two formulations find the same set of inliers.
\end{remark}

Despite being insensitive to outlier correspondences, the truncated least squares formulation~\eqref{eq:TLSRegistration} is much more challenging to solve globally, compared to the outlier-free case~\eqref{eq:standardRegistration}. This is the case even in simpler estimation problems where the feasible set of the unknowns is convex, as stated below.

\begin{remark}[Hardness of \TLS Estimation~\cite{Liu19JCGS-minSumTruncatedConvexFunctions,Barratt19arXiv-minSumClippedConvex}]\label{rmk:tlsHardness} 
The minimization of a sum of truncated convex functions, $\sumAllPointsi \min(f_i(\vxx), \barcsq)$, over a convex feasible set $\vxx \in \calX \subseteq \Real{d}$, is NP-hard in the dimension $d$~\cite{Liu19JCGS-minSumTruncatedConvexFunctions,Barratt19arXiv-minSumClippedConvex}. A simple exhaustive search can obtain a globally optimal solution in exponential time $O(2^N)$~\cite{Barratt19arXiv-minSumClippedConvex}. 
\optional{In low dimensions ($d < 3$), global optimization can be done efficiently by partitioning the feasible set $\calX$ using each convex function $f_i(\vxx)$~\cite{Liu19JCGS-minSumTruncatedConvexFunctions}.}{} 
\end{remark}

Remark~\ref{rmk:tlsHardness} states the NP-hardness of \TLS estimation over a convex feasible set. 
The \TLS problem~\eqref{eq:TLSRegistration} is even more challenging due to the non-convexity of $\SOthree$. 
While problem~\eqref{eq:TLSRegistration} is hard to solve directly, 
in the next section, we show how to decouple the estimation of scale, rotation, and translation 
 using \emph{invariant measurements}.

\spacebeforesection
\section{Decoupling Scale, Rotation, \\ and Translation Estimation}
\label{sec:decoupling}

We propose a general approach to 
decouple the estimation of scale, translation, and rotation in problem~\eqref{eq:TLSRegistration}. The key insight is that we can reformulate 
the measurements~\eqref{eq:robustGenModel} to obtain quantities that are invariant to a subset of the 
transformations (scaling, rotation, translation). 

\subsection{Translation Invariant Measurements (\TIMs)}

While the absolute positions of the points in $\calB$ depend on the translation $\vt$, 
the relative positions are invariant to $\vt$. 
Mathematically, given two points $\vb_i$ and $\vb_j$ from~\eqref{eq:robustGenModel},
the relative position of these two points is:  
\bea
 \vb_j - \vb_i = s \MR (\va_j - \va_i) + (\vo_j - \vo_i) + (\veps_j - \veps_i),
 \eea
 where the translation $\vt$ cancels out in the subtraction.
 Therefore, we can obtain a \emph{Translation Invariant Measurement} (\TIM) by computing $\TIMa_{ij} \doteq \va_j - \va_i$
 and  $\TIMb_{ij} \doteq \vb_j - \vb_i$, and the \TIM satisfies the following generative model:
 \beq
 \label{eq:TIM}
 \tag{\TIM}
 \TIMb_{ij} = s \MR \TIMa_{ij} + \vo_{ij} + \veps_{ij}, 
 \eeq
 where $\vo_{ij} \doteq \vo_j - \vo_i$ is zero if \emph{both} the $i$-th \emph{and} the $j$-th measurements are inliers (or arbitrary otherwise), while $\veps_{ij} \doteq \veps_j - \veps_i$ is the measurement noise. 
 It is easy to see that if $\|\veps_i\| \leq \beta_i$ and $\|\veps_j\| \leq \beta_j$,
 then $\|\veps_{ij}\| \leq \beta_i + \beta_j \doteq \TIMNoiseBound_{ij}$.

 The advantage of the \TIMs in eq.~\eqref{eq:TIM} is that their generative model \edit{only depends on} two unknowns, $s$ and $\MR$. 
 \edit{The number of \TIMs is upper-bounded by $\left( \substack{\nrPoints \\ 2} \right) = \nrPoints (\nrPoints-1) / 2$, where 
 pairwise relative measurements between all pairs of points are computed. 
 \prettyref{thm:TIM} below connects the \TIMs with the topology of a graph defined over the 
 3D points.}


 \begin{theorem}[Translation Invariant Measurements]\label{thm:TIM}
 Define the vectors $\va \in \Real{3\nrPoints}$ (resp. $\vb \in \Real{3\nrPoints}$), obtained by concatenating all vectors $\va_i$ (resp. $\vb_i$) in a single column vector.
 Moreover, define an arbitrary graph $\calG$ with nodes $\{1,\ldots,\nrPoints\}$ and an arbitrary set of edges $\calE$. 
 Then, the vectors $\TIMa = (\MA \kron \eye_3) \va$ and $\TIMb = (\MA \kron \eye_3) \vb$ are \TIMs, where $\MA \in \Real{\vert \calE \vert \times \nrPoints}$ is the incidence matrix of $\calG$~\cite{Chung96book}.
 \end{theorem}

A proof of the theorem is given in~\isExtended{Appendix~\ref{sec:proof:thm:TIM}.}{Appendix~\ref{sec:proof:thm:TIM}\arxivCite.} 
\revone{\TIMs generated from a complete graph on the \bunny dataset~\cite{Curless96siggraph} are illustrated in Fig.~\ref{fig:TIMGraph}.}

\begin{figure}[t]
	\begin{center}
			\includegraphics[width=0.7\columnwidth]{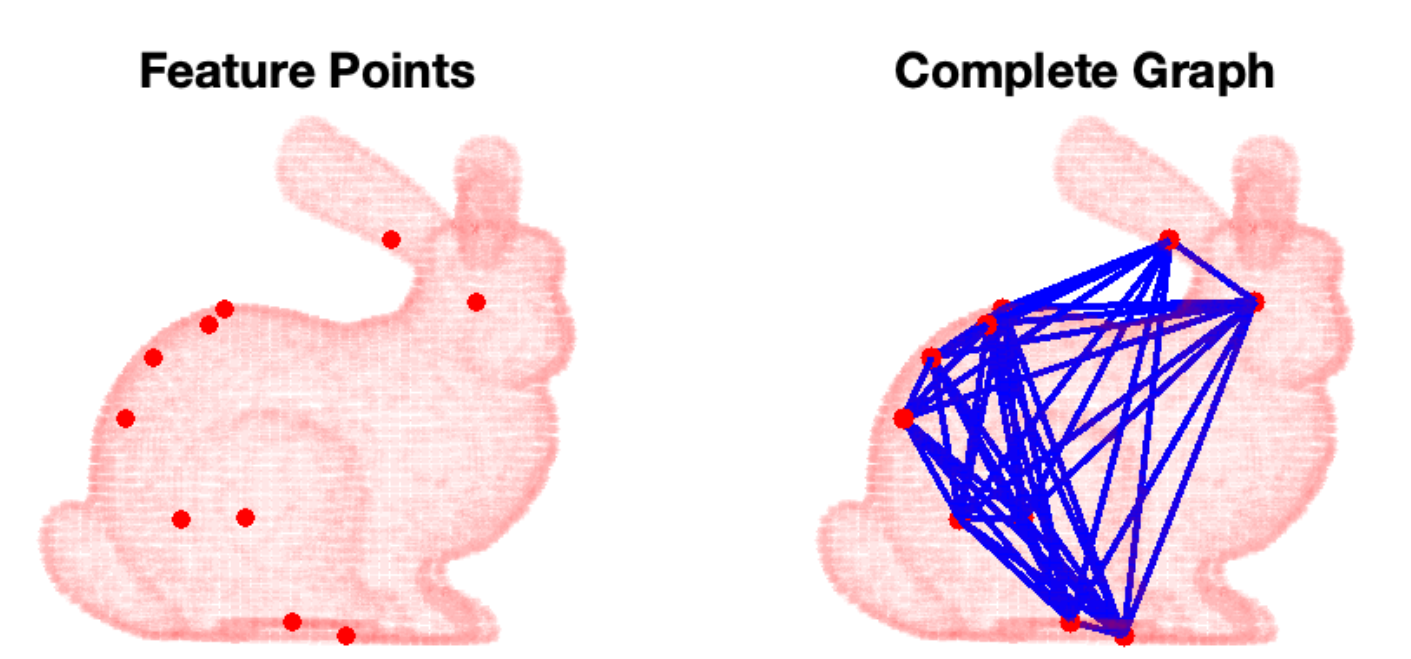}
	\caption{\TIMs generated from a complete graph in the \bunny dataset~\cite{Curless96siggraph}.
	 \label{fig:TIMGraph}}
	 	\vspace{-8mm} 
	\end{center}
\end{figure}

 \subsection{Translation and Rotation Invariant Measurements (\TRIMs)}

While the relative locations of pairs of points (\TIMs) still depend on the rotation $\MR$, their distances 
are invariant to both $\MR$ and $\vt$. Therefore, to build 
rotation invariant measurements, we compute  
the norm of each \TIM vector:
\bea
\label{eq:TRIM01}
 \| \TIMb_{ij} \|  = \| s \MR \TIMa_{ij} + \vo_{ij} + \veps_{ij} \|.
 \eea
 We now note that for the inliers ($\vo_{ij} = \bm{0}$) it holds (using $\| \veps_{ij} \| \leq \TIMNoiseBound_{ij}$ and the triangle inequality):
 \bea
  \| s \MR \TIMa_{ij} \| - \TIMNoiseBound_{ij}
 \leq 
 \| s \MR \TIMa_{ij} + \veps_{ij} \| 
 \leq 
 \| s \MR \TIMa_{ij} \| + \TIMNoiseBound_{ij},
 \eea
 hence we can write~\eqref{eq:TRIM01} equivalently as:
 \bea
\label{eq:TRIM02}
 \| \TIMb_{ij} \|  = \| s \MR \TIMa_{ij} \| + \tilde{o}_{ij} + \tilde{\eps}_{ij},
 \eea
 with $|\tilde{\eps}_{ij}| \leq \TIMNoiseBound_{ij}$, and $\tilde{o}_{ij} = 0$ if both $i$ and $j$ are inliers 
 or is an arbitrary scalar otherwise. Recalling that the norm is rotation invariant and that $s>0$, and dividing both sides of~\eqref{eq:TRIM02}
 by $\|\TIMa_{ij}\|$, we obtain new measurements $s_{ij} \doteq \frac{\| \TIMb_{ij} \|}{\| \TIMa_{ij} \|}$:
 \beq
 \label{eq:TRIM}
 \tag{\TRIM}
 s_{ij} = s + o^s_{ij} + \eps^s_{ij}, 
 \eeq
 where $\eps^s_{ij} \doteq \frac{\tilde{\eps}_{ij}}{ \| \TIMa_{ij} \| }$, 
 and $o^s_{ij} \doteq \frac{\tilde{o}_{ij}}{ \| \TIMa_{ij} \| }$. 
 It is easy to see that  $|\eps^s_{ij} | \leq \TIMNoiseBound_{ij} / \| \TIMa_{ij} \|$ since $|\tilde{\eps}_{ij}| \leq \TIMNoiseBound_{ij}$. 
 We define $\alpha_{ij} \doteq \TIMNoiseBound_{ij} / \| \TIMa_{ij} \|$.

 Eq.~\eqref{eq:TRIM} describes a \emph{Translation and Rotation Invariant Measurement} (\TRIM) whose generative model is only function of the unknown scale $s$.

\begin{remark}[Novelty of Invariant Measurements] 
Ideas similar to the \emph{translation invariant measurements} (\TIMs) have been used in recent work~\cite{Li19arxiv-fastRegistration,Bustos18pami-GORE,Liu18eccv-registration,Agarwal17icra-RFM-SLAM,Carlone14ijrr-lagoPGO2D} while (i) \revone{the novel {graph-theoretic interpretation} of~\prettyref{thm:TIM}
generalizes previously proposed methods and allows pruning outliers as described in Section~\ref{sec:maxClique},} and (ii) the notion of \emph{translation and rotation invariant measurements} (\TRIMs) is completely new. We also remark that while related work uses invariant measurements 
to filter-out outliers~\cite{Bustos18pami-GORE} or to speed up \bnb~\cite{Li19arxiv-fastRegistration,Liu18eccv-registration}, we show that they also allow 
decoupling the estimation of scale, rotation, and translation. 
\end{remark}

A summary table of the invariant measurements and the corresponding noise bounds is given \isExtended{in Appendix~\ref{sec:summaryIM}.}{in Appendix~\ref{sec:summaryIM}\arxivCite.}



\spacebeforesection
 \section{\nameLong (\name): Overview}
 \label{sec:teaser}

 We propose a decoupled approach 
 to solve in cascade for the scale, the rotation, and the translation in~\eqref{eq:TLSRegistration}.
 The approach, named \emph{\nameLong} (\name),
  works as follows:
 \begin{enumerate}
 \item we use the \TRIMs to estimate the scale $\hats$
 \item we use $\hats$ and the \TIMs to estimate the rotation $\hat\MR$
 \item we use $\hats$ and $\hat\MR$ to estimate the translation $\hat\vt$ from $(\va_i, \vb_i)$ in the original 
 \TLS problem~\eqref{eq:TLSRegistration}.
\end{enumerate}

We state each subproblem in the following subsections.

\subsection{Robust Scale Estimation}
\label{sec:scaleEstimation}

The generative model~\eqref{eq:TRIM} describes linear \emph{scalar} measurements $s_{ij}$ of the unknown scale $s$, affected 
by bounded noise $|\eps^s_{ij}| \leq\!\alpha_{ij}$ including potential outliers (when $o^s_{ij} \neq 0$).
Again, we estimate the scale given the measurements $s_{ij}$ and the bounds $\alpha_{ij}$ using a 
\TLS estimator:

\vspace{-5mm}
\bea
\label{eq:TLSscale}
\hats = \argmin_{s} \sumAllIM \min\left( \frac{  (  s-s_k )^2 }{  \alpha^2_k }  \;,\; \barcsq \right),
\eea
where for simplicity we numbered the invariant measurements from $1$ to $K = |\calE|$ and adopted the notation $s_k$ instead of $s_{ij}$. 
Section~\ref{sec:scaleTranslation_adaptiveVoting} shows that~\eqref{eq:TLSscale} can be solved exactly and in polynomial time via \emph{adaptive voting}~(Algorithm~\ref{alg:adaptiveVoting}). 

\subsection{Robust Rotation Estimation}
\label{sec:rotEstimation}

Given the scale estimate $\hats$ produced by the scale estimation, 
the generative model~\eqref{eq:TIM} describes measurements $\TIMb_{ij}$ affected 
by bounded noise $\| \veps_{ij} \| \leq\!\TIMNoiseBound_{ij}$ including potential outliers (when $\vo_{ij} \neq \bm{0}$). Again, we compute $\MR$ from the estimated scale $\hats$, the \TIM measurements $(\TIMa_{ij}, \TIMb_{ij})$ and the bounds $\TIMNoiseBound_{ij}$ using a 
\TLS estimator:
\bea
\label{eq:TLSrotation}
\hatMR = \argmin_{\MR \in \SOthree} \sumAllIM \min
\left( \frac{  \|  \TIMb_k - \hats \MR \TIMa_k \|^2 }{  \TIMNoiseBound^2_k }  
\;,\; 
\barcsq \right),
\eea
where for simplicity we numbered the measurements from $1$ to $K = |\calE|$ and adopted the notation $\TIMa_k, \TIMb_{k}$ instead of $\TIMa_{ij}, \TIMb_{ij}$. Problem~\eqref{eq:TLSrotation} is known as the \emph{Robust Wahba} or \emph{Robust Rotation Search} problem~\cite{Yang19iccv-QUASAR}\optional{, which is hard to optimize globally even if one were to drop the non-convex  constraint $\MR \in \SOthree$ (see Remark~\ref{rmk:tlsHardness}).}{.}
Section~\ref{sec:robustRotationSearch_SDPRelaxationVerification} shows that~\eqref{eq:TLSrotation} can be solved exactly and in polynomial time (in practical problems) via a \emph{tight} semidefinite relaxation. 

\subsection{Robust Component-wise Translation Estimation}
\label{sec:tranEstimation}

After obtaining the scale and rotation
 estimates $\hats$ and $\hatMR$ by solving~\eqref{eq:TLSscale}-\eqref{eq:TLSrotation}, we can substitute them back into problem~\eqref{eq:TLSRegistration} to estimate the translation $\vt$. Although~\eqref{eq:TLSRegistration} operates on the $\ell_2$ norm of the vector, we propose to 
 solve 
for the translation component-wise, i.e., 
we compute the entries $t_1, t_2, t_3$ of $\vt$ independently: 
 \bea
 \label{eq:TLStranslation}
 \hatt_j = \argmin_{ \substack{ t_j } } 
 \sumAllPointsi \min \left( \frac{(t_j - [\vb_i - \hats\hatMR\va_i]_j)^2}{\beta_i^2},  \barcsq \right),
 \eea
 for $j=1, 2, 3$, and 
 where $[\cdot]_j$ denotes the $j$-th entry of a vector.
 Since $\vb_i - \hats\hatMR\va_i$ is a known vector in this stage, it is easy to see that~\eqref{eq:TLStranslation} is a scalar \TLS problem.
 Therefore, similarly to~\eqref{eq:TLSscale},
Section~\ref{sec:scaleTranslation_adaptiveVoting} shows that~\eqref{eq:TLStranslation} can be solved exactly and in polynomial time via \emph{adaptive voting}~(Algorithm~\ref{alg:adaptiveVoting}). \revone{The interested reader can find a discussion on component-wise versus full TLS translation estimation in~\isExtended{Appendix~\ref{sec:app-compareTranslation},}{Appendix~\ref{sec:app-compareTranslation}\arxivCite.}}

\subsection{Boosting Performance: Max Clique Inlier Selection (\mcis)} 
\label{sec:maxClique}

While in principle we could simply execute the cascade of scale, rotation, and translation estimation described above, 
our graph-theoretic interpretation of~\prettyref{thm:TIM} affords further opportunities to prune outliers.

Consider the \TRIMs as edges in the complete graph $\calG(\calV,\calE)$ (where the vertices $\calV$ are the \revone{correspondences} and the 
edge set $\calE$ induces the \TIMs and \TRIMs per~\prettyref{thm:TIM}).
 After estimating the scale~\eqref{eq:TLSscale} (discussed in Section~\ref{sec:scaleTranslation_adaptiveVoting}), we can prune the edges 
$(i,j)$ in the graph 
whose associated \TRIM $s_{ij}$ have been classified as outliers by the \TLS formulation \revone{(\ie~$|s_{ij}-\hats| > \barc \alpha_{ij}$)}.
This allows us 
to obtain a pruned graph $\calG'(\calV,\calE')$, with $\calE'\subseteq \calE$, where gross outliers are discarded.
The following result ensures that inliers form a clique in the graph $\calG'(\calV,\calE')$, enabling an even more substantial rejection of outliers. 

\begin{theorem}[Maximal Clique Inlier Selection]\label{thm:maxClique}
Edges corresponding to inlier \TIMs form a clique in $\calE'$, and there is at least one maximal clique in 
$\calE'$ that contains all the inliers. 
\end{theorem} 

A proof of Theorem~\ref{thm:maxClique} is presented in 
\isExtended{Appendix~\ref{sec:proof:thm:maxClique}.}{Appendix~\ref{sec:proof:thm:maxClique}\arxivCite.} Theorem~\ref{thm:maxClique} allows us to {prune outliers} by finding the maximal cliques of $\calG'(\calV,\calE')$. \revone{Similar idea has been explored for rigid body motion segmentation~\cite{Perera12accv-MCMotionSeg}.} Although finding the maximal cliques of a graph takes exponential time in general, there exist efficient approximation algorithms that scale to graphs with {millions of nodes}~\cite{Bron73acm-allCliques,Pattabiraman15im-maxClique,Wu15ejor-reviewMCPAlgs}. Under high outlier rates, the graph $\calG'(\calV,\calE')$ is sparse and the maximal clique problem can be solved quickly in practice~\cite{Eppstein10isac-maxCliques}. 
Therefore, in this paper, after performing scale estimation and removing the corresponding gross outliers,
we compute \edit{the maximal clique with largest cardinality, \emph{i.e.,} the \emph{maximum clique}},
as the inlier set to pass to rotation estimation. 
\prettyref{sec:separateSolver} shows that this method drastically reduces the number of outliers.

\subsection{Pseudocode of \name}

The pseudocode of \name is summarized in Algorithm~\ref{alg:ITR}.

 \begin{algorithm}[h]
\SetAlgoLined
\textbf{Input:} \ points $(\va_i,\vb_i)$ and bounds $\beta_i$ ($i=1,\ldots,\nrPoints$), threshold $\barcsq$ (default: $\barcsq=1$), \revone{graph $\calG(\calV,\calE)$ (default: $\calG$ describes the complete graph)\;}  
\textbf{Output:} \  $\hats, \hatMR, \hatvt$\;
\% Compute \TIM and \TRIM \\
$\TIMb_{ij} = \vb_j\!-\!\vb_i \;,\; \TIMa_{ij} = \va_j\!-\!\va_i \;,\; \TIMNoiseBound_{ij} = \beta_i \!+\! \beta_j, \;\; \forall (i,j) \in \calE$ \\
$s_{ij} = \frac{ \| \TIMb_{ij} \| }{ \| \TIMa_{ij} \| }\;,\; \alpha_{ij} = \frac{\TIMNoiseBound_{ij}}{\| \TIMa_{ij} \|},  \;\; \forall (i,j) \in \calE$ \\
\% Decoupled estimation of $s, \MR, \vt$\\
$\hats = {\tt estimate\_s}( \{ s_{ij},\alpha_{ij} : \forall (i,j) \in \calE\}, \barcsq )$\label{line:est_s}\\\revone{
$\calG'(\calV',\calE'') = {\tt maxClique}(\calG(\calV,\calE'))$  \ \% prune outliers}  \label{line:mcis}\\ 
$\hatMR = {\tt estimate\_R}( \{ \TIMa_{ij},\TIMb_{ij},\TIMNoiseBound_{ij} : \forall (i,j) \in \calE''\}, \barcsq, \hats )$\label{line:est_R}\\
$\hatvt = {\tt estimate\_t}( \{ \va_{i},\vb_{i},\beta_{i} : i \in \calV' \}, \barcsq, \hats, \hatMR )$\label{line:est_t}\\
 \textbf{return:} $\hats, \hatMR, \hatvt$
 \caption{\emph{\nameLong} (\name).\label{alg:ITR}}
\end{algorithm}
\setlength{\textfloatsep}{0pt}%

The following Sections~\ref{sec:scaleTranslation_adaptiveVoting}-\ref{sec:robustRotationSearch_SDPRelaxationVerification} describe how to implement the functions in 
lines~\ref{line:est_s}, \ref{line:est_R}, \ref{line:est_t} of Algorithm~\ref{alg:ITR}.  In particular, we show how to
obtain global and robust estimates of scale (${\tt estimate\_s}$) and translation (${\tt estimate\_t}$) in Section~\ref{sec:scaleTranslation_adaptiveVoting}, and 
rotation (${\tt estimate\_R}$) in Section~\ref{sec:robustRotationSearch_SDPRelaxationVerification}. 
\spacebeforesection
\section{Robust Scale and Translation Estimation: \\ Adaptive Voting}
\label{sec:scaleTranslation_adaptiveVoting}
In this section, we propose an \emph{adaptive voting} algorithm to solve \emph{exactly} the robust scale estimation and the robust component-wise translation estimation. 

\subsection{Adaptive Voting for Scalar \TLS Estimation}
Both the scale estimation~\eqref{eq:TLSscale} and the component-wise translation estimation~\eqref{eq:TLStranslation} resort to finding a \TLS estimate of an unknown \emph{scalar} given 
a set of outlier-corrupted measurements. 
Using the notation for scale estimation~\eqref{eq:TLSscale}, the following theorem shows that one can solve \emph{scalar \TLS estimation} in polynomial time by a simple enumeration. 

\begin{theorem}[Optimal Scalar \TLS Estimation]\label{thm:scalarTLS}
Consider the scalar \TLS problem in~\eqref{eq:TLSscale}.
For a given $s \in \Real{}$, define the \emph{consensus set of $s$} as $\conSet(s) = \setdef{k}{ \frac{  (  s-s_k )^2 }{  \alpha^2_k } \leq \barcsq }$. 
Then, for any $s \in \Real{}$, there are at most $2\nrTIM-1$ different non-empty consensus sets. If we name these sets $\conSet_1, \ldots, \conSet_{2K-1}$, then the solution of~\eqref{eq:TLSscale} can be computed by enumeration as:
\bea \label{eq:solveScaleEnumerate}
\hats = \argmin 
\left\{ 
f_s(\hats_i) : \hats_i = 
\left(\sum_{k\in\conSet_i} \frac{1}{\alpha_k^2} \right) \inv \sum_{k\in\conSet_i} \frac{s_k}{\alpha_k^2}, \forall i 
\right\},
\eea   
where $f_s(\cdot)$ is the objective function of~\eqref{eq:TLSscale}. 
\end{theorem}
Theorem~\ref{thm:scalarTLS}, whose proof is given in~\isExtended{Appendix~\ref{sec:proof:thm:scalarTLS},}{Appendix~\ref{sec:proof:thm:scalarTLS}\arxivCite,} is based on the insight 
that the consensus set can only change at the boundaries of the intervals $[s_k - \alpha_k\barc, s_k + \alpha_k\barc]$ (Fig.~\ref{fig:consensusMax}(a)) and there are at most $2\nrTIM$ such boundaries. The theorem
also suggests a straightforward \emph{adaptive voting} algorithm to solve~\eqref{eq:TLSscale}, with pseudocode given in Algorithm~\ref{alg:adaptiveVoting}. The algorithm first builds the boundaries of the 
intervals shown in Fig.~\ref{fig:consensusMax}(a) (line~\ref{line:boundaries}). Then, for each interval, it evaluates the consensus set (line~\ref{line:consensusSet}, see also Fig.~\ref{fig:consensusMax}(b)). 
Since the consensus set does not change within an interval, we compute it at the interval centers (line~\ref{line:midPoints}, see also Fig.~\ref{fig:consensusMax}(b)). Finally, the cost of each consensus set is computed and the smallest cost is returned as optimal solution (line~\ref{line:enumeration}).

\begin{remark}[Adaptive Voting]
 The adaptive voting algorithm generalizes the histogram voting method of Scaramuzza~\cite{Scaramuzza11ijcv}  (i) to adaptively adjust the bin size in order to obtain an optimal solution and (ii) to solve a \TLS (rather than \maxConLong) formulation. Adaptive voting can be also used for \maxConLong, by simply returning the largest consensus set $\conSet_i$ in Algorithm~\ref{alg:adaptiveVoting}.
 We also refer the interested reader to the paper of Liu and Jiang~\cite{Liu19JCGS-minSumTruncatedConvexFunctions}, who recently developed a similar algorithm 
 in an independent effort and provide a generalization to 2D \TLS estimation problems.
\end{remark}

In summary, the function ${\tt estimate\_s}$ in Algorithm~\ref{alg:ITR} calls Algorithm~\ref{alg:adaptiveVoting} to compute the optimal scale estimate $\hats$, and the function ${\tt estimate\_t}$ in Algorithm~\ref{alg:ITR} calls Algorithm~\ref{alg:adaptiveVoting} three times 
(one for each entry of $\vt$) and returns the translation estimate $\hatvt = [\hat{t}_1\;\hat{t}_2\;\hat{t}_3]\tran$. 


\begin{figure}[t]
	\begin{center}
	\begin{minipage}{\columnwidth}
	\includegraphics[width=1.0\columnwidth]{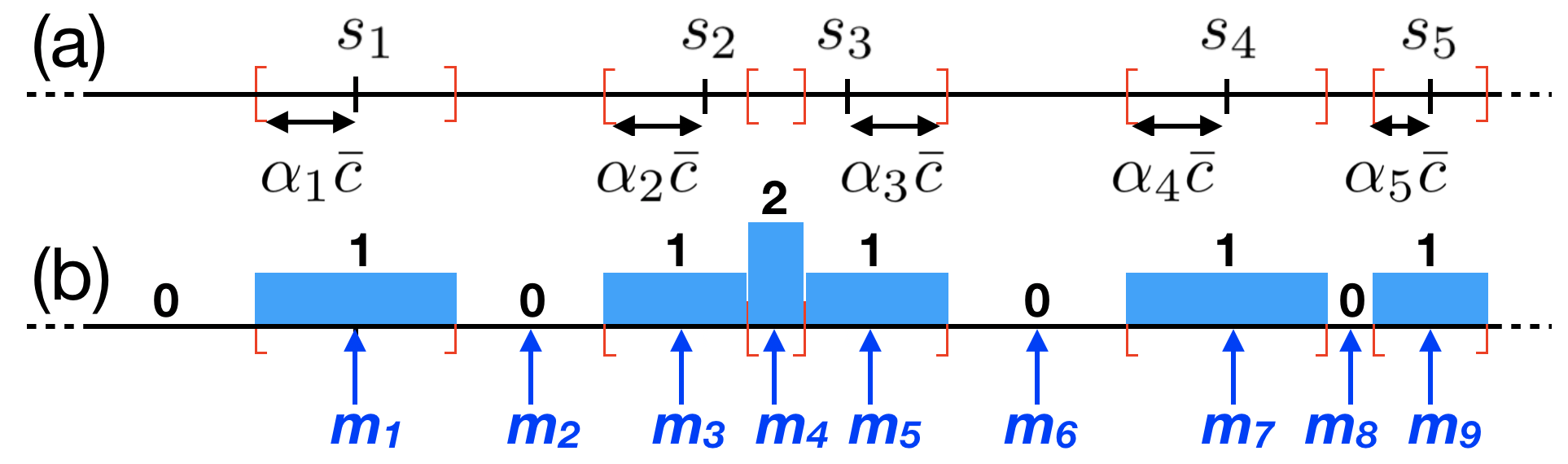}
	\end{minipage}
	\vspace{-3mm} 
	\caption{(a) confidence interval for each measurement $s_k$ (every $s$ in the $k$-th interval satisfies 
	$\frac{(s-s_k)^2}{\alpha_k^2} \leq \barcsq$;
	(b) cardinality of the consensus set for every $s$ and middle-points $m_i$ for each interval with constant consensus set. 
	 \label{fig:consensusMax}}
	\vspace{-1mm} 
	\end{center}
\end{figure} 


\begin{algorithm}[t]
\SetAlgoLined
\textbf{Input:} \ $s_k, \alpha_k, \barc$\; 
\textbf{Output:} \  $\hats$, scale estimate solving~\eqref{eq:TLSscale}\;
\% Define boundaries and sort \\
$\vv = \text{sort}([s_1 - \alpha_1\barc, s_1 + \alpha_1\barc,\ldots,s_\nrTIM - \alpha_\nrTIM\barc, s_\nrTIM + \alpha_\nrTIM\barc])$\label{line:boundaries}\\ 
\% Compute middle points \\
$m_i = \frac{\vv_{i} + \vv_{i+1}}{2}$ for $i=1,\ldots,2\nrTIM-1$ \label{line:midPoints} \\
 \% Voting \\
 \For{$i = 1,\ldots,2\nrTIM-1$}{
 	$\conSet_i = \emptyset$\\
 	\For{$k = 1,\ldots,\nrTIM$}{
		\If{$m_i \in [s_k - \alpha_k\barc, s_k + \alpha_k\barc]$}{
			$\conSet_i = \conSet_i \cup \{k\}$ \% add to consensus set \label{line:consensusSet} \\
		}
	}
}
\% Enumerate consensus sets and return best\\
 \textbf{return:} \revone{$\hats$ from Eq.~\eqref{eq:solveScaleEnumerate}.} 
\label{line:enumeration}
 \caption{Adaptive Voting.\label{alg:adaptiveVoting}}
\end{algorithm}


\spacebeforesection
\section{Robust Rotation Estimation: \\ Semidefinite Relaxation and Fast Certificates}
\label{sec:robustRotationSearch_SDPRelaxationVerification}

This section describes how to compute an optimal solution to problem~\eqref{eq:TLSrotation} or certify that a given rotation estimate is globally optimal. 
\TLS estimation is NP-hard according to Remark~\ref{rmk:tlsHardness} and 
the robust rotation estimation~\eqref{eq:TLSrotation}  has the additional complexity of involving a non-convex domain (\ie \SOthree). 
This section provides a surprising result: we can compute globally optimal solutions to~\eqref{rmk:tlsHardness} (or certify that a given estimate is optimal) in polynomial time in virtually all practical problems 
using a \emph{tight} semidefinite (SDP) relaxation. \revone{In other words, while the NP-hardness implies the presence of worst-case instances that are not solvable in polynomial time, these instances are not observed to frequently occur in practice.} 

We achieve this goal in three steps.
First,  we show that the robust rotation estimation problem~\eqref{eq:TLSrotation} can be reformulated as a \emph{Quadratically Constrained Quadratic Program} (QCQP) by adopting a quaternion formulation and using a technique we name \emph{binary cloning} (Section~\ref{sec:quatQCQP}). 
Second, we show how to obtain a \emph{semidefinite relaxation} and relax the non-convex QCQP into a convex SDP with \emph{redundant constraints} (Section~\ref{sec:quatSDPRelaxation}).
The SDP relaxation enables solving~\eqref{eq:TLSrotation} with global optimality guarantees in polynomial time. 
Third, we show that the relaxation also enables a fast \revone{\emph{optimality certifier}}, which, given a rotation guess, can test if the rotation is the optimal solution to~\eqref{eq:TLSrotation} (Section~\ref{sec:fastCertification}). 
\revone{The latter will be instrumental in developing a fast and certifiable registration approach that circumvents the time-consuming task of solving a large relaxation with existing SDP solvers.}
The derivation in Sections~\ref{sec:quatQCQP} and \ref{sec:quatSDPRelaxation} is borrowed from our previous work~\cite{Yang19iccv-QUASAR}, while the fast certification in Section~\ref{sec:fastCertification} has not been presented before.

\subsection{Robust Rotation Estimation as a QCQP}
\label{sec:quatQCQP}

This section rewrites problem~\eqref{eq:q_pointRot} as a Quadratically Constrained Quadratic Program (QCQP).
Using the quaternion preliminaries introduced in Section~\ref{sec:notationPreliminary} --and in particular eq.~\eqref{eq:q_pointRot}-- it is easy to rewrite~\eqref{eq:TLSrotation}  
using unit quaternions:
\bea \label{eq:TLSquat} 
\min_{\vq \in \usphere^3}  \sumAllIM \min \left( \frac{\Vert \hatvb_k - \vq \qProd \hatva_k \qProd \vq\inv \Vert^2}{\TIMNoiseBound_k^2}  \;,\; \barcsq \right),
\eea
where we defined 
$\hatva_k \doteq [\hats \TIMa_k\tran \; 0]\tran$ and $\hatvb_k \doteq [\TIMb_k\tran \; 0]\tran$, and ``$\qProd$'' denotes the quaternion product. 
The main advantage of using~\eqref{eq:TLSquat} is that we replaced the set \SOthree with a simpler set, the $4$-dimensional unit sphere $\usphere^3$.

{\bf From \TLS to Mixed-Integer Programming.}
Problem~\eqref{eq:TLSquat} is hard to solve globally, due to the non-convexity of both the cost function and the domain $\usphere^{3}$. As a first step towards obtaining a QCQP, we expose the non-convexity of the cost by 
rewriting the \TLS cost using binary variables. 
In particular, we rewrite the inner ``$\min$'' in~\eqref{eq:TLSquat} using the following property, that holds for any pair of scalars $x$ and $y$: 
\bea
\label{eq:minxy}
\min(x,y)=\min_{\theta \in \{+1,-1\}} \frac{1+\theta}{2}x + \frac{1-\theta}{2}y.
\eea
Eq.~\eqref{eq:minxy} can be verified to be true by inspection: the right-hand-side returns $x$ (with minimizer $\theta=+1$) 
if $x<y$, and $y$ (with minimizer $\theta=-1$) if $x>y$. This enables us to rewrite problem \eqref{eq:TLSquat} as a mixed-integer program including the quaternion
$\vq$ and binary variables $\theta_k, \; i=1,\ldots,K$:
\bea \label{eq:TLSadditiveForm}
\min_{\substack{ \vq \in \usphere^3 \\ \theta_k = \{\pm1\}} } \sumAllIM \frac{1+\theta_k}{2}\frac{\Vert \hatvb_k - \vq \qProd \hatva_k \qProd \vq\inv \Vert^2}{\TIMNoiseBound_k^2}  + \frac{1-\theta_k}{2}\barcsq.
\eea
The reformulation is related to the Black-Rangarajan duality 
between robust estimation and line processes~\cite{Black96ijcv-unification}: the \TLS cost is an extreme case of robust function that results in a binary line process.
Intuitively, the binary variables $\{\theta_k\}_{k=1}^N$ in problem~\eqref{eq:TLSadditiveForm} decide whether a given measurement $k$ is an inlier 
($\theta_k=+1$) or an outlier ($\theta_k=-1$). \revone{Exposing the non-convexity of \TLS cost function using binary variables makes the \TLS cost more favorable than other non-convex robust costs (\eg~the Geman-McClure cost used in \FGR~\cite{Zhou16eccv-fastGlobalRegistration}), because it enables binary cloning (Proposition~\ref{prop:binaryCloning}) and semidefinite relaxation (\prettyref{sec:quatSDPRelaxation}).}

{\bf From Mixed-Integer to Quaternions.}
Now we convert the mixed-integer program~\eqref{eq:TLSadditiveForm} to an optimization over $K+1$ quaternions. The intuition is that, if we define extra quaternions $\vq_k \doteq \theta_k \vq$, 
we can rewrite~\eqref{eq:TLSadditiveForm} as a function of $\vq$ and $\vq_k$ ($k=1,\ldots,K$). This is a key step towards getting a quadratic cost (Proposition~\ref{prop:qcqp}) and is formalized as follows.
\vspace{-2mm}
\begin{proposition}[Binary Cloning] \label{prop:binaryCloning}
The mixed-integer program~\eqref{eq:TLSadditiveForm} is equivalent (in the sense that they admit the same 
optimal solution $\vq$) to the following optimization:
\bea
\label{eq:TLSadditiveForm2}
\min_{\substack{ \vq \in \usphere^3 \\ \vq_k = \{\pm\vq\}} } & \hspace{-4mm}
\displaystyle\sumAllIM 
\displaystyle\frac{\Vert \hatvb_k \!-\! \vq \qProd \hatva_k \qProd \vq\inv 
+ \vq\tran\vq_k \hatvb_k \!-\! \vq \qProd \hatva_k \qProd \vq_k\inv \Vert^2  
}{4\TIMNoiseBound_k^2} \nonumber \hspace{-5mm}\\
&\displaystyle  + \frac{1-\vq\tran\vq_k}{2}\barcsq.   \hspace{-5mm}
\eea
which involves $K+1$ quaternions ($\vq$ and $\vq_k$, $i=1,\ldots,K$).
\end{proposition}
While a formal proof is given in \isExtended{Appendix~\ref{sec:proof:prop:binaryCloning},}{Appendix~\ref{sec:proof:prop:binaryCloning}\arxivCite,} it is fairly easy to see that 
if $\vq_k = \{\pm\vq\}$, or equivalently, $\vq_k = \theta_k \vq$ with $\theta_k \in \{\pm1\}$, then 
$\vq_k\tran\vq = \theta_k$, and $\vq \qProd \hatva_k \qProd \vq_k\inv =  \theta_k ( \vq \qProd \hatva_k \qProd \vq\inv)$ which exposes the relation between~\eqref{eq:TLSadditiveForm} and~\eqref{eq:TLSadditiveForm2}.  
We dubbed the re-parametrization~\eqref{eq:TLSadditiveForm2} \emph{binary cloning} since now we created a ``clone'' $\vq_k$ for each measurement, such that $\vq_k = \vq$ for inliers (recall $\vq_k = \theta_k \vq$) and $\vq_k = -\vq$ for outliers.

{\bf From Quaternions to \QCQP.}
We conclude this section by showing that~\eqref{eq:TLSadditiveForm2} can be actually written as a \QCQP. 
This observation is non-trivial since~\eqref{eq:TLSadditiveForm2} has a \emph{quartic} cost and $\vq_k = \{\pm\vq\}$ 
is not in the form of a quadratic constraint. 
The re-formulation as a \QCQP is given in the following.
\vspace{-2mm}
\begin{proposition}[Binary Cloning as a \QCQP] \label{prop:qcqp}
Define a single column vector $\vxx = [\vq\tran\ \vq_1\tran\ \dots\ \vq_K\tran]\tran$ 
stacking all variables in Problem~\eqref{eq:TLSadditiveForm2}. 
Then, Problem~\eqref{eq:TLSadditiveForm2}  
is equivalent (in the sense that they admit the same 
optimal solution $\vq$) to the following Quadratically-Constrained Quadratic Program:
\bea \label{eq:TLSBinaryClone}
\displaystyle\min_{\vxx \in \Real{4(K+1)}} & \vxx\tran \MQ \vxx \\
\subject &  \vxx_q\tran \vxx_q = 1 \hspace{30mm} \nonumber \\
& \vxx_{q_k} \vxx_{q_k}\tran = \vxx_{q} \vxx_{q}\tran, \forall \revone{k} = 1,\dots,K. \nonumber  
\eea
where $\MQ \in \calS^{4(K+1)}$ is a known symmetric matrix that depends on the \TIM measurements $\TIMa_k$ and 
$\TIMb_k$ \revone{(the explicit expression is given in \isExtended{Appendix~\ref{sec:proof:prop:qcqp}}{Appendix~\ref{sec:proof:prop:qcqp}\arxivCite})}, and the notation $\vxx_q$ (resp. $\vxx_{q_k}$) denotes the 4D subvector of $\vxx$ corresponding to $\vq$ (resp. $\vq_k$).
\end{proposition}

A complete proof of Proposition~\ref{prop:qcqp} is given in~\isExtended{Appendix~\ref{sec:proof:prop:qcqp}.}{Appendix~\ref{sec:proof:prop:qcqp}\arxivCite.}
Intuitively, (i) we developed the squares in the cost function~\eqref{eq:TLSadditiveForm2}, (ii) we used the properties of unit quaternions (Section~\ref{sec:notationPreliminary}) to simplify the expression to a quadratic cost, and (iii) we adopted the 
more compact notation afforded by the vector $\vxx$ to obtain~\eqref{eq:TLSBinaryClone}.

\subsection{Semidefinite Relaxation}
\label{sec:quatSDPRelaxation}
Problem~\eqref{eq:TLSBinaryClone} writes the \TWlong problem~\eqref{eq:TLSrotation} as a \QCQP. 
Problem~\eqref{eq:TLSBinaryClone} is still a non-convex problem (quadratic equality constraints are non-convex).
Here we develop a \emph{tight} convex semidefinite programming (SDP) relaxation for problem~\eqref{eq:TLSBinaryClone}. 

The crux of the relaxation consists in rewriting problem~\eqref{eq:TLSBinaryClone}  
as a function of the following matrix:
\bea \label{eq:matrixZ}
\MZ = \vxx \vxx\tran = \bmat{cccc}
\vq\vq\tran & \vq \vq_1\tran & \cdots & \vq\vq_K\tran \\
\star & \vq_1\vq_1\tran & \cdots & \vq_1\vq_K\tran \\
\vdots & \vdots & \ddots & \vdots \\
\star & \star & \cdots & \vq_K \vq_K\tran 
\emat \in \psdCone{4(K+1)}. 
\eea
For this purpose we note that the objective function of~\eqref{eq:TLSBinaryClone} is a linear function of $\MZ$:
\beal
\vxx\tran \MQ \vxx = 
\trace{ \MQ \vxx \vxx\tran} = \trace{\MQ \MZ},
\eeal
and that $\vxx_q\vxx_q\tran = [\MZ]_{00}$, where  
$[\MZ]_{00}$ denotes the top-left $4\times4$ diagonal block of $\MZ$, and $\vxx_{q_k} \vxx_{q_k}\tran = [\MZ]_{kk}$, where $[\MZ]_{kk}$ denotes the $k$-th diagonal block of $\MZ$ (we number the row and column block index from $0$ to $K$ for notation convenience).
Since any matrix in the form $\MZ = \vxx \vxx\tran$ is a positive-semidefinite rank-1 matrix, we obtain:

\begin{proposition}[Matrix Formulation of Binary Cloning] \label{prop:qcqpZ}
Problem~\eqref{eq:TLSBinaryClone}
is equivalent (in the sense that optimal solutions of a problem can be mapped to optimal solutions of the other) 
to the following non-convex rank-constrained program:
\bea \label{eq:qcqpZ}
\min_{\MZ \succeq 0} & \trace{\MQ \MZ} \\
\subject & \trace{[\MZ]_{00}} = 1 \nonumber \\
& [\MZ]_{kk} = [\MZ]_{00}, \;\; \forall \revone{k}=1,\dots,K \nonumber \\
& \rank{\MZ} = 1. \nonumber 
\eea
\end{proposition}

The proof is given in~\isExtended{Appendix~\ref{sec:proof:prop:matrixBinaryClone}.}{Appendix~\ref{sec:proof:prop:matrixBinaryClone}\arxivCite.}
At this point we are ready to develop our SDP relaxation 
by dropping the non-convex rank-1 constraint and adding redundant constraints.

\begin{proposition}[SDP Relaxation with Redundant Constraints] \label{prop:relaxationRedundant}
The following SDP is a convex relaxation of~\eqref{eq:qcqpZ}:
\bea \label{eq:relaxationRedundant}
\min_{\MZ \succeq 0} & \trace{\MQ \MZ} \label{eq:SDPrelax} \\
\subject 
& \trace{[\MZ]_{00}} = 1 \nonumber \\
& [\MZ]_{kk} = [\MZ]_{00}, \forall k=1,\dots,K \nonumber \\
& [\MZ]_{ij} = [\MZ]\tran_{ij}, \forall 0\leq i < j \leq K \nonumber.
\eea
\end{proposition}

The redundant constraints in the last row of eq.~\eqref{eq:SDPrelax} enforce all the off-diagonal $4\times4$ blocks of $\MZ$ to be symmetric, which is a redundant constraint for~\eqref{eq:TLSBinaryClone}, since:
\bea
[\MZ]_{ij} = \vq_i \vq_j\tran = (\theta_i \vq) (\theta_j \vq)\tran = \theta_i \theta_j \vq \vq\tran
\eea
is indeed a signed copy of $\vq\vq\tran$ and must be symmetric. \revone{Although being redundant for the QCQP~\eqref{eq:TLSBinaryClone}, these constraints are critical in tightening the SDP relaxation (\cf~\cite{Yang19iccv-QUASAR}).}
\isExtended{Appendix~\ref{sec:app-proof-globalOptimalityGuarantee}}{Appendix~\ref{sec:app-proof-globalOptimalityGuarantee}\arxivCite} provides a different approach to develop the convex relaxation using Lagrangian duality theory.

The following theorem provides readily checkable \emph{a posteriori} conditions under which~\eqref{eq:SDPrelax}  
computes an optimal solution to robust rotation estimation.

\begin{theorem}[Global Optimality Guarantee]\label{thm:quasarOptimalityGuarantee}
Let $\MZ^\star$ be the optimal solution of~\eqref{eq:relaxationRedundant}. If $\rank{\MZ^\star} = 1$, then the SDP relaxation is said to be \emph{tight}, \extraEdits{and:}
\begin{enumerate}
	\item the optimal cost of~\eqref{eq:relaxationRedundant} matches the optimal cost of the QCQP~\eqref{eq:TLSBinaryClone}, 
	\item $\MZ^\star$ can be written as $\MZ^\star = (\vxx^\star)(\vxx^\star)\tran$, where $\vxx^\star = [(\vq^\star)\tran, (\vq_1^\star)\tran,\dots,(\vq_K^\star)\tran]\tran$, with $\vq^\star_k = \theta_k^\star \vq^\star$, $\theta^\star_k \in \{\pm 1\}$, 
	\item $\pm \vxx^\star$ are the two unique global minimizers of the original QCQP~\eqref{eq:TLSBinaryClone}.
\end{enumerate}
\end{theorem}
A formal proof is given in \isExtended{Appendix~\ref{sec:app-proof-globalOptimalityGuarantee}.}{Appendix~\ref{sec:app-proof-globalOptimalityGuarantee}\arxivCite.}
Theorem~\ref{thm:quasarOptimalityGuarantee} states that if the SDP~\eqref{eq:SDPrelax} admits a rank-1 solution, then one can extract the unique global minimizer (remember: $\vq$ and $-\vq$ represent the same rotation) of the non-convex QCQP~\eqref{eq:TLSBinaryClone} from the rank-1 decomposition of the SDP solution. 

\emph{In practice, we observe that~\eqref{eq:relaxationRedundant} always returns a rank-1 solution:}
in Section~\ref{sec:experiments}, we empirically demonstrate the tightness of~\eqref{eq:SDPrelax} 
in the face of noise and extreme outlier rates {($>95\%$~\cite{Yang19iccv-QUASAR})}, and show its superior performance compared to a similar SDP relaxation using a rotation matrix parametrization (see eq.~(13) in Section V.B of~\cite{Yang19rss-teaser}).
As it is often the case, we can only partially justify with theoretical results the phenomenal performance observed in practice.
The interested reader can find a formal proof of tightness in the low-noise and outlier-free case 
in  Theorem 7 and Proposition 8 of~\cite{Yang19iccv-QUASAR}

\subsection{Fast Global Optimality Certification}
\label{sec:fastCertification}

Despite the superior robustness and (\emph{a posteriori}) global optimality guarantees, the approach outlined in the previous section requires solving the large-scale SDP~\eqref{eq:SDPrelax}, which is known to be computationally expensive 
 (\eg solving the SDP~\eqref{eq:SDPrelax} for $K=100$ takes about $1200$ seconds with \MOSEK~\cite{mosek}). On the other hand, there exist fast heuristics that compute high-quality solutions to the non-convex \TLS rotation problem~\eqref{eq:TLSrotation} in milliseconds. For example, the \emph{graduated non-convexity} (\GNC) approach in~\cite{Yang20ral-GNC}
  computes globally optimal \TLS solutions with high probability when the outlier ratio is \extraEdits{below $80\%$}, while not being 
  certifiable.

Therefore, in this section, we ask the question: \emph{given an estimate from a fast heuristic, such as \GNC,\footnote{Our certification approach also applies to \ransac, which however has the downside of being non-deterministic and typically 
less accurate  than \GNC.} can we certify the optimality of the estimate or reject it as suboptimal?} In other words: can we make \GNC certifiable?
The answer is \emph{yes}: we can leverage the insights of Section~\ref{sec:quatSDPRelaxation} to obtain a 
fast optimality \revone{certifier} that \revone{is orders of magnitude faster than solving the relaxation~\eqref{eq:relaxationRedundant}.}\footnote{\revone{\GNC (or \ransac) can be seen as a fast \emph{primal} solver that returns a rank-$1$ \emph{guess} of~\eqref{eq:relaxationRedundant} by estimating $\vxx$ (\ie~$\vq$ and $\theta_k,k=1,\dots,K$) and $\MZ = \vxx\vxx\tran$, and the optimality certifier can be seen as a fast \emph{dual-from-primal} solver that attempts to compute a dual optimality certificate for the primal rank-$1$ guess.}}
The pseudocode is presented in Algorithm~\ref{alg:optCertification} and the following theorem proves its soundness.

\begin{theorem}[Optimality Certification]\label{thm:optimalityCertification}
Given a feasible (but not necessarily optimal) solution $(\hatvq,\hattheta_1,\dots,\hattheta_K)$ of the \TLS rotation estimation problem~\eqref{eq:TLSquat}, denote with $\hatvxx = [\hatvq\tran,\hattheta_1\hatvq\tran,\dots,\hattheta_K\hatvq\tran]\tran$, and $\hatmu = \hatvxx\tran\MQ\hatvxx$, the corresponding solution and cost of the QCQP~\eqref{eq:TLSBinaryClone}.
Then,  Algorithm~\ref{alg:optCertification} produces a \emph{sub-optimality bound} $\eta$ that satisfies:
\bea \label{eq:relSuboptBound}
\frac{\hatmu - \mu^\star}{\hatmu} \leq \eta,
\eea
where $\mu^\star$ is the (unknown) global minimum of the non-convex \TLS rotation estimation problem~\eqref{eq:TLSquat} and the QCQP~\eqref{eq:TLSBinaryClone}.
Moreover, when the relaxation~\eqref{eq:relaxationRedundant} is tight and the parameter $\gamma$ in line~\ref{line:gamma} \revone{satisfies $0<\gamma<2$}, 
and if $\hatvq = \vq^\star$, $\hattheta_k = \theta^\star_k$ 
 is a global minimizer of the \TLS problem~\eqref{eq:TLSquat}, 
 then the sub-optimality bound $\eta^{(t)}$ (line~\ref{line:relativeSuboptBound}) converges to zero. 
\end{theorem}

\begin{algorithm}[t] 
\SetAlgoLined
\textbf{Input:} a feasible solution $(\hatvq,\hattheta_1,\dots,\hattheta_K)$ attaining cost $\hatmu$ in the \TLS rotation estimation problem~\eqref{eq:TLSquat}; 
homogeneous and scaled \TIMs $(\hatva_k,\hatvb_k)$, noise bounds $\TIMNoiseBound_k$ and $\barcsq$; 
maximum number of iterations $T$ (default $T=200$); 
desired relative suboptimality gap $\bareta$ (default $\bareta = 0.1\%$); \revone{initial suboptimality gap $\eta = +\infty$}; \revone{relaxation parameter $\gamma$ ($0<\gamma<2)$}  \\
\textbf{Output:} \extraEdits{bound on}  relative suboptimality gap $\eta$ \\
\% Compute data matrix $\MQ$ (Proposition~\ref{prop:qcqp}) \\
$\MQ = {\tt get\_Q}(\hatva_k,\hatvb_k,\TIMNoiseBound_k,\barcsq)$ \\
\% Compute $\barvxx$, $\barMQ$ (Appendix~\ref{sec:app-optimalityCertification}) \\
$\barvxx = [1,\hattheta_1,\dots,\hattheta_K]\tran \kron [0\;0\;0\;1]\tran$, \quad $\barMQ = \hatOmega_q\tran \MQ \hatOmega_q$ \\
\% \revone{Compute initial dual certificate $\barMM^{(0)} \in \affineSub$} \\
$\barMM^{(0)} = \barMQ - \hatmu\MJ + {\tt get\_Delta0}(\hatva_k,\hatvb_k,\TIMNoiseBound_k,\barcsq,\barMQ,\barvxx)$ \label{line:initialGuess}\\
\% \revone{Douglas-Rachford Splitting} \\
\For{$t = 0,\dots,T$}{
	\revone{$\barMM^{(t)}_{\psdSub} = \Pi_{\psdSub} (\barMM^{(t)})$} \% Project to $\psdSub$ \label{line:projectK} \\
	\revone{$\barMM^{(t)}_{\affineSub} = \Pi_{\affineSub} (2\barMM^{(t)}_{\psdSub} - \barMM^{(t)})$} \% Project to $\affineSub$ \label{line:projectM} \\
	\revone{$\barMM^{(t+1)} = \barMM^{(t)} + \gamma(\barMM^{(t)}_{\affineSub} - \barMM^{(t)}_{\psdSub})$} \% Relaxation \label{line:gamma} \\
	\% Compute relative suboptimality bound \\
	\revone{$\lambda_1^{(t)} = {\tt get\_min\_eig}(\barMM^{(t)}_{\affineSub})$, $\eta^{(t)} = \frac{|\lambda_1^{(t)}|(K+1)}{\hatmu}$} \label{line:relativeSuboptBound}\\
	\algorithmicif\ $\eta^{(t)} < \eta$\ \algorithmicthen\ $\eta = \eta^{(t)}$\ \algorithmicend \\
	\algorithmicif\ $\eta < \bareta$\ \algorithmicthen\ break\ \algorithmicend
}
\textbf{return:} $\eta$
\caption{Optimality Certification. \label{alg:optCertification}}
\end{algorithm}

Although a complete derivation of Algorithm~\ref{alg:optCertification} and the proof of Theorem~\ref{thm:optimalityCertification} is postponed to \isExtended{Appendix~\ref{sec:app-optimalityCertification},}{Appendix~\ref{sec:app-optimalityCertification}\arxivCite,} now we describe  our key intuitions. 
The first insight is that, given \revone{a globally optimal solution $(\vq^\star,\theta_1^\star,\dots,\theta_K^\star)$ to the non-convex \TLS problem~\eqref{eq:TLSquat}, whenever the SDP relaxation~\eqref{eq:SDPrelax} is tight, $\MZ^\star = (\vxxstar)(\vxxstar)\tran$, where $\vxxstar = [(\vqstar)\tran,(\theta_1^\star\vqstar)\tran,\dots,(\theta_K^\star\vqstar)\tran]$, is a globally optimal solution to the convex SDP~\eqref{eq:relaxationRedundant}. Therefore, we can compute a certificate of global optimality from the Lagrangian dual SDP of the primal SDP~\eqref{eq:relaxationRedundant}~\cite[Section 5]{Boyd04book}.}\footnote{\revone{The tight SDP relaxation ensures certifying global optimality of the non-convex problem (NP-hard) is equivalent to certifying global optimality of the convex SDP (tractable).}}
 
 The second insight is that Lagrangian duality shows that the dual certificate is a matrix 
 at the intersection between the positive semidefinite cone $\psdCone{}$  and an affine subspace $\affineSub$ (whose expression is derived in 
 \isExtended{Appendix~\ref{sec:app-optimalityCertification}).}{Appendix~\ref{sec:app-optimalityCertification}\arxivCite).} Algorithm~\ref{alg:optCertification} looks for such a matrix ($\barMM^{(t)}_{\affineSub}$ in line~\eqref{line:projectM})
  using \revone{\emph{Douglas-Rachford Splitting} (\DRS)~\cite{Henrion12handbook-conicProjection,Bauschke96SIAM-projectionFeasibility,Higham88LA-nearestSPD,Yang20arXiv-certifiablePerception,Combettes11book-proximalSplitting,Jegelka13NIPS-DRSreflection} (line~\ref{line:projectK}-\ref{line:gamma}),} 
   starting from a cleverly chosen initial guess (line~\ref{line:initialGuess}).
\revone{Since \DRS guarantees convergence to a point in $\psdCone{} \cap \calM$ from any initial guess whenever the intersection is non-empty, Algorithm~\ref{alg:optCertification} guarantees to certify global optimality if provided with an optimal solution, whenever the relaxation~\eqref{eq:SDPrelax} is tight.}

\revone{The third insight is that even when the provided solution 
is not globally optimal, or when the relaxation is not tight,}  Algorithm~\ref{alg:optCertification} can still compute a \emph{sub-optimality bound} from the minimum eigenvalue of the dual certificate matrix (line~\ref{line:relativeSuboptBound}), which is informative of the quality of the candidate solution. To the best of our knowledge, this is the first result that can assess the sub-optimality of outlier-robust estimation beyond the instances commonly 
investigated in statistics~\cite{Diakonikolas16focs-robustEstimation}.

Finally, \revone{we further speed up computation by exploiting the problem structure.} 
\isExtended{Appendix~\ref{sec:app-closedFormProjections}}{Appendix~\ref{sec:app-closedFormProjections}\arxivCite} shows that the projections on lines~\ref{line:projectK} and \ref{line:projectM} of Algorithm~\ref{alg:optCertification} can be computed in closed form, and
 \isExtended{Appendix~\ref{sec:app-initialGuess}}{Appendix~\ref{sec:app-initialGuess}}
 provides the expression of the initial guess in line~\ref{line:initialGuess}.
 Section~\ref{sec:experiments} shows the effectiveness of Algorithm~\ref{alg:optCertification} in certifying the solutions from \GNC~\cite{Yang20ral-GNC}.

\vspace{-2mm}
\section{Performance Guarantees}
\label{sec:guarantees}


This section establishes theoretical guarantees for \name. 
While Theorems~\ref{thm:scalarTLS},~\ref{thm:quasarOptimalityGuarantee}, and~\ref{thm:optimalityCertification} establish when we obtain a globally optimal solution from optimization problems involved in \name, this section uses these theorems to bound the distance between the estimate produced by \name and the ground truth.
 Section~\ref{sec:guaranteesNoiseless} investigates the noiseless (but outlier-corrupted) case, 
 while Section~\ref{sec:guaranteesNoisy} considers the general case with noise and outliers.


\subsection{Exact Recovery with Outliers and Noiseless Inliers}
\label{sec:guaranteesNoiseless}

We start by analyzing the case in which the inliers are noiseless since it allows deriving stronger 
performance guarantees and highlights the  fundamental difference the outlier generation mechanism can make.
As we show below, in a noiseless case with randomly outliers generated, \name is able to recover the \gtadj transformation from only 3 inliers (and an arbitrary number of outliers!). 
On the other hand, when the outlier generation is adversarial, \name (and any other approach) 
requires that more than 50\% of the measurements are inliers in order to retrieve the ground truth.



\begin{theorem}[Estimation Contract with Noiseless Inliers and Random Outliers]
\label{thm:certifiableRegNoiseless1}
Assume (i) the set of correspondences contains at least 3 noiseless non-collinear and distinct\footnote{Distinct in the sense that no two points occupy the same 3D location.} inliers, 
where for an inlier $i$, it holds $\vb_i = \sgt \MRgt \va_i + \vtgt$, and 
 $(\sgt, \MRgt, \vtgt)$ denotes the \gtadj transformation we want to recover. 
 Assume (ii) 
 the outliers are in \emph{generic position} (\eg they are perturbed by random noise).
  If
 (iii) the inliers computed by \name (in each subproblem) have zero residual error for the corresponding 
 subproblem, and (iv) the rotation subproblem produces a valid certificate (in the sense of Theorems~\ref{thm:quasarOptimalityGuarantee} and~\ref{thm:optimalityCertification}),
then with probability 1 the output $(\hats, \hatMR, \hatvt)$ of \name exactly recovers the ground truth,
 \ie $\hats=\sgt, \hatMR=\MRgt, \hatvt=\vtgt$.
\end{theorem}

A proof of Theorem~\ref{thm:certifiableRegNoiseless1} is given in \isExtended{Appendix~\ref{sec:proof:certifiableRegNoiseless1}.}{Appendix~\ref{sec:proof:certifiableRegNoiseless1}\arxivCite.} 
Conditions (iii) and (iv) can be readily checked using the solution computed by \name, in the spirit of 
 certifiable algorithms. 
 Assumptions (i) and (ii) only ensure that the problem is well-posed.
 In particular assumption (i) is the same assumption required to compute a unique transformation without outliers~\cite{Horn87josa}.
 Assumption (ii) is trickier since it requires \emph{assuming} the 
 outliers to be ``generic'' rather than adversarial: 
 intuitively, if we allow an adversary to perturb an arbitrary number of points, 
 s/he can create an identical copy of the points in $\calA$ with an arbitrary transformation and induce \emph{any} 
 algorithm into producing an incorrect estimate. Clearly, in such a setup outlier rejection is ill-posed and 
 no algorithm can recover the correct solution.\footnote{In such a case, it might still be feasible to enumerate 
 all potential solutions, but the optimal solution of \maxConLong or the robust estimator~\eqref{eq:TLSRegistration} can no longer be guaranteed to be close to the ground truth.} 

This fundamental limitation is captured by the following theorem, which focuses on the case with adversarial outliers.

\begin{theorem}[Estimation Contract with Noiseless Inliers and Adversarial Outliers]
\label{thm:certifiableRegNoiseless2}
Assume (i) the set of correspondences is such that the number of inliers $\nrIn$ and the number of outliers
$\nrOut$ satisfy $\nrIn \geq \nrOut+3$ and no three inliers are collinear.  
 If
 (ii) the inliers computed by \name (in each subproblem) have zero residual error for the corresponding 
 subproblem, and (iii) the rotation subproblem produces a valid certificate (in the sense of Theorems~\ref{thm:quasarOptimalityGuarantee} and~\ref{thm:optimalityCertification}), and 
 {(iv) \name returns a consensus set that satisfies (i),}
then the output $(\hats, \hatMR, \hatvt)$ of \name exactly recovers the ground truth 
transformation, \ie $\hats=\sgt, \hatMR=\MRgt, \hatvt=\vtgt$.
\end{theorem}

A proof of Theorem~\ref{thm:certifiableRegNoiseless2} is given in \isExtended{Appendix~\ref{sec:proof:certifiableRegNoiseless2}.}{Appendix~\ref{sec:proof:certifiableRegNoiseless2}\arxivCite.}
We observe that now Theorem~\ref{thm:certifiableRegNoiseless2} provides an even clearer ``contract'' 
for \name: as long as the number of inliers is sufficiently larger than the number of outliers,
 \name provides readily-checkable conditions on whether it recovered the \gtadj transformation. If the percentage of inliers is below $50\%$ (\ie $\nrIn \leq \nrOut$), \name will attempt to retrieve the transformation that is consistent with the largest set of inliers (in the proof, we show that under the conditions of the theorem, \name produces a maximum consensus solution). 

To the best of our knowledge, Theorems~\ref{thm:certifiableRegNoiseless1}-\ref{thm:certifiableRegNoiseless2}
 provide the first exact recovery results for registration with outliers. 

\subsection{Approximate Recovery with Outliers and Noisy Inliers}
\label{sec:guaranteesNoisy}

The case with noisy inliers is strictly more challenging than the noise-free case.
Intuitively, the larger the noise, the more blurred is the boundary between inliers and outliers.
From the theoretical standpoint, this translates into weaker guarantees.

\begin{theorem}[Estimation Contract with Noisy Inliers and Adversarial Outliers]
\label{thm:certifiableRegNoisy}
Assume (i) the set of correspondences contains at least 4 non-coplanar inliers, where
 for any inlier $i$, $\| \vb_i - \sgt \MRgt \va_i - \vtgt \| \leq \beta$, and 
 $(\sgt, \MRgt, \vtgt)$ denotes the \gtadj transformation. 
 Assume (ii) the inliers belong to the maximum consensus set in each 
 subproblem\footnote{In other words: the inliers belong to the largest set $\calS$  
 such that for any $i,j \in \calS$, 
 $| s_{ij} - s | \leq \alpha_{ij}$, 
 $\| \TIMb_{ij} - s R \TIMa_{ij}\| \leq \TIMNoiseBound_{ij}$, 
 and 
$\left| [\vb_{i} - s R \va_{i} - \vt]_l \right| \leq \beta_{i},\forall l=1,2,3$ 
for any transformation $(s,\MR,\vt)$. \optional{Note that for translation the assumption is applied component-wise.}{} }, and 
 (iii) the second largest consensus set is ``sufficiently smaller'' than the maximum consensus set 
(as formalized in \isExtended{Lemma~\ref{thm:TLSandMC}}{Lemma~\ref{thm:TLSandMC}\arxivCite}).
 If
 (iv) the rotation subproblem produces a valid certificate 
 (Theorems~\ref{thm:quasarOptimalityGuarantee}-\ref{thm:optimalityCertification}), 
 the output $(\hats, \hatMR, \hatvt)$ of \name satisfies:
\bea
|\hats - \sgt| &\leq& 2 \max_{ij} \alpha_{ij},\label{eq:bounds_scale} \\
\quad
\| \sgt \MRgt - \hats \hatMR \|_\frob &\leq& 2\sqrt{3}
 \frac{ \max_{ij} \alpha_{ij} }{ \min_{ijhk} \sigma_\min(\MU_{ijhk}) }
\label{eq:bounds_rot} \\ 
\quad 
\|\hatvt - \vtgt \| &\leq& (9+3\sqrt{3}) \beta, 
 \label{eq:bounds_tran}
\eea
where $\sigma_\min(\cdot)$ denotes the smallest singular value of a matrix, and $\MU_{ijhk}$ is the $3\times 3$ matrix 
$\MU_{ijhk} = 
\left[\frac{\TIMa_{ij}}{\|\TIMa_{ij}\|} \;\; 
\frac{\TIMa_{ih}}{\|\TIMa_{ih}\|} \;\;
\frac{\TIMa_{ik}}{\|\TIMa_{ik}\|}
\right]$. Moreover, $\sigma_\min(\MU_{ijhk})$ is nonzero as long as (every 4 points $\va_i$ chosen among) the inliers are not coplanar. 
\end{theorem}

A proof of Theorem~\ref{thm:certifiableRegNoisy} is given in \isExtended{Appendix~\ref{sec:proof:certifiableRegNoisy}.}{Appendix~\ref{sec:proof:certifiableRegNoisy}\arxivCite.}
To the best of our knowledge, Theorem~\eqref{thm:certifiableRegNoisy} provides the first performance bounds for noisy registration problems with outliers.
Condition (i) is similar to the standard assumption that guarantees the existence of a unique alignment in the outlier-free case, with the exception that we now require 4 non-coplanar inliers.
 While conditions (ii)-(iii) seem different and less intuitive than the ones in~Theorem~\ref{thm:certifiableRegNoiseless2}, they 
 are a generalization of conditions (i) and (ii) in~Theorem~\ref{thm:certifiableRegNoiseless2}. 
 Condition (ii) in~Theorem~\ref{thm:certifiableRegNoisy} ensures that the inliers form a large consensus set 
 (the same role played by condition (i) in Theorem~\ref{thm:certifiableRegNoiseless2}). 
 Condition (iii) in~Theorem~\ref{thm:certifiableRegNoisy} ensures that the largest consensus set cannot be  confused 
 with other large consensus sets (the assumption of zero residual played a similar role in Theorem~\ref{thm:certifiableRegNoiseless2}). These conditions essentially limit the 
 amount of ``damage'' that outliers can do by avoiding that they form large sets of mutually consistent measurements. 

 \begin{remark}[Interpretation of the Bounds]
 The bounds~\eqref{eq:bounds_scale}-\eqref{eq:bounds_tran} have a natural geometric interpretation.
 First of all, we recall from Section~\ref{sec:decoupling} that we defined $\alpha_{ij} = \TIMNoiseBound_{ij} / \| \TIMa_{ij} \|$.
 Hence $\alpha_{ij}$ can be understood as (the inverse of) a signal to noise ratio: 
$\TIMNoiseBound_{ij}$ measures the amount of noise, while $\| \TIMa_{ij} \| = \| \va_{j} - \va_{i}  \|$ measures the size of the point cloud (in terms of distance between two points). 
Therefore, the scale estimate will worsen when the noise becomes large when compared to the size of the point cloud.
The parameter $\alpha_{ij}$ also influences the rotation bound~\eqref{eq:bounds_rot}, which, however, also depends on the smallest singular value of the matrix $\sigma_\min(\MU_{ijhk})$. 
The latter measures how far is the point cloud from degenerate configurations\optional{ where the rotation cannot be 
uniquely identified.}{.} 
Finally, the translation bound does not depend on the scale of the point cloud, since for a given scale and rotation a single point would suffice to compute a translation (in other words, the distance among points does not play a role there).
 \end{remark}  

 The interested reader can find tighter (but more expensive to compute and less intelligible) bounds, under similar assumptions as Theorem~\ref{thm:certifiableRegNoisy}, in~\isExtended{Appendix~\ref{sec:tighterBounds}.}{Appendix~\ref{sec:tighterBounds}\arxivCite.} 

\spacebeforesection
\section{\namepp: A Fast C++ Implementation}
\label{sec:teaserpp}

In order to showcase the real-time capabilities of \name in real robotics, vision, and graphics applications, we have developed a fast C++ implementation of \name, named \namepp.
\namepp has been released as an open-source library and can be found at \urlTEASER.
This section describes the algorithmic and implementation choices we made in \namepp.

\namepp follows the same decoupled approach described in Section~\ref{sec:teaser}.
The only exception is that it circumvents the need to solve a large-scale SDP and uses 
the \GNC approach described in~\cite{Yang20ral-GNC} for the \TLS rotation estimation problem. 
As described in Section~\ref{sec:fastCertification}, we essentially run \GNC and certify \emph{a posteriori} that it retrieved the globally optimal solution using \revone{Algorithm~\ref{alg:optCertification}.}
 This approach is very effective in practice since --as shown in the experimental section--
the scale estimation (Section~\ref{sec:scaleEstimation}) 
and the max clique inlier selection (Section~\ref{sec:maxClique}) are already able to remove a large
 number of outliers, hence \GNC only needs to solve a more manageable problem with less outliers (a setup in which it has been shown to be very effective~\cite{Yang20ral-GNC}).
  Therefore, by combining a fast heuristic (\GNC) with \revone{a scalable optimality certifier (\DRS)}, \namepp is the first fast and certifiable registration algorithm in that 
   (i) it enjoys the performance guarantees of Section~\ref{sec:guarantees}, and (ii) it runs in milliseconds in practical problem instances.

\namepp uses \Eigen for fast linear algebra operations.
To further optimize for performance, we have implemented a large portion of \namepp using shared-memory parallelism. Using OpenMP~\cite{OpenMP08}, an industry-standard interface for scalable parallel programming, we have parallelized Algorithm~\ref{alg:adaptiveVoting}, as well as the computation of 
the \TIMs and \TRIMs. In addition, 
we use a fast exact parallel  maximum clique finder algorithm~\cite{Rossi15parallel} to find the maximum clique for outlier pruning. 

To facilitate rapid prototyping and visualization, in addition to the C++ library,
 we provide Python and MATLAB bindings.
Furthermore, we provide a ROS~\cite{Quigley09icra-ros} wrapper 
 to enable easy integration and deployment in real-time robotics applications.
\spacebeforesection
\section{Experiments and Applications}
\label{sec:experiments}

The goal of this section is to 
(i) test the performance of each module
presented in this paper, including scale, rotation, translation estimation, the \mcis 
pruning, and the optimality certification Algorithm~\ref{alg:optCertification} (\prettyref{sec:separateSolver}), 
(ii) show that \name and \namepp outperform the state of the art 
 on standard benchmarks (\prettyref{sec:benchmark}),
(iii) show that \namepp also solves 
Simultaneous Pose and Correspondence (\SPC) problems 
and 
dominates existing methods, such as \ICP (Section~\ref{sec:bruteForce}), 
(iv) show an application of \name for object localization in an RGB-D robotics dataset (\prettyref{sec:roboticsApplication1}), 
and (v) show an application of \namepp to difficult scan-matching problems with deep-learned 
correspondences (Section~\ref{sec:roboticsApplication2}).

\myParagraph{Implementation Details} 
\namepp is implemented as discussed in Section~\ref{sec:teaserpp}.
\name is implemented  in MATLAB, uses \cvx~\cite{CVXwebsite} to solve the convex relaxation~\eqref{eq:SDPrelax}, and 
uses the algorithm in~\cite{Eppstein10isac-maxCliques} to find the maximum clique in the pruned \TIM graph (see~\prettyref{thm:maxClique}). 
In all tests we set $\barc=1$. {All tests are run on a laptop with an i7-8850H CPU and 32GB of RAM.}


\subsection{Testing \name's and \namepp's Modules}
\label{sec:separateSolver}


\myParagraph{Testing Setup}
We use the \bunny~\edit{point cloud} from the Stanford 3D Scanning Repository~\cite{Curless96siggraph}. \revone{The bunny is downsampled to $\nrPoints=40$ points (unless specified otherwise) and resized to fit inside the $[0,1]^3$ cube to create point cloud $\calA$.} 
\revone{To create point cloud $\calB$ with $\nrPoints$ correspondences, first a random transformation $(s, \MR, \vt)$ (with $1\leq s \leq 5$, $\|\vt \|\leq\!\!1$, and $\MR \in \SOthree$) is applied according to eq.~\eqref{eq:robustGenModel}, and then random bounded noise $\veps_i$'s are added (we sample $\veps_i \sim \calN(\bm{0},\sigma^2 \MI)$ until $\|\veps_i\| \leq\!\!\beta_i$).} 
 We set $\sigma = 0.01$ and $\beta_i = 5.54 \sigma$ such that $\mathbb{P}\left(\|\veps_i \|^2  > \beta_i^2 \right) \leq 10^{-6}$ (\cf~Remark~\ref{rmk:chi2} in 
 \isExtended{Appendix~\ref{sec:choiceBeta}).}{Appendix~\ref{sec:choiceBeta} in~\cite{Yang20arxiv-teaser}).}
  To generate \revone{outlier correspondences}, we replace a fraction of the $\vb_i$'s with vectors uniformly sampled inside the sphere of radius 5.
  We test increasing outlier rates from $0\%$ (all inliers) to $90\%$; we test up to $99\%$ outliers when relevant (\eg we omit testing above $90\%$ if failures are already observed at $90\%$).
  All statistics are computed over 40 Monte Carlo runs unless mentioned otherwise.

\myParagraph{Scale Estimation} Given two point clouds $\calA$ and $\calB$, we first create $\nrPoints (\nrPoints-1)/2$ \TIMs corresponding to a complete graph and then use Algorithm~\ref{alg:adaptiveVoting} to solve for the scale. We compute both \emph{\maxConLong}~\cite{Speciale17cvpr-consensusMaximization} and \TLS estimates of the scale.
Fig.~\ref{fig:benchmark_separate_solvers}(a) shows box plots of the scale error with increasing outlier ratios. \edit{The scale error is computed as $|\hats - \sgt|$, where $\hats$ is the scale estimate and $\sgt$ is the ground-truth}. We observe the \TLS solver is robust against 80\% outliers, while \maxConLong failed \revone{twice} in that regime.


\renewcommand{\mpw}{4.3cm}

\begin{figure}[t]
	\begin{center}
	\begin{minipage}{\textwidth}
	\begin{tabular}{cc}%
	        \myhspace \hspace{-2mm}
			\begin{minipage}{\mpw}%
			\centering%
			\includegraphics[width=1\columnwidth]{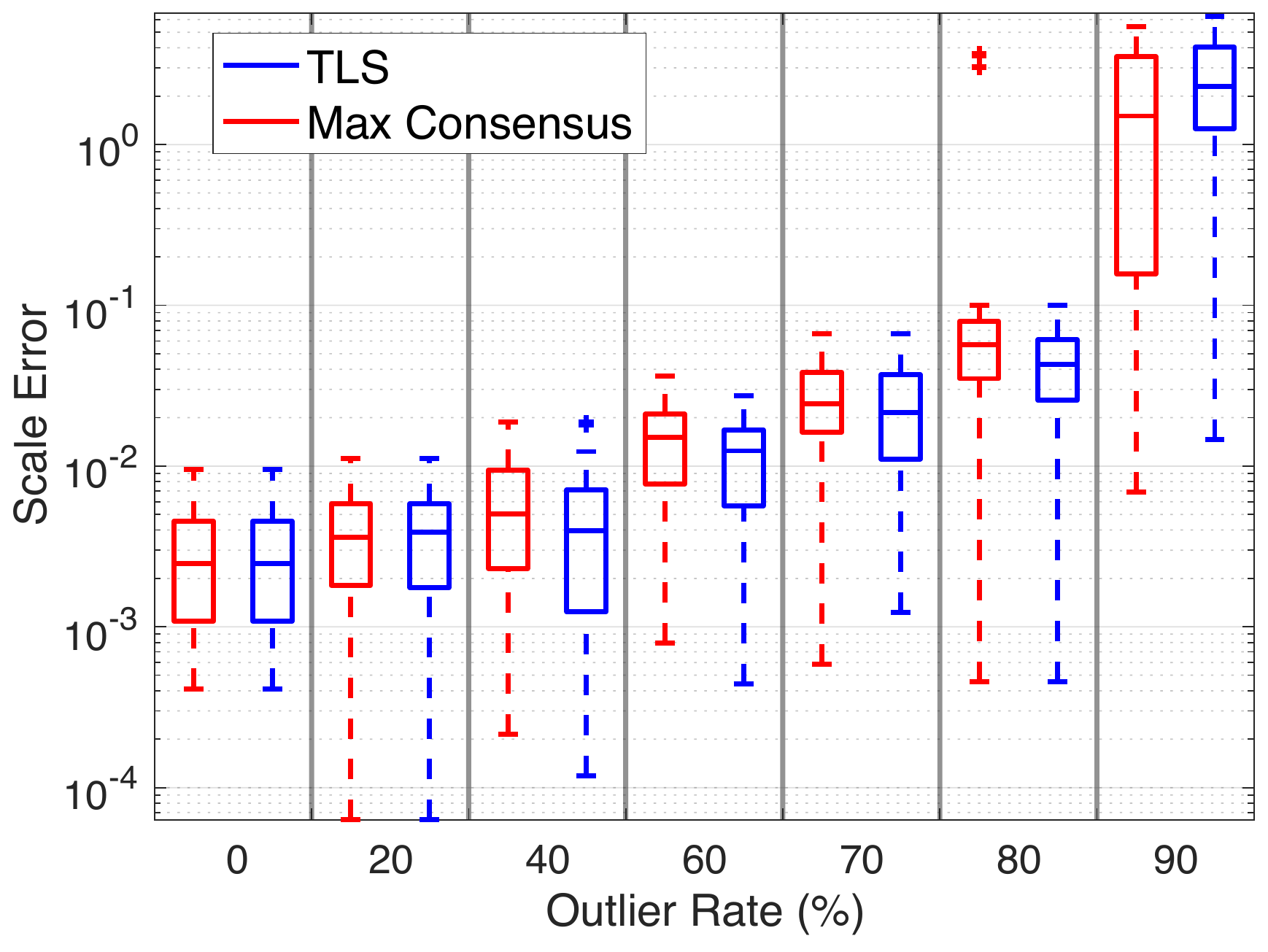} \\
			\vspace{-1mm}
			(a) Scale Estimation
			\end{minipage}
              & \myhspace \hspace{-3mm}
			\begin{minipage}{\mpw}%
			\centering%
			\includegraphics[width=1\columnwidth]{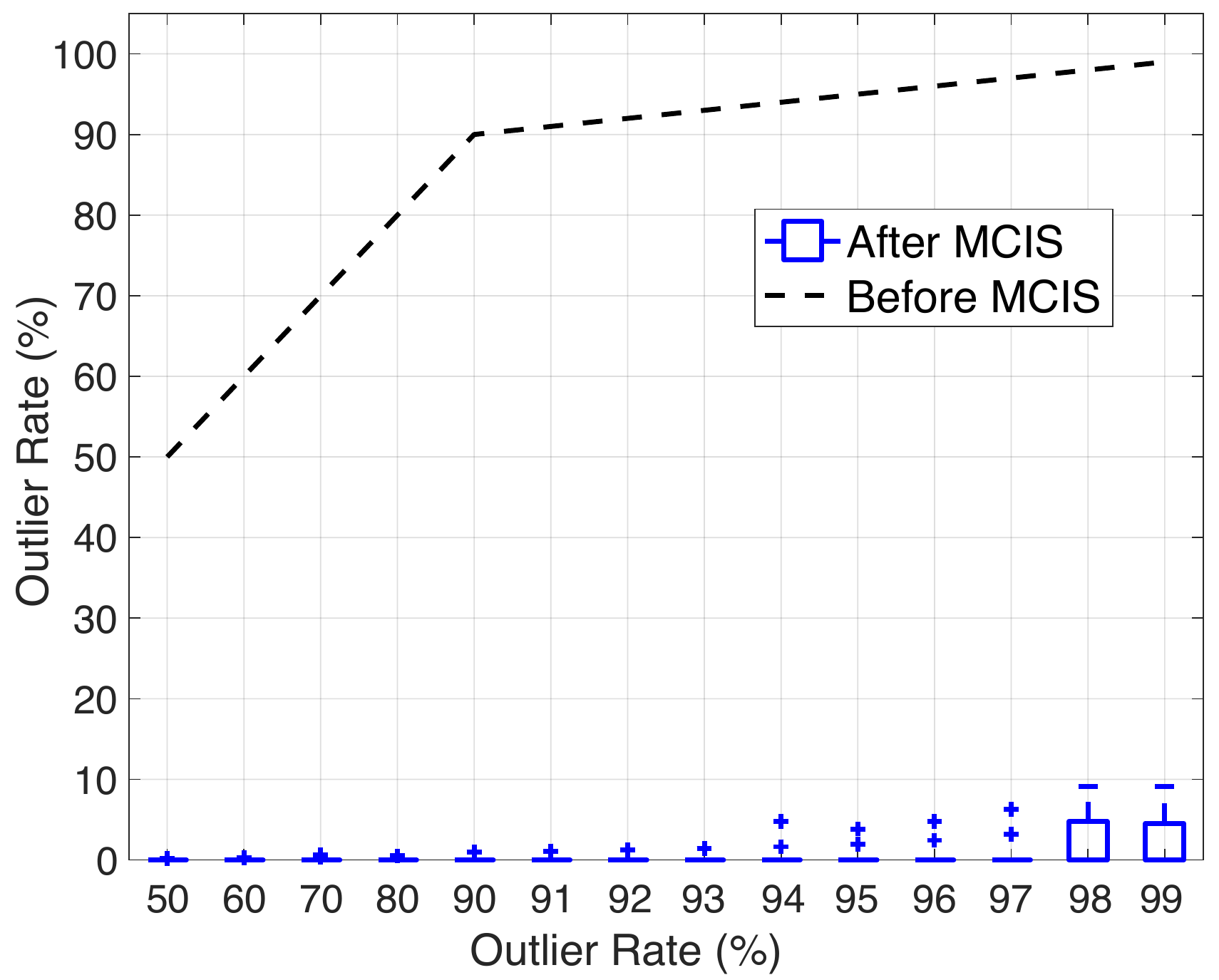} \\
			
			\vspace{-2mm}
			(b) \mcis
			\end{minipage} 
             \\
             \myhspace \hspace{-3mm}
			\begin{minipage}{\mpw}%
			\centering%
			\includegraphics[width=1.02\columnwidth]{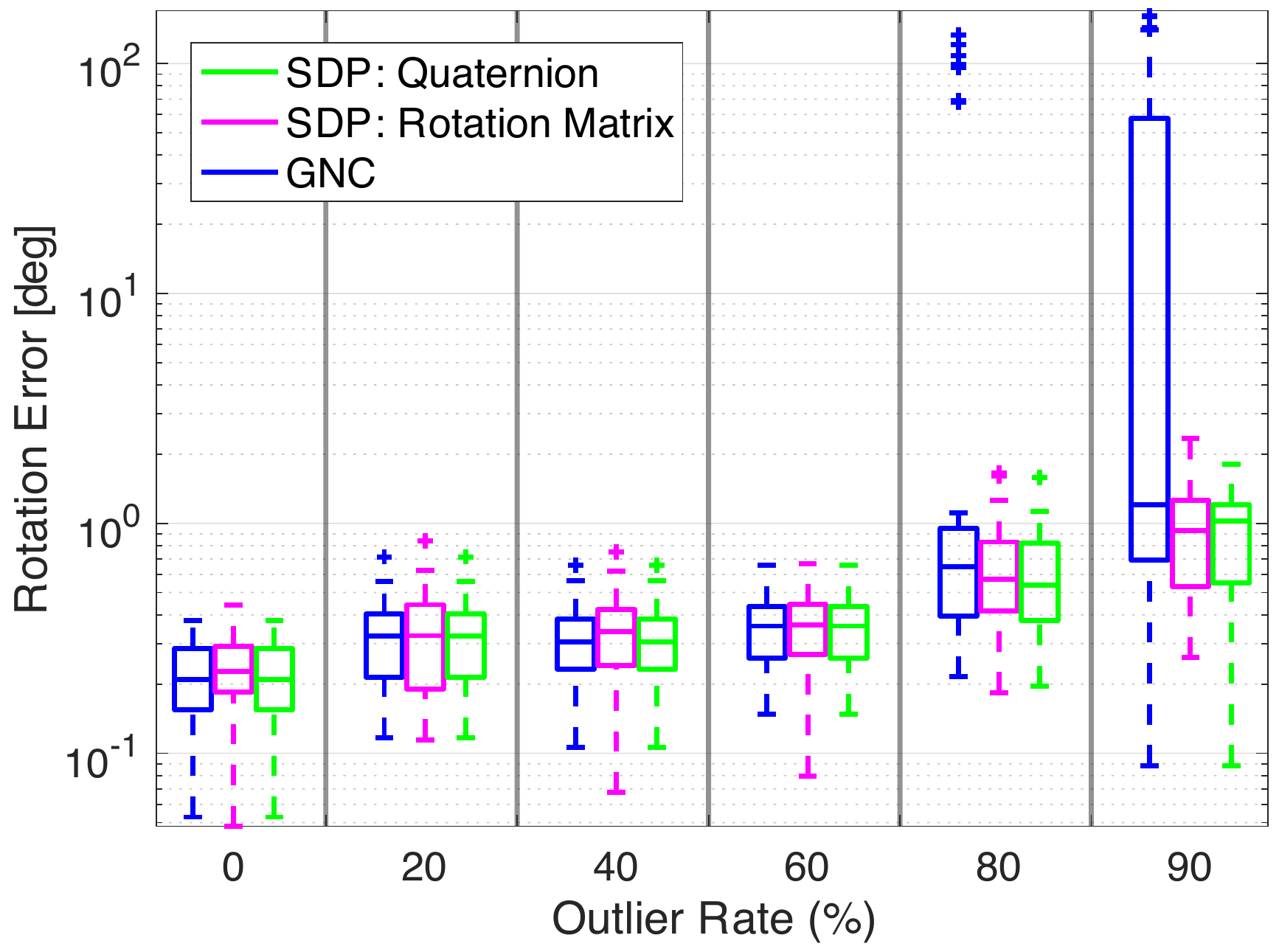} \\
			\vspace{-1mm}
			(c) Rotation Estimation
			\end{minipage}
			& \myhspace \hspace{-4mm}
			\begin{minipage}{\mpw}%
			\centering%
			\includegraphics[width=1\columnwidth]{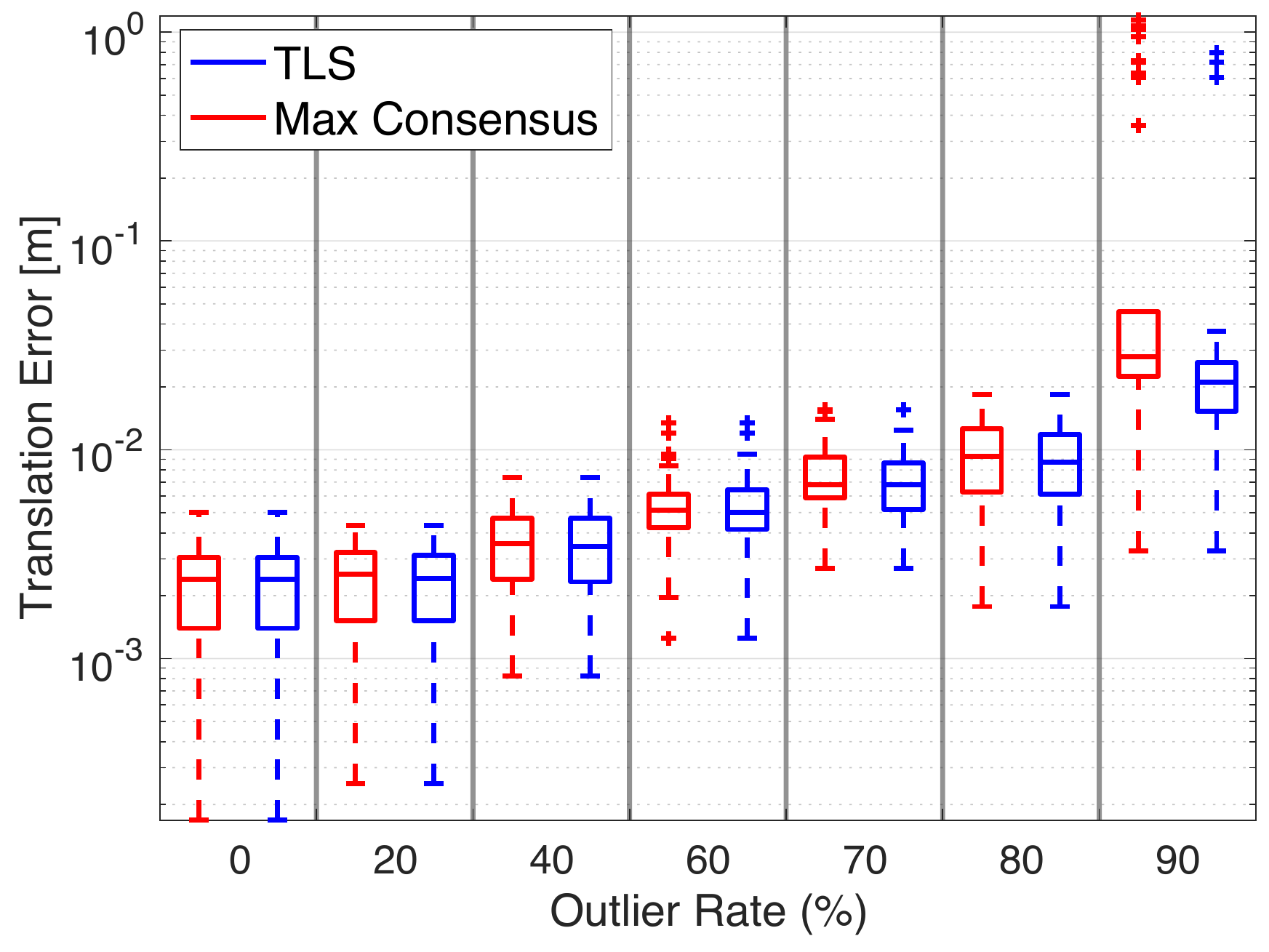} \\
			\vspace{-1mm}
			(d) Translation Estimation
			\end{minipage}
		\end{tabular}
	\end{minipage}
	\begin{minipage}{\textwidth}
	\end{minipage}
	\vspace{-4mm} 
	\caption{ Results for scale, rotation, translation estimation, and impact of  maximal clique inlier selection (\mcis) for increasing outlier rates.
	 \label{fig:benchmark_separate_solvers}}
	 	\vspace{-1mm} 
	\end{center}
\end{figure}

\myParagraph{Maximal Clique Inlier Selection (\mcis)} 
We downsample \bunny to $\nrPoints=1000$ and fix the scale to $s=1$ when applying the random transformation. 
We first prune the outlier \TIMs/\TRIMs (edges) that are not consistent with the scale $s=1$, while keeping all the points (vertices), to obtain 
the graph $\calG(\calV,\calE')$. Then we compute the \edit{\emph{maximum clique}}
in $\calG(\calV,\calE')$ 
and remove all edges and vertices outside the maximum clique, obtaining a pruned graph $\calG'(\calV',\calE'')$ \revone{(\cf~line~\ref{line:mcis} in Algorithm~\ref{alg:ITR})}.
Fig.~\ref{fig:benchmark_separate_solvers}(b) shows the outlier ratio in $\calG$ (label: ``Before \mcis'') and $\calG'$ (label: ``After \mcis'').
\mcis effectively reduces outlier rate to below 10\%, facilitating 
rotation and translation estimation, which, in isolation, can already tolerate more than \maxOutliersRot outliers ({$<90\%$ when using \GNC}). 

\myParagraph{Rotation Estimation}
We simulate \TIMs by 
applying a random rotation $\MR$ to the \bunny, and fixing $s=1$ and $\vt = \zero$ (we set the number of \TIMs to $K = 40$).
We compare three approaches to solve the \TLS rotation estimation problem~\eqref{eq:TLSrotation}:
 (i) the quaternion-based SDP relaxation in eq.~\eqref{eq:SDPrelax} and~\cite{Yang19iccv-QUASAR} \revone{(SDP: Quaternion)}, 
 (ii) the rotation-matrix-based SDP relaxation we proposed in~\cite{Yang19rss-teaser} \revone{(SDP: Rotation Matrix)}, and 
 (iii) the \GNC heuristic in~\cite{Yang20ral-GNC}.
 For each approach, we evaluate the rotation error as
{\smaller $\left| \arccos \left( \left(\trace{\hatMR\tran\MRgt} - 1 \right) / 2 \right) \right|$},  
which is the geodesic distance between the rotation estimate $\hatMR$ (produced by each approach) and the ground-truth $\MRgt$~\cite{Hartley13ijcv}.   


Fig.~\ref{fig:benchmark_separate_solvers}(c) reports the rotation error for increasing outlier rates. \revone{The \GNC heuristic performs well below $80\%$ outliers and then starts failing.}
 The two relaxations ensure similar performance, while the quaternion-based relaxation proposed in this paper 
 is slightly more accurate at high outlier rates. 
 Results in~\isExtended{Appendix~\ref{sec:app-compare_quaternion_rotationMatrix}}{Appendix~\ref{sec:app-compare_quaternion_rotationMatrix}\arxivCite} show that the quaternion-based relaxation is always tighter than the relaxation in~\cite{Yang19rss-teaser}, which often translates into better estimates. 
 The reader can find more experiments using our quaternion-based relaxation in~\cite{Yang19iccv-QUASAR}, where we also demonstrate its robustness against over $95\%$ outliers and discuss applications to panorama stitching.

 While the quaternion-based relaxation dominates in terms of accuracy and robustness, it requires solving a large SDP.
  Therefore, in \namepp, we opted to use the fast \GNC heuristic instead: this choice is motivated by the observation that \mcis is already able to remove most of the outliers (Fig.~\ref{fig:benchmark_separate_solvers}(b)) hence within \namepp the rotation estimation only requires solving a problem with less than 10\% outliers.

\begin{figure}[t]
\begin{center}
	\begin{minipage}{\columnwidth}
	\begin{center}
			\includegraphics[width=0.7\columnwidth]{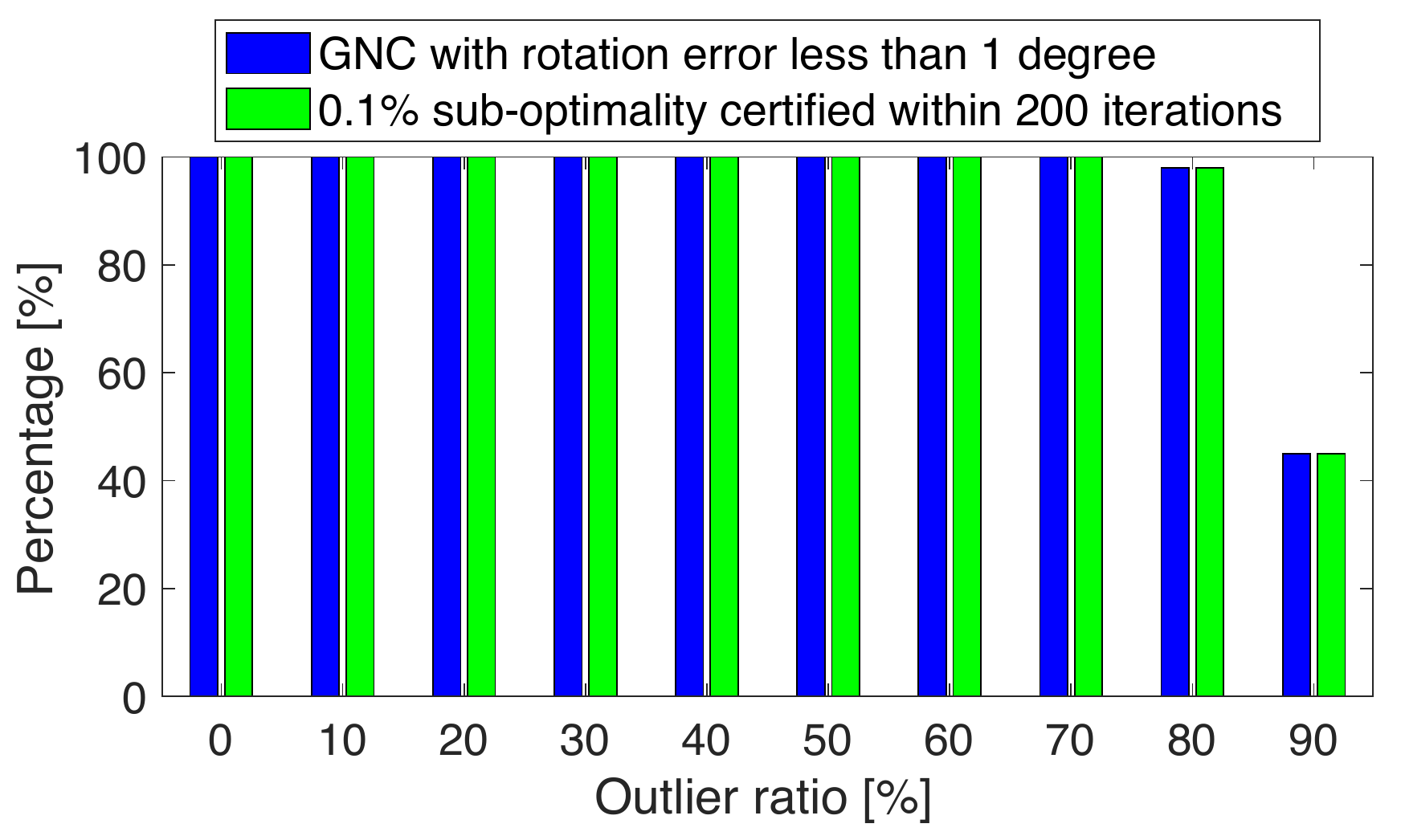}
			\vspace{-4mm}
	\end{center}
	\end{minipage} 
	\caption{\revone{\GNC and optimality certification for \TLS robust rotation search.} 
	 \label{fig:gnc_certification}}
	 	\vspace{-2mm} 
	\end{center}
\end{figure}

\myParagraph{Translation Solver} 
We apply a random translation $\vt$ to the \bunny, and fix $s=1$ and $\MR = \eye_3$.
\edit{Fig.~\ref{fig:benchmark_separate_solvers}(d) shows that component-wise translation estimation using both \maxConLong and \TLS are robust against 80\% outliers.}
\edit{ The translation error is defined as $\| \hatvt - \vtgt \|$, the  2-norm of the difference between the estimate $\hatvt$ and the ground-truth $\vtgt$. }
As mentioned above, most outliers are typically removed before translation estimation, hence enabling \name and \namepp to be robust to extreme outlier rates (more in Section~\ref{sec:benchmark}).

\myParagraph{Optimality Certification for Rotation Estimation} 
We now test the effectiveness of Algorithm~\ref{alg:optCertification} in certifying optimality of 
a rotation estimate.
 We consider the same setup  
we used for testing rotation estimation but with $K=100$ \TIMs. 
\revone{We use the \GNC scheme in~\cite{Yang20ral-GNC} to solve problem~\eqref{eq:TLSrotation} and then use Algorithm~\ref{alg:optCertification} to certify the solution of \GNC. 
Fig.~\ref{fig:gnc_certification} shows the performance of \GNC and the certification algorithm under increasing outlier rates, where 100 Monte Carlo runs are performed at each outlier rate. The blue bars show the percentage of runs for which \GNC produced a solution with less than 1 degree rotation error with respect to  the ground truth; the green bars show the percentage of runs for which Algorithm~\ref{alg:optCertification} produced a relative sub-optimality bound lower than $0.1\%$ in less than 200 iterations of \DRS. 
Fig.~\ref{fig:gnc_certification} demonstrates that: 
(i) the \GNC scheme typically produces accurate solutions to the \TLS rotation search problem with $<80\%$ outliers (a result that confirms the errors we observed in Fig.~\ref{fig:benchmark_separate_solvers}(c)); (ii) the optimality certification algorithm~\ref{alg:optCertification} can certify \emph{all} the correct \GNC solutions and reject \emph{all} incorrect \GNC solutions within 200 iterations.
On average, the certification algorithm takes $24$ iterations to obtain $<0.1\%$ sub-optimality bound, where each \DRS iteration takes 50ms in C++.}



\newcommand{\mpwthree}{5.7cm}

\begin{figure*}[h]
	\begin{center}
	\begin{minipage}{\textwidth}
	\hspace{-0.2cm}
	\begin{tabular}{ccc}%
		\begin{minipage}{\mpwthree}%
			\centering%
			\includegraphics[width=0.8\columnwidth]{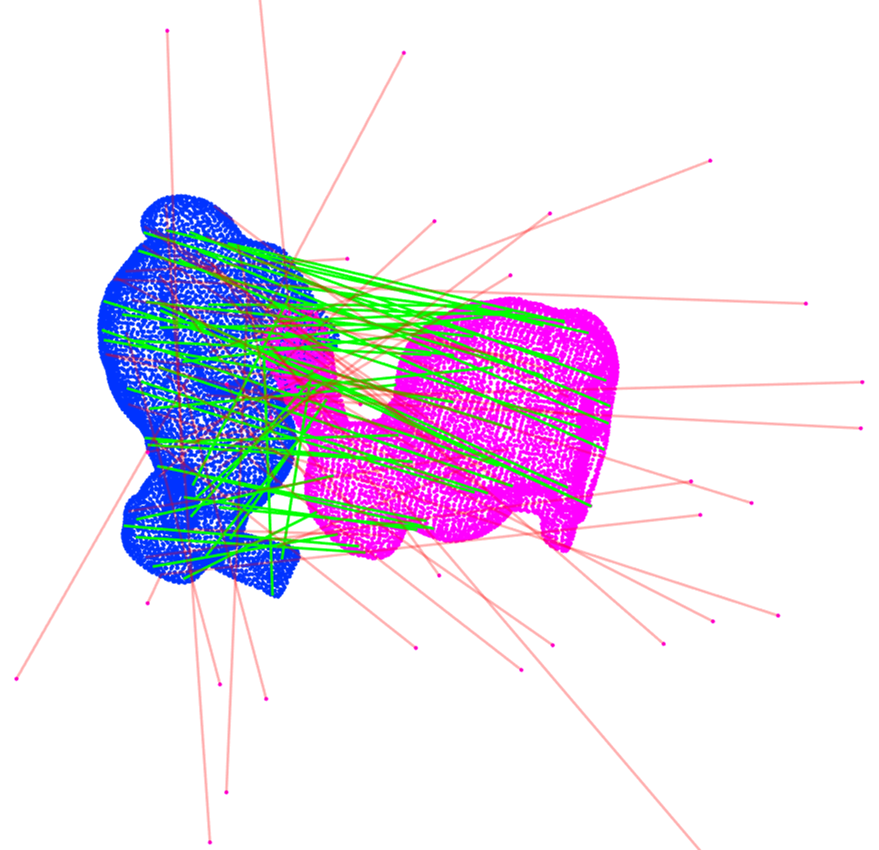} \\
			\end{minipage}
		& \myhspace
			\begin{minipage}{\mpwthree}%
			\centering%
			\includegraphics[width=0.8\columnwidth]{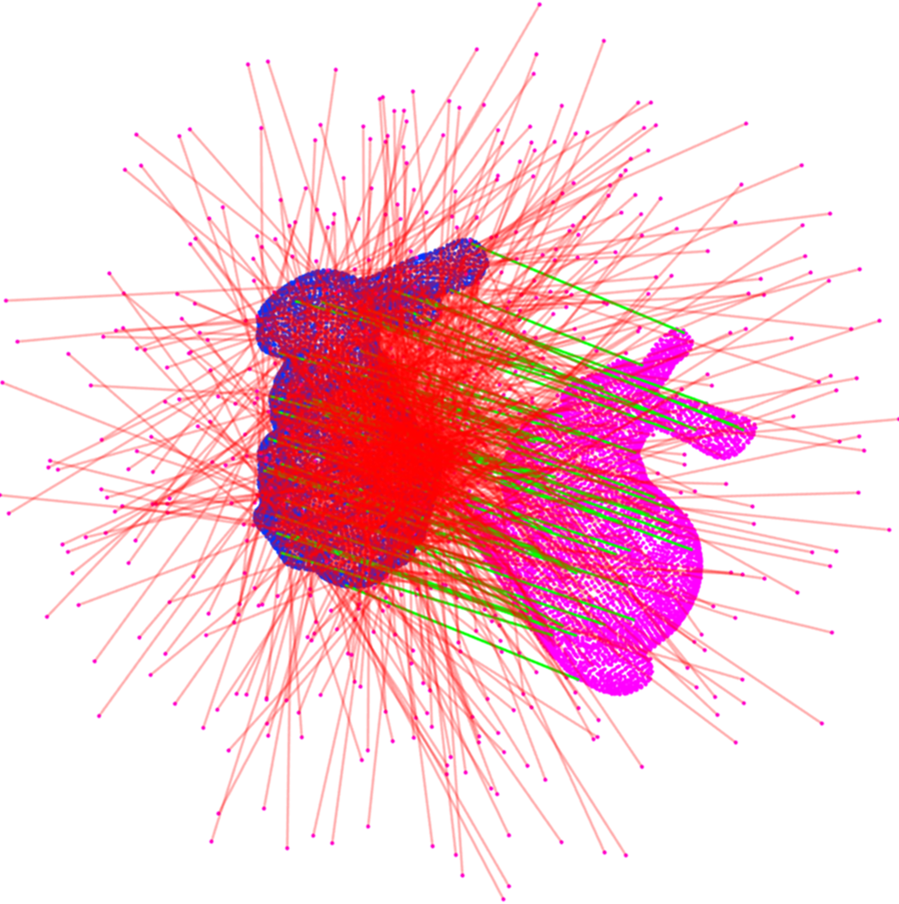}
			\\
			\end{minipage}
		& \myhspace
			\begin{minipage}{\mpwthree}%
			\centering%
			\includegraphics[width=0.9\columnwidth]{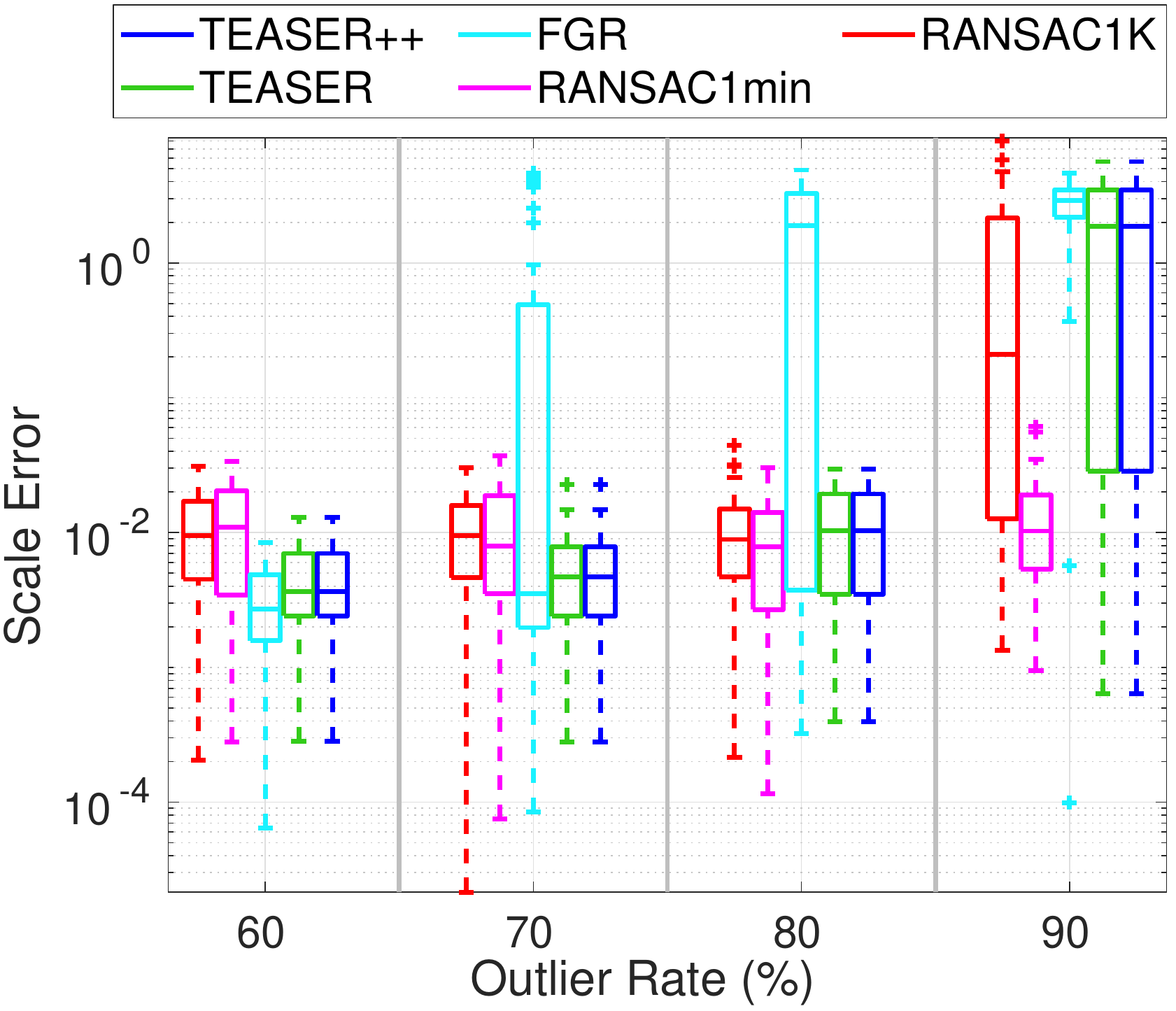}  \\
		\end{minipage}  \\
		\begin{minipage}{\mpwthree}%
			\centering%
			\includegraphics[width=0.9\columnwidth]{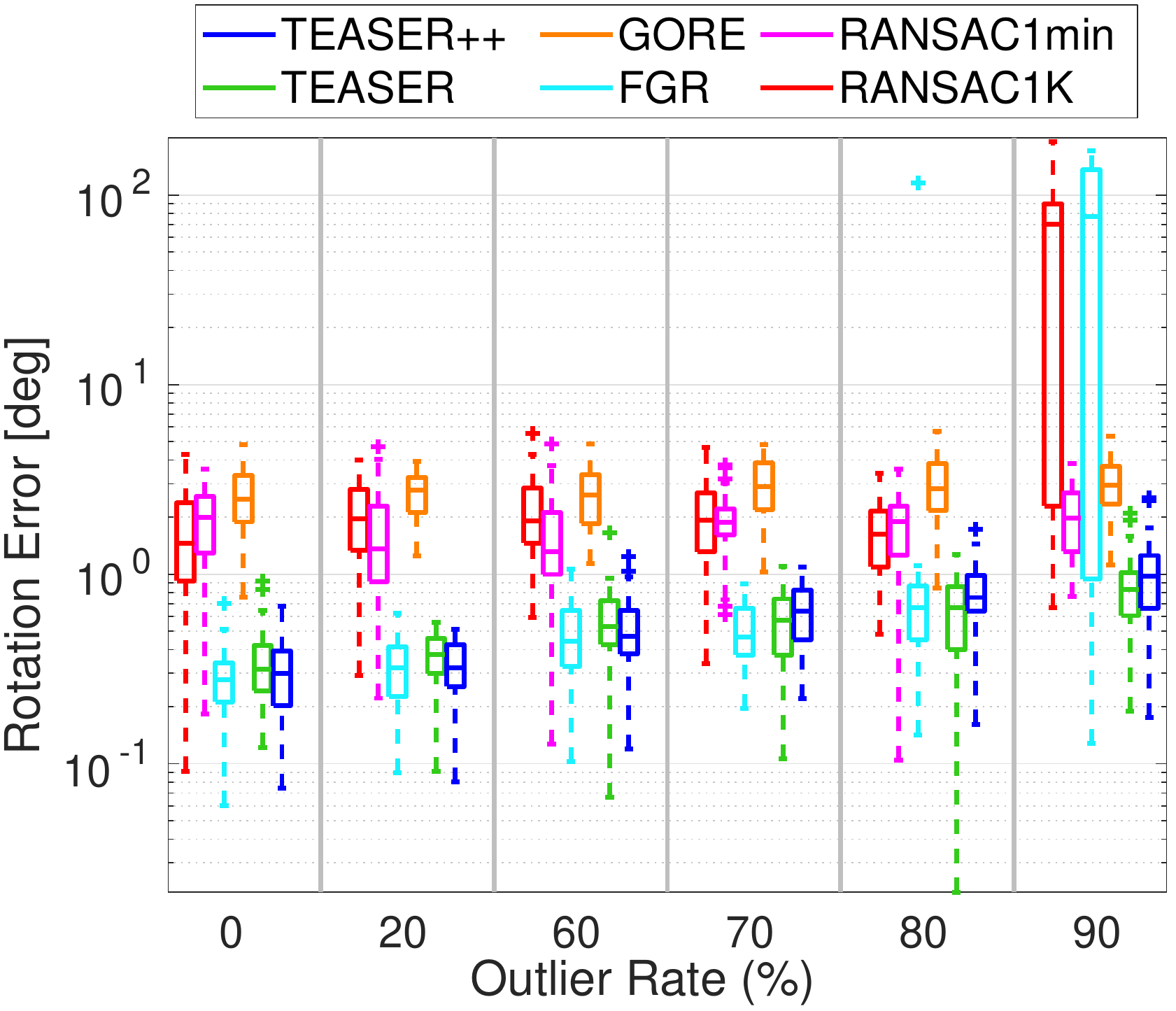} \\
			\end{minipage}
		& \myhspace
			\begin{minipage}{\mpwthree}%
			\centering%
			\includegraphics[width=0.9\columnwidth]{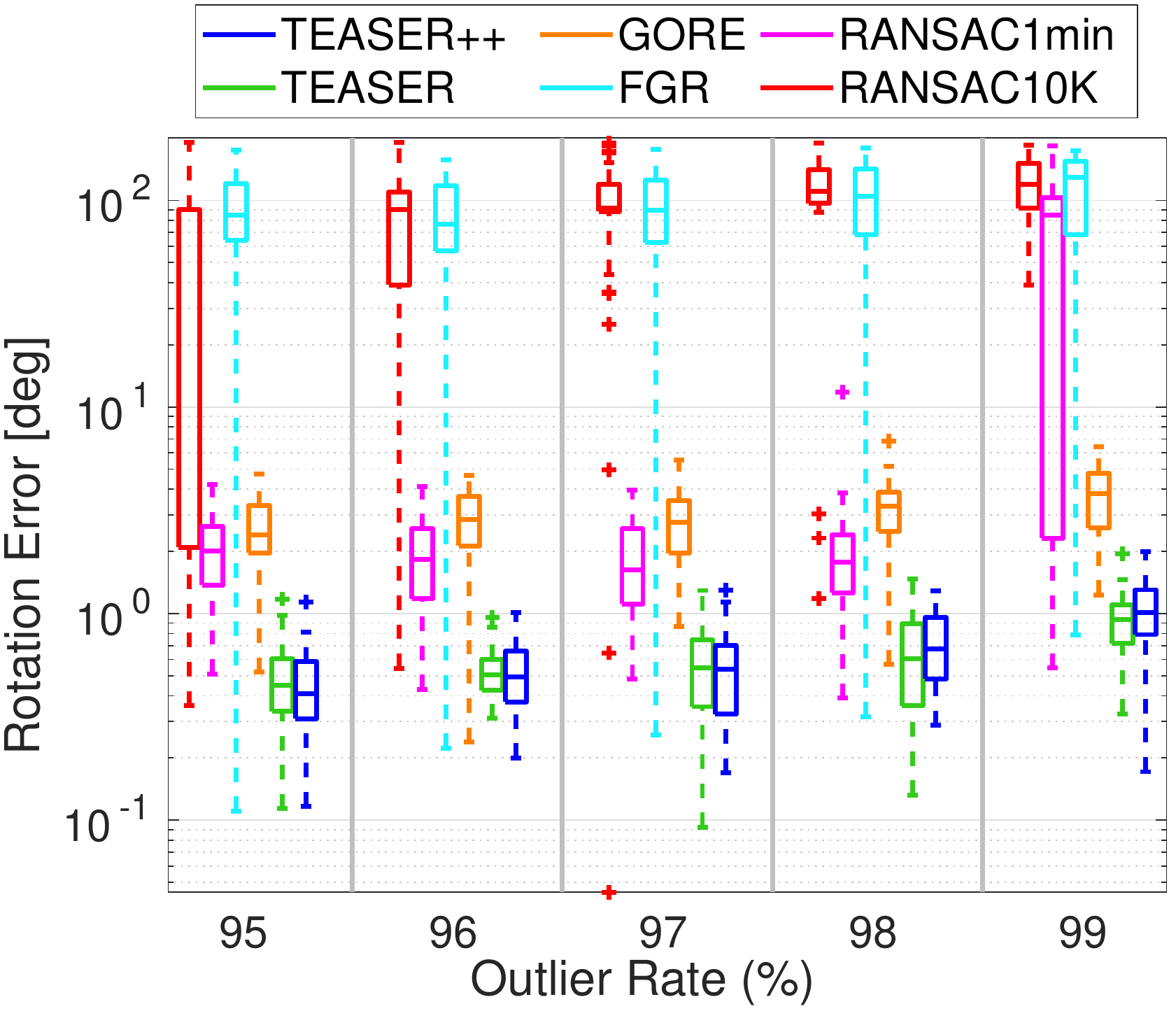}\\
			\end{minipage}
		& \myhspace
			\begin{minipage}{\mpwthree}%
			\centering%
			\includegraphics[width=0.9\columnwidth]{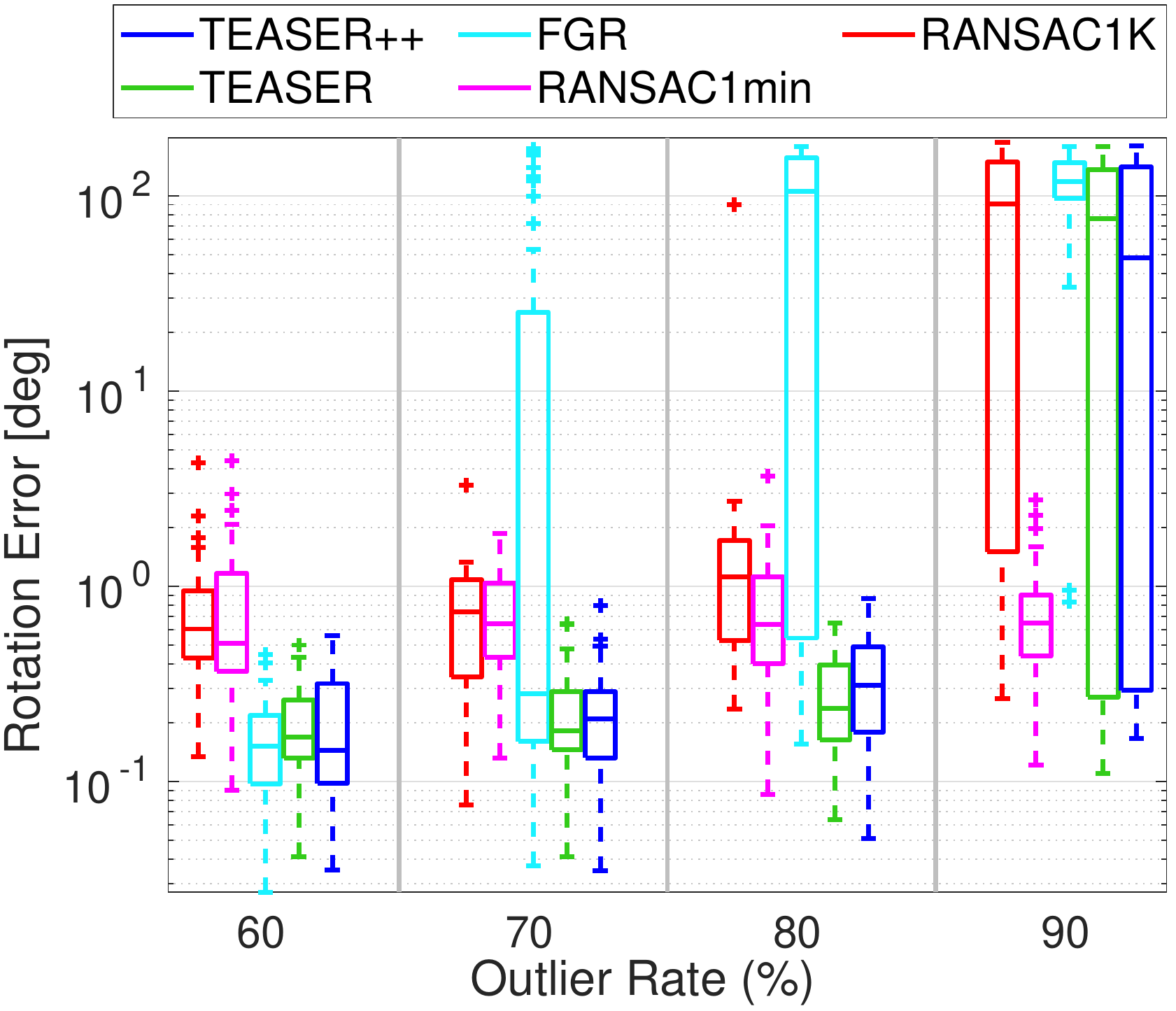} \\
			\end{minipage}\\
		\begin{minipage}{\mpwthree}%
			\centering%
			\includegraphics[width=0.9\columnwidth]{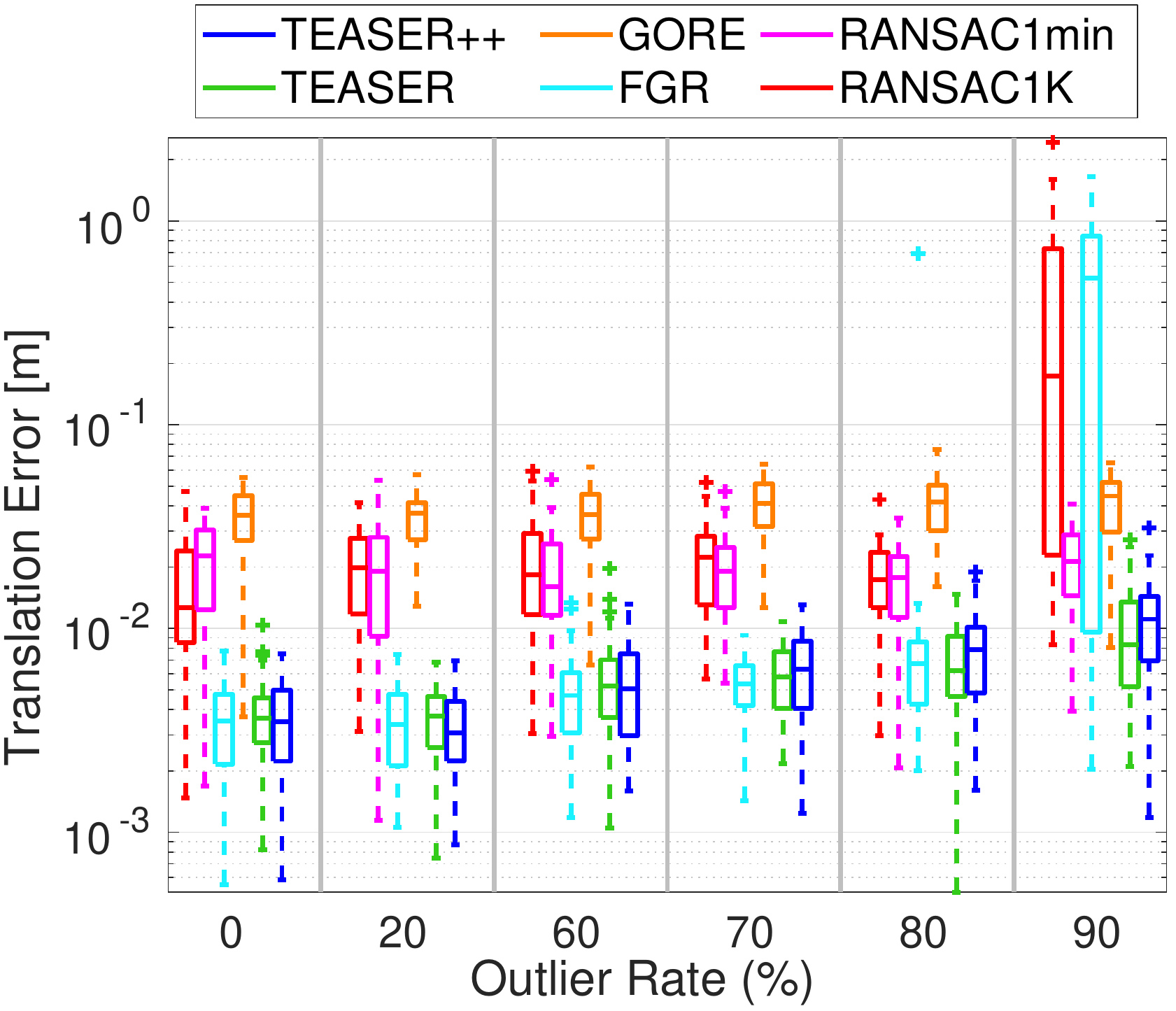} 
			\end{minipage}
		& \myhspace
			\begin{minipage}{\mpwthree}%
			\centering%
			\includegraphics[width=0.9\columnwidth]{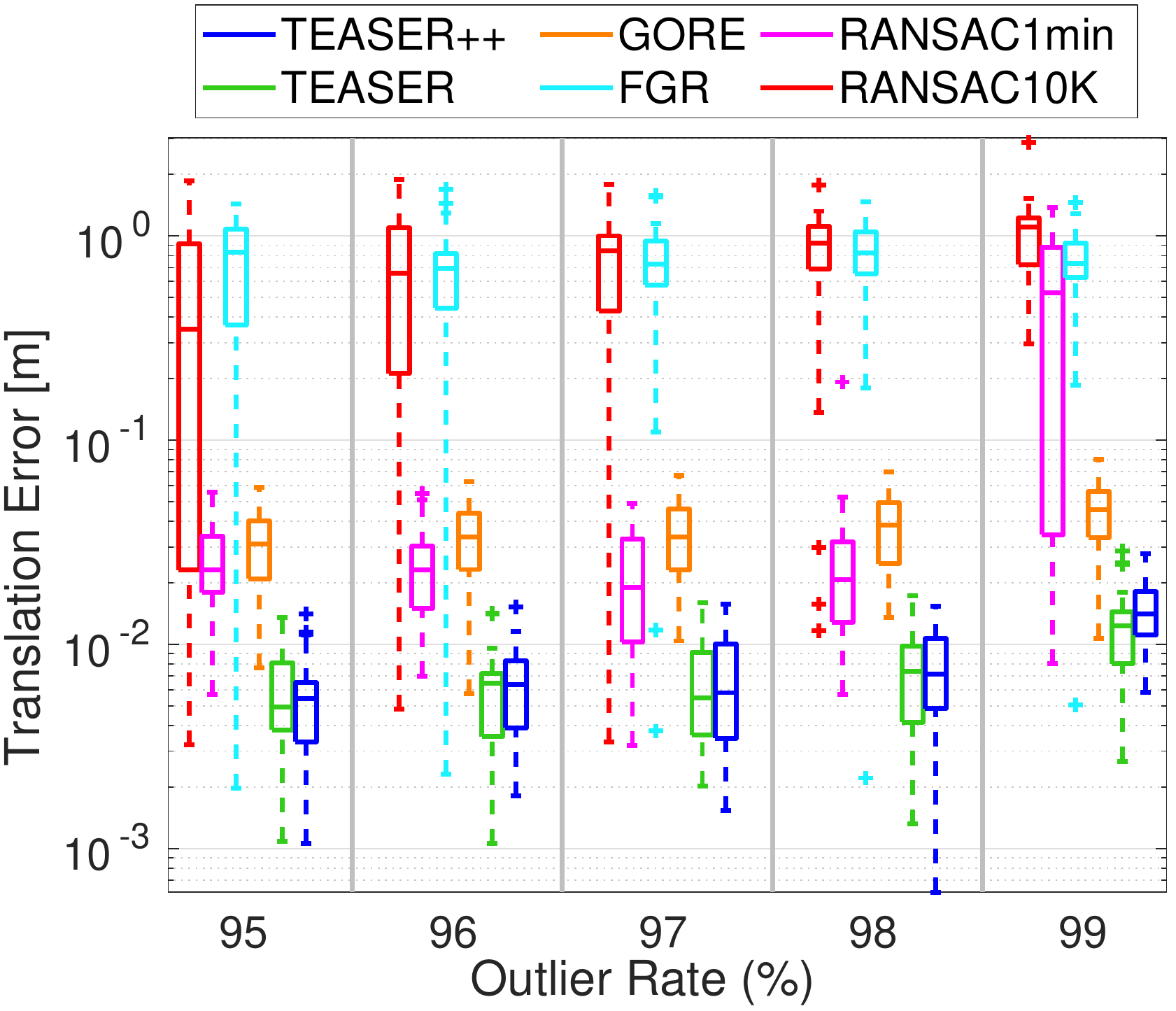}  
			\end{minipage}
		& \myhspace
			\begin{minipage}{\mpwthree}%
			\centering%
			\includegraphics[width=0.9\columnwidth]{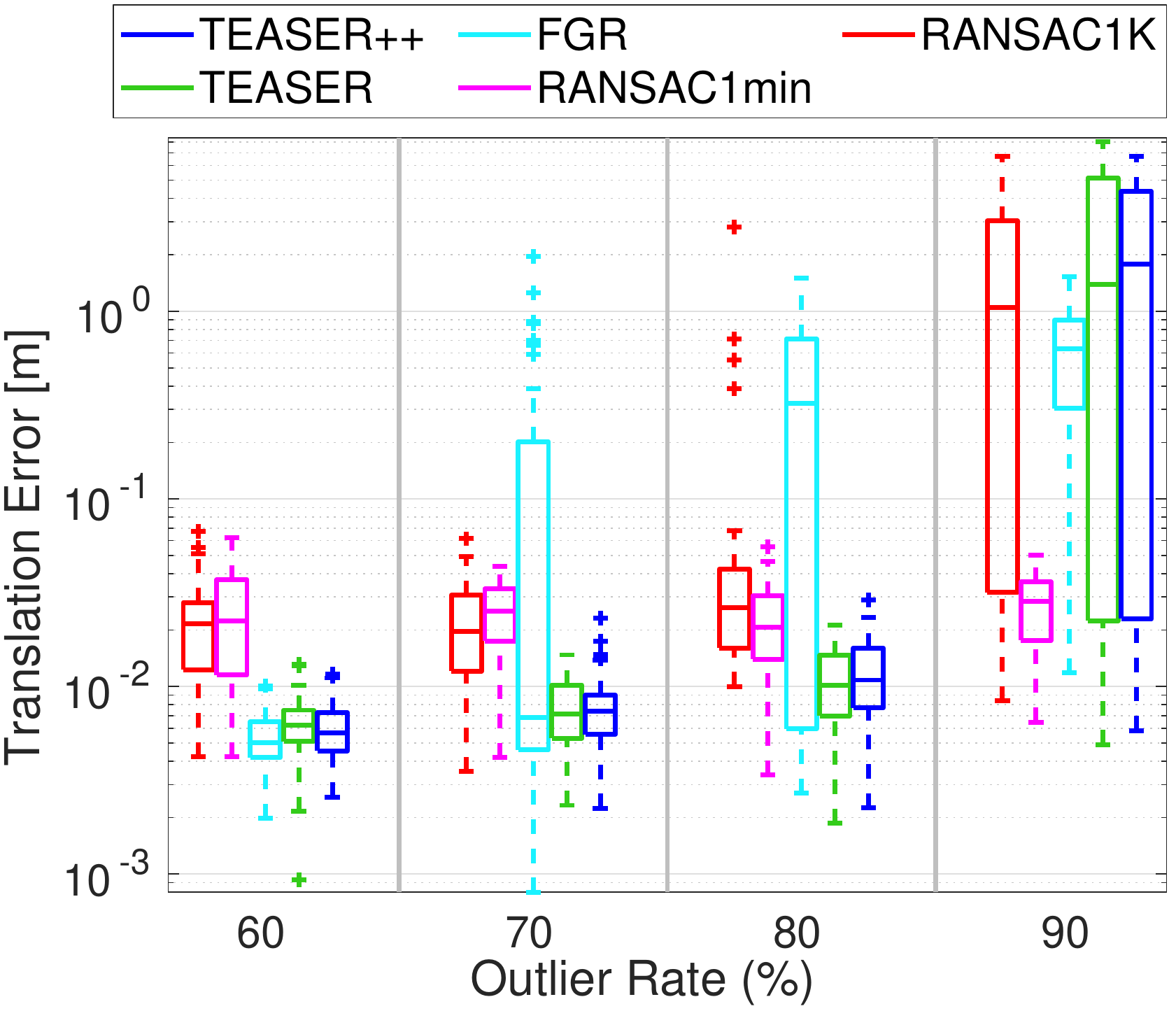}
			\end{minipage} \\
		\begin{minipage}{\mpwthree}%
			\centering%
			\includegraphics[width=0.9\columnwidth]{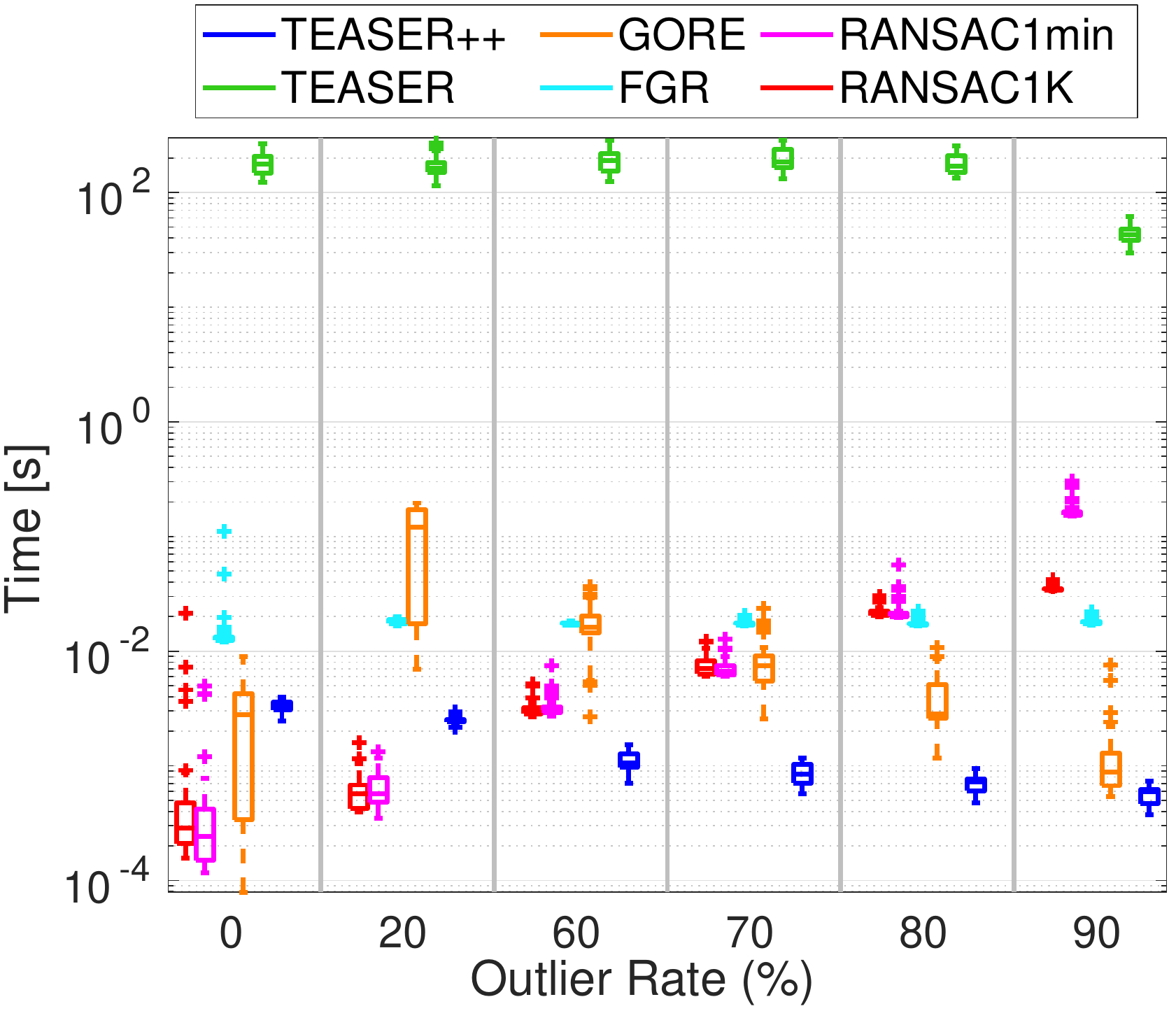} \\
			(a) Known Scale
			\end{minipage}
		& \myhspace
			\begin{minipage}{\mpwthree}%
			\centering%
			\includegraphics[width=0.9\columnwidth]{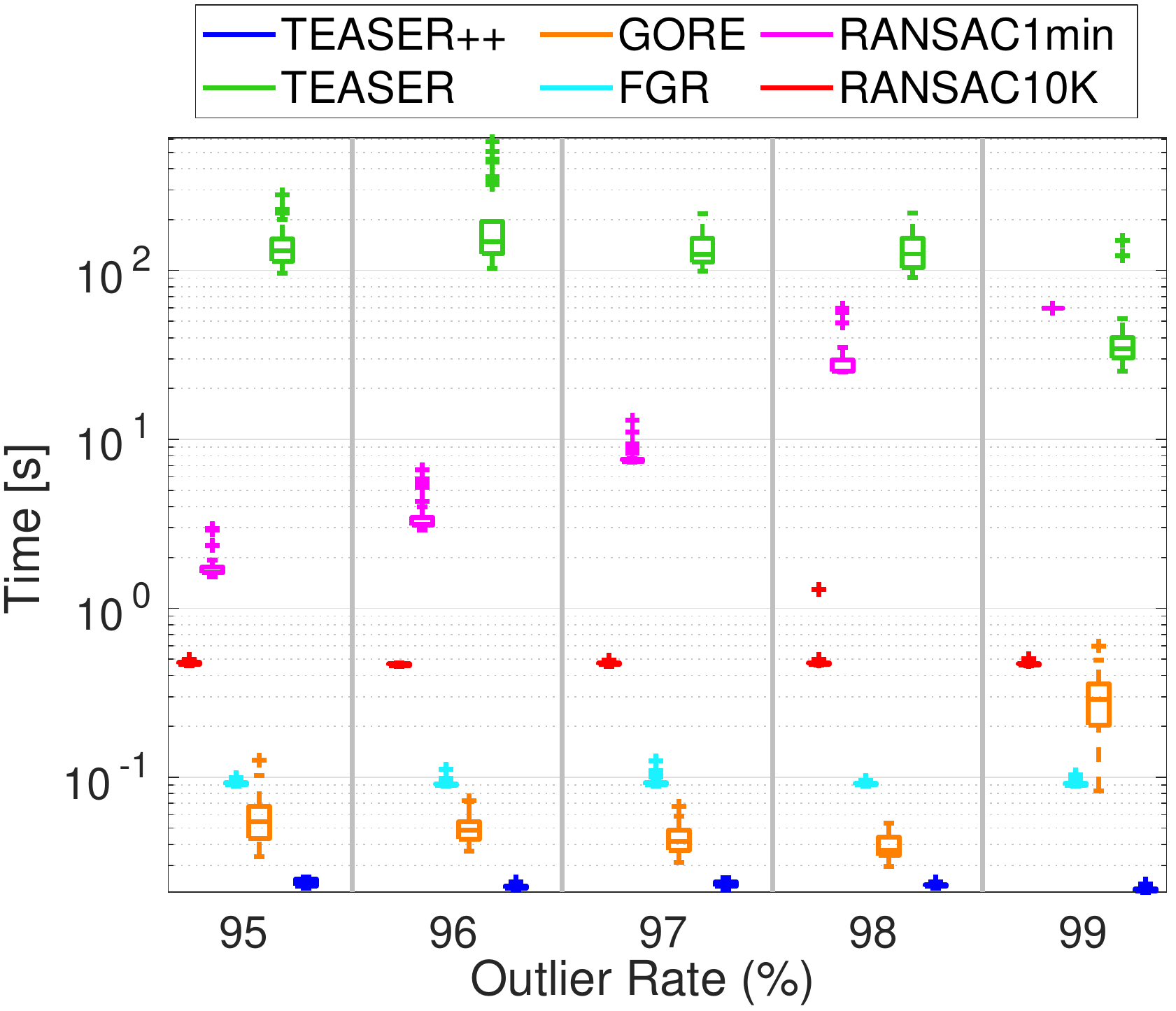}  \\
			(b) Known Scale (Extreme Outliers)
			\end{minipage}
		& \myhspace
			\begin{minipage}{\mpwthree}%
			\centering%
			\includegraphics[width=0.9\columnwidth]{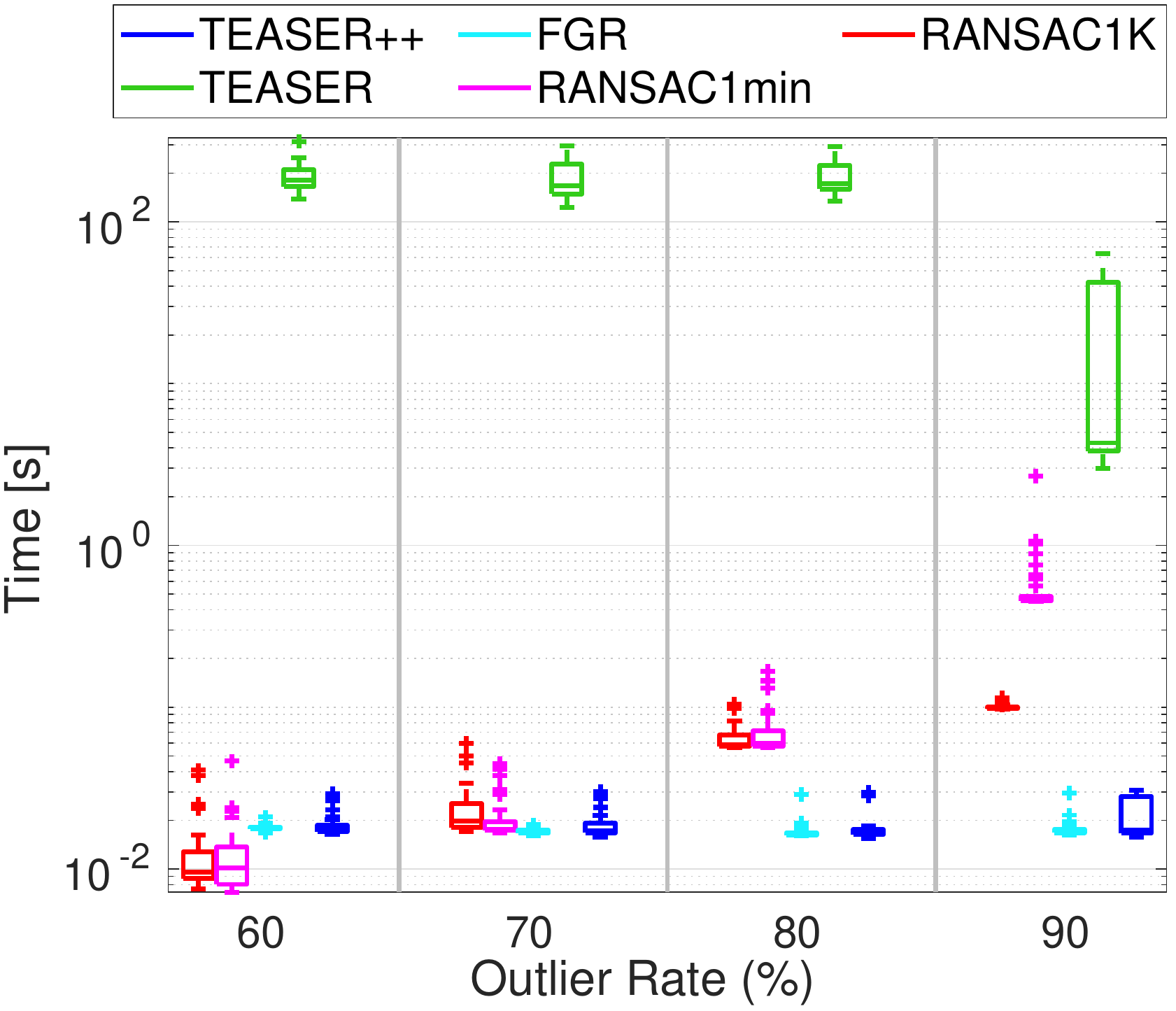} \\
			(c) Unknown Scale
			\end{minipage}
		\end{tabular}
	\end{minipage}
	\vspace{-3mm} 
	\caption{Benchmark results. 
	(a) Boxplots of rotation errors, translation errors, and timing for the six compared methods on the \bunny dataset with known scale (the top figure shows a  registration example with 50\% outlier correspondences). 
	(b) Same as before, but for outlier rates between $95\%$ and $99\%$ (the top figure shows an example with 95\% outlier correspondences).
	(c) Boxplots of scale, rotation, translation errors and timing for the five registration methods that support scale estimation on the \bunny dataset with unknown scale.}
	 \label{fig:benchmark}
	\vspace{-7.5mm} 
	\end{center}
\end{figure*}


\subsection{Benchmarking on Standard Datasets}
\label{sec:benchmark}

\myParagraph{Testing Setup}
We benchmark \name and \namepp against two state-of-the-art robust registration techniques: \emph{Fast Global Registration} (\FGR)~\cite{Zhou16eccv-fastGlobalRegistration} and \emph{Guaranteed Outlier REmoval} (\GORE)~\cite{Bustos18pami-GORE}. 
In addition, we test two \ransac variants \revone{(with $99\%$ confidence)}: a fast version where we terminate \ransac after a maximum of 1,000 iterations (\ransaconek) and 
a slow version where we terminate \ransac after 60 seconds (\ransaconemin). 
We use four datasets, 
{\emph{\bunny}, \emph{\armadillo}, \emph{\dragon}, and \emph{\buddha},} from the Stanford 3D Scanning Repository~\cite{Curless96siggraph} and downsample them to $\nrPoints=100$ points. 
The tests below follow the same protocol of \prettyref{sec:separateSolver}. 
Here we focus on the results on the \bunny dataset and we postpone the (qualitatively similar) results 
obtained on the other three datasets to  \isExtended{Appendix~\ref{sec:experiments_supp}.}{Appendix~\ref{sec:experiments_supp}\arxivCite.}
 Appendix~\ref{sec:experiments_supp} also  \extraEdits{showcases the proposed algorithms} on registration problems with high noise ($\sigma=0.1$), \extraEdits{and with up to 10,000 correspondences}.

{\bf Known Scale.} We first evaluate the compared techniques with known scale $s=1$. 
Fig.~\ref{fig:benchmark}(a) shows the rotation error, translation error, and timing for increasing outlier rates on the \bunny dataset. 
From the rotation and translation errors, we note that \name, \namepp, \GORE, and \ransaconemin are robust against up to 90\% outliers, although \name and \namepp tend to produce more accurate estimates than \GORE, and \ransaconemin typically requires over $10^5$ iterations for convergence at 90\% outlier rate. \FGR can only resist 70\% outliers and \ransaconek starts breaking at 90\% outliers. 
\namepp's performance is on par with \name for all outlier rates, which is expected from the observations in Section~\ref{sec:teaserpp}. 
The timing subplot at the bottom of Fig.~\ref{fig:benchmark}(a) shows that \name is impractical for real-time robotics applications. On the other hand, \namepp  is one of the fastest techniques across the spectrum, and is able to solve problems with large number of outliers in less than 10ms {on a standard laptop}.  

{\bf Extreme Outlier Rates.}
We further benchmark the performance of \name and \namepp under extreme outlier rates from 95\% to 99\% with known scale and $N=1000$ correspondences on the \bunny. We replace \ransaconek with \ransactenk, since \ransaconek already performs poorly at 90\% outliers.
Fig.~\ref{fig:benchmark}(b) shows the boxplots of the rotation errors, translation errors, and timing.
\name, \namepp, and \GORE are robust against up to 99\% outliers, while \ransaconemin with 60s timeout can resist 98\% outliers with about 
$10^6$ iterations. \ransactenk and \FGR perform poorly under extreme outlier rates. While \GORE, \name and \namepp are both robust against 99\% outliers, \name, and \namepp produce lower estimation errors, with \namepp being one order of magnitude faster than \GORE (bottom subfigure). \revone{We remark that \namepp's robustness against $99\%$ outliers is due in large part to the drastic reduction of outlier rates by \mcis.}

{\bf Unknown Scale.} 
\GORE is unable to solve for the scale, hence we only benchmark \name and \namepp against \FGR,\footnote{Although the original algorithm in~\cite{Zhou16eccv-fastGlobalRegistration} did not solve for the scale, we extend it by using Horn's method to compute the scale at each iteration} \ransaconek, and \ransaconemin. Fig.~\ref{fig:benchmark}(c) plots scale, rotation, translation error and timing for increasing outliers on the \bunny dataset. 
All the compared techniques perform well when the outlier ratio is below 60\%. \FGR has the lowest breakdown point and {fails at 70\%}. \ransaconek, \name, and \namepp only fail at 90\% outlier ratio when the scale is unknown. Although \ransaconemin with 60s timeout outperforms other methods at 90\% outliers, it typically requires more than $10^5$ iterations to converge, which is not practical for real-time applications. 
\namepp consistently runs in less than 30ms.


\subsection{Simultaneous Pose and Correspondences (\SPC)}
\label{sec:bruteForce}

Here we provide a proof-of-concept of how to use \namepp in the case where we do not have correspondences. 

{\bf Testing Setup.} 
\revone{We obtain the source point cloud $\calA$ by downsampling the \bunny dataset to $100$ points.
Then, we create the point cloud $\calB$ by applying a random rotation and translation to $\calA$. 
Finally. we remove a percentage of the points in $\calB$ to simulate \emph{partial overlap} between $\calA$ and $\calB$. For instance, when the overlap is $80\%$, we discard $20\%$ (randomly chosen) points 
from $\calB$. 
We compare \namepp against (i) \ICP initialized with the identity transformation; and (ii) \goICP~\cite{Yang16pami-goicp} with $30\%$, $60\%$ and $90\%$ trimming percentages to be robust to partial overlap.}
For \namepp, we generate \emph{all possible correspondences}: in other words, for each point in $\calA$ we add all points in $\calB$ as a potential correspondence (for a total of $|\calA| \cdot |\calB|$ correspondences). We then feed the correspondences to \namepp that computes a transformation without the need for an initial guess. 
Clearly, most of the correspondences fed to \name are outliers, but we rely on \namepp to find the small set of inliers.

Fig.~\ref{fig:bruteForce}(a)-(b) show the rotation and translation errors for different levels of overlap between $\calA$ and $\calB$. \ICP fails to compute the correct transformation in practically all instances, since the initial guess is not in the basin of convergence of the optimal solution. \revone{\goICP is a global method which does not require an initial guess and performs much better than \ICP. However, \goICP requires the user to set a trimming percentage to deal with partial overlap, and Fig.~\ref{fig:bruteForce}(a)-(b) show the performance of \goICP is highly sensitive to the trimming percentage (\goICP ($60\%$) performs the best but is still less accurate than \namepp). Moreover, \goICP takes 16 seconds on average due to its usage of \bnb.}
\emph{\namepp computes a correct solution across the spectrum without the need for an initial guess}. 
\namepp only starts failing when the overlap drops below $10\%$. 
The price we pay for this enhanced robustness is an increase in runtime. We feed $|\calA| \cdot |\calB| \approx 10^4$  correspondences to \namepp, which increases the runtime compared to the correspondence-based setup in which the number of correspondences scales linearly (rather than quadratically) with the point cloud size. 
Fig.~\ref{fig:bruteForce}(c) reports the timing breakdown for the different modules in \namepp.
From the figure, we observe that (i) for small overlap, \namepp is not far from real-time, (ii) the timing is dominated by the maximum clique computation and scale estimation, where the latter includes the computation of the translation and rotation invariant measurements (\TRIMs), \revone{and (iii) \namepp is orders of magnitude faster than \goICP when the overlap is low (\eg~below $40\%$).} 


\renewcommand{\mpw}{4.3cm}

\begin{figure}[t]
	\begin{center}
	\begin{minipage}{\textwidth}
	\begin{tabular}{cc}%
	        \myhspace \hspace{-4mm}
			\begin{minipage}{\mpw}%
			\centering%
			\includegraphics[trim=0.1cm 0.4cm 0.2cm 0cm, width=1\columnwidth]{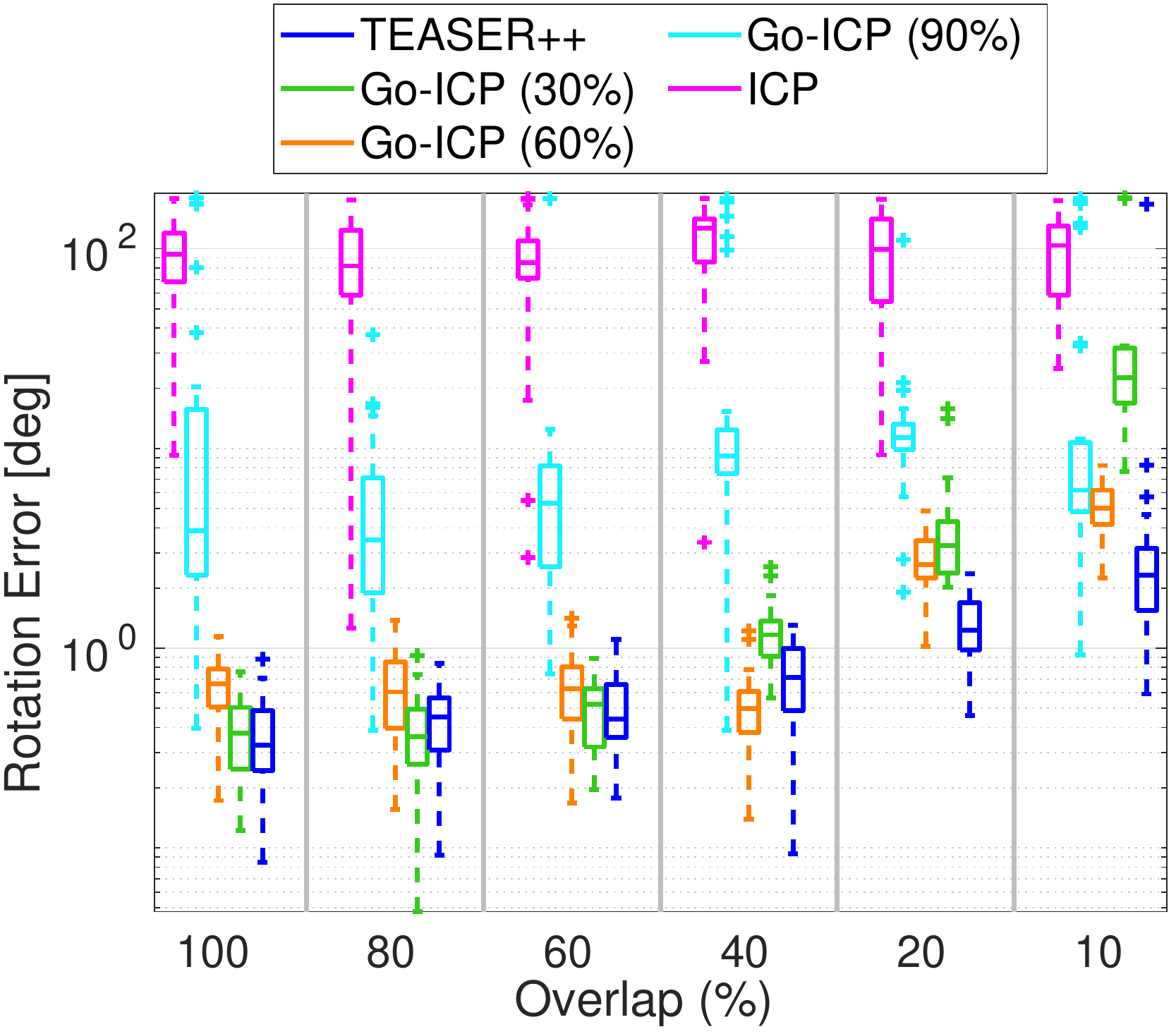} \\
			\vspace{-1mm}
			(a) Rotation Error
			\end{minipage}
              &
              \myhspace 
			\begin{minipage}{\mpw}%
			\centering%
			\includegraphics[trim=0.1cm 0cm 0.2cm 0cm, width=1\columnwidth]{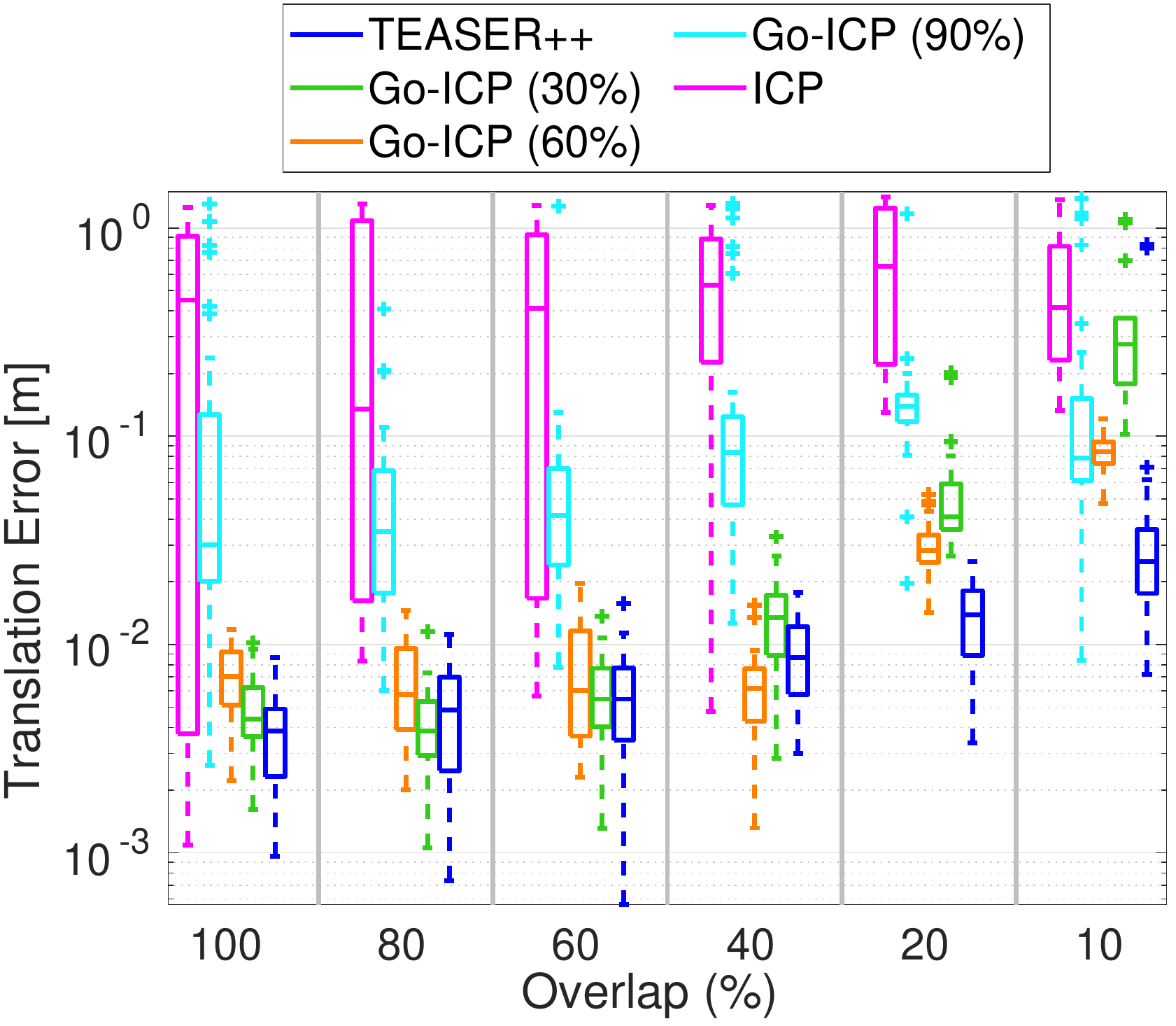} \\
			
			\vspace{-2mm}
			(b) Translation Error
			\end{minipage} 
		\\
		\multicolumn{2}{c}{
		\begin{minipage}{7cm}%
			\centering%
			\includegraphics[width=0.7\columnwidth]{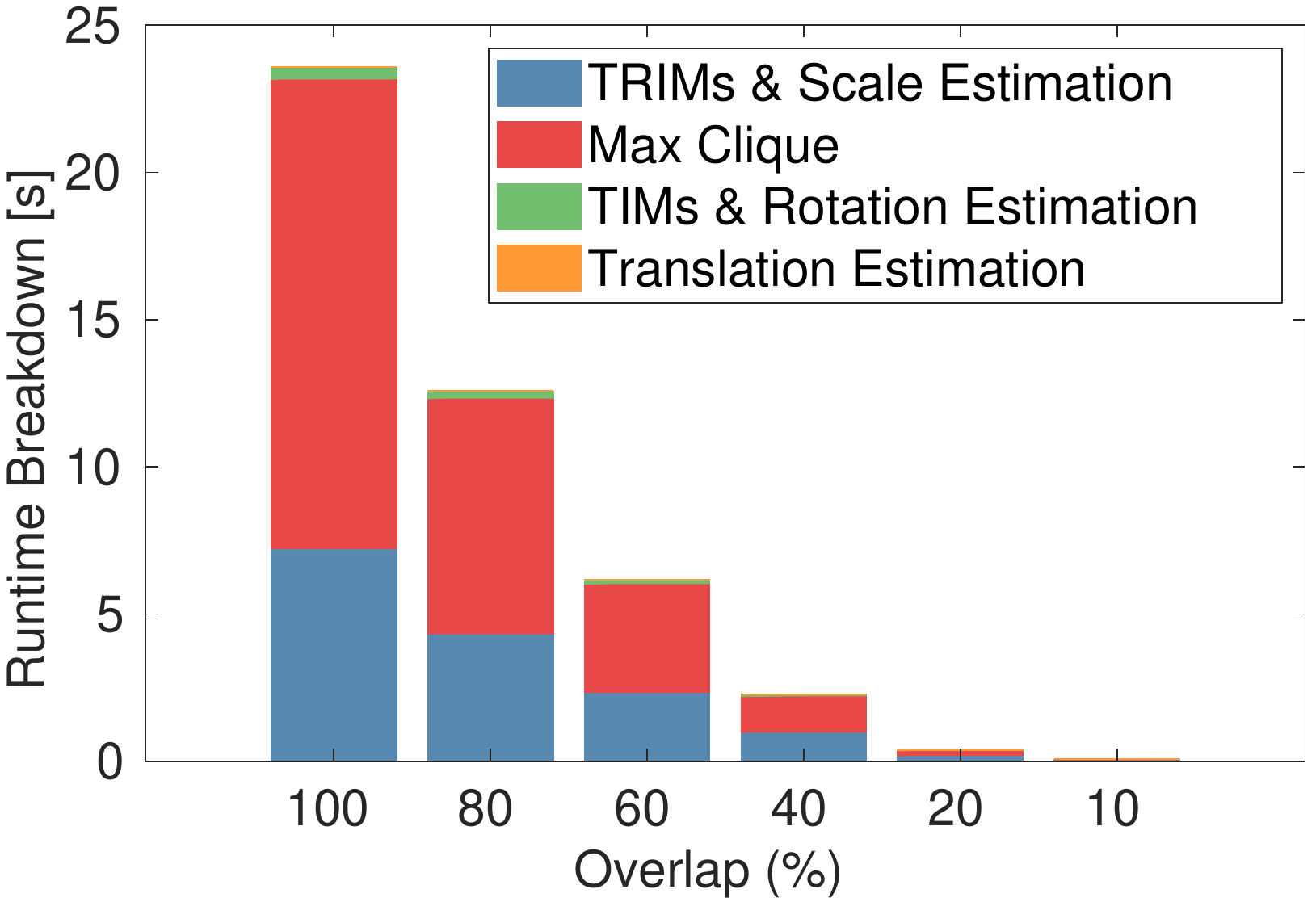} \\
			(c) Timing Breakdown
			\end{minipage}
		}
	\end{tabular}			
	\end{minipage}	
	\vspace{-2mm} 
	\caption{\revone{(a)-(b) Rotation and translation errors for \namepp, \ICP, and \goICP in a correspondence-free problem. (c) Timing breakdown for \namepp.
	 \label{fig:bruteForce}}}
	 	\vspace{-3mm} 
	\end{center}
\end{figure}


\newcommand{\mpwfour}{3.3cm}
\newcommand{\mpwtbl}{4.1cm}
\begin{figure*}[h]
	\begin{center}
	\begin{minipage}{0.58\textwidth}
		\vspace{-28mm}
		\hspace{-0.2cm}
		\begin{tabular}{ccc}%
			\begin{minipage}{\mpwfour}%
				\centering%
				\includegraphics[trim=1cm 0.3cm 1cm 0.2cm, width=1\columnwidth]{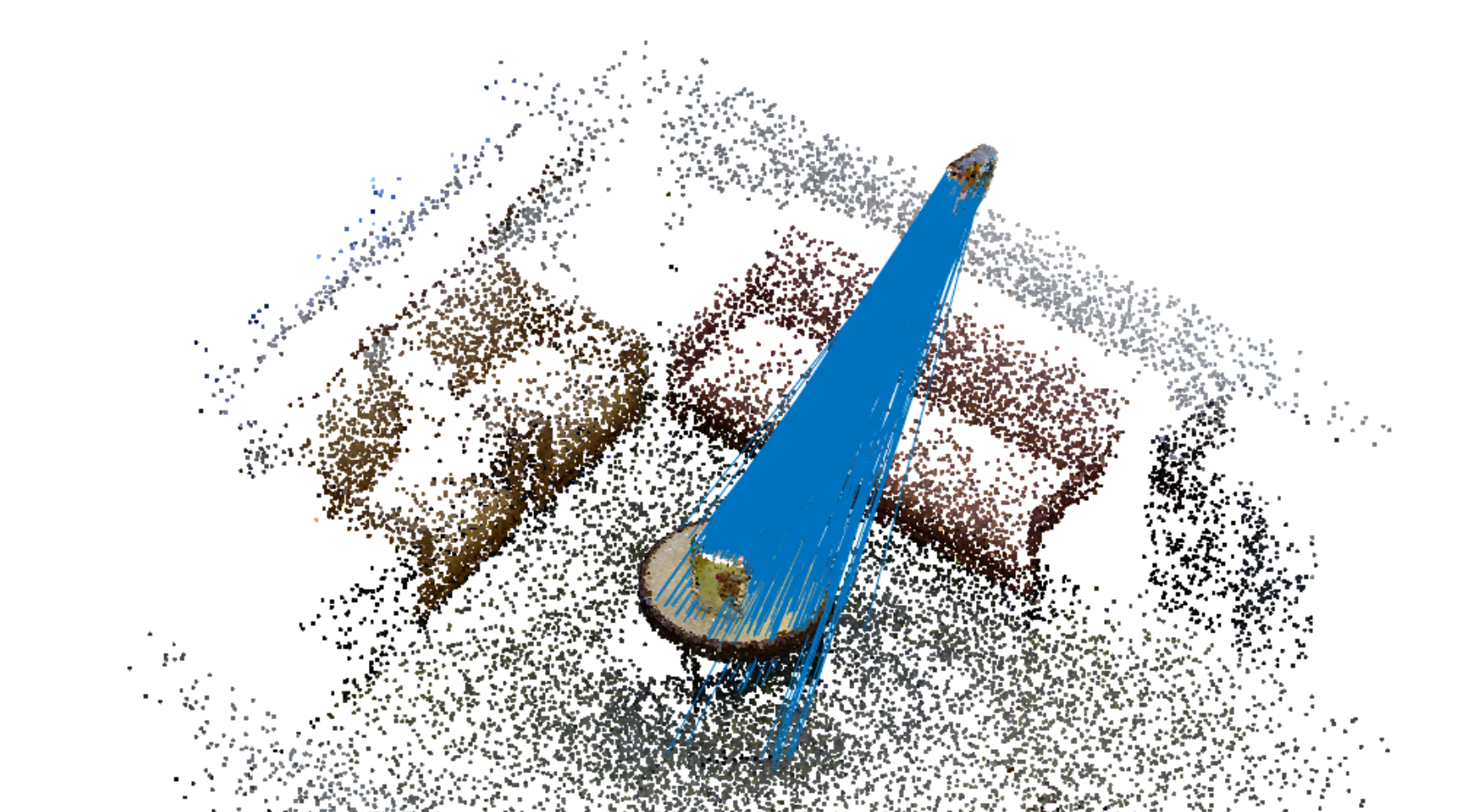} \\
			\end{minipage}
			\begin{minipage}{\mpwfour}%
				\centering%
				\includegraphics[trim=1cm 0.3cm 1cm 0.2cm, width=1\columnwidth]{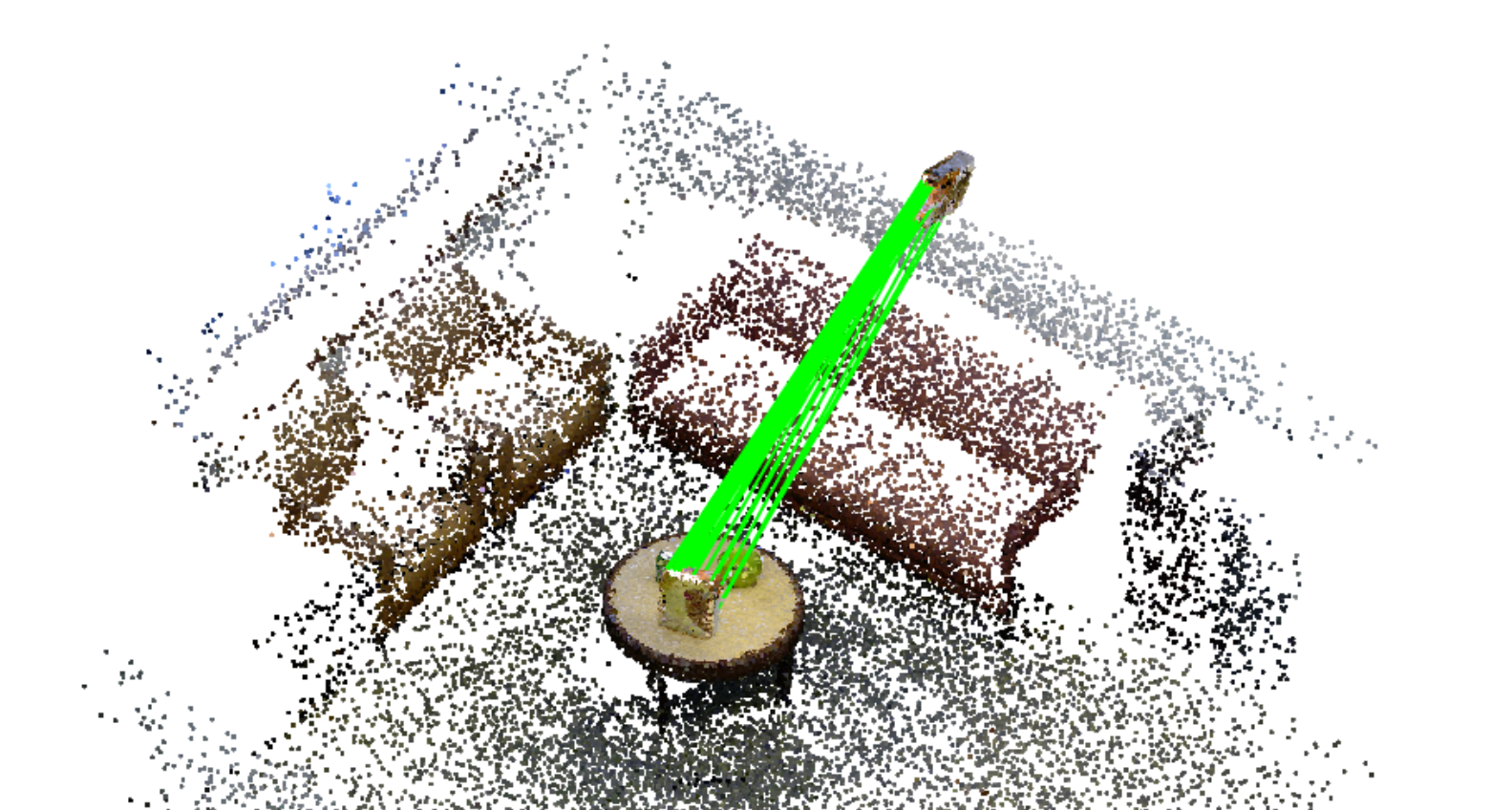} \\
			\end{minipage}
			\begin{minipage}{\mpwfour}%
				\centering%
				\includegraphics[trim=1cm 0.3cm 1cm 0.2cm, width=1\columnwidth]{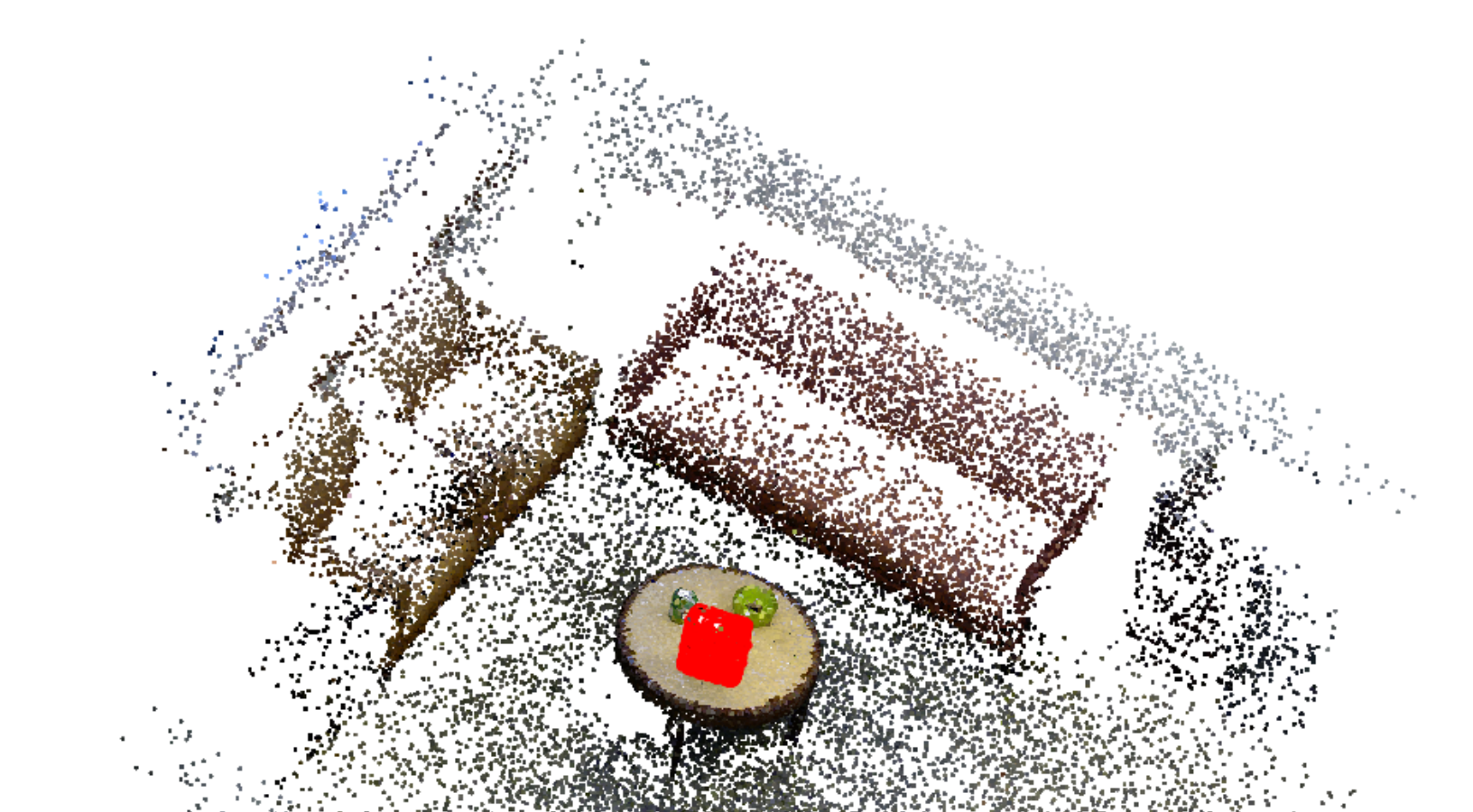} \\
			\end{minipage} 
		\end{tabular}
		\vspace{-3mm} 
		\caption{Successful object pose estimation by \name on a real RGB-D dataset. Blue lines are the original FPFH~\cite{Rusu09icra-fast3Dkeypoints} correspondences with outliers, green lines are the inlier correspondences computed by \name, and the final registered object is highlighted in red.}
		\label{fig:RGBD}
		\vspace{-35mm} 
	\end{minipage}%
	\hspace{1mm}
	\begin{minipage}[t][][t]{0.34\textwidth}
		\centering
		\scalebox{0.8}{
		\begin{tabular}{ccc}
			& Mean & SD \\
			\hline
			Rotation error [rad] & 0.066 & 0.043 \\
			\hline
			Translation error [m] & 0.069 & 0.053 \\
			\hline
			\# of FPFH correspondences & 525 & 161 \\
			\hline
			FPFH inlier ratio [\%] & 6.53 & 4.59
		\end{tabular}}
		\captionof{table}{Registration results {on eight scenes} of the RGB-D dataset~\cite{Lai11icra-largeRGBD}.}
		\label{tab:objectPoseEstimation}
		\vspace{-5mm} 
	\end{minipage}
	\end{center}
\end{figure*}

\subsection{Application 1: Object Pose Estimation and Localization}
\label{sec:roboticsApplication1}

{\bf Testing Setup.}
We use the large-scale point cloud datasets from~\cite{Lai11icra-largeRGBD} to test \name in \emph{object pose estimation and localization} applications. 
We first use the ground-truth object labels to extract the \emph{cereal box/cap} out of the scene and treat it as the 
object, then apply a random transformation to the scene, to get an object-scene pair. To register the object-scene pair, we first use FPFH feature descriptors~\cite{Rusu09icra-fast3Dkeypoints} to establish putative correspondences.
Given correspondences from FPFH, \name is used to find the relative pose between the object and scene. We downsample the object and scene  using the same ratio (about 0.1) to make the object have 2,000 points.

{\bf Results.}
Fig.~\ref{fig:RGBD} shows the noisy FPFH correspondences, the inlier correspondences obtained by \name, and successful localization and pose estimation of the \emph{cereal box}.  
{Another example is given in Fig.~\ref{fig:overview}(g)-(h).
Qualitative results for eight scenes are given in \isExtended{Appendix~\ref{sec:objectPoseEstimationSupp}.}{Appendix~\ref{sec:objectPoseEstimationSupp}\arxivCite.}}
 The inlier correspondence ratios for \emph{cereal box} are all below 10\% and typically below 5\%. \name is able to compute a highly accurate estimate of the pose using a handful of inliers. 
Table~\ref{tab:objectPoseEstimation} shows the mean and standard deviation (SD) of the rotation and translation errors, {the number of FPFH correspondences, and the inlier ratio estimated by \name on the eight scenes.}


%
%
%
%

\subsection{Application 2: Scan Matching}
\label{sec:roboticsApplication2}

{\bf Testing Setup.}
\namepp can also be used in robotics applications that need robust scan matching, such as 3D reconstruction and loop closure detection in SLAM~\cite{Cadena16tro-SLAMsurvey}. We evaluate \namepp's performance in such scenarios using the \emph{\matchTD} dataset~\cite{Zeng17cvpr-3dmatch}, which consists of RGB-D scans from 62 real-world indoor scenes. The dataset is divided into 54 scenes for training, and 8 scenes for testing. The authors provide 5,000 randomly sampled keypoints 
for each scan. 
On average, there are 205 pairs of scans per scene 
(maximum: 519 in the Kitchen scene, minimum: 54 in the Hotel 3 scene).


\newcommand*\rot{\rotatebox{90}}

\begin{figure*}[h]
	\begin{center}
    \begin{minipage}{0.34\textwidth}
    \vspace{-5mm}
        \begin{tabular}{c}
			\begin{minipage}{\columnwidth}%
				\centering%
		        \includegraphics[width=\columnwidth]{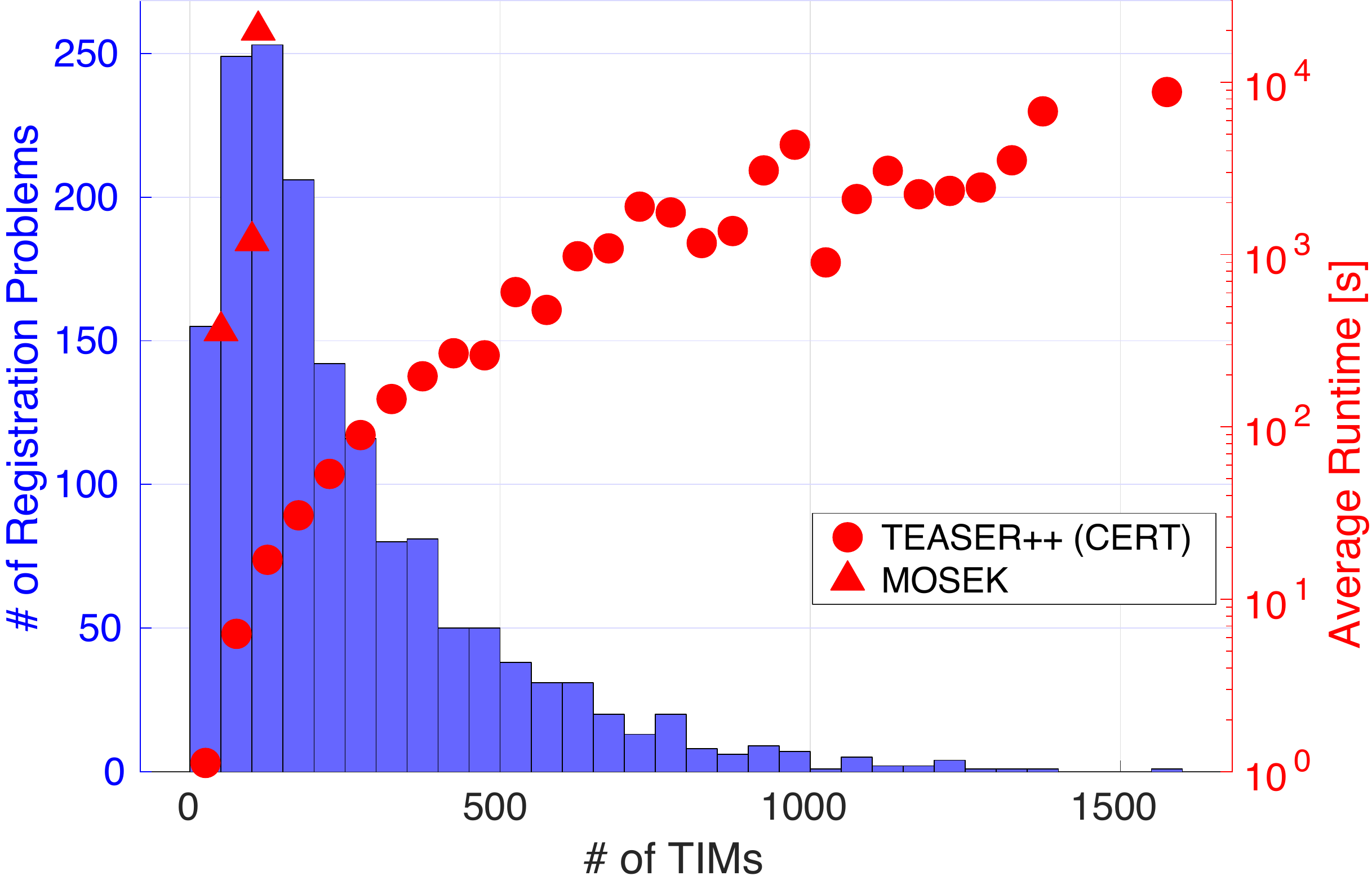} 
			\end{minipage}
            \vspace{-3mm}
        \end{tabular}
        \caption{\revone{\hspace{-3mm}\# of registration problems (histogram) and average runtime (scatter plot, compared with \MOSEK) for \namepp (CERT) \wrt \# of \TIMs passed to certification in \emph{3DMatch}.}}
        \label{fig:cert-runtime-vs-tims}
	\end{minipage}%
	\hspace{5mm}
	\begin{minipage}{0.58\textwidth}
    \vspace{-5mm}
        \hspace{-0.2cm}
        \resizebox{\columnwidth}{!}{
        \begin{tabular}{cccccccccc}
            & \multicolumn{8}{c}{Scenes}\\
        \cline{2-9}
             & \rot{\shortstack{Kitchen (\%)}} & \rot{\shortstack{Home 1  (\%)}}  & \rot{\shortstack{ Home 2 (\%)}} & \rot{\shortstack{Hotel 1 (\%)}} & \rot{\shortstack{Hotel 2  (\%)}} & \rot{\shortstack{ Hotel 3 (\%)}} & \rot{\shortstack{ Study (\%)}} & \rot{\shortstack{ MIT Lab (\%)}} & \rot{\shortstack{Avg.  Runtime {[}s{]}}} \\
        \hline
        RANSAC-1K      & 91.3        & 89.1        & 74.5      & 94.2       & 84.6       & 90.7       & 86.3          & 81.8   & 0.008                   \\
        \hline
        RANSAC-10K     & 97.2         & 92.3        & 79.3      & 96.5       & 86.5       & 94.4       & 90.4         & 85.7   & 0.074                  \\
        \hline
        TEASER++        & 98.6         & 92.9        & 86.5      & 97.8       & 89.4       & 94.4       & 91.1          & 83.1   & 0.059                  \\
        \hline
        TEASER++ (CERT) & 99.4         & 94.1       & 88.7      & 98.2       & 91.9       & 94.4       & 94.3          & 88.6   & 238.136     
        \end{tabular}
        }
        \captionof{table}{Percentage of correct registration results \revone{and average runtime} using \namepp, \namepp certified, and \ransac on the \emph{3DMatch} dataset.}
        \label{tab:scanMatching}
		\vspace{-3mm} 
	\end{minipage}%
	\end{center}
\vspace{-9mm}
\end{figure*}

We use 3DSmoothNet \cite{gojcic19cvpr-3dsmoothnet}, a state-of-the-art neural network, to compute local descriptors for each 3D keypoint, and generate correspondences using nearest-neighbor search.
We then feed the correspondences to \namepp and \ransac (implemented in Open3D \cite{Zhou18arxiv-open3D}) and compare their performance in terms of percentage of successfully matched scans and runtime. \revone{Due to the large number of pairs we need to test, we run the experiments on a server with a Xeon Platinum 8259CL CPU at 2.50GHz, and allocate 12 threads for each algorithm under test.}
Two scans are successful matched when the transformation computed by a technique 
 has (i) rotation error smaller than \SI{10}{\degree}, and (ii) translation error less than \SI{30}{\centi\metre}. 
 We report \ransac's results with maximum number of iterations equal to 1,000 (\ransaconek) and 10,000 (\ransactenk).
 We also compare the percentage of successful registrations out of the cases where \namepp certified the rotation estimation as optimal, a setup we refer to as \namepp (CERT). 
 The latter \extraEdits{executes Algorithm~\ref{alg:optCertification} to evaluate the poses from \namepp and uses a desired sub-optimality gap $\bar{\eta} = 3\%$ to certify optimality.} 
We use $\beta_i = \beta = 5\ \mathrm{cm}$ for \namepp in all tests.

{\bf Results.}
 Table \ref{tab:scanMatching} shows the percentage of successfully matched scans and the average timing for 
 the four compared techniques.
 \name dominates both \ransac variants \revone{with exception of the Lab scene.} 
	\ransaconek has a success rate up to \revone{$12\%$} worst than \namepp. 
  \ransactenk is an optimized C++ implementation and, while running slower than \namepp, it cannot reach the same accuracy for most scenes.  
 These results further highlight that \namepp can be safely used as a faster and more robust replacement for \ransac in SLAM pipelines. 
 This conclusion is further reinforced by the last row in Table \ref{tab:scanMatching}, where we show the 
 success rate for the poses certified as optimal by \namepp. 
 The success rate strictly dominates both \ransac variants  and \namepp, since \namepp (CERT) is able to identify and 
 reject unreliable registration results. This is a useful feature  when scan matching is used for loop closure detection in SLAM, since bad registration results lead to incorrect loop closures and can compromise the quality of the resulting map (see~\cite{Lajoie19ral-DCGM,Yang20ral-GNC}). 
 \extraEdits{While running Algorithm~\ref{alg:optCertification} in \namepp (CERT) requires more than 200 seconds in average, the majority of instances are solved by \namepp (CERT) within 100 seconds and problems with fewer than 50 \TIMs can be certified in 1.12 seconds. Fig.~\ref{fig:cert-runtime-vs-tims} compares the runtime of our certification approach against \MOSEK (which directly solves the SDP relaxation) for 50, 100, and 110 \TIMs. We can see that \namepp (CERT) is orders of magnitude faster, and can certify large-scale problems beyond the reach of \MOSEK, which runs out of memory for over 150 \TIMs.}

\emph{Why is \namepp not able to solve 100\% of the tests in Table \ref{tab:scanMatching}?}
This should not come as a surprise from the statements in Theorem~\ref{thm:certifiableRegNoiseless1}-\ref{thm:certifiableRegNoisy}. \emph{The estimation contracts discussed in Section~\ref{sec:guarantees} require a minimum number of inliers:} \namepp can mine a small number of inliers among a large number of outliers, but it cannot
 solve problems where no inliers are given or where the inliers are not enough to identify a unique registration!
 For the experiments in this section, it is not uncommon to have no or fewer than 3 inliers in a scene due 
 to the quality of the keypoint descriptors. 
 Moreover, it is not uncommon to have symmetries in the scene, which make the registration non-unique. 
 To intuitively highlight this issue, 
 Fig.~\ref{fig:scanMatchingScatter} shows the rotation errors of \namepp, with different markers for certified (blue dots) and non-certified (red crosses) solutions. 
 The figure shows that (i) \namepp (CERT) is able to reject a large number of incorrect estimates (red crosses with large errors), and (ii) some of the incorrect but certified solutions exhibit errors around \SI{90}{\degree} and \SI{180}{\degree} which correspond to symmetries of the scene. 
 A visual example of the phenomenon is shown in Fig.~\ref{fig:symmetry}. 
 These results highlight the need for 
 better keypoint detectors, since even state-of-the-art deep-learning-based methods struggle to produce acceptable outlier rates in real problems.


\renewcommand{\mpw}{1.3cm}
\begin{figure}[ht!]
	\begin{minipage}{\columnwidth}
	\begin{center}
			\includegraphics[width=1\columnwidth]{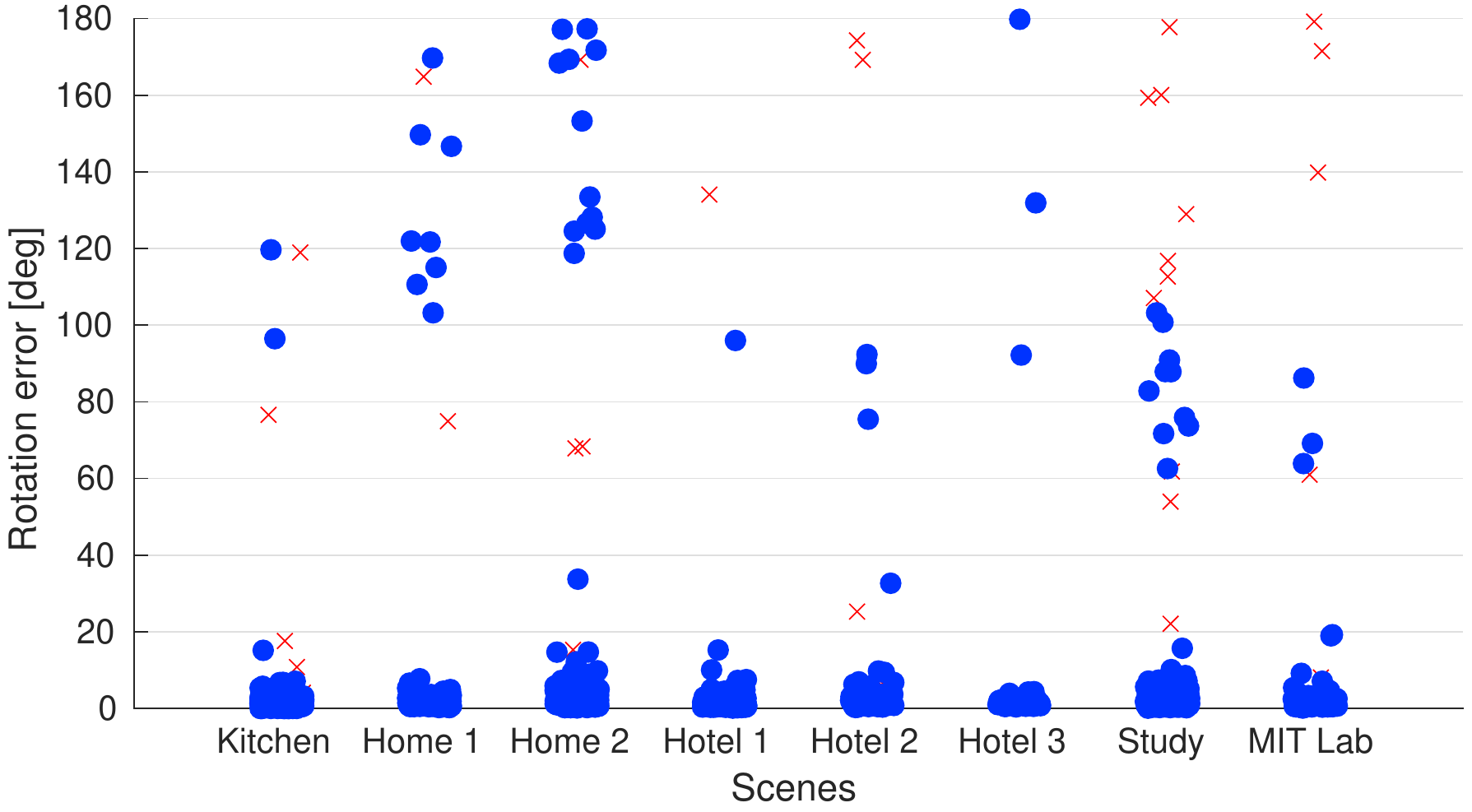}
			\vspace{-8mm}
	\end{center}
	\end{minipage} 
	\caption{Rotation errors for each scene (data points correspond to pairs of scans in the scene), with certified \namepp solutions (blue dots) vs. non-certified (red crosses).
	 \label{fig:scanMatchingScatter}}
	 	\vspace{-4mm} 
\end{figure}


\renewcommand{\mpw}{4.3cm}
\newcommand{\mpww}{8.6cm}

\newbox{\bigpicturebox}

\begin{figure}[t]
	\begin{center}

	\sbox{\bigpicturebox}{%
	\hspace{-12mm}
	\begin{tabular}{c}%
		\scalebox{1}[1.1]{\includegraphics[width=.45\columnwidth]{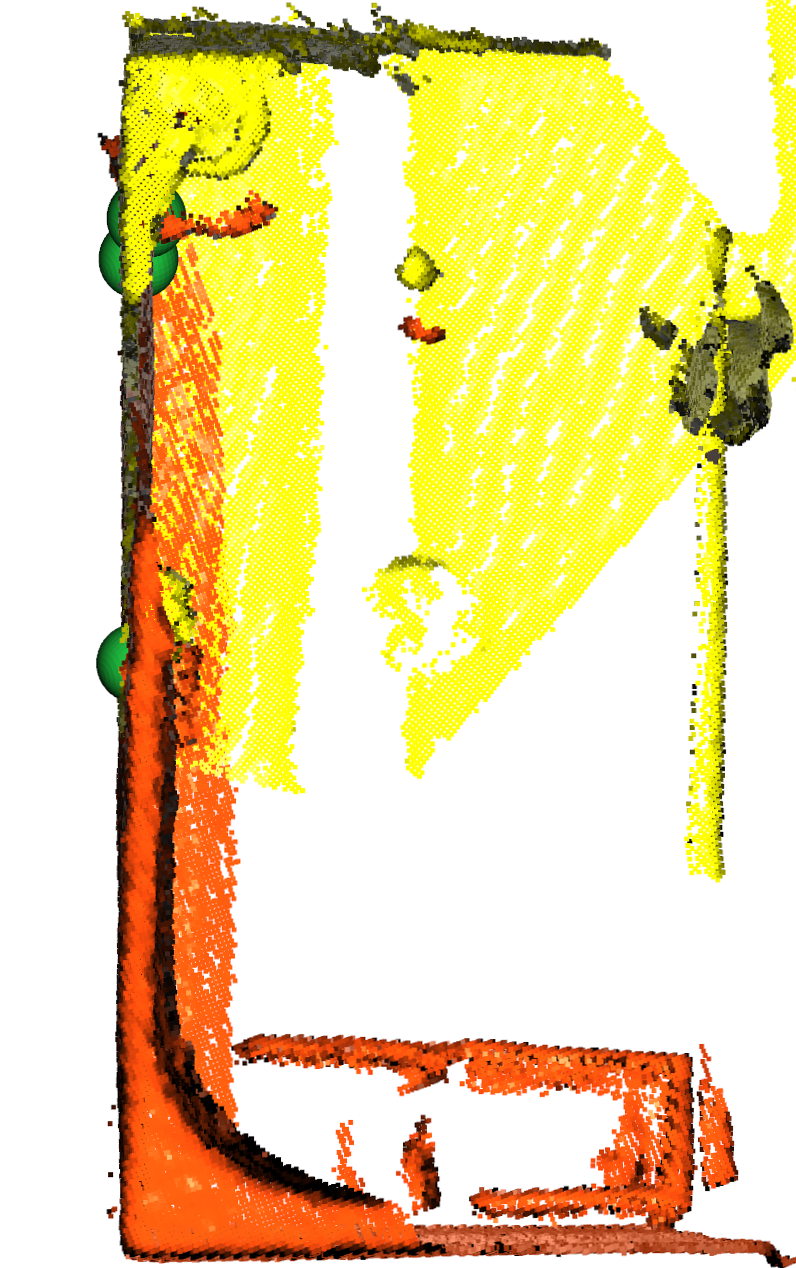}} \\
		(a) Ground truth \\
		(top-down view)
	\end{tabular} %
	}
	\usebox{\bigpicturebox} \hspace{-2mm}
	\begin{minipage}[b][\ht\bigpicturebox]{0.4\columnwidth}
		\begin{tabular}{c}%
			\includegraphics[width=0.92\columnwidth]{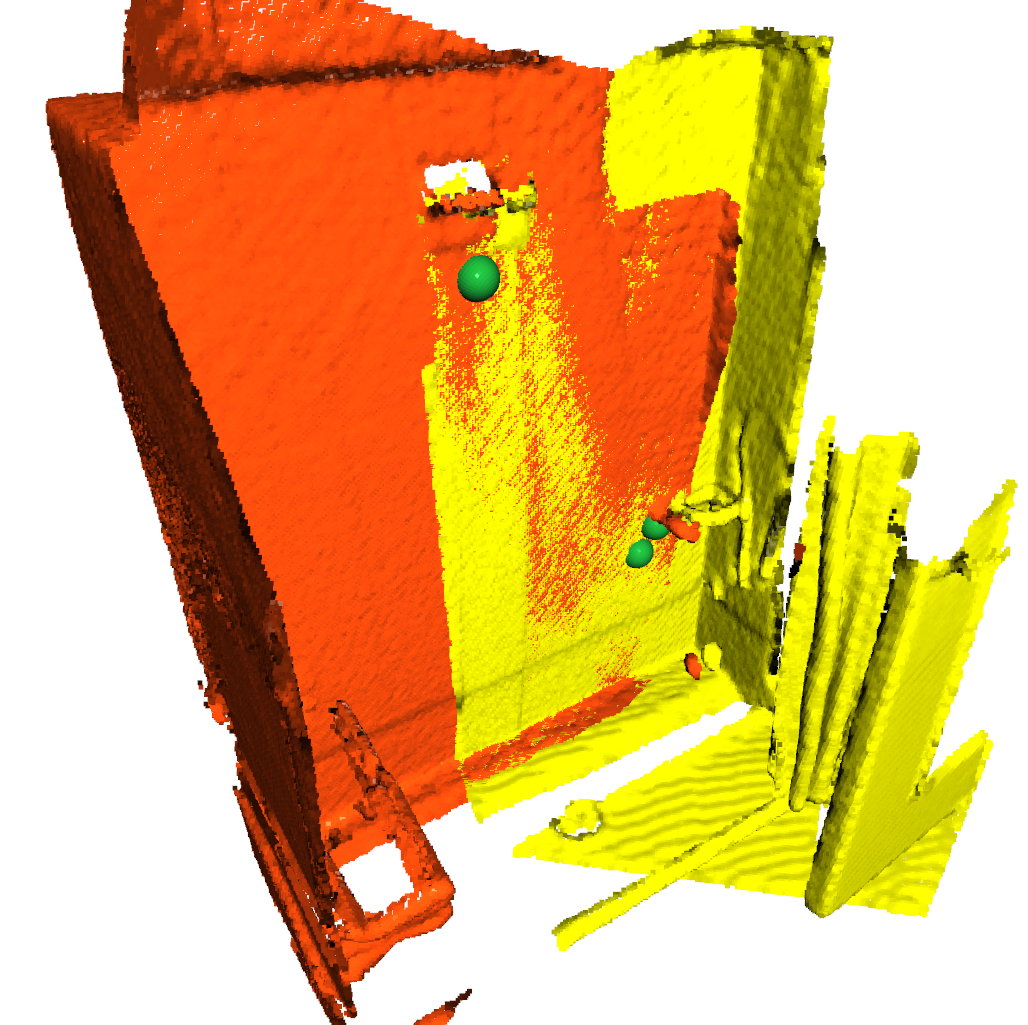} \\
			(b) Ground truth \\
			(side view) \\
			\includegraphics[width=0.8\columnwidth]{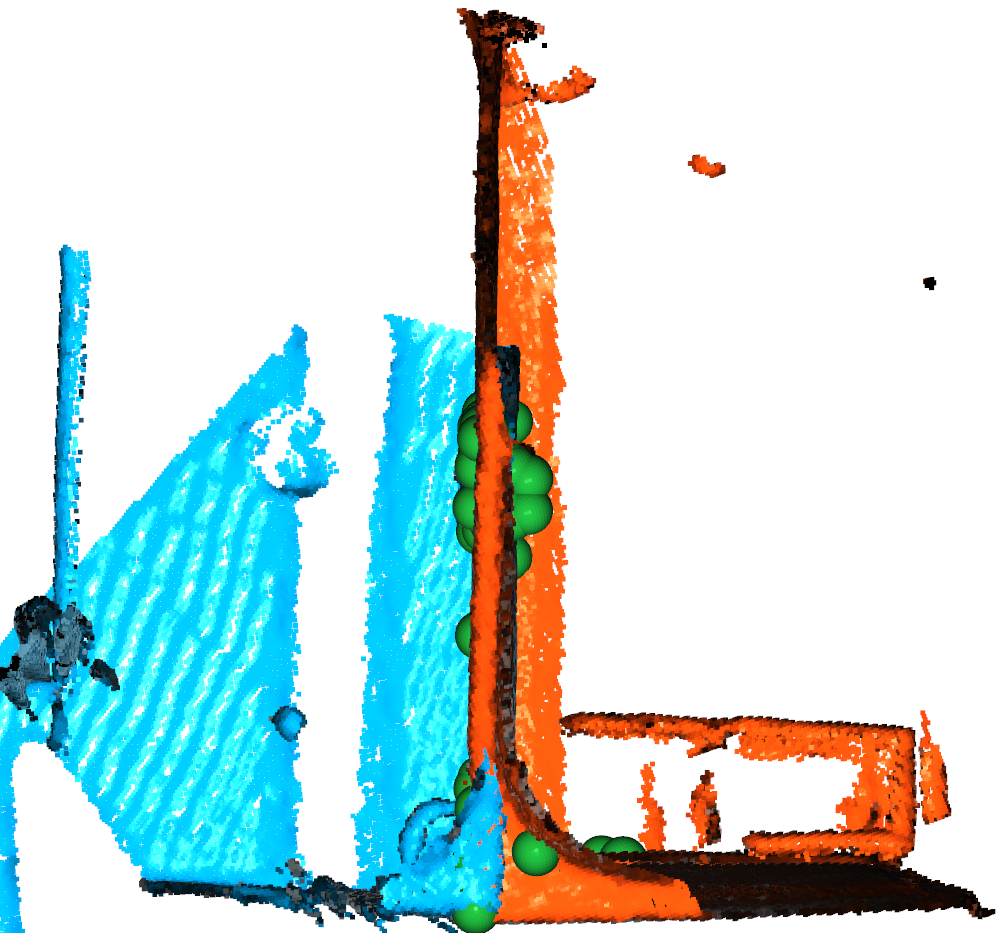} \\
			(c) \namepp's estimate \\
			(top-down view)
		\end{tabular}
	\end{minipage}

	\caption{Example from the Hotel 3 scene of the \matchTD dataset~\cite{Zeng17cvpr-3dmatch}.
	The scene pictures a hotel shower where symmetries in the keypoint distribution cause \namepp to fail. 
	(a)-(b) Ground-truth alignment (top view and side view). The ground-truth alignment admits 3 inlier correspondences (in green). 
	(c) \namepp's estimate of the alignment. \namepp is able to find an estimate with 57 inliers. The symmetries of the scene allow for multiple potential registrations, and the ground truth inliers are not part of the maximum clique of the \TIM graph. In this case, the max clique selected by \namepp contains an incorrect solution that differs from the ground truth by \SI{180}{\degree}.}
	 \label{fig:symmetry}
	\end{center}
\end{figure}

\spacebeforesection
\section{Conclusion} 
\label{sec:conclusion}
We propose the first fast and certifiable algorithm for correspondence-based registration with extreme outlier rates. We leverage insights from estimation theory (\eg unknown-but-bounded noise), 
geometry (\eg invariant measurements), graph theory (\eg maximum clique for inlier selection), and optimization (\eg tight SDP relaxations). 
These insights lead to the design of two certifiable registration algorithms. 
\name is accurate and robust but requires solving a large SDP. 
\namepp has similar performance in practice but circumvents the need to solve an SPD and can run in milliseconds. \revone{\namepp is also certifiable by leveraging Douglas-Rachford Splitting to compute a dual optimality certificate.}  
For both algorithms, we provide theoretical bounds on the estimation errors, which are 
the first of their kind for robust registration problems. 
Moreover, we test their performance on standard benchmarks, object detection datasets, and 
the \emph{\matchTD} scan matching dataset and show that (i) both algorithms dominate the state of the art (\eg~\ransac, \bnb, heuristics) and are robust to more than $99\%$ outliers, (ii) \namepp can run in milliseconds and it is currently the fastest robust registration algorithm at high outlier rates, (iii) \namepp is so robust it can also solve problems without correspondences (\eg hypothesizing all-to-all correspondences) where it \extraEdits{outperforms \ICP and \goICP.}
We release a fast open-source C++ implementation of \namepp.

While not central to the technical contribution, 
this paper also establishes the foundations of \emph{certifiable perception}, as discussed in 
\isExtended{Appendix~\ref{sec:certifiableAlgorithms}.}{Appendix~\ref{sec:certifiableAlgorithms} in~\cite{Yang20arxiv-teaser}.}
This research area offers many research opportunities and the resulting progress has the potential to boost trustworthiness and reliability in safety-critical application of robotics and computer vision.
Future work includes developing certifiable algorithms for other spatial perception problems 
(ideally leading to certifiable approaches for all the applications discussed in~\cite{Yang20ral-GNC} and more). 
Another avenue for future research is the use of the certifiably robust algorithms presented in this paper for self-supervision of deep learning methods for keypoint detection and matching. 
\vspace{-4mm}
\bibliographystyle{IEEEtran}
\bibliography{../../references/refs,myRefs}

\newcommand{\mySpaceBetwenBios}{\vspace{-15mm}}
\mySpaceBetwenBios 
\begin{IEEEbiography}[{\includegraphics[width=1in,height=1.25in,clip,keepaspectratio]{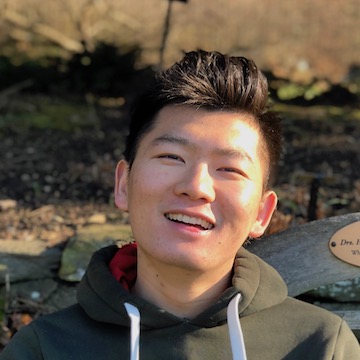}}]{Heng Yang} is a PhD candidate in the Department of Mechanical Engineering and the Laboratory for Information \& Decision Systems (LIDS) at the Massachusetts Institute of Technology, where he is working with Prof. Luca Carlone at the SPARK Lab. He has obtained a B.S. degree in Mechanical Engineering (with honors) from the Tsinghua University, Beijing, China, in 2015; and an S.M. degree in Mechanical Engineering from the Massachusetts Institute of Technology in 2017. His research interests include convex optimization, semidefinite and sums-of-squares relaxation, robust estimation and machine learning, applied to robotic perception and computer vision. His work includes developing robust and certifiable algorithms for geometric understanding. Heng Yang is a recipient of the Best Paper Award in Robot Vision at the 2020 IEEE International Conference on Robotics and Automation (ICRA).
\end{IEEEbiography}
\mySpaceBetwenBios
\begin{IEEEbiography}[{\includegraphics[width=1in,height=1.25in,clip,keepaspectratio]{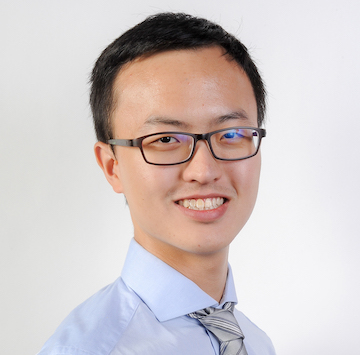}}]{Jingnan Shi} is a Master student 
in the Department of Aeronautics and Astronautics at the Massachusetts Institute of Technology. He is currently a member of the SPARK lab, led by Professor Luca Carlone. 
He has obtained a B.S. degree in Engineering from Harvey Mudd College in 2019.
His research interests include robust perception and estimation, with applications specific to various robotic systems.  
\end{IEEEbiography}
\mySpaceBetwenBios
\begin{IEEEbiography}[{\includegraphics[width=1in,height=1.25in,clip,keepaspectratio]{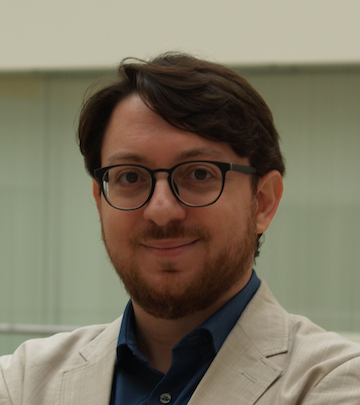}}]{Luca Carlone} is the Leonardo CD Assistant Professor in the Department of Aeronautics
  and Astronautics at the Massachusetts Institute of
  Technology, and a Principal Investigator in the Laboratory for Information \& Decision Systems (LIDS).  He joined LIDS as a postdoctoral associate (2015) and later as a Research
  Scientist (2016), after spending two years as a postdoctoral fellow at the
  Georgia Institute of Technology (2013-2015). 
  He has obtained a B.S.~degree in mechatronics
  from the Polytechnic University of Turin, Italy, in
  2006; an S.M.~degree in mechatronics from the
  Polytechnic University of Turin, Italy, in 2008; an
  S.M.~degree in automation engineering from the
  Polytechnic University of Milan, Italy, in 2008; and
  a Ph.D.~degree in robotics also from the Polytechnic University of Turin in 2012.
  His research interests include
  nonlinear estimation, numerical and distributed optimization, and probabilistic
  inference, applied to sensing, perception, and decision-making in single and
  multi-robot systems. His work includes seminal results on certifiably correct
  algorithms for localization and mapping, as well as approaches for visual-inertial navigation and distributed mapping.
  He is a recipient of the
  2020 RSS  Early Career Award, the 2017 Transactions on Robotics King-Sun Fu Memorial Best Paper Award, 
  the Best Paper Award in Robot Vision at ICRA 2020,
  the Best Paper award at WAFR 2016, and the Best Student Paper award at the 2018
  Symposium on VLSI Circuits. 
\end{IEEEbiography}

\clearpage

 \isTwoCols{}{\onecolumn}
\section*{Appendices}
\addcontentsline{toc}{section}{Appendices}
\renewcommand{\thesubsection}{\Alph{subsection}}

\subsection{Manifesto of Certifiable Perception}
\label{sec:certifiableAlgorithms}

This section provides a broader context for the algorithms developed in this paper, and motivates 
 our efforts towards perception  algorithms with formal performance guarantees.

{\bf What is a \emph{certifiable algorithm}?} 
Outlier-robust estimation is a hard combinatorial problem, where one has to separate inliers from outliers 
while computing an estimate for the variables of interest.
Many algorithms, such as \ransac and \ICP, are heuristic solvers for such a combinatorial problem~\cite{Chin17slcv-maximumConsensusAdvances}. These algorithms are not certifiable (formalized later) in the sense that 
the estimate they return can be arbitrarily far from the optimal solution of the combinatorial 
problem, and there is no easy way to check how suboptimal an estimate is.

Then a natural question is: \emph{can we directly compute an optimal solution for outlier-robust estimation
or develop algorithms that always compute near-optimal solutions?} 
Unfortunately, in general the answer is no. A recent set of papers~\cite{Tzoumas19iros-outliers,Chin18eccv-robustFitting} has shown that a broad family of robust estimation problems (including \maxConLong~\cite{Chin17slcv-maximumConsensusAdvances} and the Truncated Least Squares formulation considered in this paper) are \emph{inapproximable} in polynomial-time: there exist worst-case instances in which no polynomial-time algorithm can compute a near-optimal solution (see~\cite{Chin18eccv-robustFitting} for a more formal treatment and~\cite{Tzoumas19iros-outliers} for even more pessimistic inapproximability results ruling our quasi-polynomial-time algorithms). 

This fundamental intractability calls for a paradigm shift: since it is not possible to solve every 
robust estimation problem in polynomial time, we claim that a useful goal is to design algorithms that 
perform well in typical instances and are able to certify the correctness of the resulting estimates, 
but at the same time can detect worst-case instances and declare ``failure'' on those rather 
than blindly returning an invalid estimate. 

We formalize this notion in the definition below.

\begin{definition}[Certifiable Algorithms]\label{def:certifiableAlgorithm}
Given an optimization problem $\myProb(\myData)$ that depends on input data $\myData$, 
we say that an algorithm $\myAlgo$ is \emph{certifiable} if, after solving $\myProb(\myData)$, 
algorithm $\myAlgo$ either provides a certificate for the quality of its solution (\eg a proof of optimality, a finite bound on its sub-optimality, or a finite bound on the distance of the estimate from the optimal solution), or declares failure otherwise.
\end{definition}

The notion of certifiable algorithms is inspired by (and is indeed a 
particularization of) the notion of \emph{Probably Certifiably Correct (PCC) Algorithm} introduced by 
Bandeira in~\cite{Bandeira16crm}.
We observe that a certifiable algorithm is \emph{sound} (the algorithm does not certify incorrect solutions), 
but is not necessarily \emph{complete} (the algorithm may declare failure in a problem that can be solved in polynomial time by a different algorithm).  
\emph{In this paper, we provide the first certifiable algorithm for registration (and for rotation search, one of the subproblems we encounter) and what we believe are 
the first certifiable algorithms for outlier-robust estimation in robotics and vision.}
We sometimes refer to the algorithms in Definition~\ref{def:certifiableAlgorithm} as \emph{certifiably robust} 
to stress that in outlier-robust estimation the notion of optimality has implications on the robustness to outliers
of the resulting estimate. However, the definition also applies to hard outlier-free problems such as~\cite{Carlone15icra-verification,Rosen18ijrr-sesync}. 

When applied to estimation problems in robotics or vision,
we note that an optimal solution is not necessarily a ``good'' solution (\ie close to the ground truth for the quantity we want to estimate). For instance, if the data $\myData$ we feed to the algorithm is completely random and uncorrelated with our unknown, solving $\myProb(\myData)$ would not bring us any closer to knowing the value of the unknown. In this sense,  
any meaningful performance guarantee 
 will need to take assumptions on the generative model of the data. 
We refer to the corresponding theoretical results as 
 ``estimation contracts'', which, informally, say that 
as long as the data is informative about the ground truth, 
a certifiable algorithm can produce an estimate close to the ground truth. 
In other words, while a certifiable algorithm is one that can assess if it completed the assigned task to 
optimality, an estimation contract ensures that optimality is useful towards estimating 
a variable of interest.


{\bf Why certifiable algorithms?}
Besides being an intellectual pursuit towards the design of better algorithms for robot perception, 
 we believe that the current adoption of algorithms that can fail without notice is a major cause of 
 brittleness in modern robotics applications, ranging from self-driving cars to autonomous drones. 
 This is true across spatial perception applications, including object detection~\cite{RahmanIROS19-falseNegative} and Simultaneous Localization and Mapping~\cite{Cadena16tro-SLAMsurvey}. 
 A certifiable algorithm can detect failures before they cascade to other modules in the autonomy stack.
 Moreover, we hope that the development of estimation contracts can pave the way towards 
  formal verification and monitoring of complex perception systems involving multiple modules and algorithms, establishing connections with parallel efforts on safe autonomy and decision-making, \eg~\cite{Seshia16arxiv-verifiedAutonomy,Desai17icrv-verification}. 

\subsection{Choice of $\beta_i$ and $\barc$}
\label{sec:choiceBeta}

The parameters $\beta_i$ and $\barc$ are straightforward to set in~\eqref{eq:TLSRegistration}. 
In the following we discuss how to set them depending on the assumptions we make on the \emph{inlier} noise.

\begin{remark}[Probabilistic inlier noise]\label{rmk:chi2}
If we assume the inliers follow the generative model~\eqref{eq:robustGenModel} with  
$\vepsilon_i \sim \calN(\MZero_3, \sigma_i^2\eye_3)$ and $\vo_i=\zero$, 
it holds:
\bea
\frac{1}{\betaNoise_i^2} \| \vb_i - \MR \va_i \|^2 = \frac{1}{\betaNoise_i^2} \| \vepsilon_i \|^2
\sim \chi^2(3),
\eea
where $\chi^2(3)$ is the Chi-squared distribution with three degrees of freedom. 
Therefore, with desired probability $p$, the weighted error $\frac{1}{\betaNoise_i^2} \| \vepsilon_i \|^2$ for the inliers satisfies:
\bea \label{eq:choiceofp}
\mathbb{P} \left( \frac{\Vert \vepsilon_i \Vert^2}{\sigma_i^2} \leq \myQuantile^2 \right) = p,
\eea
where $\myQuantile^2$ is the quantile of the $\chi^2$ distribution with three degrees of freedom and lower tail probability equal to $p$ (\eg $\myQuantile=3$ for $p=0.97$). 
Therefore, one can simply set the noise bound $\beta_i$ in Problem~\eqref{eq:TLSRegistration} to be $\beta_i = \sigma_i$, set the $\barc = \myQuantile$. 
As mentioned in Section~\ref{sec:TLSregistration}, from optimization standpoint, this is equivalent to setting $\beta_i = \myQuantile \sigma_i$ and setting $\barc = 1$.
The parameter $\myQuantile$ monotonically increases with $p$; therefore, setting $p$ close to 1 makes the formulation~\eqref{eq:TLSRegistration} more prone to accept measurements with large residuals, while a small $p$ makes~\eqref{eq:TLSRegistration} more selective. 
\end{remark}

\begin{remark}[Set membership inlier noise] 
If we assume the inliers follow the generative model~\eqref{eq:robustGenModel} with  
unknown-but-bounded noise
$\| \vepsilon_i \| \leq \beta_i$, where $\beta_i$ is a given noise bound,\footnote{This is the typical setup assumed in 
\emph{set membership estimation}~\cite{Milanese89chapter-ubb}.} 
it is easy to see that the inliers satisfy: 
\bea
\| \vepsilon_i \| \leq \beta_i &\iff& \| \vb_i - \MR \va_i \|^2 \leq \beta_i^2 
\isTwoCols{\nonumber \\ &\iff&}{\iff}
\frac{1}{\beta_i^2}\| \vb_i - \MR \va_i \|^2 \leq 1
\eea
Therefore, we directly plug $\beta_i$ into~\eqref{eq:TLSRegistration} and choose $\barc = 1$. 
\end{remark} 

\subsection{Truncated Least Squares vs. \maxConLong}
\label{sec:TLSvsMC}

Truncated Least Squares (and Algorithm~\ref{alg:adaptiveVoting}) are related to \emph{\maxConLong} (\maxCon), a popular approach for outlier rejection in vision~\cite{Speciale17cvpr-consensusMaximization,Liu18eccv-registration}. 
\maxConLong looks for an estimate that maximizes the number of inliers (or, equivalently, minimizes the number of outliers):
\beq
\label{eq:MC}
\min_{ \substack{\calO \subseteq \calM \\ \vxx \in \calX} } |\calO|, \subject \| \resF_i(\vxx) \|^2 \leq \barcsq \;\;
\forall \; i\in {\calM \setminus \calO}
\eeq
where $\vxx$ is the variable we want to estimate (possibly belonging to some domain $\calX$),
$\calM$ is the available set of measurements,  
$\resF_i(\cdot)$ is a given residual error function, 
$\barc$ is the maximum admissible error for an inlier, and  
 $|\cdot|$ denotes the cardinality of a set.
Problem~\eqref{eq:MC} looks for the smallest set of outliers ($\calO$) such that all the other measurements (\ie the inliers $\calM \setminus \calO$) have residual error below $\barc$ for some $\vxx$. While \maxConLong is intractable in general, by following the same lines of Theorem~\ref{thm:scalarTLS}, it is easy to show it can be solved in polynomial time in the scalar case.

\maxCon and \TLS do not return the same solution in general, since \TLS may prefer discarding measurements that induce a large bias in the estimate, as shown by the simple example below.

\begin{example}[\TLS $\neq$ \maxCon]
\label{example:TLSvsMC}
Consider a simple scalar estimation problem, where we are given three measurements, 
and compare the two formulations:
\bea
\min_{s} \sum_{k \in \calM} \min\left(  (  s-s_k )^2   \;,\; \barcsq \right) &&\hspace{-3mm} \text{(TLS)} \\
\min_{\substack{\calO \subseteq \calM, s}} |\calO| 
 \subject {}
 (  s-s_k )^2 \leq  \barcsq \;\; \forall k \in \calM \setminus \calO && \hspace{-3mm} \text{(MC)}
\eea
where $\calM = \{1,2,3\}$.
Assume $s_1\!=\!s_2\!=\!0$, 
 $s_3 = 3$, and $\barc = 2$. 
Then, it is possible to see that $s_{\maxCon} = 1.5$ attains a maximum consensus set including all measurements $\{1,2,3\}$, {while the \TLS estimate is $\hats = 0$ which attains an optimal cost equal to 2, 
and has consensus set $\conSet(\hats) = \{1,2\}$}. 
\end{example}

Note that \maxConLong is not necessarily preferable over \TLS: indeed (i) a human might have arguably preferred to flag $s_3 = 3$ as an outlier in the toy example above, and more importantly (ii) in the experiments in Section~\ref{sec:separateSolver} we observe that often \maxConLong has a lower breakdown point than \TLS (intuitively, \TLS is more sensitive to the \emph{distribution} of the inliers, since it penalizes their residuals). 

While \maxCon and \TLS do not choose the same set of inliers in general,
in cases where there is a large set of inliers, we expect \TLS and \maxCon to produce similar solutions, as shown below.

\begin{lemma}[Necessary Conditions for \TLS $\equiv$ \maxCon]
\label{thm:TLSandMC}
Assume the maximum consensus set returned by \maxConLong (\maxCon) has size $\nrInMC$ and the sum of the squared residual errors for the inliers is $\resInMC$.
If the second largest consensus set has size smaller than $\nrInMC - \resInMC/\barcsq$, then
 \TLS also selects the maximum consensus set as inliers.
\end{lemma}

\begin{proof}
The lemma establishes a general condition under which the set computed by~\eqref{eq:MC} matches the set of inliers found by the \TLS formulation:
\beq
\label{eq:TLSgeneral}
\min_{ \vxx \in \calX } \sum_{i\in \calM} \min \left( \| \resF_i(\vxx) \|^2 , \barcsq \right) 
\eeq
For the \TLS formulation, the inliers are the measurements that --at the optimal solution-- have residual smaller than $\barc$.
We denote as $f_\TLS(\conSet)$ the value of the \TLS cost in~\eqref{eq:TLSgeneral} for a given choice of the consensus set $\conSet$. 

Denote the size of the measurement set as $N$, \ie $N \doteq |\calM|$.
If $\calO_\maxCon$ is the optimal solution of~\eqref{eq:MC}, we define $\nrOutMC \doteq |\calO_\maxCon|$.
Lemma~\ref{thm:TLSandMC} denotes the number of 
inliers found by \maxConLong as 
$\nrInMC \doteq |\calM \setminus \calO| =  N - \nrOutMC$. 
Moreover, Lemma~\ref{thm:TLSandMC} denotes the sum of the squared residual errors for the inliers as:
\beq
\resInMC \doteq \sum_{i \in {\calM \setminus \calO_\maxCon}}\| \resF_i(\vxx) \|^2
\eeq
Then the lemma claims that if the size of the second largest consensus set is smaller than
$\nrInMC - \resInMC/\barcsq$, \TLS returns the maximum consensus set as inliers.

Assume by contradiction that the maximum consensus set $\conSet_\maxCon$ leads to a suboptimal \TLS solution.
Then, there exists another  \TLS solution with consensus set $\conSet'$ such that:
\beq
\label{eq:cont1}
f_\TLS(\conSet') < \nrOutMC \barcsq + \resInMC  
\eeq
which follows from the assumption that the solution corresponding to $\conSet_\maxCon$  (attaining cost $\nrOutMC \barcsq + \resInMC$) is suboptimal.

Now, if we call $\nrIn' = |\conSet'|$ and define $\nrOut' = N - \nrIn'$, the assumption that any consensus set other than $\conSet_\maxCon$ has size smaller than $\nrInMC - \resInMC/\barcsq$  implies:
\bea
\nrIn' < \nrInMC - \resInMC/\barcsq & \iff \nonumber \\
N - \nrOut' < N - \nrOutMC - \resInMC/\barcsq & \iff \nonumber  \\
\nrOut' > \nrOutMC + \resInMC/\barcsq
\label{eq:appIneq1}
\eea

Using~\eqref{eq:appIneq1} and recalling that the residual error for the inliers is nonnegative (it is a sum of squares):
\beq
\label{eq:cont2}
f_\TLS(\conSet') \geq \nrOut' \barcsq >  (\nrOutMC + \resInMC/\barcsq)\barcsq = \nrOutMC \barcsq + \resInMC
\eeq
We conclude the lemma by observing that~\eqref{eq:cont1} and~\eqref{eq:cont2} cannot be simultaneously satisfied, leading to contradiction. 
\end{proof} 
\isTwoCols{

\begin{table*}[t]
\centering
\begin{tabular}{cccc}
Measurements & Points & \TIMs & \TRIMs \\
\hline 
Symbol & $\va_i$, $\vb_i$ & $\barva_{ij}$, $\barvb_{ij}$ & $s_{ij}$ \\
\hline 
Definition & - & $\begin{cases} \barva_{ij} = \va_j - \va_i \\ \barvb_{ij} = \barvb_j - \barvb_i \end{cases}$ & $s_{ij} = \frac{\| \barvb_{ij} \| }{ \| \barva_{ij} \| }$ \\
\hline 
Generative model & $\vb_i = s\MR\va_i + \vt + \vo_i + \vepsilon_i $ & $\barvb_{ij} = s\MR \barva_{ij} + \vo_{ij} + \vepsilon_{ij}$ & $s_{ij} = s + o_{ij}^s + \epsilon_{ij}^s$ \\
\hline 
Noise bounds & $\|\vepsilon_i\| \leq \beta_i$ & $\|\vepsilon_{ij}\| \leq \TIMNoiseBound_{ij} \doteq \beta_i+\beta_j$ & $|\epsilon_{ij}^s| \leq \alpha_{ij} \doteq \TIMNoiseBound_{ij}/\|\barva_{ij}\|$ \\
\hline 
Dependent transformations & $(s,\MR,\vt)$ & $(s,\MR)$ & s \\
\hline
Number & $N$ & $K \leq \frac{N(N-1)}{2}$ & $K$ \\
\hline
\end{tabular}
\caption{Summary of invariant measurements.}
\label{tab:summaryIMs}
\vspace{-5mm}
\end{table*}}{} 

\subsection{Proof of~\prettyref{thm:TIM}: Translation Invariant Measurements}
\label{sec:proof:thm:TIM}

 Using the vector notation $\va \in \Real{3\nrPoints}$ and $\vb \in \Real{3\nrPoints}$ already introduced in the statement of the theorem, we can write the generative model~\eqref{eq:robustGenModel} compactly as:
 \bea
 \label{eq:robustGenModelVect}
 \vb = s (\eye_\nrPoints \kron \MR) \va + (\ones_\nrPoints \kron \vt) + \vo + \veps
 \eea
 where $\vo \doteq [\vo_1\tran \; \ldots \; \vo_\nrPoints\tran]\tran$,
  $\veps \doteq [\veps_1\tran \; \ldots \; \veps_\nrPoints\tran]\tran$, and
  $\ones_\nrPoints$ is a column vector of ones of size $\nrPoints$. Denote $K=|\calE|$ as the cadinality of $\calE$, such that $\MA \in \Real{K \times \nrPoints}$ is the incidence matrix of the graph with edges $\calE$.
  Let us now multiply both members by $(\MA \kron \eye_3)$:
  \bea
  \label{eq:proof-tim1}
 \TIMb = 
  (\MA \kron \eye_3) [s (\eye_\nrPoints \kron \MR) \va +  (\ones_\nrPoints \kron \vt) + (\vo + \veps)]
 \eea
 Using the property of the Kronecker product we simplify:
 \beal
 (i) &  (\MA \kron \eye_3)(\eye_\nrPoints \kron \MR) \va = (\MA \kron \MR) \va \\
 &=  (\eye_K \kron \MR) (\MA \kron \eye_3) \va \\
 &=(\eye_K \kron \MR) \TIMa\\
 (ii)& (\MA \kron \eye_3)(\ones_\nrPoints \kron \vt) = (\MA \ones_\nrPoints \kron \vt ) = \zero
 \eeal
where we used the fact that $\ones_\nrPoints$ is in the Null space of the 
 incidence matrix $\MA$~\cite{Chung96book}. Using (i) and (ii), eq.~\eqref{eq:proof-tim1} becomes:
 \bea
  \label{eq:proof-tim2}
 \TIMb = s (\eye_K \kron \MR)  \TIMa + (\MA \kron \eye_3)(\vo + \veps)
 \eea
 which is invariant to the translation $\vt$, concluding the proof.  

\subsection{Summary of Invariant Measurements}
\label{sec:summaryIM}
Table~\ref{tab:summaryIMs} provides a summary of the invariant measurements introduced in Section~\ref{sec:decoupling}.

\isTwoCols{}{

\begin{table*}[t]
\centering
\begin{tabular}{cccc}
Measurements & Points & \TIMs & \TRIMs \\
\hline 
Symbol & $\va_i$, $\vb_i$ & $\barva_{ij}$, $\barvb_{ij}$ & $s_{ij}$ \\
\hline 
Definition & - & $\begin{cases} \barva_{ij} = \va_j - \va_i \\ \barvb_{ij} = \barvb_j - \barvb_i \end{cases}$ & $s_{ij} = \frac{\| \barvb_{ij} \| }{ \| \barva_{ij} \| }$ \\
\hline 
Generative model & $\vb_i = s\MR\va_i + \vt + \vo_i + \vepsilon_i $ & $\barvb_{ij} = s\MR \barva_{ij} + \vo_{ij} + \vepsilon_{ij}$ & $s_{ij} = s + o_{ij}^s + \epsilon_{ij}^s$ \\
\hline 
Noise bounds & $\|\vepsilon_i\| \leq \beta_i$ & $\|\vepsilon_{ij}\| \leq \TIMNoiseBound_{ij} \doteq \beta_i+\beta_j$ & $|\epsilon_{ij}^s| \leq \alpha_{ij} \doteq \TIMNoiseBound_{ij}/\|\barva_{ij}\|$ \\
\hline 
Dependent transformations & $(s,\MR,\vt)$ & $(s,\MR)$ & s \\
\hline
Number & $N$ & $K \leq \frac{N(N-1)}{2}$ & $K$ \\
\hline
\end{tabular}
\caption{Summary of invariant measurements.}
\label{tab:summaryIMs}
\vspace{-5mm}
\end{table*}} 

\subsection{Proof of~\prettyref{thm:maxClique}: Max Clique Inlier Selection}
\label{sec:proof:thm:maxClique}

Consider a graph $\calG'(\calV,\calE')$ whose edges where selected as inliers during scale estimation, by pruning the complete graph $\calG(\calV,\calE)$. 
An edge $(i,j)$ (and the corresponding \TIM) is an inlier if both $i$ and $j$ are correct correspondences (see discussion before~\prettyref{thm:TIM}). Therefore, $\calG'$ contains edges connecting all points for which we have inlier correspondences. 
Therefore, these points are vertices of a clique in the graph $\calG'$ 
and the edges (or equivalently the \TIMs) connecting those points form a clique in  $\calG'$.
We conclude the proof by observing 
that the clique formed by the inliers has to belong to at least one maximal clique of $\calG'$. 

\subsection{Proof of~\prettyref{thm:scalarTLS}: Optimal Scalar \TLS Estimation}
\label{sec:proof:thm:scalarTLS}

Let us first prove that there are at most $2\nrTIM-1$ different non-empty consensus sets. 
We attach a confidence interval $[s_k - \alpha_k\barc, s_k + \alpha_k\barc]$ 
to each measurement $s_k$, $\forall k \in \{1,\ldots,\nrTIM\}$.
For a given scalar $s\in \Real{}$, a measurement $k$ is in the consensus set of $s$ if $s\in [s_k-\alpha_k\barc, s_k+\alpha_k\barc]$ (satisfies $\frac{\|s-s_k\|^2}{\alpha_k^2} \leq \barcsq$), see Fig.~\ref{fig:consensusMax}(a).
Therefore, the only points on the real line where the 
consensus set may change are the boundaries (shown in red in Fig.~\ref{fig:consensusMax}(a)) of the intervals 
$[s_k-\alpha_k\barc, s_k+\alpha_k\barc]$, $k \in \{1,\ldots,\nrTIM\}$. Since there are at most   $2\nrTIM-1$  such intervals, there are at most $2\nrTIM-1$ non-empty consensus sets (Fig.~\ref{fig:consensusMax}(b)), concluding the first part of the proof. 
The second part follows from the fact that the consensus set of $\hats$ is necessarily one of the $2\nrTIM-1$ possible consensus sets, and problem~\eqref{eq:TLSscale} simply computes the least squares estimate of the measurements for every possible consensus set (at most $2\nrTIM-1$) and chooses the estimate that induces the lowest cost as the optimal estimate. 


\subsection{Proof of Proposition~\ref{prop:binaryCloning}: Binary Cloning}
\label{sec:proof:prop:binaryCloning}

Here we prove the equivalence between the mixed-integer program~\eqref{eq:TLSadditiveForm} and the optimization in~\eqref{eq:TLSadditiveForm2} involving $N+1$ quaternions. To do so, 
we note that since $\theta_k \in \{+1,-1\}$ and $\frac{1+\theta_k}{2} \in \{0,1\}$, we can safely move $\frac{1+\theta_k}{2}$ inside the squared norm (because $0=0^2,1=1^2$) in each summand of the cost function~\eqref{eq:TLSadditiveForm}:
\bea \label{eq:moveInsideNorm}
& \sumAllPointsk \frac{1+\theta_k}{2}\frac{\Vert \hatvb_k - \vq \qProd \hatva_k \qProd \vq\inv \Vert^2}{\TIMNoiseBound_k^2} + \frac{1-\theta_k}{2} \barcsq 
\isTwoCols{
 \hspace{-5mm}\\
& \hspace{-9mm} =}{=} 
\sumAllPointsk \frac{\Vert \hatvb_k - \vq \qProd \hatva_k \qProd \vq\inv + \theta_k\hatvb_k - \vq \qProd \hatva_k \qProd (\theta_k \vq\inv) \Vert^2}{4\TIMNoiseBound_k^2} + \frac{1-\theta_k}{2}\barcsq  \nonumber\hspace{-5mm}  
\eea
Now we introduce $N$ new unit quaternions $\vq_k = \theta_k \vq$, $k=1,\dots,N$ by multiplying $\vq$ by the $N$ binary variables $\theta_k \in \{+1,-1\}$, a re-parametrization we called \emph{binary cloning}. One can easily verify that $\vq\tran \vq_k = \theta_k (\vq\tran \vq) = \theta_k$. Hence, by substituting $\theta_k=\vq\tran \vq_k$ into~\eqref{eq:moveInsideNorm}, we can rewrite the mixed-integer program~\eqref{eq:TLSadditiveForm} as:
\bea
& \min\limits_{\substack{\vq \in \usphere^3 \\ \vq_k \in \{\pm \vq\}}} 
\sumAllPointsk \frac{\Vert \hatvb_k - \vq \qProd \hatva_k \qProd \vq\inv + \vq\tran \vq_k \hatvb_k - \vq \qProd \hatva_k \qProd \vq_k\inv \Vert^2 }{4\TIMNoiseBound_k^2} 
\isTwoCols{\nonumber \\
& +}{+} 
\frac{1-\vq\tran\vq_k}{2}\barcsq,
\eea
which is the same as the optimization in~\eqref{eq:TLSadditiveForm2}. 

\subsection{Proof of Proposition~\ref{prop:qcqp}: Binary Cloning as a \QCQP}
\label{sec:proof:prop:qcqp}

Here we show that the optimization involving $K+1$ quaternions in~\eqref{eq:TLSadditiveForm2} can be reformulated as the Quadratically-Constrained Quadratic Program (\QCQP) in~\eqref{eq:TLSBinaryClone}. Towards this goal, 
we prove that the objective function and the constraints in the \QCQP are a 
re-parametrization of the ones in~\eqref{eq:TLSadditiveForm2}.

{\bf Equivalence of the objective functions.}
We start by developing the squared 2-norm term in~\eqref{eq:TLSadditiveForm2} (we scale $\TIMa_k$ by $\hats\TIMa_k$,~\ie~$\TIMa_k \doteq \hats\TIMa_k$ for notation simplicity):
\bea
& \Vert \hatvb_k - \vq \qProd \hatva_k \qProd \vq\inv + \vq\tran \vq_k \hatvb_k - \vq \qProd \hatva_k \qProd \vq_k\inv \Vert^2 \nonumber \\
& \expl{$\Vert \vq\tran \vq_k \hatvb_k \Vert^2 = \Vert \hatvb_k \Vert^2 = \Vert \TIMb_k \Vert^2$, $\hatvb_k\tran(\vq\tran\vq_k)\hatvb_k = \vq\tran\vq_k\Vert \TIMb_k \Vert^2$}
& \expl{$\Vert \vq \qProd \hatva_k \qProd \vq\inv \Vert^2 = \Vert \MR \TIMa_k\Vert^2 = \Vert \TIMa_k \Vert^2$}
& \expl{$\Vert \vq \qProd \hatva_k \qProd \vq_k\inv \Vert^2 = \Vert \theta_k \MR \TIMa_k\Vert^2 = \Vert \TIMa_k \Vert^2$}
& \expl{${\scriptstyle (\vq \qProd \hatva_k \qProd \vq\inv)\tran (\vq \qProd \hatva_k \qProd \vq_k\inv) = (\MR \TIMa_k)\tran(\theta_k \MR \TIMa_k) = \vq\tran\vq_k\Vert \TIMa_k \Vert^2}$}
& = 2\Vert \TIMb_k \Vert^2 + 2\Vert \TIMa_k \Vert^2 + 2\vq\tran \vq_k \Vert \TIMb_k \Vert^2 + 2\vq\tran \vq_k \Vert \TIMa_k \Vert^2 \nonumber \\
& - 2\hatvb_k\tran (\vq\qProd \hatva_k \qProd \vq\inv) -2\hatvb_k\tran (\vq\qProd\hatva_k\qProd\vq_k\inv) \nonumber \\
& \hspace{-8mm}- 2\vq\tran\vq_k\hatvb_k\tran(\vq\qProd \hatva_k \qProd \vq\inv) -2\vq\tran\vq_k\hatvb_k\tran(\vq\qProd\hatva_k\qProd\vq_k\inv) \\
& \expl{${\scriptstyle \vq\tran\vq_k\hatvb_k\tran(\vq\qProd\hatva_k\qProd\vq_k\inv)=(\theta_k)^2\hatvb_k\tran(\vq\qProd\hatva_k\qProd\vq\inv)= \hatvb_k\tran(\vq\qProd\hatva_k\qProd\vq\inv)}$ }
& \expl{ $\hatvb_k\tran(\vq\qProd\hatva_k\qProd\vq_k\inv) = \vq\tran \vq_k \hatvb_k\tran(\vq\qProd\hatva_k\qProd\vq\inv) $ }
& = 2\Vert \TIMb_k \Vert^2 + 2\Vert \TIMa_k \Vert^2 + 2\vq\tran \vq_k \Vert \TIMb_k \Vert^2 + 2\vq\tran \vq_k \Vert \TIMa_k \Vert^2 \nonumber \\
& - 4\hatvb_k\tran (\vq\qProd \hatva_k \qProd \vq\inv) - 4\vq\tran\vq_k\hatvb_k\tran (\vq\qProd \hatva_k \qProd \vq\inv) \label{eq:afterDevelopSquares}
\eea
where we have used multiple times the binary cloning equalities $\vq_k = \theta_k \vq, \theta_k = \vq\tran \vq_k$, the equivalence between applying rotation to a homogeneous vector $\hatva_k$ using quaternion product and using rotation matrix in eq.~\eqref{eq:q_pointRot} from the main document, as well as the fact that vector 2-norm is invariant to rotation and homogenization (with zero padding). 

Before moving to the next step, we make the following observation by combining eq.~\eqref{eq:Omega_1} and $\vq\inv = [-\vv\tran,s]\tran$:
\bea
\MOmega_1(\vq\inv) = \MOmega_1\tran(\vq),\quad \MOmega_2(\vq\inv) = \MOmega_2\tran(\vq)
\eea
which states the linear operators $\MOmega_1(\cdot)$ and $\MOmega_2(\cdot)$ of $\vq$ and its inverse $\vq\inv$ are related by a simple transpose operation.
In the next step, we use the equivalence between quaternion product and linear operators in $\MOmega_1(\vq)$ and $\MOmega_2(\vq)$ as defined in eq.~\eqref{eq:quatProduct}-\eqref{eq:Omega_1} to simplify $\hatvb_k\tran(\vq\qProd\hatva_k \qProd \vq\inv)$ in eq.~\eqref{eq:afterDevelopSquares}:
\bea
& \hatvb_k\tran(\vq\qProd\hatva_k \qProd \vq\inv) \nonumber \\
& \expl{${\scriptstyle \vq\qProd\hatva_k = \MOmega_1(\vq)\hatva_k \;,\; \MOmega_1(\vq)\hatva_k \qProd \vq\inv = \MOmega_2(\vq\inv)\MOmega_1(\vq)\hatva_k = \MOmega_2\tran(\vq)\MOmega_1(\vq)\hatva_k}$} 
& = \hatvb_k\tran (\MOmega_2\tran(\vq)\MOmega_1(\vq) \hatva_k) \\
& \hspace{-8mm} \expl{$\MOmega_2(\vq)\hatvb_k = \hatvb_k \qProd \vq = \MOmega_1(\hatvb_k) \vq$ \;,\; $\MOmega_1(\vq)\hatva_k = \vq \qProd \hatva_k = \MOmega_2(\hatva_k)\vq$}
& = \vq\tran \MOmega_1\tran(\hatvb_k) \MOmega_2(\hatva_k) \vq. \label{eq:developQuatProduct}
\eea
Now we can insert eq.~\eqref{eq:developQuatProduct} back to eq.~\eqref{eq:afterDevelopSquares} and write:
\bea
& \Vert \hatvb_k - \vq \qProd \hatva_k \qProd \vq\inv + \vq\tran \vq_k \hatvb_k - \vq \qProd \hatva_k \qProd \vq_k\inv \Vert^2 \nonumber \\
& = 2\Vert \TIMb_k \Vert^2 + 2\Vert \TIMa_k \Vert^2 + 2\vq\tran \vq_k \Vert \TIMb_k \Vert^2 + 2\vq\tran \vq_k \Vert \TIMa_k \Vert^2 \nonumber \\
& \hspace{-5mm} - 4\hatvb_k\tran (\vq\qProd \hatva_k \qProd \vq\inv) - 4\vq\tran\vq_k\hatvb_k\tran (\vq\qProd \hatva_k \qProd \vq\inv) \\
& = 2\Vert \TIMb_k \Vert^2 + 2\Vert \TIMa_k \Vert^2 + 2\vq\tran \vq_k \Vert \TIMb_k \Vert^2 + 2\vq\tran \vq_k \Vert \TIMa_k \Vert^2 \nonumber \\
& \hspace{-5mm} - 4\vq\tran\MOmega_1\tran(\hatvb_k)\MOmega_2(\hatva_k)\vq - 4\vq\tran\vq_k \vq\tran\MOmega_1\tran(\hatvb_k)\MOmega_2(\hatva_k)\vq \\
& \expl{${\scriptstyle \vq\tran\vq_k \vq\tran\MOmega_1\tran(\hatvb_k)\MOmega_2(\hatva_k)\vq=\theta_k \vq\tran\MOmega_1\tran(\hatvb_k)\MOmega_2(\hatva_k)\vq = \vq\tran\MOmega_1\tran(\hatvb_k)\MOmega_2(\hatva_k)\vq_k}$}
& = 2\Vert \TIMb_k \Vert^2 + 2\Vert \TIMa_k \Vert^2 + 2\vq\tran \vq_k \Vert \TIMb_k \Vert^2 + 2\vq\tran \vq_k \Vert \TIMa_k \Vert^2 \nonumber \\
& \hspace{-5mm} - 4\vq\tran\MOmega_1\tran(\hatvb_k)\MOmega_2(\hatva_k)\vq - 4 \vq\tran\MOmega_1\tran(\hatvb_k)\MOmega_2(\hatva_k)\vq_k \\
& \expl{$-\MOmega_1\tran(\hatvb_k) = \MOmega_1(\hatvb_k)$} 
& = 2\Vert \TIMb_k \Vert^2 + 2\Vert \TIMa_k \Vert^2 + 2\vq\tran \vq_k \Vert \TIMb_k \Vert^2 + 2\vq\tran \vq_k \Vert \TIMa_k \Vert^2 \nonumber \\
& \hspace{-5mm} + 4\vq\tran\MOmega_1(\hatvb_k)\MOmega_2(\hatva_k)\vq + 4 \vq\tran\MOmega_1(\hatvb_k)\MOmega_2(\hatva_k)\vq_k, \label{eq:finishDevelopSquareNorm}
\eea
which is quadratic in $\vq$ and $\vq_k$. Substituting eq.~\eqref{eq:finishDevelopSquareNorm} back to~\eqref{eq:TLSadditiveForm2}, we can write the cost function as:
\bea \label{eq:costInQuadraticFormBeforeLifting}
& \sum\limits_{k=1}^K \frac{\Vert \hatvb_k - \vq \qProd \hatva_k \qProd \vq\inv + \vq\tran \vq_k \hatvb_k - \vq \qProd \hatva_k \qProd \vq_k\inv \Vert^2 }{4\TIMNoiseBound_k^2} + \frac{1-\vq\tran\vq_k}{2}\barcsq \nonumber \hspace{-5mm}\\
& \hspace{-10mm}= \sum\limits_{k=1}^K \vq_k\tran \left( \underbrace{ \frac{(\Vert \TIMb_k \Vert^2+ \Vert \TIMa_k \Vert^2)\eye_4 + 2\MOmega_1(\hatvb_k)\MOmega_2(\hatva_k)}{2\TIMNoiseBound_k^2} + \frac{\barcsq}{2}\eye_4 }_{:=\MQ_{kk}}\right) \vq_k \nonumber\hspace{-10mm} \\
& \hspace{-15mm} + 2\vq\tran \left( \underbrace{ \frac{(\Vert \TIMb_k \Vert^2+ \Vert \TIMa_k \Vert^2)\eye_4 + 2\MOmega_1(\hatvb_k)\MOmega_2(\hatva_k)}{4\TIMNoiseBound_k^2}  - \frac{\barcsq}{4}\eye_4 }_{:=\MQ_{0k}} \right) \vq_k, \nonumber\hspace{-15mm}\\
\eea
where we have used two facts: (i) $\vq\tran \MA \vq = \theta_k^2 \vq\tran \MA \vq = \vq_k\tran \MA \vq_k$ for any matrix $\MA\in\Real{4 \times 4}$, (kk) $c=c\vq\tran \vq = \vq\tran (c\eye_4) \vq$ for any real constant $c$, which allowed writing the quadratic forms of $\vq$ and constant terms in the cost as quadratic forms of $\vq_k$. Since we have not changed the decision variables $\vq$ and $\{\vq_k\}_{k=1}^K$, the optimization in~\eqref{eq:TLSadditiveForm2} is therefore equivalent to the following optimization:
\bea \label{eq:optQCQPbeforeLifting}
\min_{\substack{\vq \in S^3 \\ \vq_k \in \{\pm \vq\}}} \sumAllPointsk \vq_k\tran \MQ_{kk} \vq_k + 2\vq\tran \MQ_{0k} \vq_k
\eea
where $\MQ_{kk}$ and $\MQ_{0k}$ are the known $4\times 4$ data matrices as defined in eq.~\eqref{eq:costInQuadraticFormBeforeLifting}.

Now it remains to prove that the above optimization~\eqref{eq:optQCQPbeforeLifting} is equivalent to the \QCQP in~\eqref{eq:TLSBinaryClone}. Recall that $\vxx$ is the column vector stacking all the $K+1$ quaternions, \ie, $\vxx=[\vq\tran\ \vq_1\tran\ \dots\ \vq_N\tran]\tran \in \Real{4(K+1)}$. 
Let us introduce 
symmetric matrices $\MQ_k \in \Real{4(K+1) \times 4(K+1) },k=1,\dots,K$ and let the $4 \times 4$ sub-block of $\MQ_k$ corresponding to sub-vector $\vu$ and $\vv$, be denoted as $[\MQ_k]_{uv}$; each $\MQ_k$ is defined as:
\bea \label{eq:Qi}
[\MQ_k]_{uv} = \begin{cases}
\MQ_{kk} & \text{if } \vu=\vq_k \text{ and }\vv=\vq_k \\
\MQ_{0k} & \substack{ \text{if } \vu=\vq \text{ and }\vv=\vq_k \\ \text{or } \vu=\vq_k \text{ and }\vv=\vq} \\
\zero_{4\times 4} & \text{otherwise}
\end{cases}
\eea
\ie $\MQ_k$ has the diagonal $4\times 4$ sub-block corresponding to $(\vq_k,\vq_k)$ be $\MQ_{kk}$, has the two off-diagonal $4\times 4$ sub-blocks corresponding to $(\vq,\vq_k)$ and $(\vq_k,\vq)$ be $\MQ_{0k}$, and has all the other $4\times 4$ sub-blocks be zero. Then we can write the cost function in eq.~\eqref{eq:optQCQPbeforeLifting} compactly using $\vxx$ and $\MQ_k$:
\bea
& \displaystyle \sumAllPointsk \vq_k\tran \MQ_{kk} \vq_k + 2\vq\tran \MQ_{0k} \vq_k = \sumAllPointsk \vxx\tran \MQ_k \vxx
\eea
Therefore, by denoting $\MQ = \sumAllPointsk \MQ_k$, we proved that the objective functions in~\eqref{eq:TLSadditiveForm2} and the \QCQP~\eqref{eq:TLSBinaryClone} are the same.

{\bf Equivalence of the constraints.}
  We are only left to prove that~\eqref{eq:TLSadditiveForm2} and~\eqref{eq:TLSBinaryClone} have the same feasible set, 
\ie, the following two sets of constraints are equivalent:
\bea \label{eq:equivalentConstraints}
\begin{cases}
\vq \in \usphere^3 \\
\vq_k \in \{\pm \vq\},\\
k=1,\dots,K
\end{cases} \Leftrightarrow
\begin{cases}
\vxx_q\tran \vxx_q = 1 \\
\vxx_{q_k}\vxx_{q_k}\tran = \vxx_q \vxx_q\tran,\\
k=1,\dots,K
\end{cases}
\eea
We first prove the ($\Rightarrow$) direction. Since $\vq \in \usphere^3$, it is obvious that $\vxx_q\tran \vxx_q = \vq\tran \vq = 1$. In addition, since $\vq_k \in \{+\vq,-\vq\}$, it follows that $\vxx_{q_k} \vxx_{q_k}\tran = \vq_k \vq_k\tran = \vq\vq\tran = \vxx_q \vxx_q\tran$. Then we proof the reverse direction ($\Leftarrow$). Since $\vxx_q\tran \vxx_q = \vq\tran \vq$, so $\vxx_q\tran \vxx_q = 1$ implies $\vq\tran \vq =1$ and therefore $\vq \in \usphere^3$. On the other hand, $\vxx_{q_k} \vxx_{q_k}\tran = \vxx_q \vxx_q\tran$ means $\vq_k \vq_k\tran = \vq\vq\tran$. If we write $\vq_k = [q_{k1},q_{k2},q_{k3},q_{k4}]\tran$ and $\vq = [q_1,q_2,q_3,q_4]$, then the following matrix equality holds:
\smaller
\bea
\hspace{-6mm}\bmat{cccc}
q_{k1}^2 & q_{k1}q_{k2} & q_{k1}q_{k3} & q_{k1}q_{k4} \\
\star & q_{k2}^2 & q_{k2}q_{k3} & q_{k2}q_{k4} \\
\star & \star & q_{k3}^2 & q_{k3}q_{k4} \\
\star & \star & \star & q_{k4}^2
\emat = 
\bmat{cccc}
q_1^2 & q_1q_2 & q_1q_3 & q_1q_4 \\
\star & q_2^2 & q_2q_3 & q_2q_4 \\
\star & \star & q_3^2 & q_3q_4 \\
\star & \star & \star & q_4^2
\emat
\nonumber\hspace{-10mm}\\
\eea
\normalsize
First, from the diagonal equalities, we can get $q_{ki}=\theta_i q_i, \theta_i\in\{+1,-1\}, i=1,2,3,4$. Then we look at the off-diagonal equality: $q_{ki}q_{kj}=q_i q_j, i\neq j$, since $q_{ki}=\theta_i q_i$ and $q_{kj} = \theta_j q_j$, we have $q_{ki}q_{kj} = \theta_i \theta_j q_j q_k$, from which we obtain $\theta_i \theta_j =1, \forall i \neq j$. This implies that all the binary values $\{\theta_i\}_{i=1}^4$ have the same sign, and therefore they are equal to each other. As a result, $\vq_k = \theta_k \vq = \{+\vq, -\vq\}$, showing the two sets of constraints in eq.~\eqref{eq:equivalentConstraints} are indeed equivalent. Therefore, the \QCQP in eq.~\eqref{eq:TLSBinaryClone} is equivalent to the optimization in~\eqref{eq:optQCQPbeforeLifting}, and the original optimization in~\eqref{eq:TLSadditiveForm2} that involves $K+1$ quaternions, concluding the proof. 

\subsection{Proof of Proposition~\ref{prop:qcqpZ}: {Matrix Formulation of Binary Cloning}}
\label{sec:proof:prop:matrixBinaryClone}

Here we show that the non-convex \QCQP written in terms of the vector $\vxx$ in Proposition~\ref{prop:qcqp} (and eq.~\eqref{eq:TLSBinaryClone}) is equivalent to the non-convex problem written using the matrix $\MZ$ in Proposition~\ref{prop:qcqpZ} (and eq.~\eqref{eq:qcqpZ}). We do so by showing that the objective function and the constraints in~\eqref{eq:qcqpZ} are a re-parametrization of the ones in~\eqref{eq:TLSBinaryClone}.

\myParagraph{Equivalence of the objective function} 
Since $\MZ = \vxx \vxx\tran$ and using the cyclic property of the trace, we rewrite the objective in ~\eqref{eq:TLSBinaryClone} as:
\bea
\vxx\tran \MQ \vxx   = \trace{\MQ \vxx \vxx\tran} = \trace{\MQ \MZ}
\eea
showing the equivalence of the objectives in~\eqref{eq:TLSBinaryClone} and~\eqref{eq:qcqpZ}.

\myParagraph{Equivalence of the constraints} It is trivial to see that $\vxx_q\tran \vxx_q = \trace{\vxx_q \vxx_q\tran} = 1$ is equivalent to $\trace{[\MZ]_{qq}}=1$ by using the cyclic property of the trace operator and inspecting the structure of $\MZ$. In addition, $\vxx_{q_k}\vxx_{q_k}\tran = \vxx_q \vxx_q\tran$ also directly maps to $[\MZ]_{q_k q_k} = [\MZ]_{qq}$ for all $i=1,\dots,K$. Lastly, requiring $\MZ\succeq 0$ and $\rank{\MZ} = 1$ is equivalent to restricting $\MZ$ to the form $\MZ = \vxx\vxx\tran$ for some vector $\vxx \in \Real{4(K+1)}$. Therefore, the constraint sets of eq.~\eqref{eq:TLSBinaryClone} and~\eqref{eq:qcqpZ} are also equivalent, concluding the proof.  

\subsection{Proof of Theorem~\ref{thm:quasarOptimalityGuarantee}: SDP Relaxation with Redundant Constraints and Global Optimality Guarantee}
\label{sec:app-proof-globalOptimalityGuarantee}

We first restate the QCQP~\eqref{eq:TLSBinaryClone} and explicitly add the redundant constraints: 

\begin{lemma}[QCQP with Redundant Constraints]\label{lemma:QCQPRedundantConstraints}
The QCQP~\eqref{eq:TLSBinaryClone} is equivalent to the following QCQP with redundant constraints:
\begin{align} \label{eq:QCQPRedundant}
\quad \min_{\vxx \in \Real{4(K+1)}}  & \;\; \vxx\tran \MQ \vxx \tag{P}\\
\subject & \vxx_q\tran \vxx_q = 1 , \nonumber \\
& \vxx_{q_k}\vxx_{q_k}\tran = \vxx_q \vxx_q\tran,\forall k = 1,\dots,K ,  \nonumber \\
& \vxx_{q_i}\vxx_{q_j}\tran = (\vxx_{q_i}\vxx_{q_j}\tran)\tran,\forall 0 \leq i < j \leq K , \nonumber
\end{align}
where $\vq_0 \doteq \vq$ for notation simplicity.
\end{lemma}
The redundancy of the last set of constraints can be easily understood from:
\bea
\vxx_{q_i} \vxx_{q_j}\tran = (\theta_i\vq)(\theta_j\vq\tran) = (\theta_j\vq)(\theta_i\vq\tran) = (\vxx_{q_i} \vxx_{q_j}\tran)\tran.
\eea
Since the original QCQP~\eqref{eq:TLSBinaryClone} is equivalent to the rank-constrained SDP~\eqref{eq:qcqpZ}, problem~\eqref{eq:QCQPRedundant} is also equivalent to~\eqref{eq:qcqpZ}.

Now let us prove Theorem~\ref{thm:quasarOptimalityGuarantee}. 
While the SDP relaxation of~\eqref{eq:QCQPRedundant} can be simply obtained by adding the redundant constraints and dropping the rank constraint in~\eqref{eq:qcqpZ}. 
Here we take a longer path, using Lagrangian duality, which is useful towards getting the guarantees stated in the theorem.

\begin{lemma}[Dual of QCQP with Redundant Constraints]\label{lemma:dualofQCQPwithRedundantConstraints}
The Lagrangian dual problem of the QCQP with redundant constraints~\eqref{eq:QCQPRedundant} is the following convex semidefinite program:
\begin{align}\label{eq:dualQCQPRedundant}
 \quad \max_{\mu, \MLambda, \MW} &  \;\; \mu \tag{D}\\
\subject & (\MQ - \mu \MJ + \MLambda + \MW) \succeq 0, \nonumber
\end{align}
where $\MJ \in \calS^{4(K+1)}$ is an all-zero matrix except the top-left diagonal $4\times4$ block $[\MJ]_{00} = \eye_4$; $\MLambda \in \calS^{4(K+1)}$ is a block-diagonal matrix $\MLambda = \blkdiag(\MLambda_{00},\MLambda_{11},\dots,\MLambda_{KK})$, where each $\MLambda_{kk}$ is a $4\times4$ symmetric matrix and the sum $\sum_{k=0}^{K} \MLambda_{kk} = \MZero$; and $\MW \in \calS^{4(K+1)}$ satisfies $[\MW]_{ij} = \MZero$ for $i=j$ and $[\MW]_{ij} \in \ssym^4$ (\ie $[\MW]_{ij} $ is a skew symmetric matrix) for $i \neq j$.
\end{lemma}
The dual problem~\eqref{eq:dualQCQPRedundant} is derived from the redundant QCQP~\eqref{eq:QCQPRedundant} by following standard Lagrangian duality~\cite{Boyd04book}. For any $\vxx$ that is feasible for~\eqref{eq:QCQPRedundant}, one can verify that: 
\bea \label{eq:primalDualProd}
\vxx\tran \MJ \vxx = 1, \quad \vxx\tran \MLambda \vxx = 0, \quad \vxx\tran \MW \vxx = 0.
\eea

\begin{lemma}[Dual of the Dual]\label{lemma:dualofDual}
The dual SDP of the dual problem~\eqref{eq:dualQCQPRedundant} is the following SDP:
\begin{align} \label{eq:dualofDual}
\quad \min_{\MZ \succeq 0} & \;\; \trace{\MQ \MZ} \tag{DD} \\
\subject 
& \trace{[\MZ]_{00}} = 1 \nonumber \\
& [\MZ]_{kk} = [\MZ]_{00}, \forall k=1,\dots,K \nonumber \\
& [\MZ]_{ij} = [\MZ]\tran_{ij}, \forall 0\leq i < j \leq K \nonumber,
\end{align}
which is the same as the SDP relaxation with redundant constraints~\eqref{eq:relaxationRedundant}.
\end{lemma}
Now we are ready to prove Theorem~\ref{thm:quasarOptimalityGuarantee}. First of all, we have the following weak duality:
\bea \label{eq:weakDuality}
f^\star_{D} = f^{\star}_{DD} \leq f^\star_P,
\eea
where $f^\star_D = f^\star_{DD}$ (strong duality of the SDP pair) is due to the fact that \eqref{eq:dualofDual} has a strictly feasible solution, by having $[\MZ]_{00} = [\MZ]_{11} = \dots = [\MZ]_{KK} = \frac{1}{4}\eye_4$ and $[\MZ]_{ij} = \MZero$ for $i\neq j$ (in this case $\MZ = \frac{1}{4}\eye_{4(K+1)} \succ 0$), and $f^\star_{DD} \leq f^\star_P$ is due to the fact that \eqref{eq:dualofDual} is a convex relaxation for \eqref{eq:QCQPRedundant}. 

If the optimal solution of \eqref{eq:dualofDual} satisfies $\rank{\MZ^\star} = 1$, then this means that $\MZ^\star$ is also feasible for the rank-constrained SDP~\eqref{eq:qcqpZ}. Therefore, the minimum $f^\star_{DD}$ can indeed be attained by a point in the feasible set of problem~\eqref{eq:qcqpZ}. Because the rank-constrained problem~\eqref{eq:qcqpZ} is equivalent to \eqref{eq:QCQPRedundant}, $f^{\star}_{DD} = f^\star_P$ must hold and the rank-1 decomposition of $\MZ^\star$, $\vxx^\star$ must be a global minimizer of problem \eqref{eq:QCQPRedundant}.

To prove the uniqueness of $\pm\vxx^\star$ as global minimizers of \eqref{eq:QCQPRedundant}, denote $\mu^\star = (\vxx^\star)\tran \MQ (\vxx^\star) = \trace{\MQ\MZ^\star}$ as the global minimum of \eqref{eq:QCQPRedundant},~\eqref{eq:dualQCQPRedundant} and~\eqref{eq:dualofDual}, and $\MLambda^\star$, $\MW^\star$ as the corresponding dual variables. Let $\MM^\star = \MQ - \mu^\star\MJ + \MLambda^\star + \MW^\star$, we have $\rank{\MM^\star} = 4(K+1) - 1$ and its kernel is $\ker(\MM^\star) = \{\vxx: \vxx = a \vxx^\star\}$ from complementary slackness of the pair of SDP duals. Denote the feasible set of \eqref{eq:QCQPRedundant} as $\Omega\eqref{eq:QCQPRedundant}$, then we have $\ker(\MM^\star) \cap \Omega\eqref{eq:QCQPRedundant} = {\pm \vxx^\star}$. Therefore, for any $\vxx \in \Omega\eqref{eq:QCQPRedundant}/\{\pm \vxx^\star\}$, we have:
\bea
\vxx\tran \MM^\star \vxx > 0  \\
\Leftrightarrow \vxx\tran \MQ \vxx - \mu^\star + \vxx\tran\MLambda^\star\vxx + \vxx\tran\MW^\star\vxx > 0  \\
\Leftrightarrow \vxx\tran \MQ \vxx > \mu^\star,  
\eea
where $\vxx\tran\MLambda^\star\vxx = \vxx\tran\MW^\star\vxx = 0$ holds due to~\eqref{eq:primalDualProd}.


\subsection{Proof of Theorem~\ref{thm:optimalityCertification}: Optimality Certification}
\label{sec:app-optimalityCertification}

We prove Theorem~\ref{thm:optimalityCertification} in three steps. 
First, we show that that proving global optimality of an estimate for the rotation estimation problem~\eqref{eq:TLSquat} relates to the existence of a matrix at the intersection of the positive semidefinite (PSD) cone and an affine subspace (Appendix~\ref{sec:app-certificate1}).
 Second, we show that we
 can compute a sub-optimality gap (this is the bound $\eta$ in the theorem) for an estimate using any matrix in the affine subspace (Appendix~\ref{sec:app-certificate2}). 
 Finally, we show that under the conditions of Theorem~\ref{thm:optimalityCertification} 
 such bound converges to zero when the provided estimate is globally optimal (Appendix~\ref{sec:app-certificate3}).


\subsubsection{Matrix Certificate for Global Optimality}
\label{sec:app-certificate1}

Here we derive necessary and sufficient conditions for global optimality of the \TLS rotation estimation problem~\eqref{eq:TLSquat}.

\begin{theorem}[Sufficient Condition for Global Optimality]\label{thm:sufficientConditionGlobalOptimality}
Given a feasible solution $(\hatvq,\hattheta_1,\dots,\hattheta_K)$ of the \TLS rotation estimation problem~\eqref{eq:TLSquat}, denote $\hatvxx = [\hatvq\tran,\hattheta_1\hatvq\tran,\dots,\hattheta_K\hatvq\tran]\tran$, and $\hatmu = \hatvxx\tran\MQ\hatvxx$, as the corresponding solution and cost of the QCQP~\eqref{eq:QCQPRedundant}. Then $\hatmu$ (resp. $(\hatvq,\hattheta_1,\dots,\hattheta_K)$) is the global minimum (resp. minimizer) of problem~\eqref{eq:TLSquat} if there exists a matrix $\MM$ that satisfies the following:
\bea
& \MM \in \affineSubNonRot, \label{eq:Msubspace}\\
& \MM \succeq 0, \label{eq:MPSD}
\eea
where $\affineSubNonRot$ is the following affine subspace:
\bea
\label{eq:MsubspaceNonRot}
\affineSubNonRot \doteq \left\{ \MM \in \calS^{4(K+1)} \Big| \substack{ \displaystyle \MM \hatvxx = \MZero, 
\\ \displaystyle \MM - \MQ + \hatmu \MJ \in \calH. } \right\}
\eea
and $\calH$ is defined as: 
\bea
\calH \doteq \left\{ \MDelta \in \calS^{4(N+1)}  \Bigg| \substack{ \displaystyle  \sumAllIM [\MDelta]_{kk} = \MZero, \\ \displaystyle [\MDelta]_{ij} \in \ssym^4, \forall 0 \leq i < j \leq K. } \right\},
\eea
The matrix $\MJ \in \calS^{4(K+1)}$ is an all-zero matrix except the top-left $4 \times 4$ block $[\MJ]_{00} = \eye_4$ (the same defined in~Appendix~\ref{sec:app-proof-globalOptimalityGuarantee}).
\end{theorem}
\begin{proof}
Because $\MM - \MQ + \hatmu\MJ \in \calH$, and from the definition of $\calH$, we can write $\MM$ as:
\bea
\label{eq:MMasDualVar}
\MM = \MQ - \hatmu\MJ + \MLambda + \MW,
\eea
where $\MLambda , \MW \in \calS^{4(K+1)}$ are the dual variables defined in Lemma~\ref{lemma:dualofQCQPwithRedundantConstraints}. From $\MM \succeq 0$, we have the following inequality for any $\vxx$ that is primal feasible for~\eqref{eq:QCQPRedundant}:
\bea
\vxx\tran \MM \vxx \geq 0  \Leftrightarrow \vxx\tran (\MQ - \hatmu \MJ + \MLambda + \MW)\vxx \geq 0 \\
\Leftrightarrow \vxx\tran \MQ \vxx \geq \hatmu + \vxx\tran \MLambda \vxx + \vxx\tran \MW \vxx = \hatmu,
\eea
which states that $\hatmu$ (resp. $\hatvxx$) is the global minimum (resp. minimizer) of~\eqref{eq:QCQPRedundant}.

In addition, denote $\hatMZ = \hatvxx\hatvxx\tran$, then $\hatMZ$ and $(\hatmu, \MLambda, \MW)$ are a pair of primal-dual optimal solutions for \eqref{eq:dualofDual} and \eqref{eq:dualQCQPRedundant}, which leads to the following SDP complementary slackness condition:
\bea
(\MQ - \hatmu\MJ + \MLambda + \MW) \hatMZ = \MZero \\
\Leftrightarrow (\MQ - \hatmu\MJ + \MLambda + \MW) \hatvxx = \MZero.
\eea
Therefore, $\MM\hatvxx = \MZero$ must hold in eq.~\eqref{eq:Msubspace}.
\end{proof}

\begin{corollary}[Sufficient and Necessary Condition for Global Optimality]\label{thm:sufficientNecessaryCondition}
If strong duality is achieved, \ie $f^\star_D = f^\star_{DD} = f^\star_P$ is attained, then the existence of a matrix $\MM$ that satisfies eq.~\eqref{eq:Msubspace} and~\eqref{eq:MPSD} is also a necessary condition for $\hatmu$ (resp. $(\hatvq,\hattheta_1,\dots,\hattheta_K)$) to be the global minimum (resp. minimizer) of problem~\eqref{eq:TLSquat}.
\end{corollary}
\begin{proof}
If strong duality is achieved, then there must exist dual variables $\MLambda$ and $\MW$ that satisfy $\MQ - \hatmu\MJ + \MLambda + \MW \succeq 0$. The rest follows from the proof of Theorem~\ref{thm:sufficientConditionGlobalOptimality}.
\end{proof}


\subsubsection{Sub-optimality Bound}
\label{sec:app-certificate2}


Theorem~\ref{thm:sufficientConditionGlobalOptimality} suggests that finding a matrix $\MM$ that lies at the intersection of the PSD cone $\psdSub$ and the affine subspace $\affineSubNonRot$ gives a \emph{certificate} for $\hatvxx$ to be globally optimal for \eqref{eq:QCQPRedundant}. Corollary~\ref{corollary:relativeSubopt} below states that even if $\hatvxx$ is not a globally optimal solution, finding any $\MM \in \affineSubNonRot$ (not necessarily PSD) still gives a valid sub-optimality bound using the minimum eigenvalue of $\MM$.\footnote{
Note that the expression of $\affineSubNonRot$ depends on the estimate $\hatvxx$ as
 per~\eqref{eq:MsubspaceNonRot}.} 

\begin{corollary}[Relative Sub-optimality Bound]\label{corollary:relativeSubopt}
Given a feasible solution $(\hatvq,\hattheta_1,\dots,\hattheta_K)$ of the \TLS rotation estimation problem~\eqref{eq:TLSquat}, denote $\hatvxx = [\hatvq\tran,\hattheta_1\hatvq\tran,\dots,\hattheta_K\hatvq\tran]\tran$, and $\hatmu = \hatvxx\tran\MQ\hatvxx$, as the corresponding solution and cost of the QCQP~\eqref{eq:QCQPRedundant}. Define the affine subspace  $\affineSubNonRot$ as in~\eqref{eq:MsubspaceNonRot}. 
For any $\MM \in \affineSubNonRot$, let $\lambda_1(\MM)$ be its minimum eigenvalue. Then the relative sub-optimality gap of $\hatvxx$ is bounded by:
\bea
\frac{\hatmu - \mu^\star}{\hatmu} \leq \eta \doteq \frac{|\lambda_1(\MM)|(K+1)}{\hatmu},
\eea
where $\mu^\star$ is the true global minimum of~\eqref{eq:QCQPRedundant}.
\end{corollary}
\begin{proof}
For any $\vxx$ that is primal feasible for \eqref{eq:QCQPRedundant}, we have $\| \vxx \|^2 = \sum_{k=1}^K \| \vq_k \|^2 = K+1$. Therefore, for any $\vxx$ in the feasible set of \eqref{eq:QCQPRedundant} and $\MM \in \affineSubNonRot$, the following holds:
\bea
\vxx\tran \MM \vxx \geq \lambda_1(\MM) \| \vxx \|^2 = \lambda_1(\MM) (K+1) \\
\overset{(i)}{\Leftrightarrow} \vxx\tran (\MQ - \hatmu\MJ + \MLambda + \MW)\vxx \geq \lambda_1(\MM) (K+1) \\
\overset{(ii)}{\Leftrightarrow} \vxx\tran \MQ \vxx \geq \hatmu + \lambda_1(\MM) (K+1). \label{eq:anyxbound}
\eea
where (i) follows from the fact that any $\MM \in \affineSubNonRot$ can be written as $\MQ - \hatmu\MJ + \MLambda + \MW$ (see eq.~\eqref{eq:MMasDualVar}), and (ii) follows from~\eqref{eq:primalDualProd}. 
Since $\mu^\star$ is also achieved in the feasible set, eq.~\eqref{eq:anyxbound} also holds for $\mu^\star$:
\bea
\mu^\star \geq \hatmu + \lambda_1(\MM) (K+1) \\
\Leftrightarrow \frac{\hatmu - \mu^\star}{\hatmu} \leq \frac{|\lambda_1(\MM)|(K+1)}{\hatmu}, \label{eq:relSuboptBoundApp}
\eea
{where in eq.~\eqref{eq:relSuboptBoundApp} we have used $\lambda_1(\MM) \leq 0$ since $\MM$ has a zero eigenvalue with eigenvector $\hatvxx$ ($\MM \hatvxx = \MZero$ in eq.~\eqref{eq:Msubspace}).}
\end{proof}


\subsubsection{\revone{Douglas-Rachford Splitting} and Algorithm~\ref{alg:optCertification}}
\label{sec:app-certificate3}

Theorems~\ref{thm:sufficientConditionGlobalOptimality}-\ref{thm:sufficientNecessaryCondition} relate global optimality to the existence of a matrix $\MM \in \psdSub \cap \affineSubNonRot$, while Theorem~\ref{corollary:relativeSubopt} obtains a suboptimality bound from any matrix $\MM \in \affineSubNonRot$. 
The next Theorem states that finding a matrix $\MM \in \psdSub \cap \affineSubNonRot$ is equivalent to finding a matrix $\barMM \in \psdSub \cap \affineSub$, where $\affineSub$ is a ``rotated'' version of $\affineSubNonRot$ (this will enable faster projections to this subspace).

\begin{theorem}[Rotated Affine Subspace]\label{thm:rotateAffineSubspace}
Given a feasible solution $(\hatvq,\hattheta_1,\dots,\hattheta_K)$ of the \TLS rotation estimation problem~\eqref{eq:TLSquat},
define the block-diagonal matrix $\hatOmega_q = \eye_{K+1} \kron \MOmega_1(\hatvq)$.
 Then a matrix $\MM \in \affineSubNonRot \cap \psdSub$ if and only if $\barMM \doteq \hatOmega_q\tran \MM \hatOmega_q \in \affineSub \cap \psdSub$ with $\affineSub$ defined as:
\bea \label{eq:defbarcalM}
\affineSub \doteq \left\{ \barMM \in \calS^{4(K+1)} \Big| \substack{ \displaystyle \barMM \barvxx = \MZero, \\ \displaystyle \barMM - \barMQ + \hatmu \MJ \in \calH. } \right\},
\eea
where $\barvxx = [1,\hattheta_1,\dots,\hattheta_K]\tran \kron \ve$ (with $\ve = [0,0,0,1]\tran$) and $\barMQ = \hatOmega_q\tran \MQ \hatOmega_q$.  Moreover, $\barMM$ and $\MM$ produce the same relative sub-optimality bound in Theorem~\ref{corollary:relativeSubopt}.
\end{theorem}

\begin{proof}
Using the expression of $\MOmega_1(\vq)$ in~\eqref{eq:Omega_1}, one can verify that $\MOmega_1(\vq)\tran \vq = \ve$, $\hatOmega_q\tran\hatvxx = \barvxx$ and $\hatOmega_q\tran \hatOmega_q = \hatOmega_q \hatOmega_q\tran = \eye_{4(K+1)}$. Therefore, we have the following equivalence:
\bea
\hspace{-5mm}\barMM \barvxx = \MZero \Leftrightarrow \hatOmega_q\tran\MM\hatOmega_q \hatOmega_q\tran\hatvxx = \MZero \nonumber \\
\Leftrightarrow \hatOmega_q \MM \hatvxx = \MZero \Leftrightarrow \MM \hatvxx = \MZero.
\eea 
In addition, it is easy to verify that $\MDelta \in \calH$ if and only if $\barMDelta \doteq \hatOmega_q\tran \MDelta \hatOmega_q \in \calH$, and $\hatOmega_q\tran\MJ\hatOmega_q = \MJ$. Therefore,
\bea
\MM - \MQ - \hatmu\MJ \in \calH \Leftrightarrow \barMM - \barMQ - \hatmu\MJ \in \calH,
\eea
and we conclude that $\MM \in \affineSubNonRot \Leftrightarrow \barMM \in \affineSub$. Finally, $\barMM = \hatOmega_q\tran \MM \hatOmega_q = \hatOmega_q\inv \MM \hatOmega_q$ is a similarity transform on $\MM$ and shares the same eigenvalues with $\MM$, which trivially leads to $\MM \in \psdSub \Leftrightarrow \barMM \in \psdSub$. 
In addition, since $\MM$ and $\barMM$ share the same eigenvalues, we have $\lambda_1(\barMM) = \lambda_1(\MM)$, which means that $\barMM$ and $\MM$ produces the same relative sub-optimality bound in Corollary~\ref{corollary:relativeSubopt}.
\end{proof}

Therefore, to certify optimality of an estimate, we have to search for a matrix $\barMM  \in \affineSub \cap \psdSub$, \ie at the intersection between two convex sets. \revone{Although finding a point in the intersection of two convex sets can be expensive (\eg~finding $\barMM  \in \affineSub \cap \psdSub$ is equivalent to solving a feasibility SDP), the following lemma states that one can find a point in the intersection of two convex sets by cleverly leveraging the individual projections onto the two convex sets, a method that is commonly known as \emph{Douglas-Rachford Splitting} (\DRS)~\cite{Combettes11book-proximalSplitting,Yang20arXiv-certifiablePerception,Jegelka11cvpr}.}
\begin{lemma}[Douglas-Rachford Splitting]\label{lemma:douglasrachfordsplit}
\revone{Let $\calX_1$, $\calX_2$ be two closed convex sets in the Hilbert space $\calX$, and $\Pi_{\calX_1}$, $\Pi_{\calX_2}$ be their corresponding projection maps, then for any $\vxx_0 \in \calX$, the following iterates:}
\bea \label{eq:douglasrachfordsplitting}
& (i) \ \vxx_\tau^{\calX_1} = \Pi_{\calX_1} (\vxx_\tau), \\
& (ii) \ \vxx_\tau^{\calX_2} = \Pi_{\calX_2} (2\vxx_\tau^{\calX_1} - \vxx_\tau), \\
& (iii) \ \vxx_{\tau+1} = \vxx_{\tau} + \gamma (\vxx_\tau^{\calX_2} - \vxx_\tau^{\calX_1}),
\eea
\revone{generates a sequence $\{ \vxx_\tau \}_{\tau \geq 0}$ that converges to the intersection of $\calX_1$ and $\calX_2$ when $0<\gamma<2$, provided that the intersection is non-empty.}
\end{lemma}

Therefore, the \DRS (line~\ref{line:projectK}-line~\ref{line:gamma}) in Algorithm~\ref{alg:optCertification} guarantee to converge to a matrix $\barMM \in \affineSub \cap \psdSub$ when $\affineSub \cap \psdSub \neq \varnothing$, and the sub-optimality bound $\eta$ converges to zero. Moreover, Algorithm~\ref{alg:optCertification} produces a sub-optimality bound even if it fails to certify global optimality, thanks to Corollary~\ref{corollary:relativeSubopt}.


\subsection{Closed-form Projections}
\label{sec:app-closedFormProjections}

Although Lemma~\ref{lemma:douglasrachfordsplit} is theoretically sound, in practice computing the projection onto an arbitrary convex set could be computationally expensive\footnote{Even for an affine subspace, $\calA = \{\vxx \in \Real{n}: \MA \vxx = \vb\}$, the naive projection of any $\vxx$ onto $\calA$ is: $\Pi_{\calA}(\vxx) = \vxx - \MA\tran(\MA\MA\tran)\inv(\MA\vxx - \vb)$, which could be expensive for large $n$ due to the computation of the inverse $(\MA\MA\tran)\inv$.}. However, we will show that both projections, $\Pi_{\affineSub}$ and $\Pi_{\psdSub}$ in Algorithm~\ref{alg:optCertification} can be computed efficiently in closed form.

The projection onto the PSD cone is presented in the following Lemma due to Higham~\cite{Higham88LA-nearestSPD}.
\begin{lemma}[Projection onto $\calS^n_+$]\label{lemma:projectionPSDcone}
Given any matrix $\MM \in \calS^n$, let $\MM = \MU\diag{\lambda_1,\dots,\lambda_n}\MU\tran$ be its spectral decomposition, then the projection of $\MM$ onto the PSD cone $\calS^n_+$ is:
\bea
\Pi_{\calS^n_+}(\MM) = \MU \diag{\max(0,\lambda_1),\dots,\max(0,\lambda_n)}\MU\tran.
\eea
\end{lemma}

Now let us focus on the affine projection onto $\affineSub$.
The expression of the projection $\Pi_{\affineSub}$ is given in Proposition~\ref{prop:proj2barcalM} below. Before stating the result, let us introduce some notation. For two $4\times4$ matrices $\MX \in \calS^{4}$ and $\MY \in \ssym^4$, we partition $\MX$ and $\MY$ as:
\bea
\scriptsize \hspace{-2mm}
\MX = \bmat{cc}
\MX^m \in \calS^{3} & \MX^v \in \Real{3} \\
(\MX^v)\tran & \MX^s \in \Real{}
\emat,
\MY = \bmat{cc}
\MY^m \in \ssym^3 & \MY^v \in \Real{3} \\
-(\MY^v)\tran & \MY^s = 0
\emat,
\eea
where the superscripts $m,v,s$ denotes the top-left $3\times 3$ matrix part, the top-right $3 \times 1$ vector part and the bottom-right scalar part, respectively. For a matrix $\MM \in \calS^{4(K+1)}$, we partition it into $(K+1)^2$ blocks of size $4\times 4$, and we create a set of $L \doteq \frac{K(K+1)}{2}$ ordered indices (left to right, top to bottom) that enumerate the upper-triangular off-diagonal blocks:
\bea \label{eq:setUupperidx}
\hspace{-3mm}\myBlocks = \{(0,1),\dots,(0,K),(1,2),\dots,(1,K),\dots,(K-1,K)\},
\eea
where each $[\MM]_{\myBlocks(l)},l=1,\dots,L,$ takes a $4\times 4$ off-diagonal block from $\MM$. Then, we define a linear map $\calC: \calS^{4(K+1)} \rightarrow \Real{L \times 3}$ that assembles the top-right $3 \times 1$ vector part of each upper-triangular off-diagonal block into a matrix:
\bea \label{eq:defLinearMapC}
\calC(\MM) = \left[ [\MM]^v_{\myBlocks(1)},\dots,[\MM]^v_{\myBlocks(L)} \right]\tran.
\eea
Given a solution $\hatvxx = [\hatvq\tran,\hattheta_1\hatvq\tran,\dots,\hattheta_K\hatvq\tran]\tran$ to problem~\eqref{eq:QCQPRedundant}, denote $\hatMR$ to be the unique rotation matrix corresponding to $\hatvq$, and define the following residual vectors:
\bea
\hatxi_k = \hatMR\tran (\barvb_k - \hatMR \barva_k), k = 1,\dots,K,
\eea
for each pair of \TIMs $(\TIMa_k,\TIMb_k)$ (scale $\TIMa_k$ by: $\TIMa_k \leftarrow \hats\TIMa_k$).
\begin{proposition}[Projection onto $\affineSub$]\label{prop:proj2barcalM}
Given any matrix $\MM \in \calS^{4(K+1)}$, its projection onto $\affineSub$, denoted $\barMM \doteq \Pi_{\affineSub}(\MM)$, is: 
\bea
\barMM = \Pi_{\barcalH}(\MH) + \barMQ - \hatmu\MJ,
\eea
where $\MH \doteq \MM - \barMQ + \hatmu\MJ$ and $\barMH \doteq \Pi_{\barcalH}(\MH)$ can be computed as follows:
\begin{enumerate}
	\item Project the matrix part of each diagonal block:
	\bea \label{eq:proj2barcalH_diag_matrix}
	[\barMH]_{kk}^m = [\MH]_{kk}^m -  \frac{\sum_{k=0}^K [\MH]_{kk}^m }{K+1}, \forall k = 0,\dots,K.
	\eea

	\item Project the scalar part of each diagonal block:
	\bea \label{eq:proj2barcalH_diag_scalar}
	\hspace{-6mm}[\barMH]_{kk}^s = \begin{cases}
    -(\quarter \hattheta_k + \half) \| \hatxi_k \|^2 + (\quarter \hattheta_k - \half)\barcsq & k \neq 0 \\
    \sum_{k=1}^K (\quarter \hattheta_k + \half ) \| \hatxi_k \|^2 - (\quarter\hattheta_k - \half)\barcsq & k=0.
	\end{cases}
	\eea

	\item Project the matrix part of each off-diagonal block:
	\bea \label{eq:proj2barcalH_offdiag_matrix}
	[\barMH]_{\myBlocks(l)}^m = \frac{[\MH]_{\myBlocks(l)}^m - ([\MH]_{\myBlocks(l)}^m)\tran}{2}, \forall l=1,\dots,L.
	\eea

	\item Project the scalar part of each off-diagonal block:
	\bea \label{eq:proj2barcalH_offdiagscalar}
	[\barMH]_{\myBlocks(l)}^{s} = 0, \forall l=1,\dots,L.
	\eea

	\item \label{projstep:proj2barcalH_offdiagvector}Project the vector part of each off-diagonal block:
	\bea \label{eq:proj2barcalH_offdiagvector}
	\calC(\barMH) = \MP \cdot \calF(\MH),
	\eea
	where $\MP \in \calS^L$ and $\calF(\MH) \in \Real{L\times 3}$ are defined in {Appendix~\ref{sec:app-proof-prop-projectionAffineSpace}}.

	\item Project the vector part of each diagonal block using the results of step~\ref{projstep:proj2barcalH_offdiagvector}:
	\bea \label{eq:proj2barcalH_diagvector}
	& [\barMH]_{kk}^v = -\sum_{i=0,i\neq k}^K \hattheta_k\hattheta_i [\barMH]_{ki}^v + \vphi_k, \\
	& \hspace{-6mm} \vphi_k = \begin{cases}
	-\left( \half \hattheta_k + 1\right)\vSkew{\hatxi_k}\TIMa_k & k=1,\dots,K\\
	\sum_{k=1}^K \left( \half \hattheta_k + 1\right)\vSkew{\hatxi_k}\TIMa_k  & k = 0
	\end{cases}.
	\eea
\end{enumerate}
\end{proposition}
A complete proof of Proposition~\ref{prop:proj2barcalM} is algebraically involved and is given in {Appendix~\ref{sec:app-proof-prop-projectionAffineSpace}}.


\subsection{Proof of Theorem~\ref{thm:certifiableRegNoiseless1}: Estimation Contract with Noiseless Inliers and Random Outliers}
\label{sec:proof:certifiableRegNoiseless1}

We prove Theorem~\ref{thm:certifiableRegNoiseless1} in two steps. 
First we show that under assumptions (iii)-(iv) of the theorem, the inliers found by \name match the maximum consensus solution. Then, we prove that under assumptions (i)-(ii) the maximum consensus solution is unique and recovers the ground truth, 
from which the claims of the theorem follow.

\subsubsection{\TLS $\equiv$ \maxCon}
Let the number of inliers in the maximum consensus be $\nrIn^{\maxCon}$, since the inliers are noise-free, the sum of the squared residual errors in the maximum consensus set, $r^{\maxCon}_{\mathrm{in}}$, must be zero. By Lemma~\ref{thm:TLSandMC}, if \TLS selects a different consensus set than the maximum consensus set, then the consensus set that \TLS selects must have size $\nrIn^{\TLS}$ greater than $\nrIn^{\maxCon} - r^{\maxCon}_{\mathrm{in}}/\barcsq$,~\ie~$\nrIn^{\TLS}> \nrIn^{\maxCon} - r^{\maxCon}_{\mathrm{in}}/\barcsq = \nrIn^{\maxCon}$, contradicting the fact that $\nrIn^{\maxCon}$ is the size of the maximum consensus set. Therefore, \TLS $\equiv$ \maxCon for each subproblem.

\subsubsection{Exact recovery of maximum consensus} We now prove that maximum consensus recovers the ground truth under the assumptions of the theorem. Let us start with the scale subproblem.  
The inliers are noiseless and distinct (assumption (i) in the theorem), so each pair of inliers will produce a \TRIM $s_k = \sgt$. 
Since we have at least 3 inliers, there are at least 3 such \TRIMs. Since the outliers are in \emph{generic position} 
(assumption (ii) in the theorem), 
the event that more than 3 \TRIMs formed by the \emph{outliers} will have the same scale estimate $\bar{s}$ that is different from $\sgt$ happens with probability zero. 
Therefore, the maximum consensus solution coincides with the \gtadj scale. 
The same logic can be repeated for the rotation and translation subproblems, leading to identical conclusions 
 under the assumption that the rotation subproblem finds a certifiably optimal \TLS estimate (condition (iv)).  

\subsubsection{Exact recovery of \name} Since under the assumptions of the theorem, \name matches the maximum consensus solution in each subproblem, and maximum consensus recovers the ground truth in each subproblem with probability 1, \name almost surely recovers the \gtadj transformation, \ie $\hats=\sgt, \hatMR=\MRgt, \hatvt=\vtgt$, proving the theorem.

\subsection{Proof of Theorem~\ref{thm:certifiableRegNoiseless2}: Estimation Contract with Noiseless Inliers and Adversarial Outliers}
\label{sec:proof:certifiableRegNoiseless2}

As for Theorem~\ref{thm:certifiableRegNoiseless1}, 
we  prove Theorem~\ref{thm:certifiableRegNoiseless2} in two steps. 

\subsubsection{\TLS $\equiv$ \maxCon}
The proof of this part is identical to the corresponding proof in Theorem~\ref{thm:certifiableRegNoiseless1}.
 
\subsubsection{Exact recovery of maximum consensus}
\label{subsec:adversaryMaxCon}
Different than the case with random outliers, we need condition (i), a stronger assumption on the number of inliers to guarantee recovery of the ground truth. Let the ground-truth consensus set be $\calC\gt$ whose size is at least $\nrIn$ (the outliers can accidentally become inliers) in condition (i). Suppose an adversary wants to trick \maxCon to return a solution $(\bar{s},\bar{\MR},\bar{\vt})$ that is different from the ground truth, s/he must have a consensus set, $\bar{\calC}$, of size $\bar{N}_{\mathrm{in}}$ that is no smaller than $\nrIn$, in which the measurements are consistent \wrt $(\bar{s},\bar{\MR},\bar{\vt})$. Now because the total number of measurement is $\nrIn + \nrOut$, and both $\calC\gt$ and $\bar{\calC}$ have size at least $\nrIn$, this means $\calC\gt$ and $\bar{\calC}$ must share at least $2\nrIn - (\nrIn + \nrOut) = \nrIn - \nrOut$ common correspondences. From condition (i), we know that $\nrIn - \nrOut \geq 3$. However, if $\calC\gt$ and $\bar{\calC}$ share at least 3 common noise-free correspondences, then $(\bar{s},\bar{\MR},\bar{\vt}) \equiv (\sgt,\MRgt,\vtgt)$ must be true. Therefore, we conclude that there is no chance for the adversary to trick \maxCon, and \maxCon must recover the \gtadj transformation when condition (i) holds.

\subsubsection{Exact recovery of \name}
Condition (iv) guarantees that  
\name returned a consensus set that satisfies condition (i).
However, according to Section~\ref{subsec:adversaryMaxCon} only the ground truth can produce an estimate that satisfies condition (i), hence \name's solution must match the ground truth, proving the theorem. 
Note that recovering the ground truth does not guarantee that all outliers are rejected, since in the adversarial regime outliers can 
be indistinguishable from inliers (\ie an outlier $i$ can satisfy $\vb_i = \sgt \MRgt \va_i + \vtgt$), a fact that, however, does not have negative repercussions on the accuracy of \name.


\subsection{Proof of Theorem~\ref{thm:certifiableRegNoisy}: 
Estimation Contract with Noisy Inliers and
Adversarial Outliers}
\label{sec:proof:certifiableRegNoisy}

We prove Theorem~\ref{thm:certifiableRegNoisy} in two steps. 
First we show that under assumptions (ii) and (iii) of the theorem, the inliers produced by \name match the maximum consensus set and contain the set of true inliers. Then, we prove that under assumption (i) and (iv) the maximum consensus solution is also close to the  ground truth, in the sense of inequalities~\eqref{eq:bounds_scale}-\eqref{eq:bounds_tran}. 

\subsubsection{\TLS $\equiv$ \maxCon and inliers are preserved}
We prove that under assumptions (ii) and (iii), the inliers produced by \name match the maximum consensus set and contain the set of true inliers. 
In each subproblem, Lemma~\ref{thm:TLSandMC} and assumption (iii) guarantee that \name computes a maximum consensus set. Assumption (iii) implies the maximum consensus set contains the inliers (and hence \name preserves the inliers).

\subsubsection{Noisy recovery of maximum consensus} We now prove that maximum consensus produces a solution ``close'' to the ground truth under the assumptions of the theorem. Let us start with the scale subproblem. 
Assumption (ii) in the theorem ensures that the maximum consensus set contains all inliers in each subproblem.
Let us now derive the bounds on the resulting estimate, starting from the scale estimation.

\myparagraph{Scale error bound}
Since the scale estimation selects all the inliers, plus potentially some outliers, it holds for some of the 
selected measurements $i,j$:
\bea
|\sgt - \hats | = |\sgt - \hats + \eps^s_{k} - \eps^s_{k} | \\
 \overset{(a)}{\leq}  |\sgt + \eps^s_{k} - \hats| + |\eps^s_{k} | \\
 \overset{(b)}{\leq}  | s_{k} - \hats |  + \alpha_{ij}   \overset{(c)}{\leq} 2 \alpha_{ij} 
\eea
where in (a) we used the triangle inequality, in (b) we noticed that for inliers $i,j$, it holds 
$|\eps^s_{k}| \leq \alpha_{ij}$ and $\sgt + \eps^s_{k} = s_k$, and in (c) we noticed that for $s_k$ to be 
considered an inlier by \TLS, $| s_{k} - \hats | \leq \alpha_{ij}$.
Since the inequality has to hold for all the true inliers, but
we do not know which ones of the measurements selected by \TLS are true inliers, we substitute the bound with:
\bea
|\sgt - \hats |  \leq 2 \max_{ij} \alpha_{ij}
\eea
which proves the first bound in~\eqref{eq:bounds_scale}. 

\myparagraph{Rotation error bound}
Let us now move to rotation estimation. 
Since the rotation subproblem also selects all the inliers, plus potentially some outliers, it holds for some of the selected measurements $i,j$:
\bea
\| \sgt \MRgt \TIMa_{ij} - \hats \hatMR \TIMa_{ij} \|  = 
\| \sgt \MRgt \TIMa_{ij} - \hats \hatMR \TIMa_{ij} + \veps_{ij} -\veps_{ij}  \| \\
 \overset{(a)}{=}
\| \TIMb_{ij}- \hats \hatMR \TIMa_{ij}  -\veps_{ij}  \| \\
 \overset{(b)}{\leq}
\| \TIMb_{ij} - \hats \hatMR \TIMa_{ij} \|  + \| \veps_{ij}  \|
 \overset{(c)}{\leq} 2 \TIMNoiseBound_{ij}
 \label{eq:appR0}
\eea
where in (a) we noticed that for inliers $i,j$, $\TIMb_{ij} = \sgt \MRgt \TIMa_{ij} + \veps_{ij}$, 
in (b) we used the triangle inequality, and in (c) we noticed 
that it holds 
$\|\veps_{ij}\| \leq \TIMNoiseBound_{ij}$ and that for $(i,j)$ to be 
considered an inlier by \TLS, $\| \TIMb_{ij}- \hats \hatMR \TIMa_{ij} \| \leq \TIMNoiseBound_{ij}$. 

Squaring the inequality and defining $\vu_{ij} \doteq \frac{\TIMa_{ij}}{\|\TIMa_{ij}\|}$: 
\bea
\| \sgt \MRgt \TIMa_{ij} - \hats \hatMR \TIMa_{ij} \|^2  \leq 4 \TIMNoiseBound_{ij}^2 \iff \\
\| \sgt \MRgt \vu_{ij} - \hats \hatMR \vu_{ij} \|^2  \leq 4 \frac{ \TIMNoiseBound_{ij}^2 } { \|\TIMa_{ij}\|^2 } = 4 \alpha_{ij}^2
\label{eq:appR1}
\eea

\noindent
Since we have at least 4 inliers $i,j,h,k$, whose invariant measurements include $\TIMa_{ij}, \TIMa_{ih}, \TIMa_{ik}$, 
and  each invariant measurement satisfies~\eqref{eq:appR1}, 
we can sum (member-wise) the 3 corresponding inequalities and obtain:
 \bea
 \| \sgt \MRgt \vu_{ij} - \hats \hatMR \vu_{ij} \|^2 + 
 \| \sgt \MRgt \vu_{ih} - \hats \hatMR \vu_{ih} \|^2 +  \\
  \| \sgt \MRgt \vu_{ik} - \hats \hatMR \vu_{ik} \|^2 \leq 12 \alpha_{ij}^2 \\
  \iff  
\| (\sgt \MRgt  - \hats \hatMR) \MU_{ijhk}\|_\frob^2  \leq 12 \alpha_{ij}^2
\label{eq:appR2}
\eea
where $\MU_{ijhk} = [\vu_{ij} \; \vu_{ih} \; \vu_{ik}]$ and matches the definition in the statement of the theorem.
Now we note that for any two matrices $\MA$ and $\MB$ it holds $\| \MA \MB \|_\frob^2 \; \geq \sigma_\min(\MB)^2 \|\MA\|_\frob^2$, which applied to~\eqref{eq:appR2} becomes:
\bea
\| (\sgt \MRgt - \hats \hatMR) \MU_{ijhk}\|_\frob^2 \geq \\
\sigma_\min(\MU_{ijhk})^2 \| \sgt \MRgt - \hats \hatMR \|_\frob^2
\label{eq:appR3}
\eea
Chaining the inequality~\eqref{eq:appR3} back into~\eqref{eq:appR2}:
\bea
\sigma_\min(\MU_{ijhk})^2 \| \sgt \MRgt - \hats \hatMR \|_\frob^2  \leq 12 \alpha_{ij}^2 \iff \\
\| \sgt \MRgt - \hats \hatMR \|_\frob  \leq 2\sqrt{3} \frac{ \alpha_{ij} }{ \sigma_\min(\MU_{ijhk}) }
\label{eq:appR4}
\eea 
Since we do not know what are the true inliers, we have to take the worst case over all possible $i,j,h,k$ in~\eqref{eq:appR4}, yielding: 
\bea
\label{eq:rotBound}
\| \sgt \MRgt - \hats \hatMR \|_\frob \leq 2\sqrt{3}
 \frac{ \max_{ij} \alpha_{ij} }{ \min_{ijhk} \sigma_\min(\MU_{ijhk}) }
\eea
We finally observe that $\sigma_\min(\MU_{ijhk})$ is different from zero as long as the vectors $\vu_{ij}, \vu_{ih}, \vu_{ik}$ are linearly independent (\ie do not lie in the same plane), which is the case when the four inliers are not coplanar.
This proves the second bound in~\eqref{eq:bounds_rot}.

\myparagraph{Translation error bound}
Since the translation subproblem also selects all the inliers, plus potentially some outliers, it holds for each inlier point $i$:
\beal
\vb_i = \sgt \MRgt \va_i + \vtgt + \veps_i  &  \text{and} \\
\vb_i = \hats \hatMR \va_i + \hatvt + \vphi_i &  \text{with} \;\; \|\vphi_i\| \leq \sqrt{3}\beta_i 
\label{eq:appT1}
\eeal
where the first follows from the definition of inlier, and the second from the maximum error of a measurement considered inlier by \TLS (the factor $\sqrt{3}$ comes from the fact that translation is estimated component-wise, which means $\|\vphi_i\|^2 = [\vphi_i]_1^2 + [\vphi_i]_2^2 + [\vphi_i]_3^2 \leq 3\beta_i^2$ ).
Combining the equalities~\eqref{eq:appT1}:
 \beal
\hats \hatMR \va_i + \hatvt + \vphi_i = \sgt \MRgt \va_i + \vtgt + \veps_i   \\
\iff \vtgt - \hatvt =  (\hats \hatMR - \sgt \MRgt)\va_i + (\vphi_i-\veps_i)
\label{eq:appT2}
\eeal
Since we have at least 4 inliers $i,j,h,k$, each satisfying the equality~\eqref{eq:appT2}, we 
subtract (member-wise) the first equality (including inlier $i$) from the sum of the last 3 equalities (including points $j,h,k$):
 \bea
2(\vtgt - \hatvt) &=&  \sum_{l\in\{j,h,k\}}(\hats \hatMR - \sgt \MRgt)(\va_l - \va_i)  \\
&+& \sum_{l\in\{j,h,k\}} (\vphi_l-\veps_l - \vphi_i+\veps_i)
\label{eq:appT3}
\eea
Taking the norm of both members and using the triangle inequality:
\bea
2\|\vtgt - \hatvt\| &\leq&  \sum_{l\in\{j,h,k\}} \| (\hats \hatMR - \sgt \MRgt)(\va_l - \va_i) \|  \\
&+& \sum_{l\in\{j,h,k\}} (\|\vphi_l\|+\|\veps_l\| + \|\vphi_i\|+\|\veps_i\|)
\label{eq:appT4}
\eea
Noting that $\| (\hats \hatMR - \sgt \MRgt)(\va_l - \va_i) \| = \| (\hats \hatMR - \sgt \MRgt) \TIMa_{il} \|$
must satisfy the inequality~\eqref{eq:appR0}, and the last 4 terms in~\eqref{eq:appT4} have norm bounded by $\beta$ and $\sqrt{3}\beta$:
\bea
\hspace{-4mm} 2\|\vtgt - \hatvt\| &\leq&  \sum_{l\in\{j,h,k\}} 2 \TIMNoiseBound_{ij} + \sum_{l\in\{j,h,k\}} (2 + 2\sqrt{3}) \beta
\label{eq:appT5}
\eea
Since we assume all the points to have the same noise bound $\beta$, it holds $\TIMNoiseBound_{ij} = 2\beta$ and 
the right-hand-side of~\eqref{eq:appT5} becomes $(18+6\sqrt{3}) \beta$, from which it follows:
 \bea
\hspace{-2mm} 2\|\vtgt - \hatvt\| \leq  (18+6\sqrt{3}) \beta
\iff
\|\vtgt - \hatvt\| \leq  (9+3\sqrt{3}) \beta
\eea
which proves the last bound in~\eqref{eq:bounds_tran}.

\subsection{\revone{\TLS Translation Estimation: Component-wise vs. Full}}
\label{sec:app-compareTranslation}
Component-wise translation estimation is a simple approximation for the full translation estimate. The full \TLS translation estimate can be obtained using techniques described in~\cite{Liu19JCGS-minSumTruncatedConvexFunctions,Yang20arXiv-certifiablePerception}. In this section, we provide numerical experiments to show that the component-wise translation estimation is sufficiently close to the full \TLS translation estimate.

The \TLS estimator of the full translation vector $\vt \in \Real{3}$, given noisy and outlying observations $\vt_i,i=1,\dots,N$, is the globally optimal solution of the following non-convex optimization:
\bea \label{eq:fullTLStranslation}
\min_{\vt \in \Real{3}} \sum_{i=1}^N \min \left(\frac{\| \vt - \vt_i \|^2}{\beta_i^2}, \barcsq \right),
\eea
where $\vt_i =\vb_i - \hats\hatMR\va_i$ in the context of \name (\cf eq.~\eqref{eq:TLStranslation}). As presented in Section~\ref{sec:tranEstimation}, instead of estimating the full translation by solving problem~\eqref{eq:fullTLStranslation}, \name seeks to estimate the translation vector component-wise by leveraging the adaptive voting algorithm for solving the scalar \TLS problem. In this section, we show that component-wise translation estimation is sufficiently close to the full translation estimation.

In order to show the accuracy of component-wise translation estimation, we design two methods to solve problem~\eqref{eq:fullTLStranslation} globally. The two methods bear great similarity to the semidefinite relaxation~\cite{Yang19iccv-QUASAR} and graduated non-convexity (\gnc) algorithms~\cite{Yang20ral-GNC} presented for solving the rotation subproblem.
(i) (\tlssdp) We use Lasserre's hierarchy of moment relaxations~\cite{lasserre10book-momentsOpt} to relax the nonconvex problem into a convex SDP. By solving the SDP, we can obtain an estimate of the full translation, together with a global optimality certificate. This method has similar flavor as to the SDP relaxation we presented for estimating rotation, but is slightly more involved. The interested reader can refer to~\cite{Yang20arXiv-certifiablePerception} for more technical details. (i) (\tlsgnc) We use \GNC to solve the \tls translation estimation problem, where the non-minimal solver is simply the weighted average of the noisy measurements $\vt_i$'s. \gnc is very efficient but it offers no optimality guarantees. Comparison of these two methods (for solving the full translation) with adaptive voting (for solving the component-wise translation) are shown in Fig.~\ref{fig:robustTranslation} (using the same experimental setup as in Section~\ref{sec:separateSolver}). We can see that: (i) \tlssdp~performs the best. It is robust to $90\%$ outliers and returns the most accurate translation estimates. The estimates of \tlssdp~are indeed the \emph{globally optimal} solutions, because the relaxation is empirically always tight (duality gap is zero). (ii) As expected, solving the translation component-wise is an approximation, and it is slightly less accurate than solving the whole translation. However, the estimates are sufficiently close to the globally optimal solutions (when the outlier rate is below $90\%$). (iii) The \gnc method obtains the globally optimal solution when the outlier rate is below $90\%$. This is similar to the performance of the \gnc method for solving the rotation subproblem presented in \teaser.


\begin{figure}[t]
\begin{center}
	\begin{minipage}{\columnwidth}
	\begin{center}
			\includegraphics[width=0.8\columnwidth]{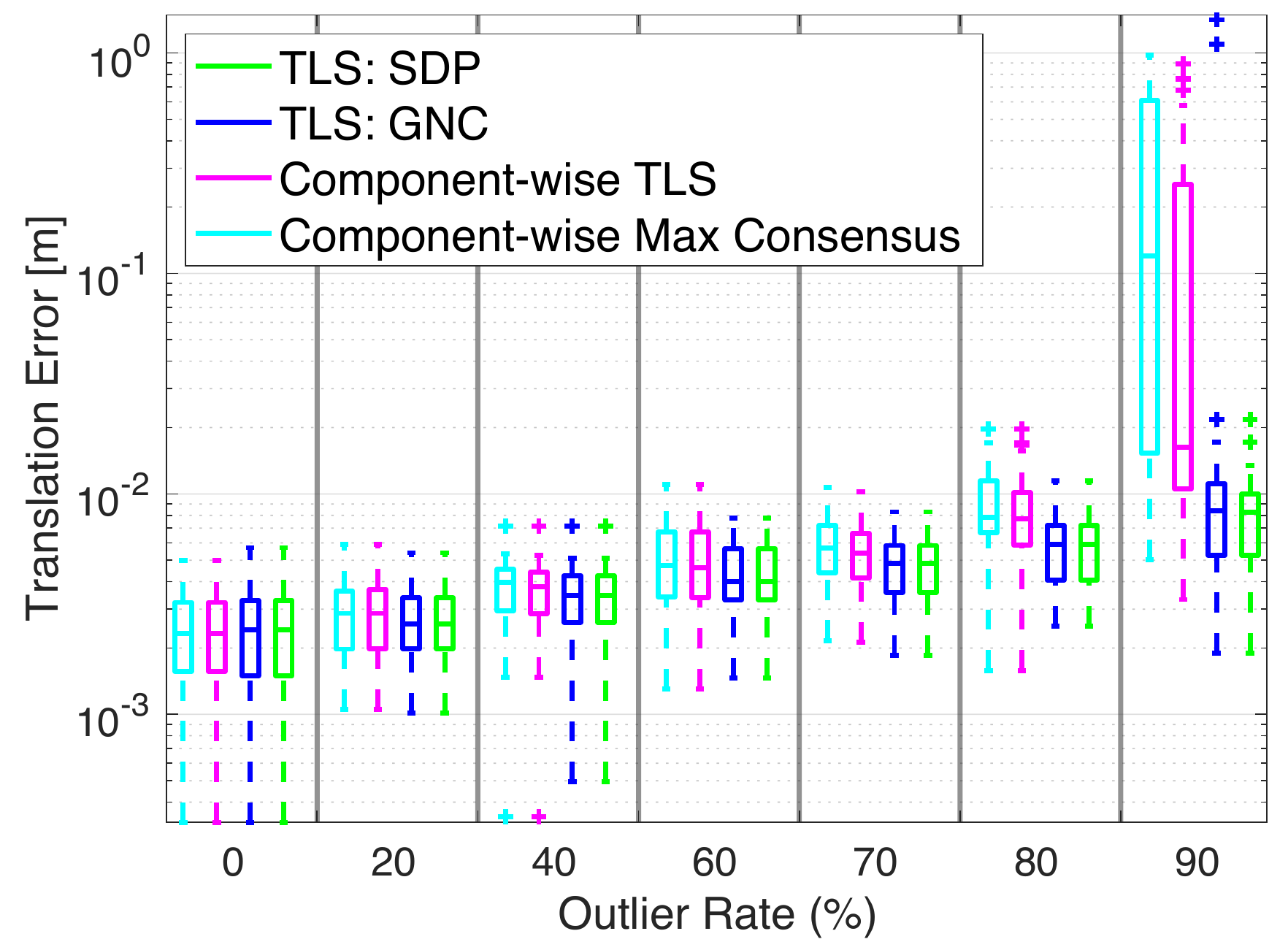}
			\vspace{-3mm}
	\end{center}
	\end{minipage} 
	\caption{Comparison of four robust translation estimation methods. 
	 \label{fig:robustTranslation}}
	\end{center}
\end{figure}

\subsection{SDP Relaxations: Quaternion vs. Rotation Matrix}
\label{sec:app-compare_quaternion_rotationMatrix}


\begin{figure}[t]
	\begin{center}
	\begin{minipage}{\columnwidth}
	\begin{tabular}{c}%
			\myhspace 
			\begin{minipage}{\columnwidth}%
			\centering%
			\includegraphics[width=1\columnwidth]{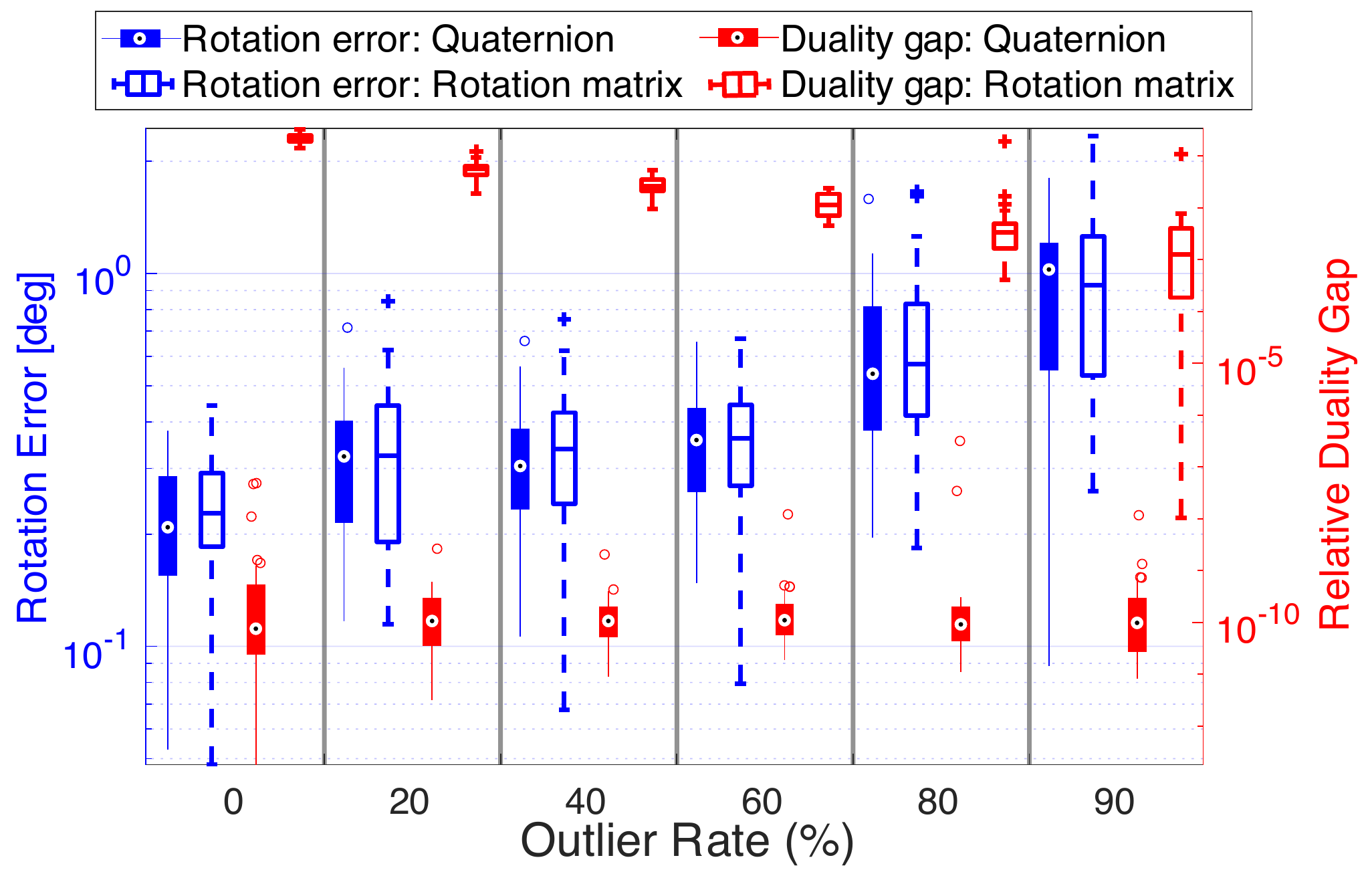} \\
			\vspace{-1mm}
			(a) Results on the \bunny dataset.
			\end{minipage}
              \\
            \myhspace
			\begin{minipage}{\columnwidth}%
			\centering%
			\includegraphics[width=1\columnwidth]{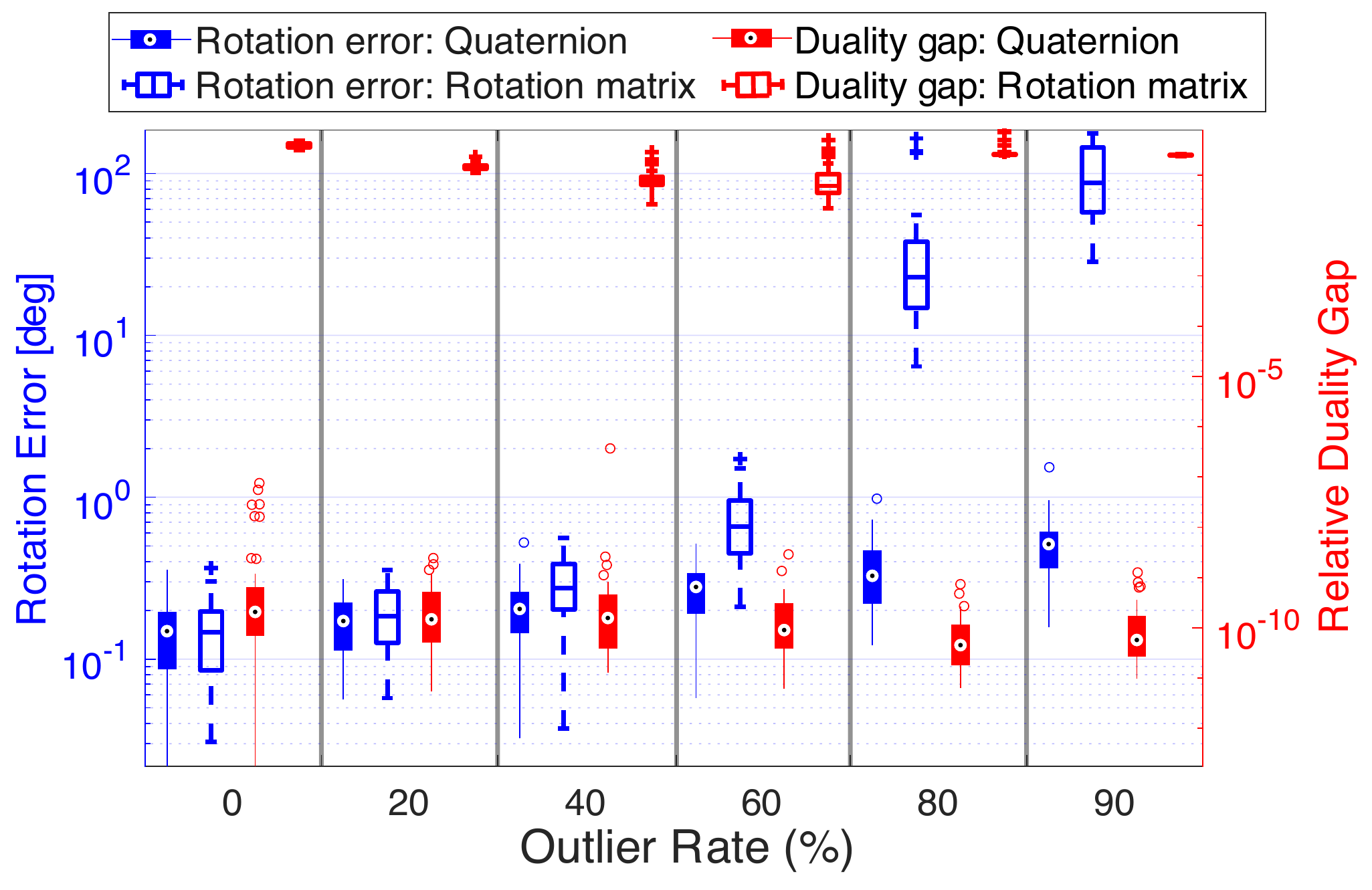} \\
			\vspace{-1mm}
			(b) Results on random simulated unit vectors.
			\end{minipage}
		\end{tabular}
	\end{minipage}
	\vspace{-2mm}
	\caption{Comparison of rotation errors (left-axis, blue) and relative duality gap (right-axis, red) between the quaternion-based SDP relaxation (filled boxplots) and the rotation-matrix-based SDP relaxation (empty boxplots) on (a) the \bunny dataset and (b) random simulated unit vectors.
	\label{fig:app-compare_quaternion_rotm}}
	 	\vspace{-2mm} 
	\end{center}
\end{figure}

In Section~\ref{sec:experiments}, we compare the rotation estimation error between the quaternion-based SDP relaxation (eq.~\eqref{eq:SDPrelax},~\cite{Yang19iccv-QUASAR}) and the rotation-matrix-based SDP relaxation (\cite{Yang19rss-teaser}) on the \bunny dataset. In this section, we also report the \emph{relative duality gap} for the two relaxations on the \bunny dataset. 
Fig.~\ref{fig:app-compare_quaternion_rotm}(a) shows the relative duality gap (boxes in red, scale on the right-hand y-axis) and the rotation error (boxes in blue, scale on the left-hand y-axis).
The figure shows that the the quaternion-based relaxation is always \emph{tighter} than the rotation-matrix-based relaxation, although both relaxations achieve similar performance in terms of rotation estimation errors.

To further stress-test the two techniques, 
 we compare their performance in a more randomized test setting~\cite{Yang19iccv-QUASAR}: at each Monte Carlo run, we first randomly sample $K=40$ unit vectors from the 3D unit sphere, then apply a random rotation and isotropic Gaussian noise (standard deviation $\sigma=0.01$) to the $K$ unit vectors. Fig.~\ref{fig:app-compare_quaternion_rotm}(b) shows the results under this setting. We can see that the quaternion-based relaxation remains tight, while the rotation-matrix-based relaxation has a much larger relative duality gap. In addition, a tighter relaxation translates to more accurate rotation estimation: the quaternion-based relaxation dominates the rotation-matrix relaxation and obtains much more accurate rotation estimates, especially at $>80\%$ outlier rates.

\subsection{Stanford Dataset: Extra Results}
\label{sec:experiments_supp}

In all the following experiments, $\bar{c}^2=1$ for scale, rotation and translation estimation.

\subsubsection{Benchmark on Synthetic Datasets}
\label{sec:benchmark_supp}
As stated in \prettyref{sec:benchmark} in the main document, we have benchmarked \name against four state-of-the-art methods in point cloud registration (\FGR~\cite{Zhou16eccv-fastGlobalRegistration}, \GORE~\cite{Bustos18pami-GORE}, \ransaconek, \ransac) at increasing level of outliers, using four datasets from the Stanford 3D Scanning Repository~\cite{Curless96siggraph}: \bunny, \armadillo, \dragon and \buddha. Results for the \bunny are showed in the main document and here we show the results for the other three datasets in Fig.~\ref{fig:benchmarkSupp}. The figure confirms that \name dominates the other techniques is insensitive to extreme outlier rates.


\renewcommand{\mpwthree}{6cm}
\renewcommand{\myhspace}{\hspace{-3mm}}

\begin{figure*}[t]
	\begin{center}
	\begin{minipage}{\textwidth}
	\hspace{-0.2cm}
	\begin{tabular}{ccc}%
			\begin{minipage}{\mpwthree}%
			\centering%
			\includegraphics[width=\columnwidth]{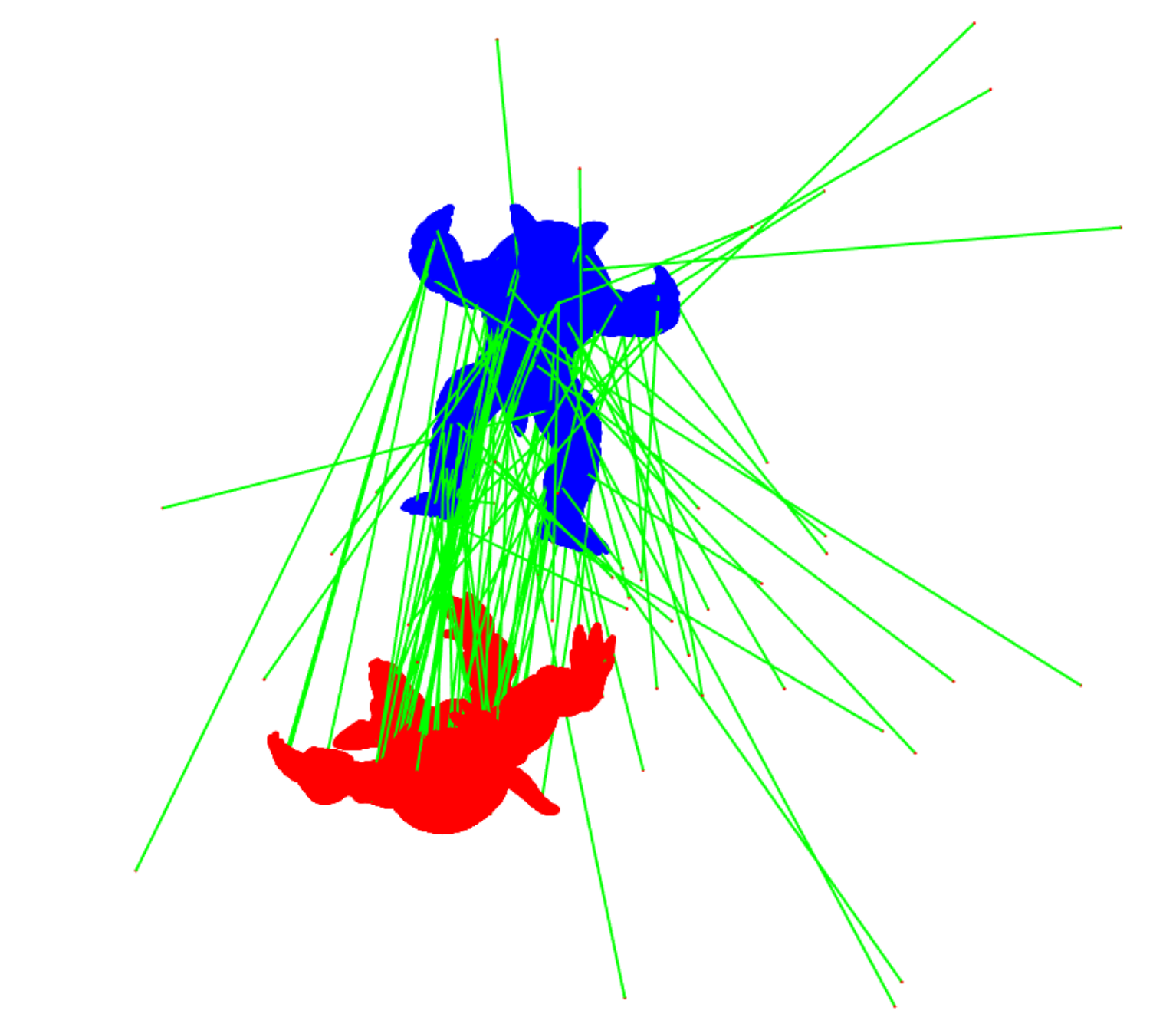} \\
			\end{minipage}
		& \myhspace
			\begin{minipage}{\mpwthree}%
			\centering%
			\includegraphics[width=\columnwidth]{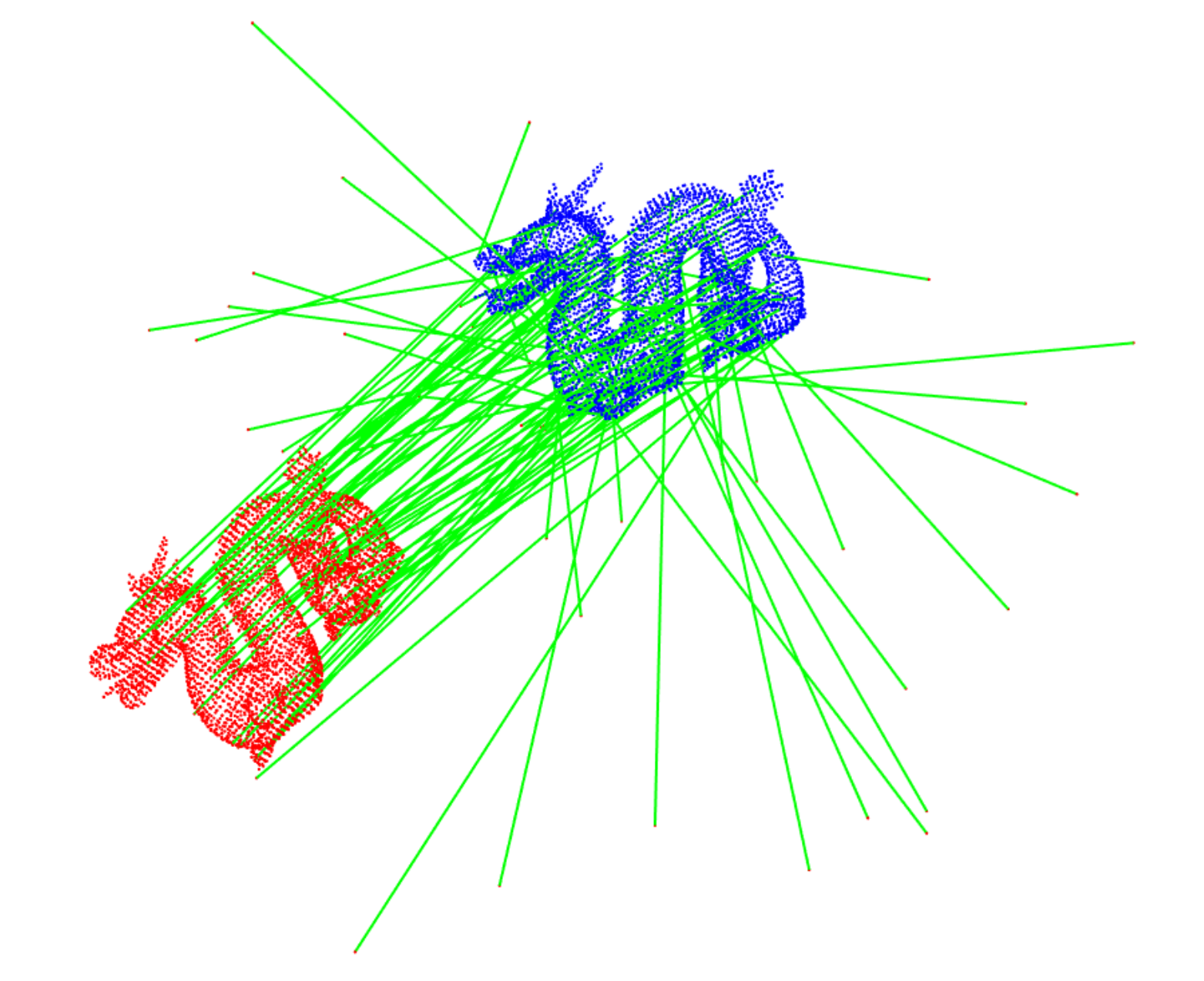} \\
			\end{minipage}
		& \myhspace
			\begin{minipage}{\mpwthree}%
			\centering%
			\includegraphics[width=\columnwidth]{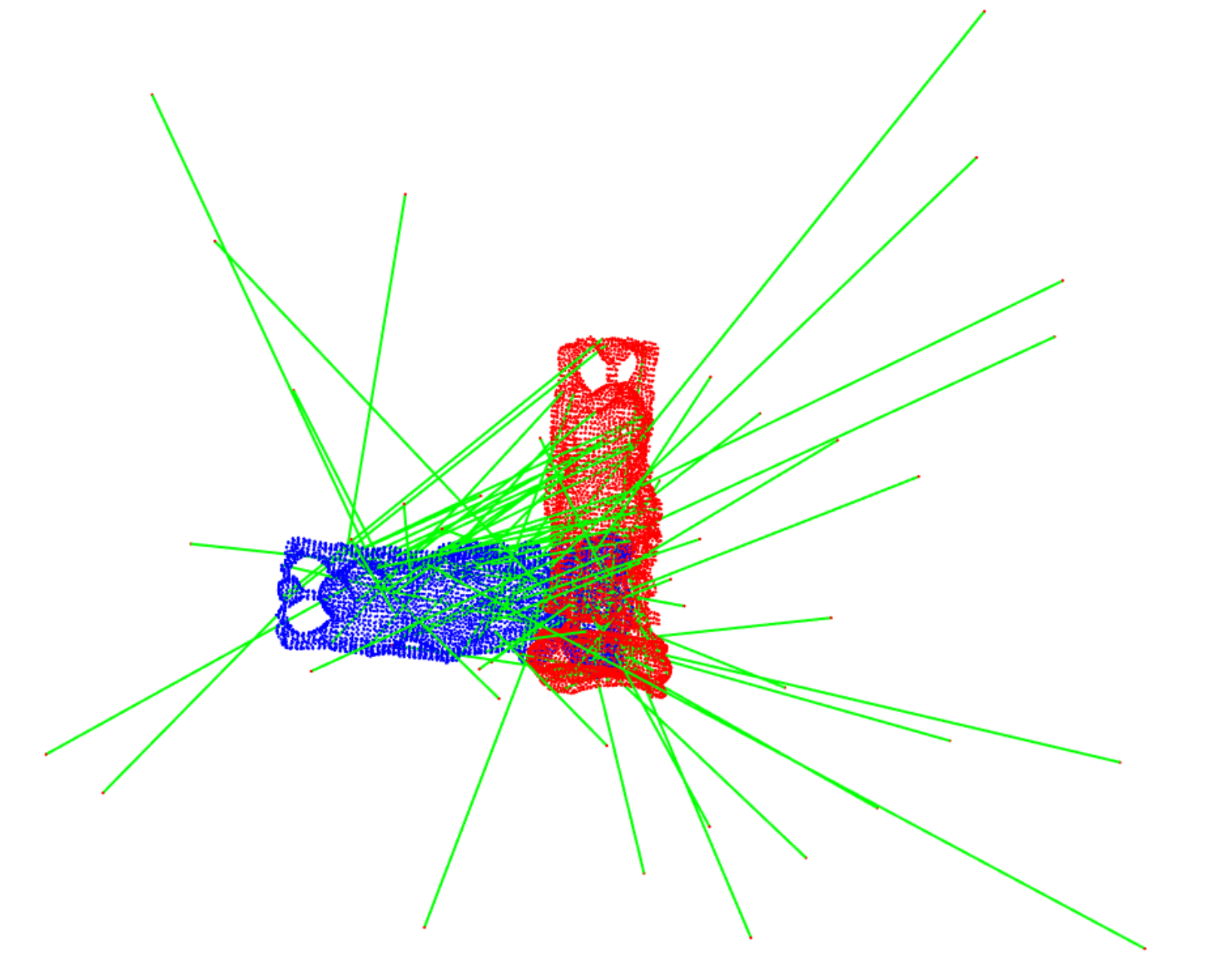} \\
			\end{minipage}  \\
			
		\begin{minipage}{\mpwthree}%
			\centering%
			\includegraphics[width=\columnwidth]{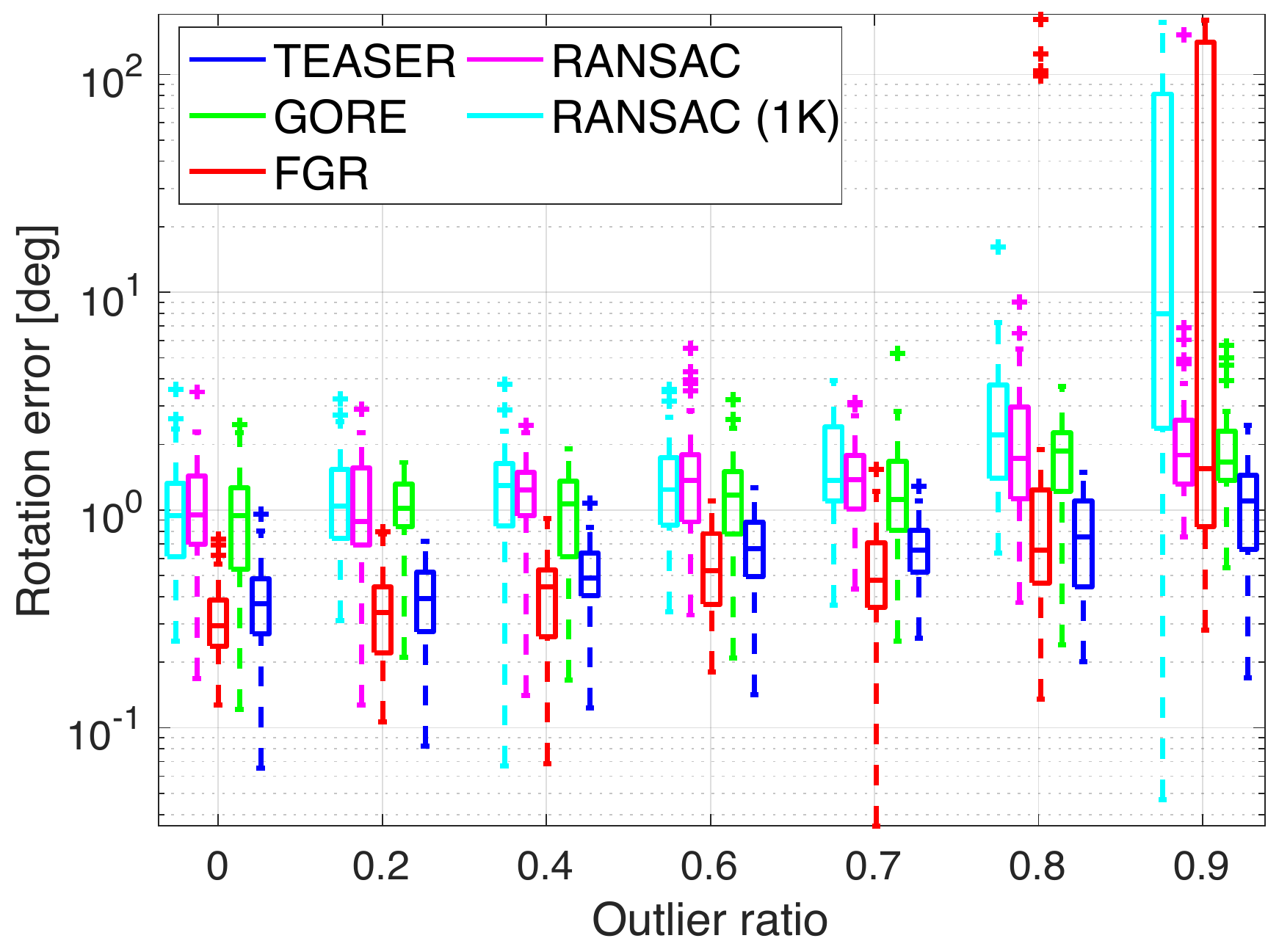} \\
			\end{minipage}
		& \myhspace
			\begin{minipage}{\mpwthree}%
			\centering%
			\includegraphics[width=\columnwidth]{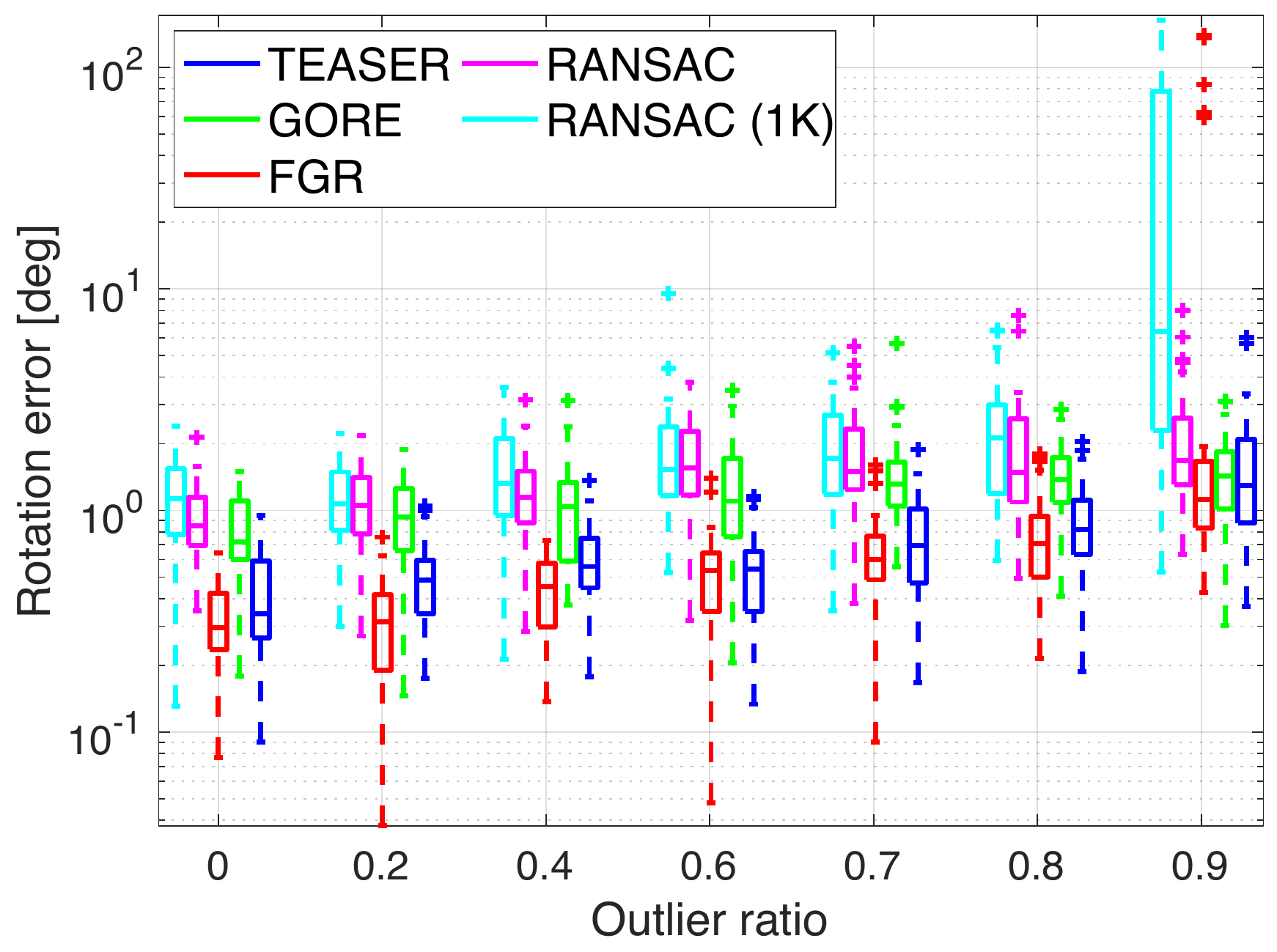} \\
			\end{minipage}
		& \myhspace
			\begin{minipage}{\mpwthree}%
			\centering%
			\includegraphics[width=\columnwidth]{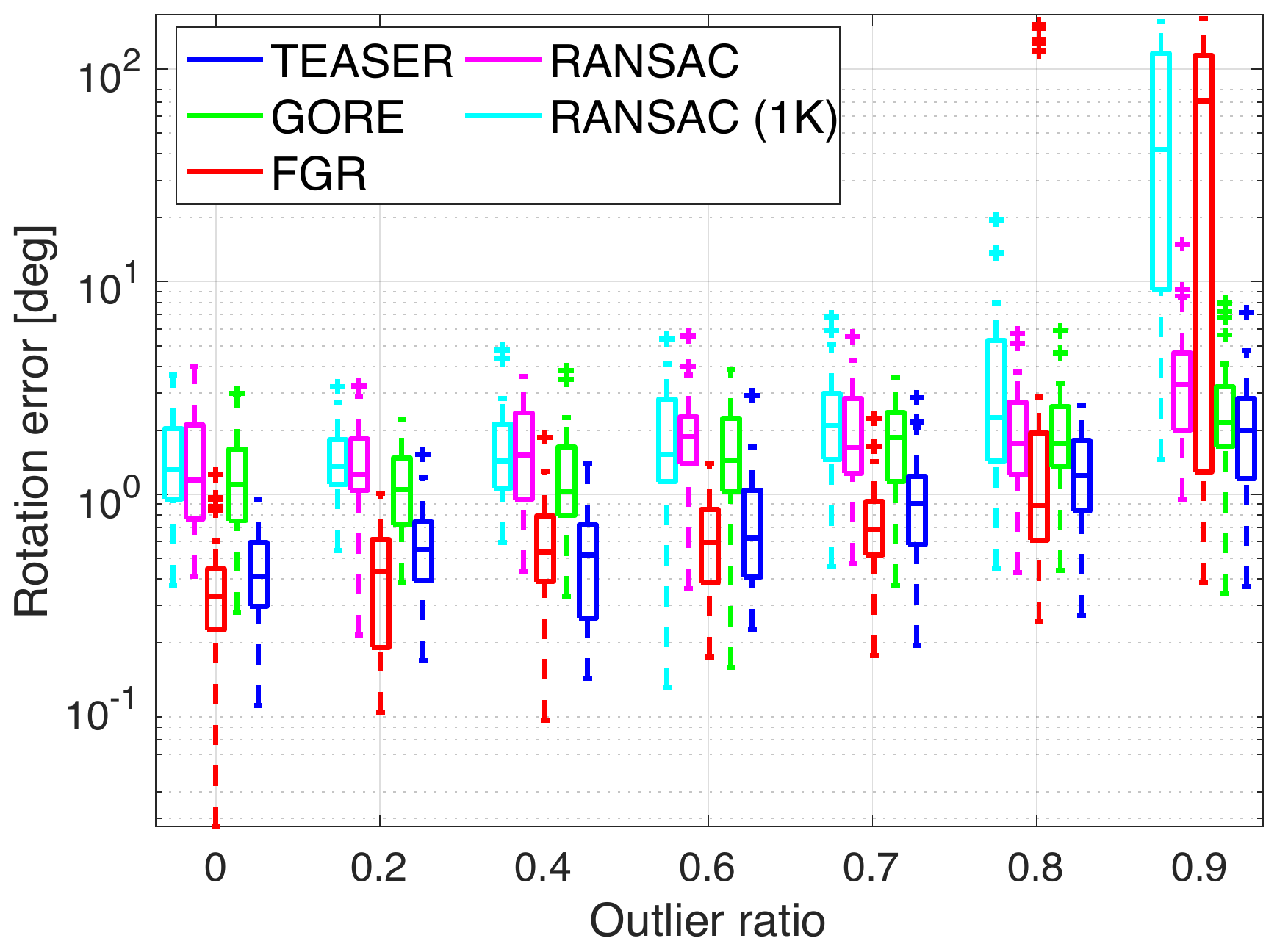} \\
			\end{minipage}\\
			
		\begin{minipage}{\mpwthree}%
			\centering%
			\includegraphics[width=\columnwidth]{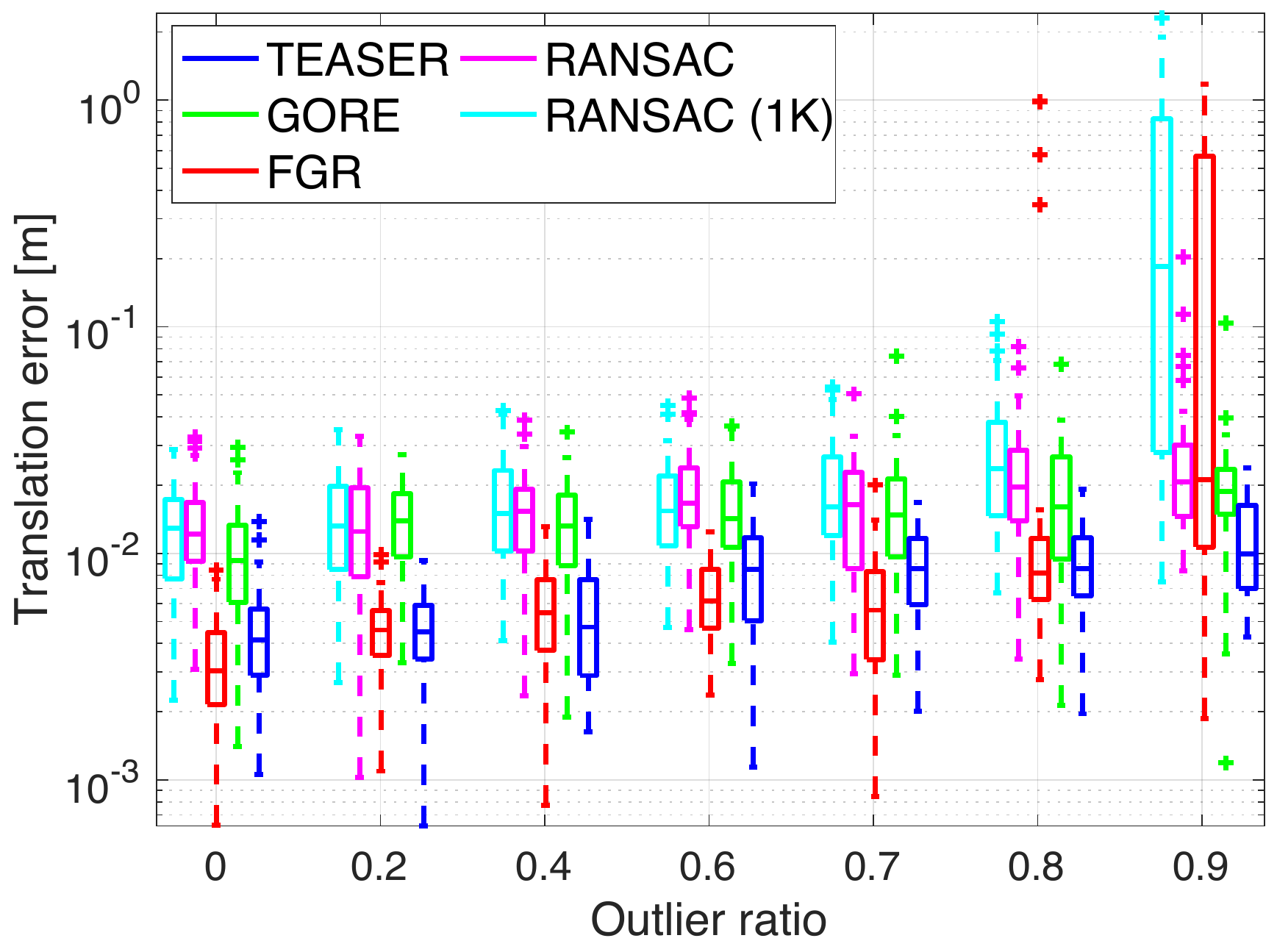} \\
			(a) \armadillo
			\end{minipage}
		& \myhspace
			\begin{minipage}{\mpwthree}%
			\centering%
			\includegraphics[width=\columnwidth]{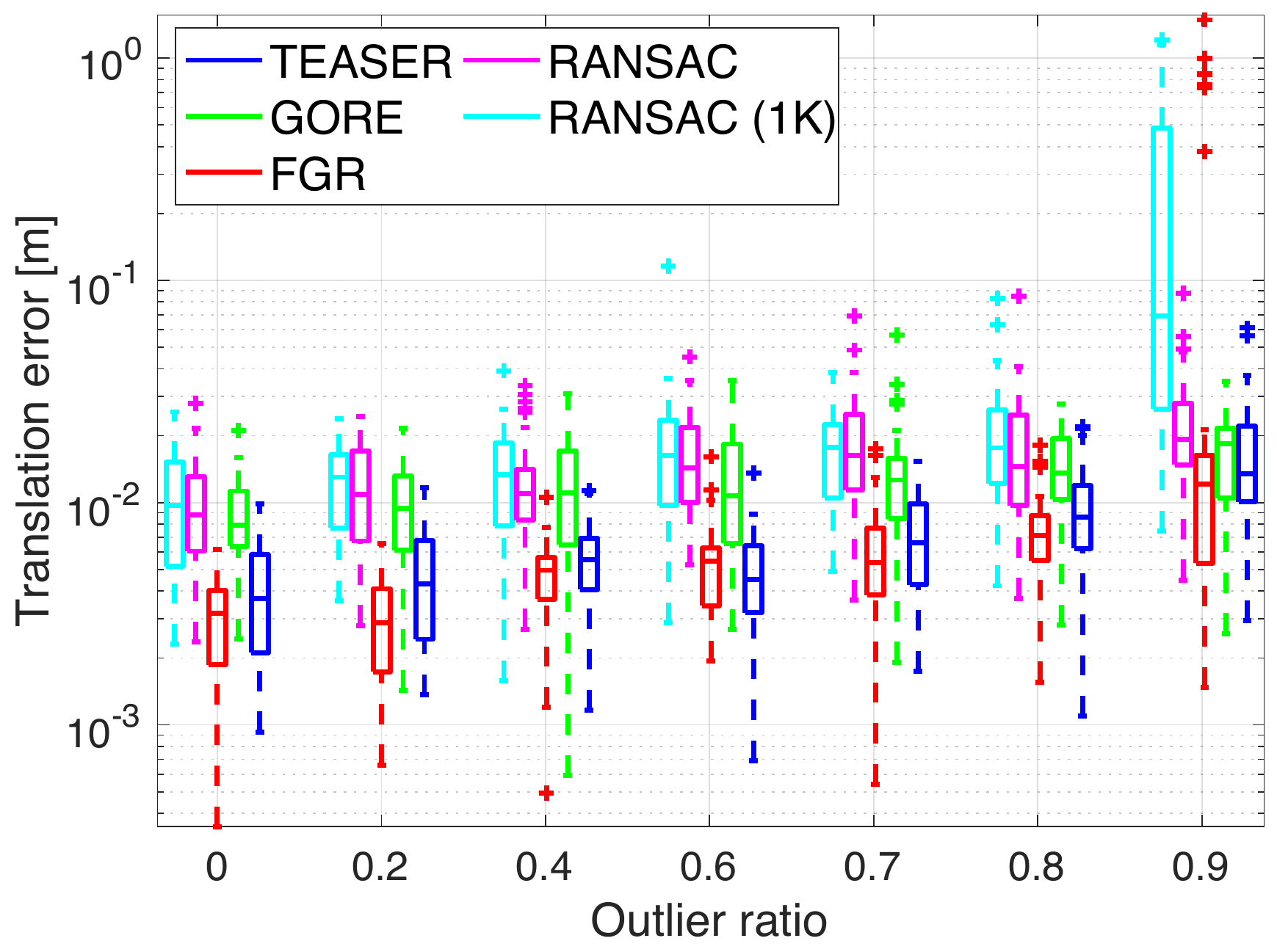} \\
			(b) \dragon
			\end{minipage}
		& \myhspace
			\begin{minipage}{\mpwthree}%
			\centering%
			\includegraphics[width=\columnwidth]{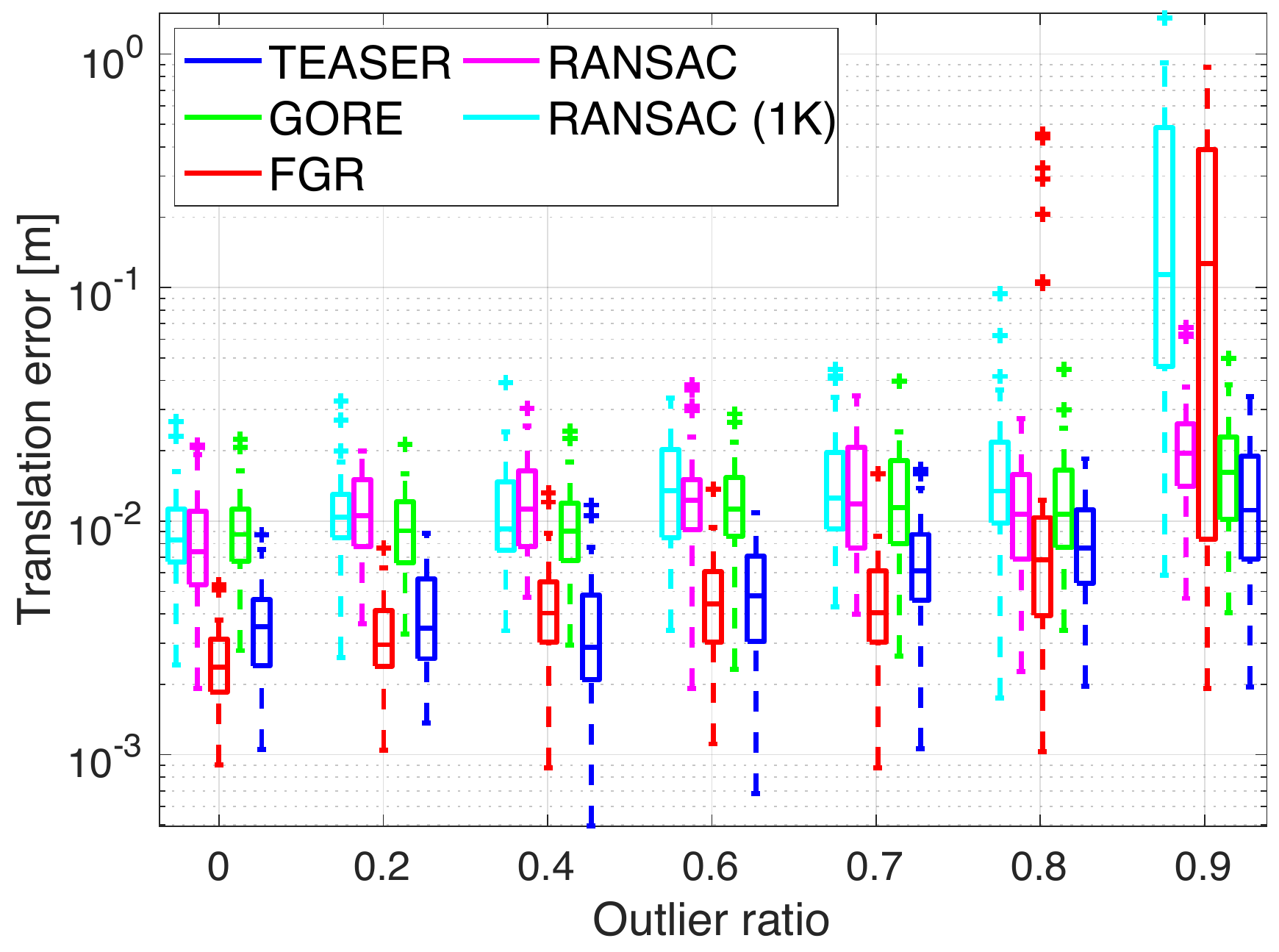} \\
			(c) \buddha
			\end{minipage}
		\end{tabular}
	\end{minipage}
	\vspace{-3mm} 
	\caption{Results for the \armadillo, \dragon and \buddha datasets at increasing levels of outliers. First row: example of putative correspondences with 50\% outliers. Blue points are the model point cloud, red points are the transformed scene model. Second row: rotation error produced by five methods including \name. Third row: translation error produced by the five methods. }
	 \label{fig:benchmarkSupp}
	\end{center}
\end{figure*}

\subsubsection{\name on High Noise}
In the main document, the standard deviation $\sigma$ of the isotropic Gaussian noise is set to be $0.01$, with the point cloud scaled inside a unit cube $[0,1]^3$. In this section, we give a visual illustration of how corrupted the point cloud is after adding noise of such magnitude. In addition, we increase the noise standard deviation $\sigma$ to be $0.1$, visualize the corresponding corrupted point cloud, and show that \name can still recover the ground-truth transformation with reasonable accuracy. Fig.~\ref{fig:bunny_with_noise} illustrates a \emph{clean} (noise-free) \bunny model (Fig.~\ref{fig:bunny_with_noise}(a)) scaled inside the unit cube and its variants by adding different levels of isotropic Gaussian noise and outliers. From Fig.~\ref{fig:bunny_with_noise}(b)(c), we can see $\sigma=0.01$ 
(noise used in the main document) is a reasonable noise standard deviation for real-world applications, while $\sigma=0.1$ destroys the geometric structure of the \bunny and it is beyond the noise typically encountered in robotics and computer vision applications. 
Fig.~\ref{fig:bunny_with_noise}(d) shows the \bunny with high noise ($\sigma = 0.1$) and 50\% outliers. Under this setup, we run \name to register the two point clouds and recover the relative transformation. \name can still return rotation and translation estimates that are  close to the ground-truth transformation. However, due to the severe noise corruption, even a successful registration 
fails to yield a visually convincing registration result to human perception. For example, Fig.~\ref{fig:bunny_teaser_high_noise} shows a representative \name registration result with $\sigma=0.1$ noise corruption and $50\%$ outlier rate, where accurate transformation estimates are obtained (rotation error is $3.42^{\circ}$ and translation error is $0.098$m), but the \bunny in the cluttered scene is hardly visible, even when super-imposed to the clean model.


\renewcommand{\myhspace}{\hspace{-6mm}}
\renewcommand{\mpw}{4.5cm}

\begin{figure}[h]
	\begin{center}
	\begin{minipage}{\columnwidth}
	\vspace{-12mm}
	\centering
	\begin{tabular}{cc}%
	\myhspace
			\begin{minipage}{\mpw}%
			\centering%
			\includegraphics[trim=0mm 0mm 0mm 10mm, clip, width=\columnwidth]{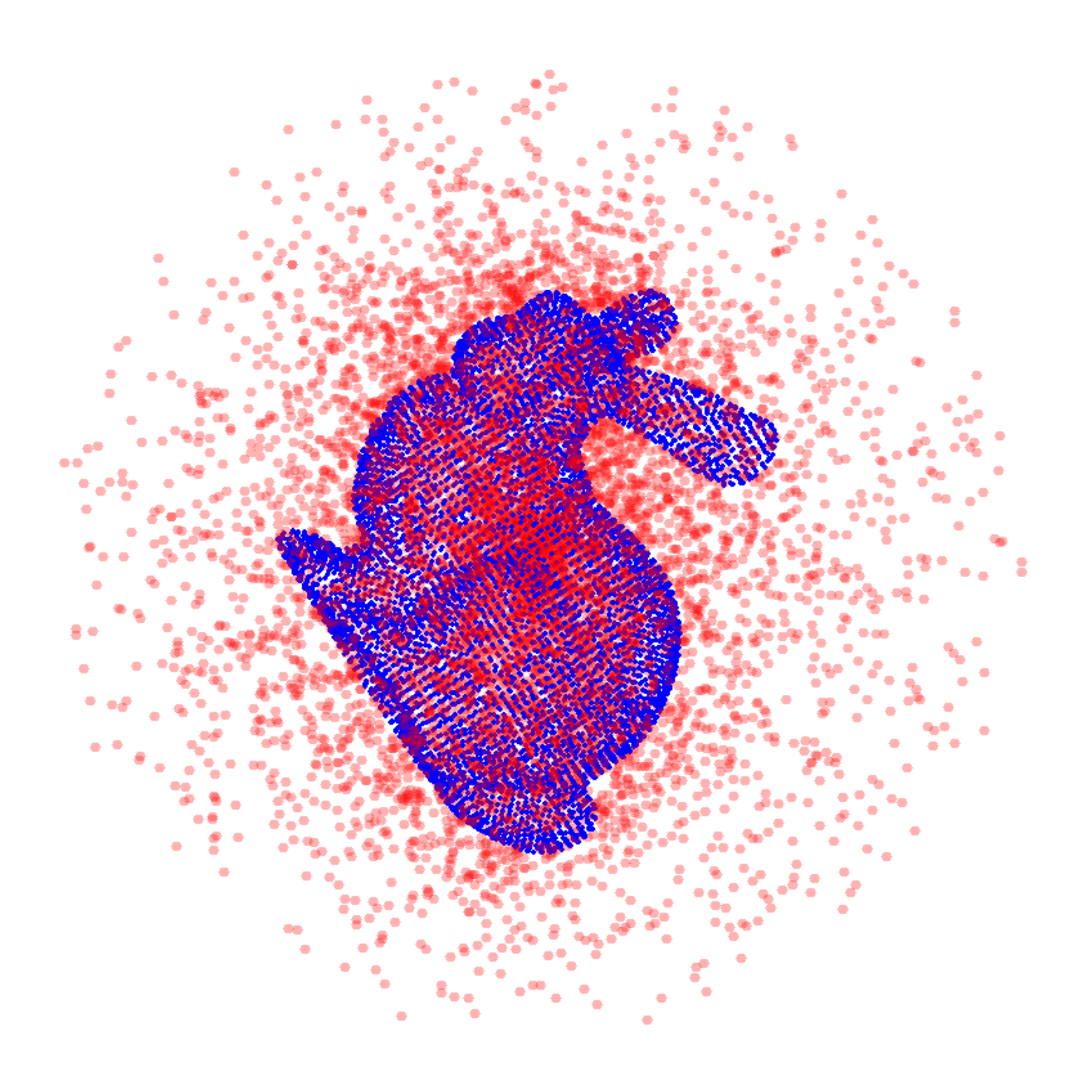} \\
			\end{minipage}
		& \myhspace 
			\begin{minipage}{\mpw}%
			\centering%
			\includegraphics[trim=0mm 0mm 0mm 10mm, clip,width=0.97\columnwidth]{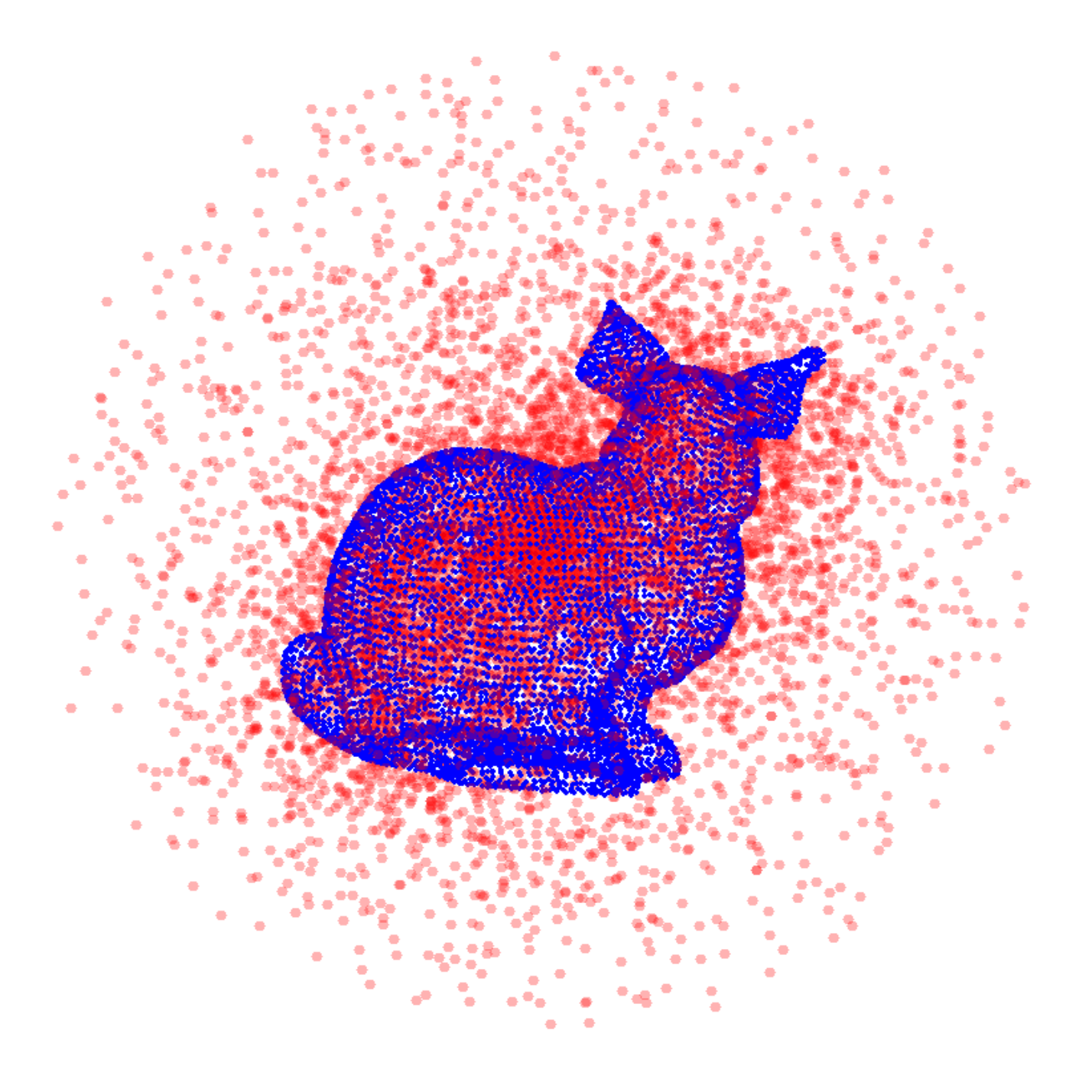} \\
			\end{minipage} 
		\end{tabular}
	\end{minipage}
	\vspace{-3mm} 
	\caption{ A single representative \name registration result with $\sigma=0.1$ noise corruption and $50\%$ outliers, viewed from two distinctive perspectives. Clean \bunny model showed in blue and cluttered scene showed in red. Although the rotation error compared to ground-truth is $3.42^{\circ}$ and the translation error compared to ground-truth is $0.098$m, it is challenging for a human 
	to confirm the correctness of the registration result. 
	 \label{fig:bunny_teaser_high_noise}}
	\end{center}
\end{figure}

\subsubsection{\namepp on Problems with Increasing Number of Correspondences}


\renewcommand{\myhspace}{\hspace{-6mm}}
\renewcommand{\mpw}{4.5cm}

\begin{figure*}[]%
	\vspace{-8mm}
	\begin{center}
	\begin{minipage}{\textwidth}
		\begin{center}
			\includegraphics[width=0.7\columnwidth]{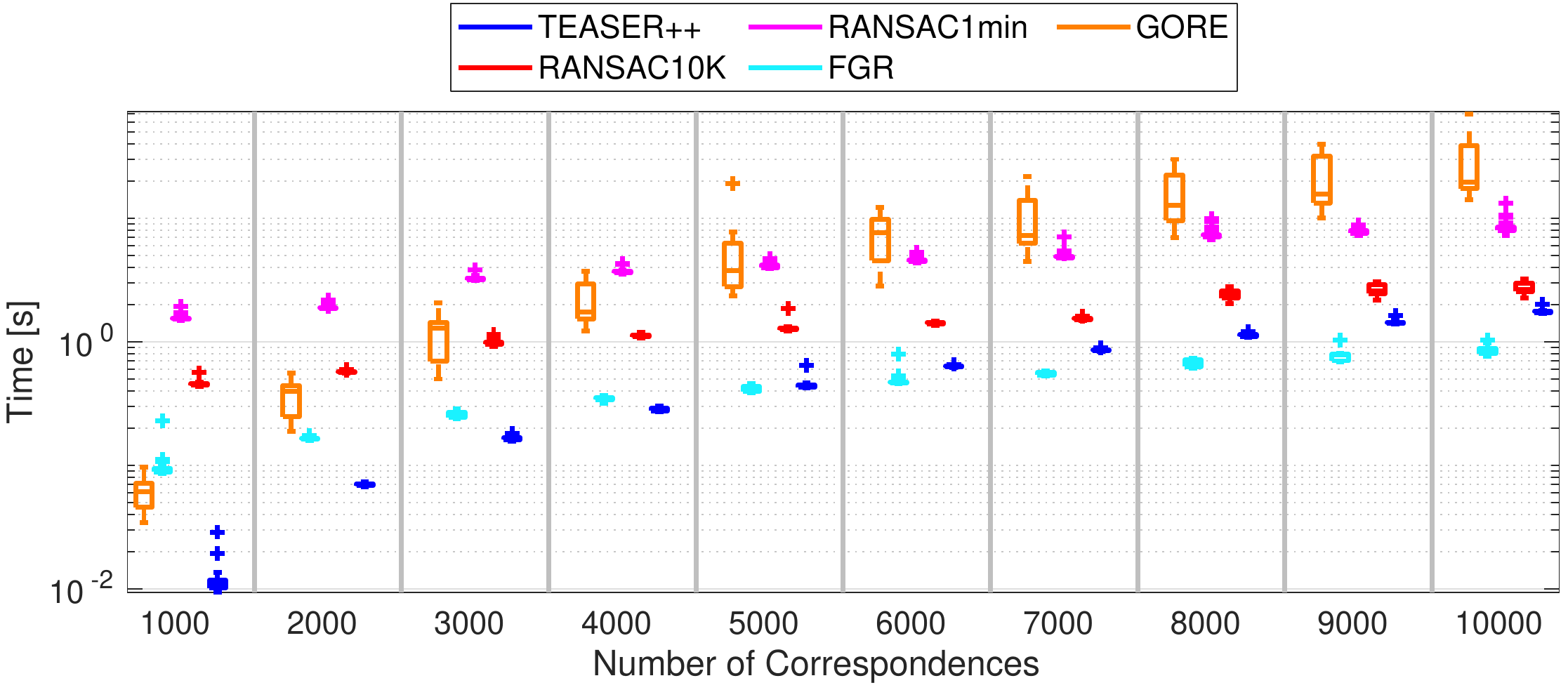}\\
	(a) Runtimes vs. number of correspondences 
		\end{center}
		\vspace{1mm}
	\end{minipage} \\
	\begin{minipage}{\textwidth}
		\begin{center}
			\includegraphics[width=0.7\columnwidth]{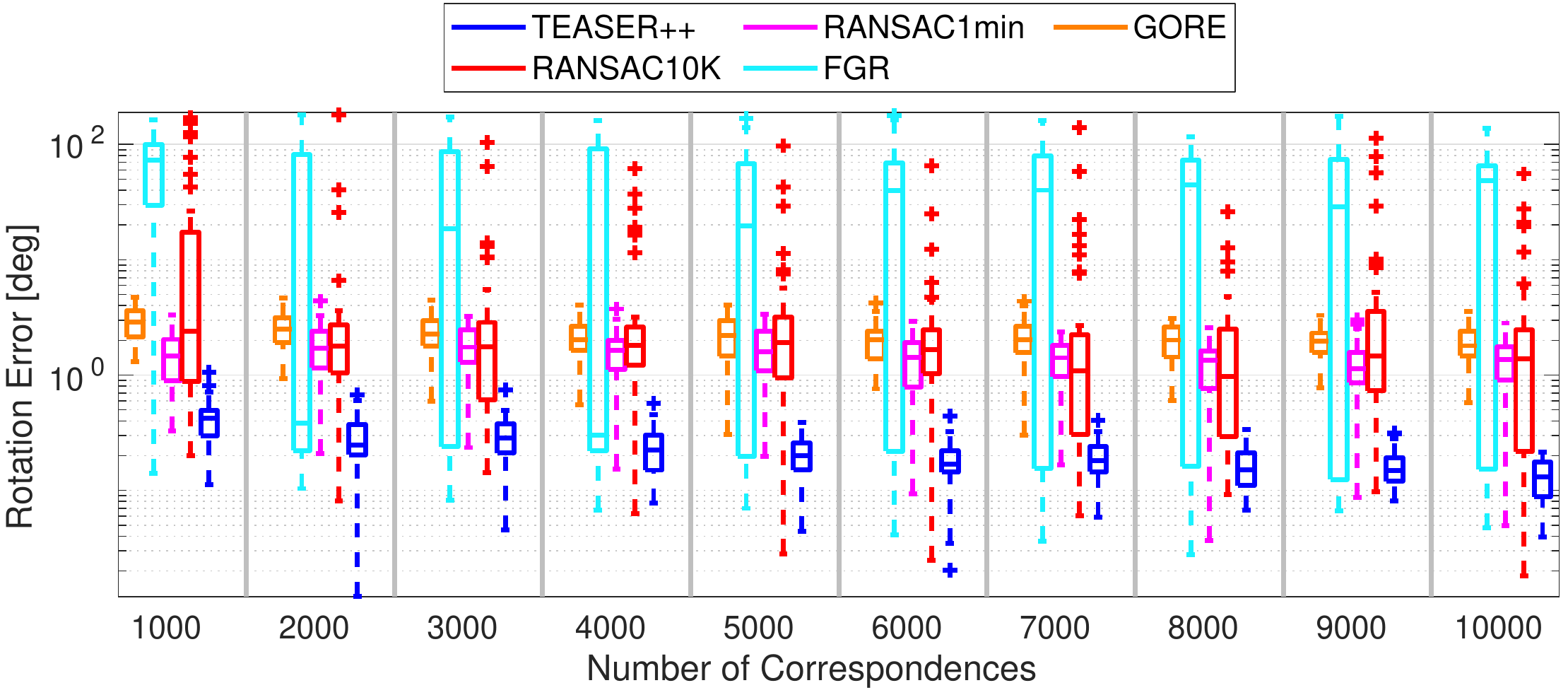}\\
	(b) Rotation errors vs. number of correspondences 
		\end{center}
		\vspace{1mm}
	\end{minipage} \\
	\begin{minipage}{\textwidth}
		\begin{center}
			\includegraphics[width=0.7\columnwidth]{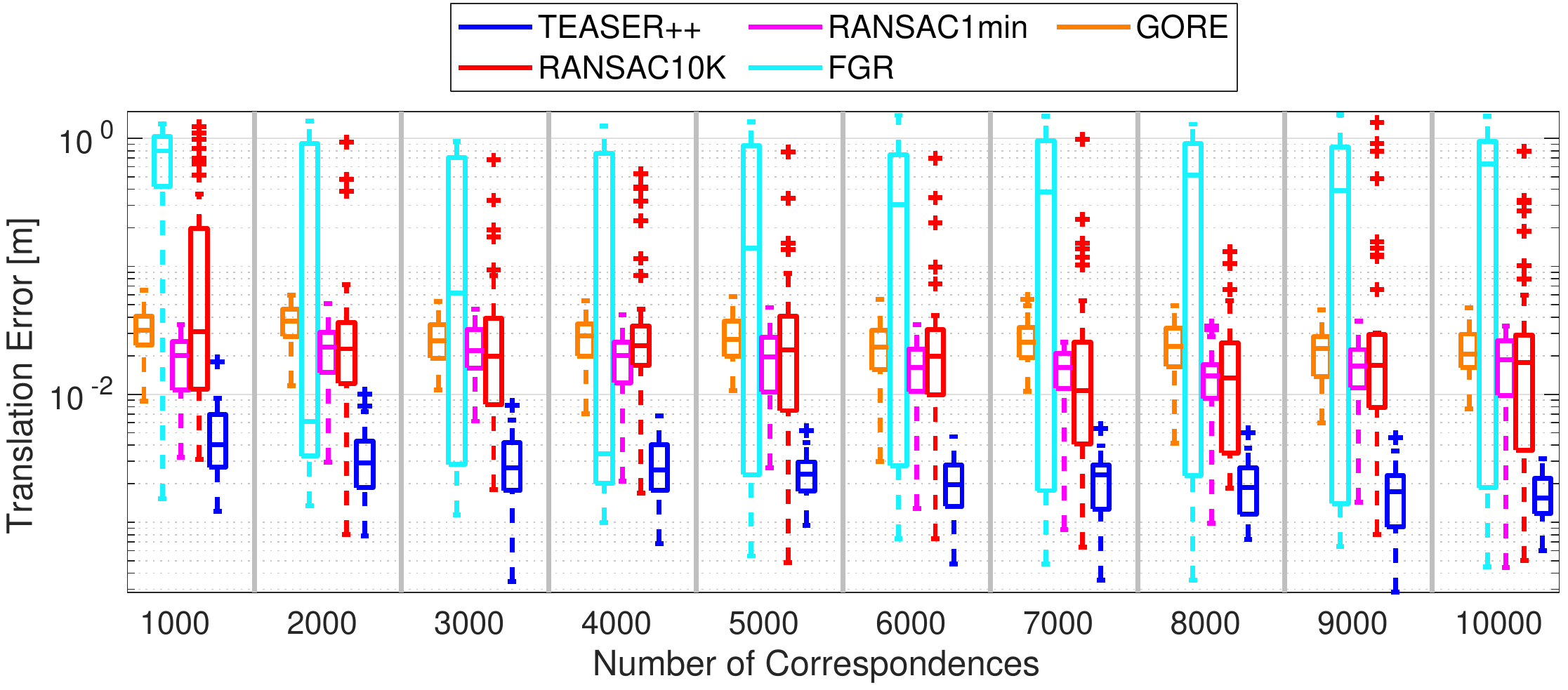}\\
	(c) Translation errors vs. number of correspondences 
		\end{center}
	\end{minipage}
	\caption{Comparison of different robust translation estimation methods with different number of correspondences. 
	 \label{fig:bunny_teaser_high_nrpoints}}
	\end{center}
\end{figure*}

\revone{Here we provide more results for \namepp on problems with higher number of correspondences. We follow the same experimental setup as detailed in Section \ref{sec:benchmark}, with the exception that we fixed our outlier rate at 95\%, and test the algorithms with increasing numbers of correspondences.
Fig.~\ref{fig:bunny_teaser_high_nrpoints} shows the results of the experiments.

Overall, in terms of rotation errors and translation errors, \namepp dominates other methods by a large margin. In terms of speed, \namepp outperforms other robust estimation algorithms such as \GORE. While \namepp's runtime scales with the number of correspondences, it still outperforms methods such as \GORE on all cases. At 10000 correspondences, \namepp only takes around 2 seconds to solve the registration problem.  
This demonstrates that \namepp can be used to solve registration problems with high number of correspondences within reasonable amount of time.
}

\clearpage

\clearpage
 
\isTwoCols{\onecolumn}{}

\renewcommand{\myhspace}{\hspace{-6mm}}
\renewcommand{\mpw}{4.5cm}
\newcommand{\myRate}{0.7}

\begin{figure}[h]
	\begin{center}
	\begin{minipage}{\textwidth}
	\begin{tabular}{cccc}%
	\myhspace \hspace{3mm}
			\begin{minipage}{\mpw}%
			\centering%
			\includegraphics[width=\myRate\columnwidth]{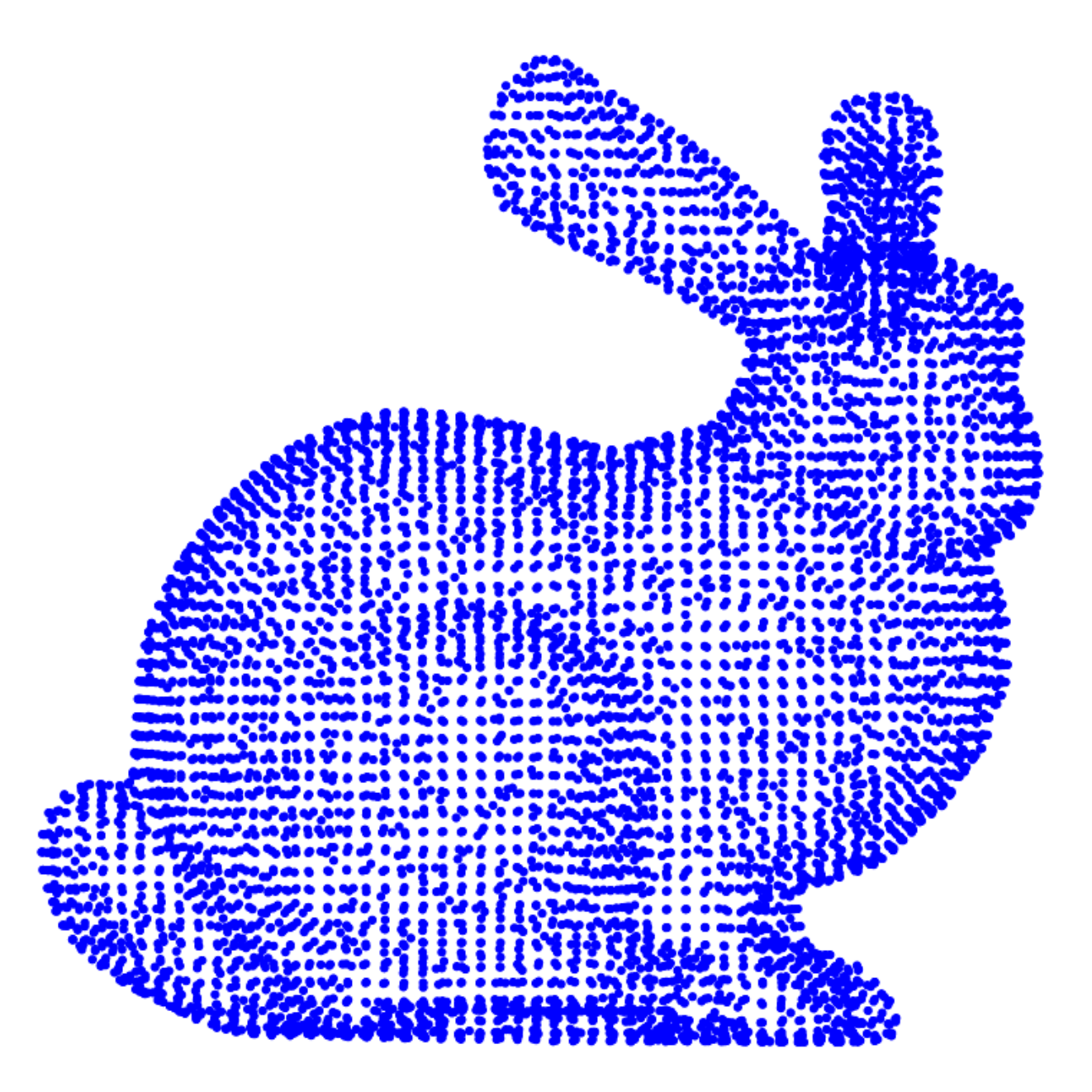} \\
			(a) \bunny model
			\end{minipage}
		& \myhspace 
			\begin{minipage}{\mpw}%
			\centering%
			\includegraphics[width=\myRate\columnwidth]{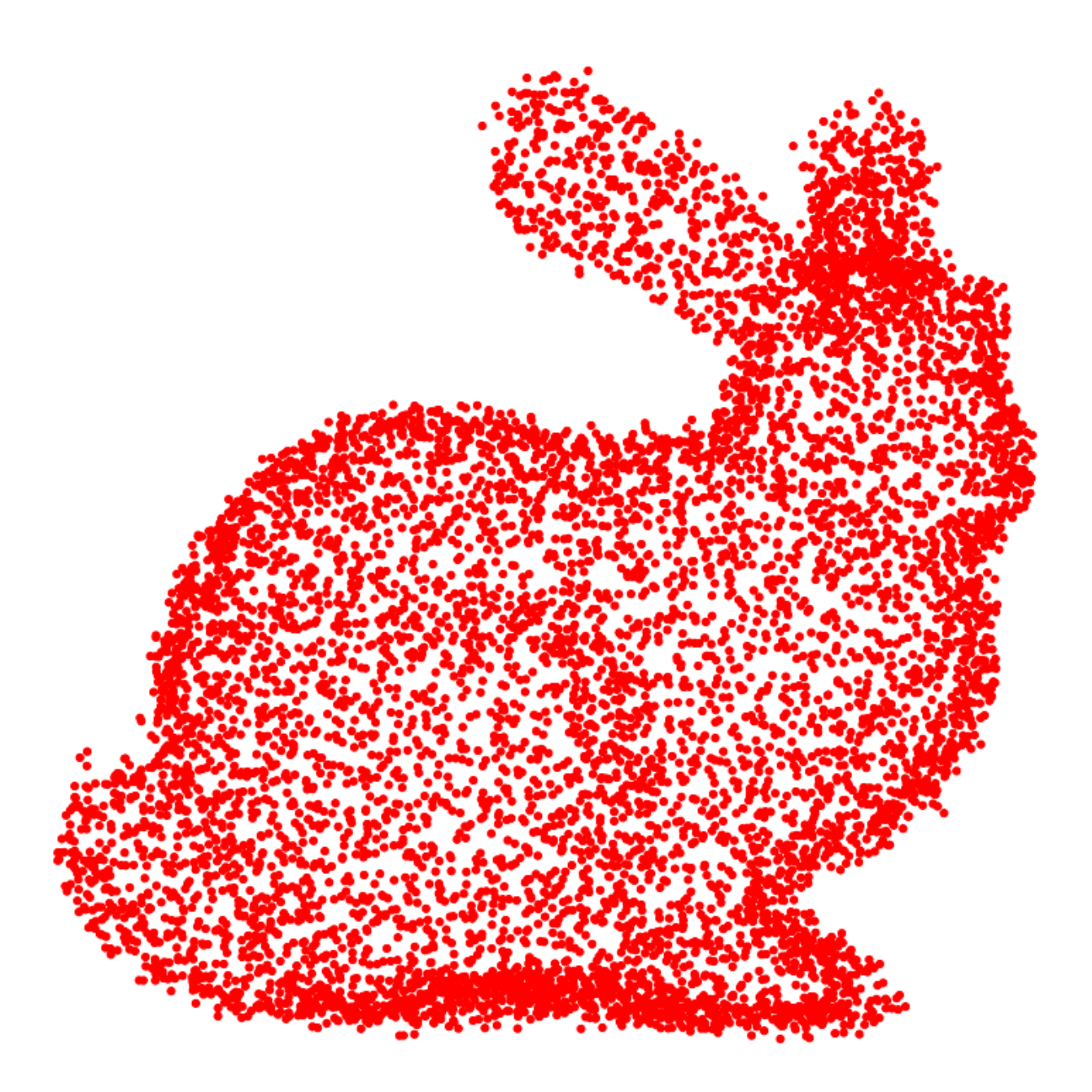} \\
			(b) \bunny scene, $\sigma = 0.01$
			\end{minipage}
		& \myhspace 
			\begin{minipage}{\mpw}%
			\centering%
			\includegraphics[width=\myRate\columnwidth]{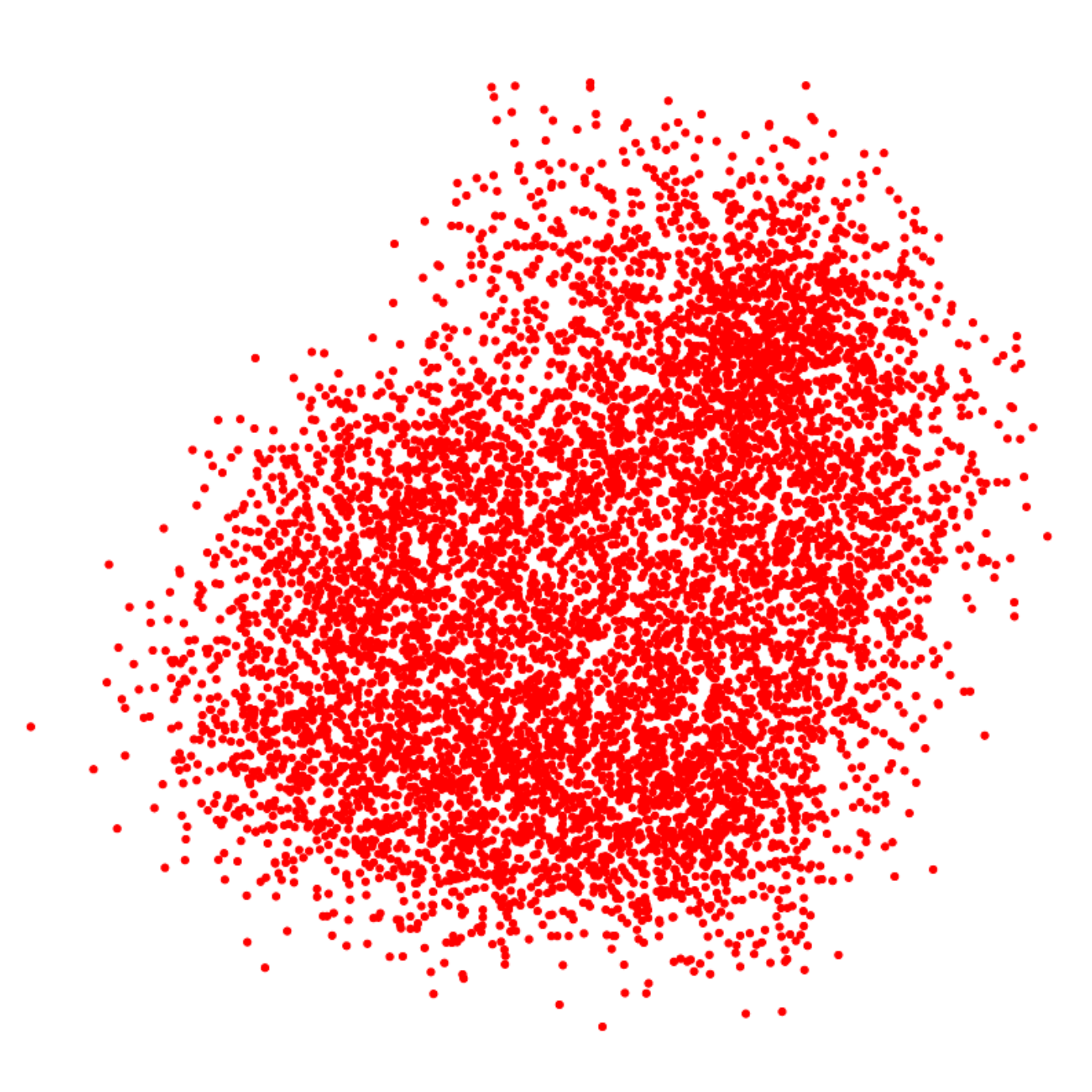} \\
			(c) \bunny scene, $\sigma = 0.1$
			\end{minipage}
		& \myhspace 
			\begin{minipage}{\mpw}%
			\centering%
			\includegraphics[width=\myRate\columnwidth]{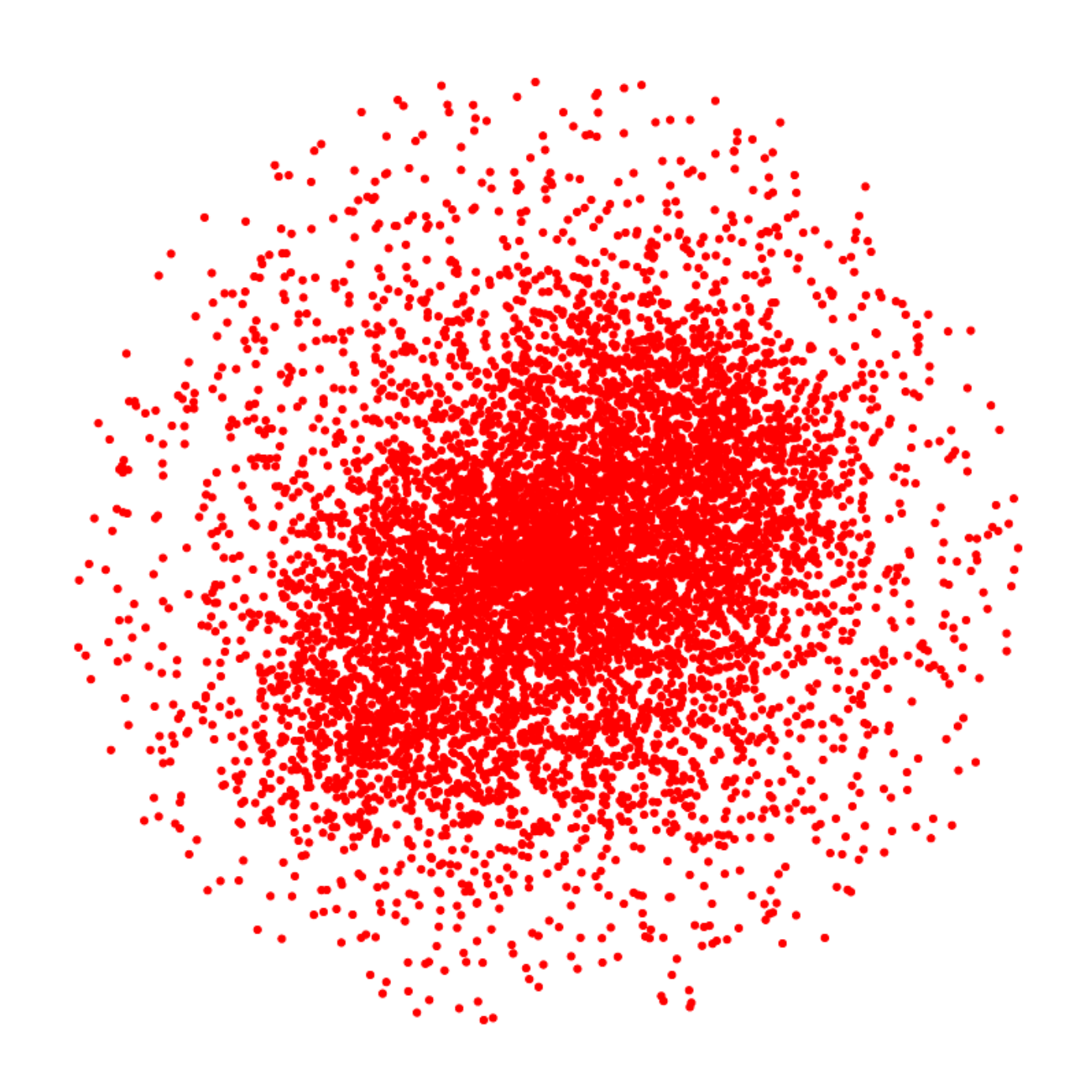} \\
			(d) \bunny scene, $\sigma = 0.1$, \\ 50\% outliers
			\end{minipage}
		\end{tabular}
	\end{minipage}
	\begin{minipage}{\textwidth}
	\end{minipage}
	\vspace{-3mm} 
	\caption{\bunny point cloud scaled inside unit cube $[0,1]^3$ and corrupted by different levels of noise and outliers, all viewed from the same perspective angle. (a) Clean \bunny model point cloud, scaled inside unit cube $[0,1]^3$. (b) \bunny scene, generated from (a) by adding isotropic Gaussian noise with standard deviation $\sigma=0.01$. (c) \bunny scene, generated from (a) by adding isotropic Gaussian noise with $\sigma=0.1$. (d) \bunny scene, generated from (a) by adding isotropic Gaussian noise with $\sigma = 0.1$ and $50\%$ random outliers.
	 \label{fig:bunny_with_noise}}
	\vspace{-6mm} 
	\end{center}
\end{figure}

\subsection{Object Pose Estimation: Extra Results}
\label{sec:objectPoseEstimationSupp}

Fig.~\ref{fig:objectPoseEstimation} shows the registration results for the 8 selected scenes 
from the  University of Washington RGB-D datasets~\cite{Lai11icra-largeRGBD}. The inlier correspondence ratios for \emph{cereal box} are all below 10\% and typically below 5\%. \name is able to compute a highly-accurate estimate of the relative pose using a handful of inliers. 
Due to the distinctive shape of \emph{cap}, FPFH produces a higher number of inliers and \name is able to register the object-scene pair without problems. 


\renewcommand{\mpw}{6cm}
\renewcommand{\myRate}{0.8}

\begin{figure*}[h]
	\begin{center}
	\begin{minipage}{\textwidth}
	\hspace{-0.2cm}
	\begin{tabular}{ccc}%
			\begin{minipage}{\mpw}%
			\centering%
			\includegraphics[width=\myRate\columnwidth]{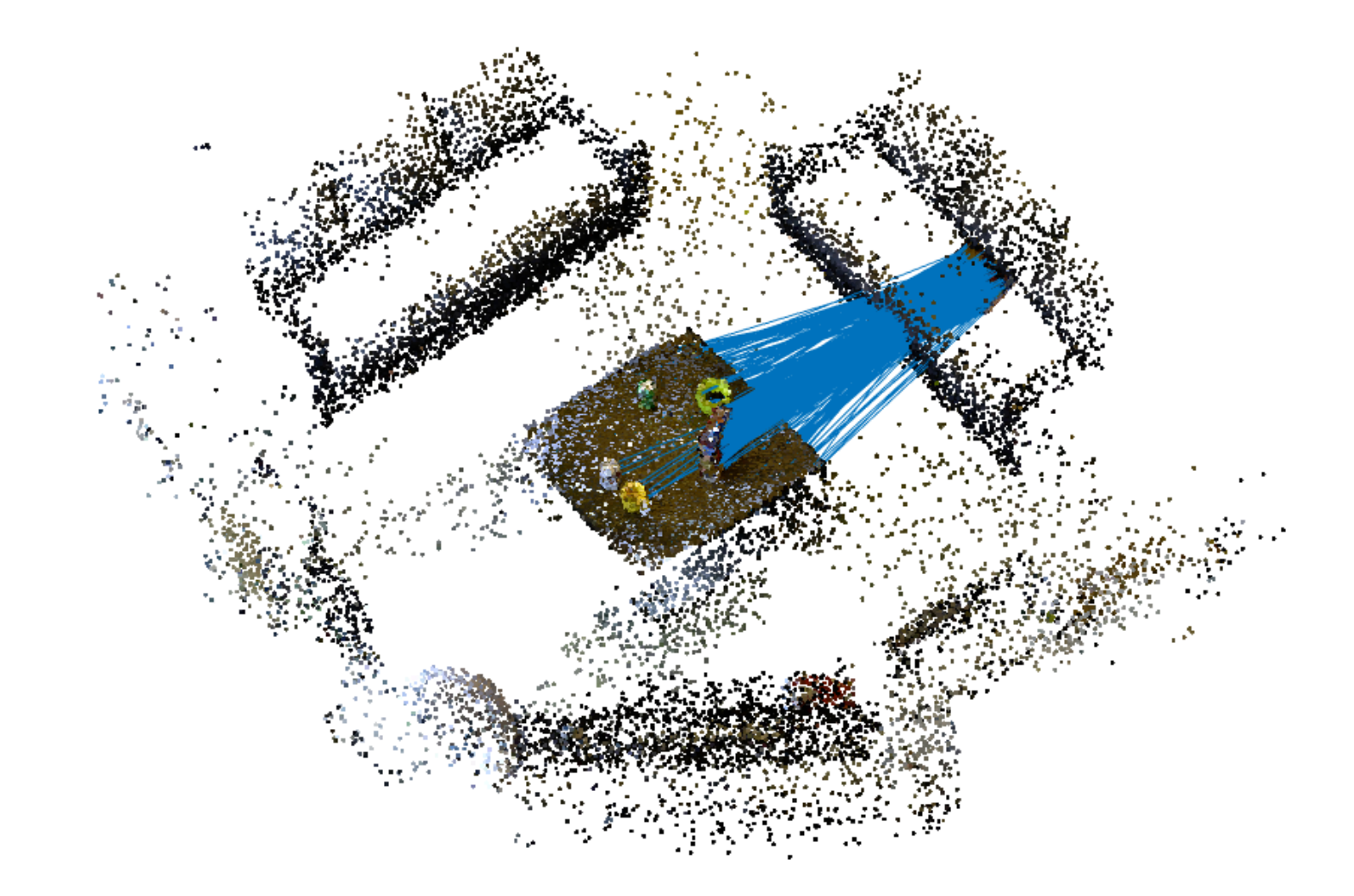} \\
			\end{minipage}
		& \myhspace
			\begin{minipage}{\mpw}%
			\centering%
			\includegraphics[width=\myRate\columnwidth]{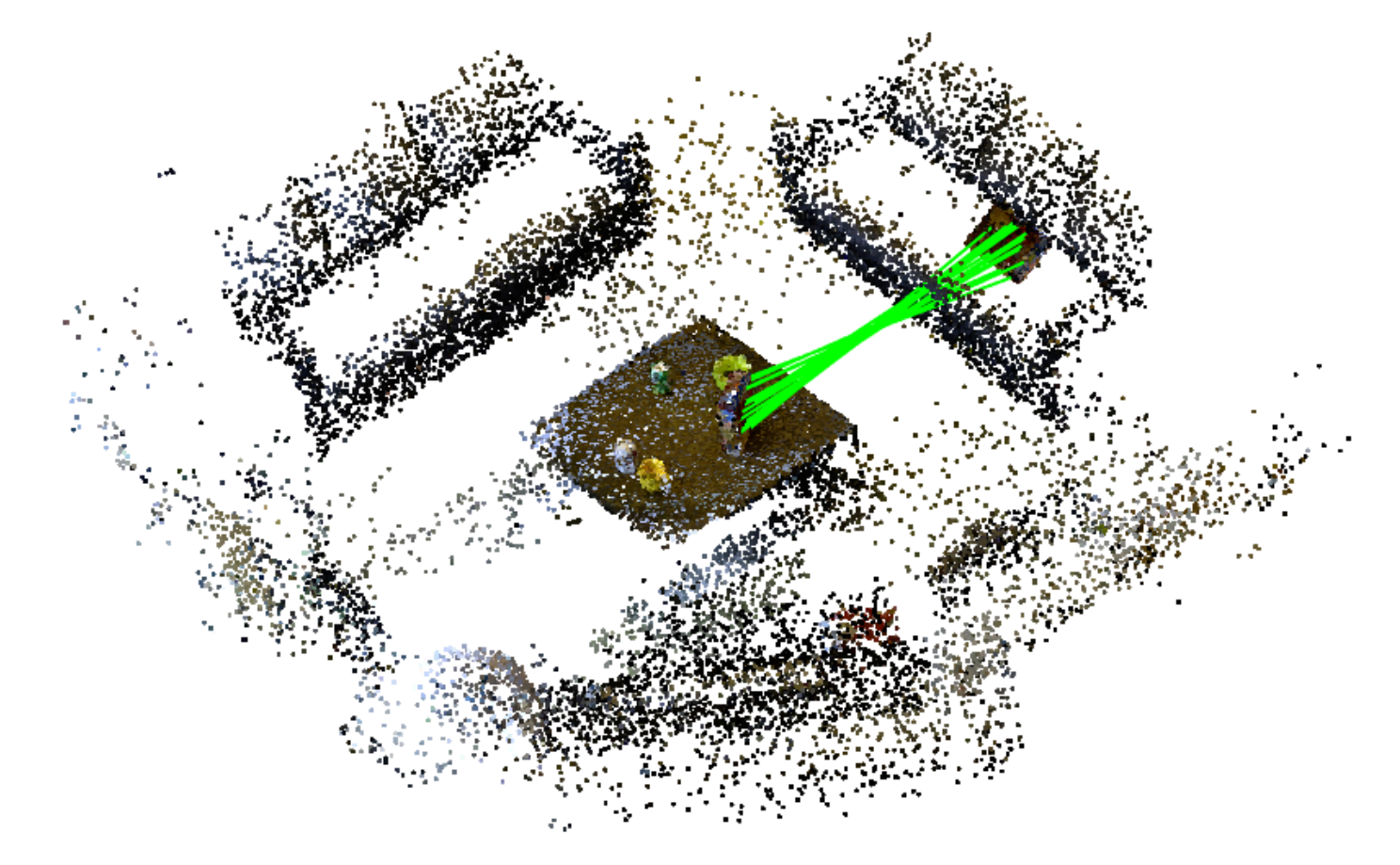} \\
			\end{minipage}
		& \myhspace
			\begin{minipage}{\mpw}%
			\centering%
			\includegraphics[width=\myRate\columnwidth]{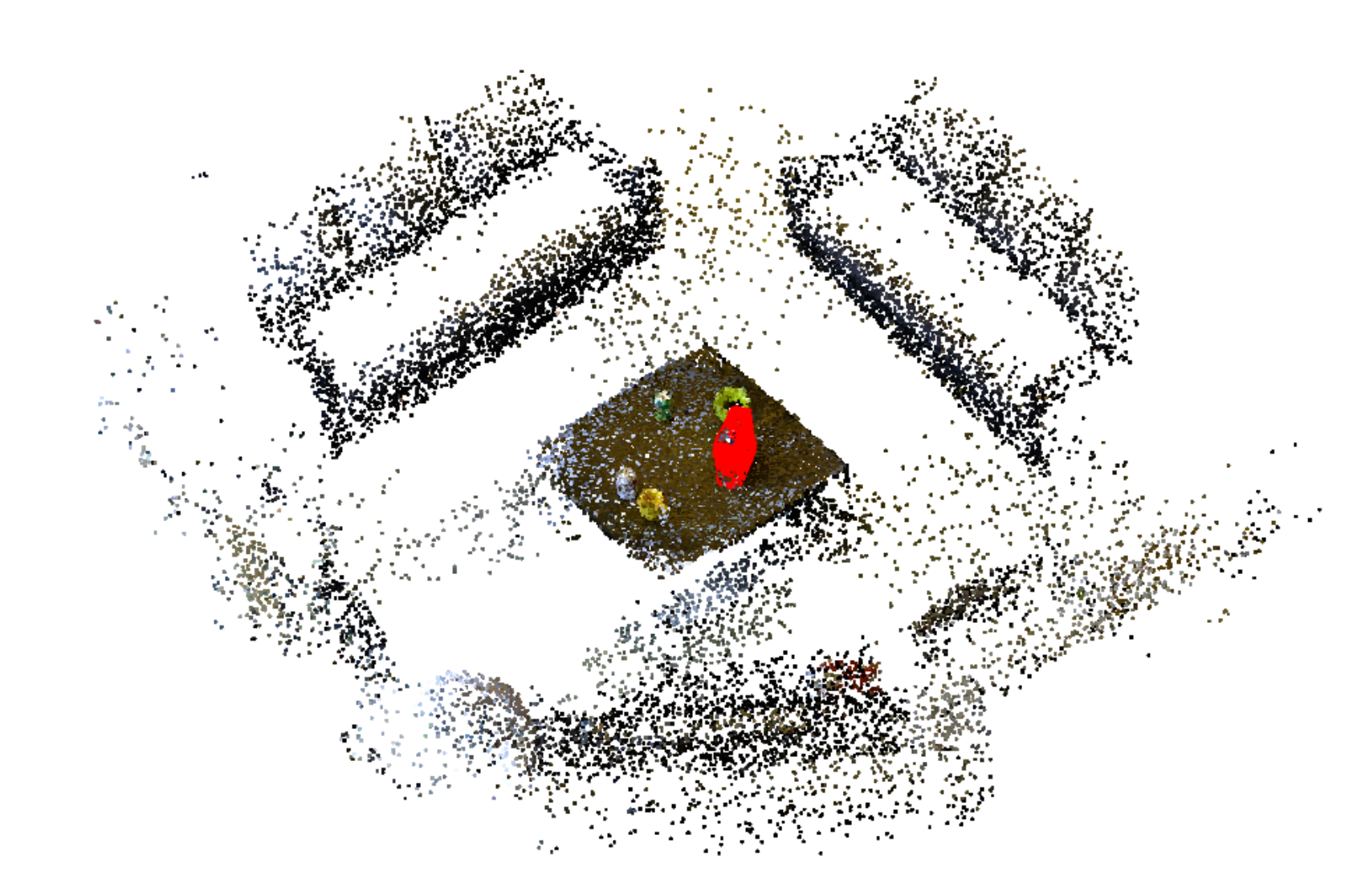} \\
			\end{minipage} \\
		\multicolumn{3}{c}{\emph{scene-2}, \# of FPFH correspondences: 550, Inlier ratio: 4.55\%, Rotation error: 0.120, Translation error: 0.052.}\\
		
		\begin{minipage}{\mpw}%
			\centering%
			\includegraphics[width=\myRate\columnwidth]{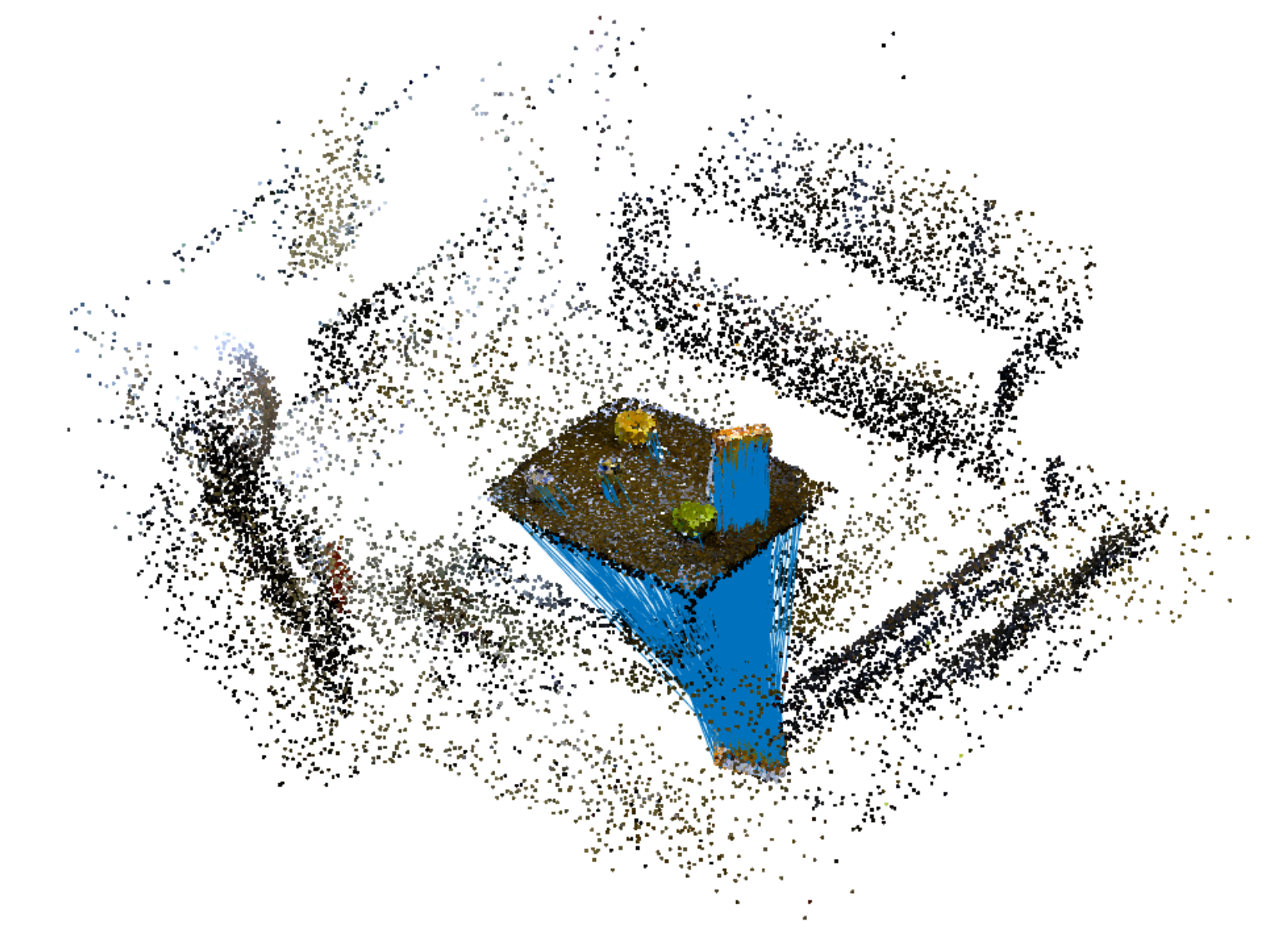} \\
			\end{minipage}
		& \myhspace
			\begin{minipage}{\mpw}%
			\centering%
			\includegraphics[width=\myRate\columnwidth]{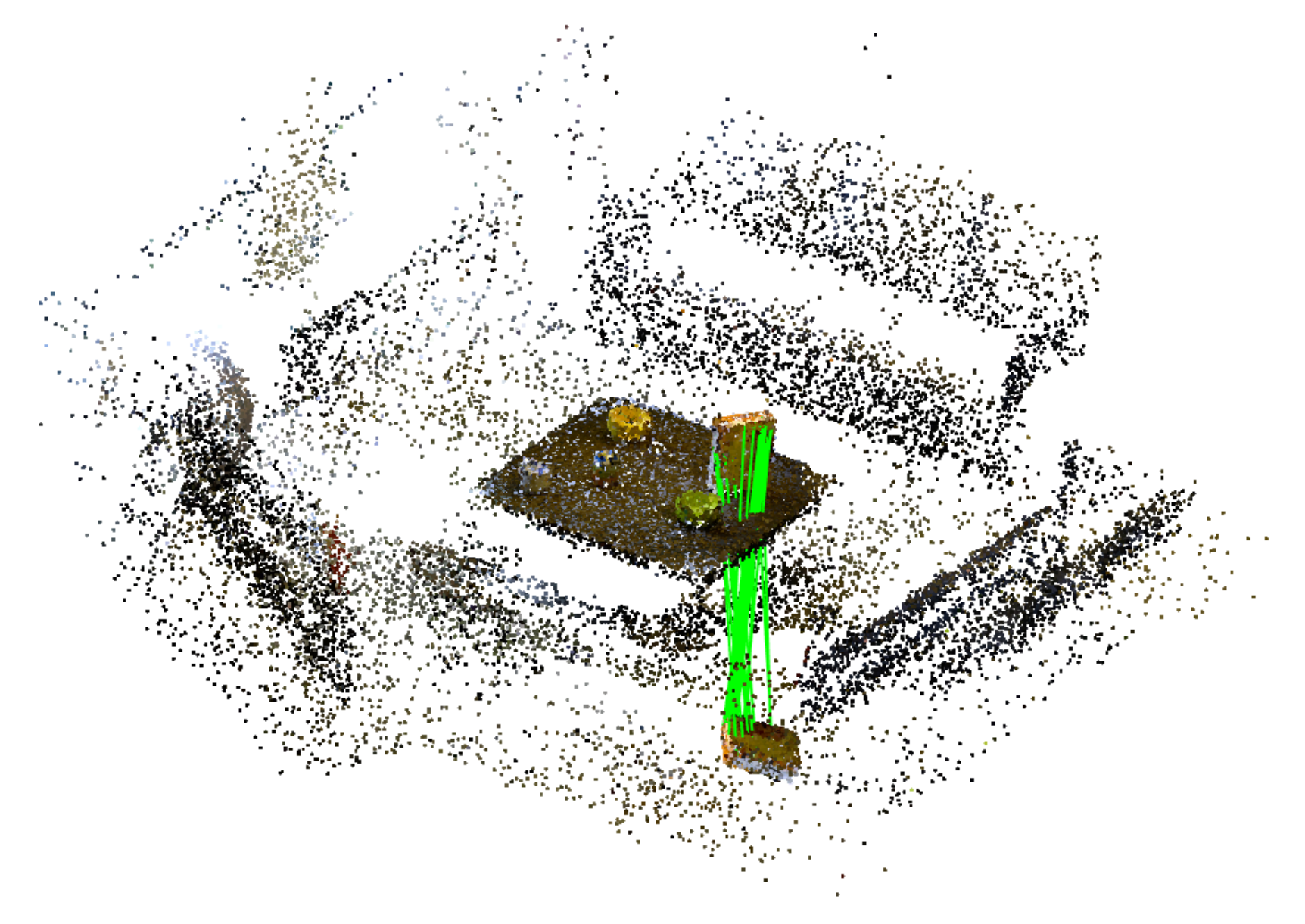} \\
			\end{minipage}
		& \myhspace
			\begin{minipage}{\mpw}%
			\centering%
			\includegraphics[width=\myRate\columnwidth]{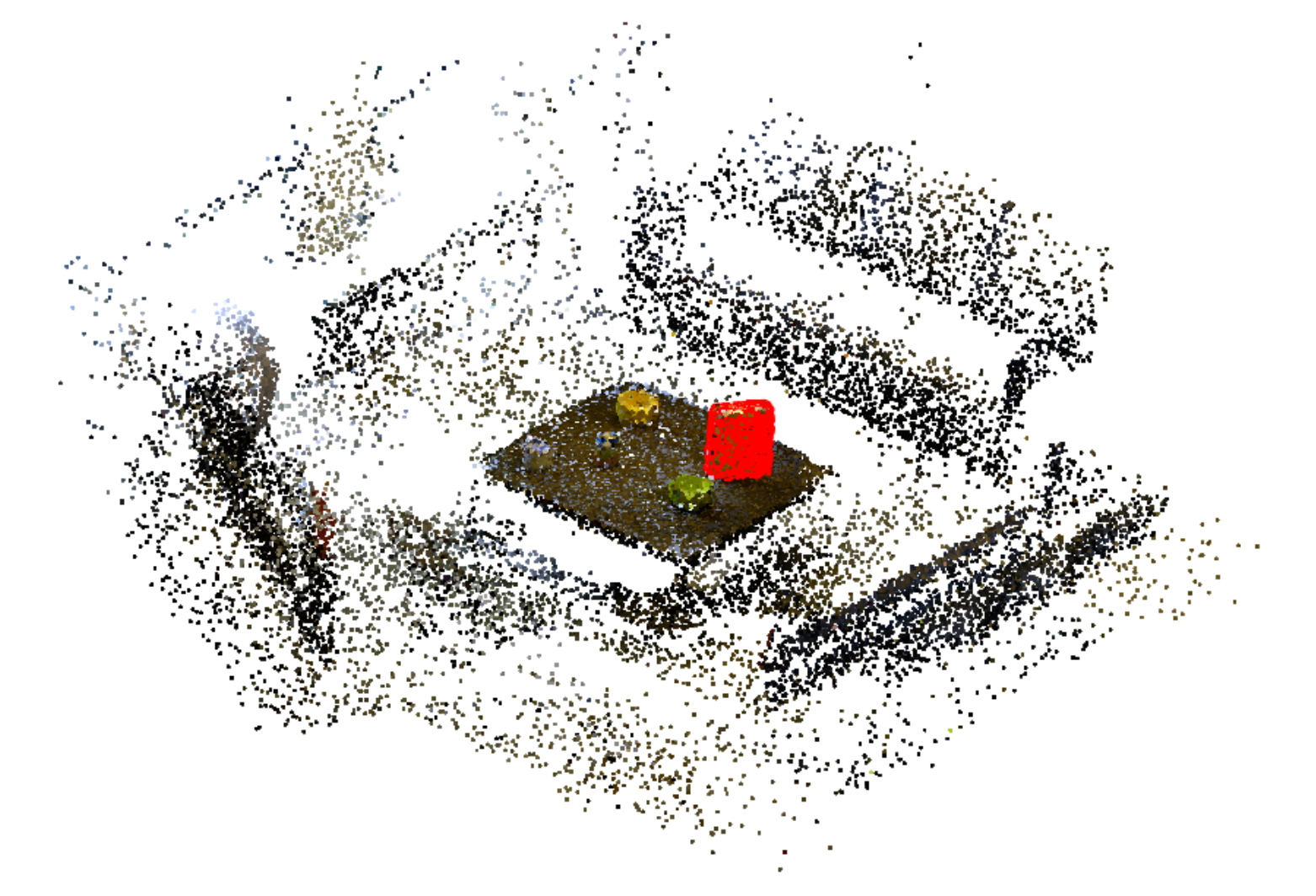} \\
			\end{minipage} \\
\multicolumn{3}{c}{\emph{scene-4}, \# of FPFH correspondences: 636, Inlier ratio: 4.56\%, Rotation error: 0.042, Translation error: 0.051.}\\

		\begin{minipage}{\mpw}%
			\centering%
			\includegraphics[width=\myRate\columnwidth]{scene_3_a_5.pdf} \\
			\end{minipage}
		& \myhspace
			\begin{minipage}{\mpw}%
			\centering%
			\includegraphics[width=\myRate\columnwidth]{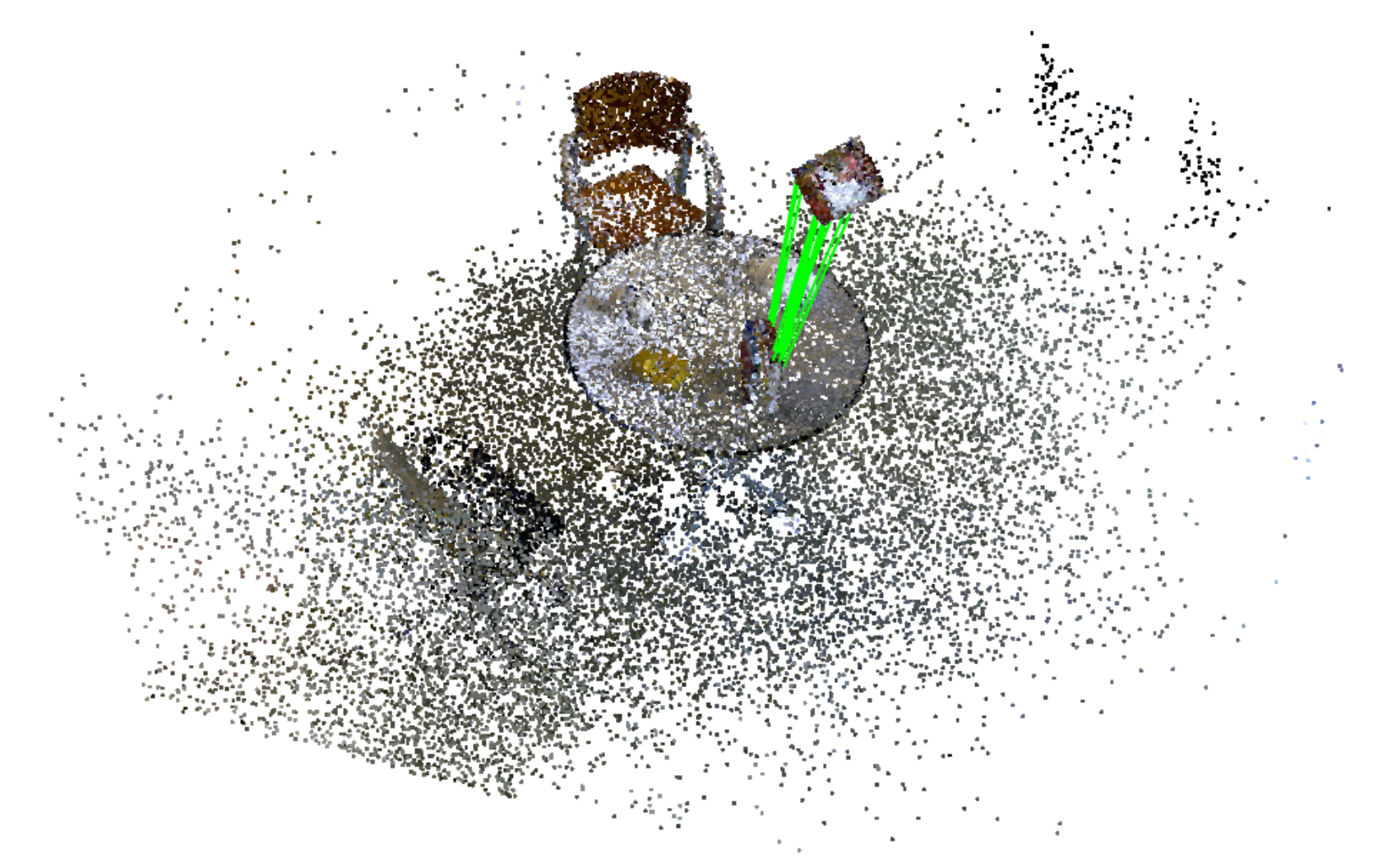} \\
			\end{minipage}
		& \myhspace
			\begin{minipage}{\mpw}%
			\centering%
			\includegraphics[width=\myRate\columnwidth]{scene_3_c_5.pdf} \\
			\end{minipage} \\
		\multicolumn{3}{c}{\emph{scene-5}, \# of FPFH correspondences: 685, Inlier ratio: 2.63\%, Rotation error: 0.146, Translation error: 0.176.}\\
		
		\begin{minipage}{\mpw}%
			\centering%
			\includegraphics[width=\myRate\columnwidth]{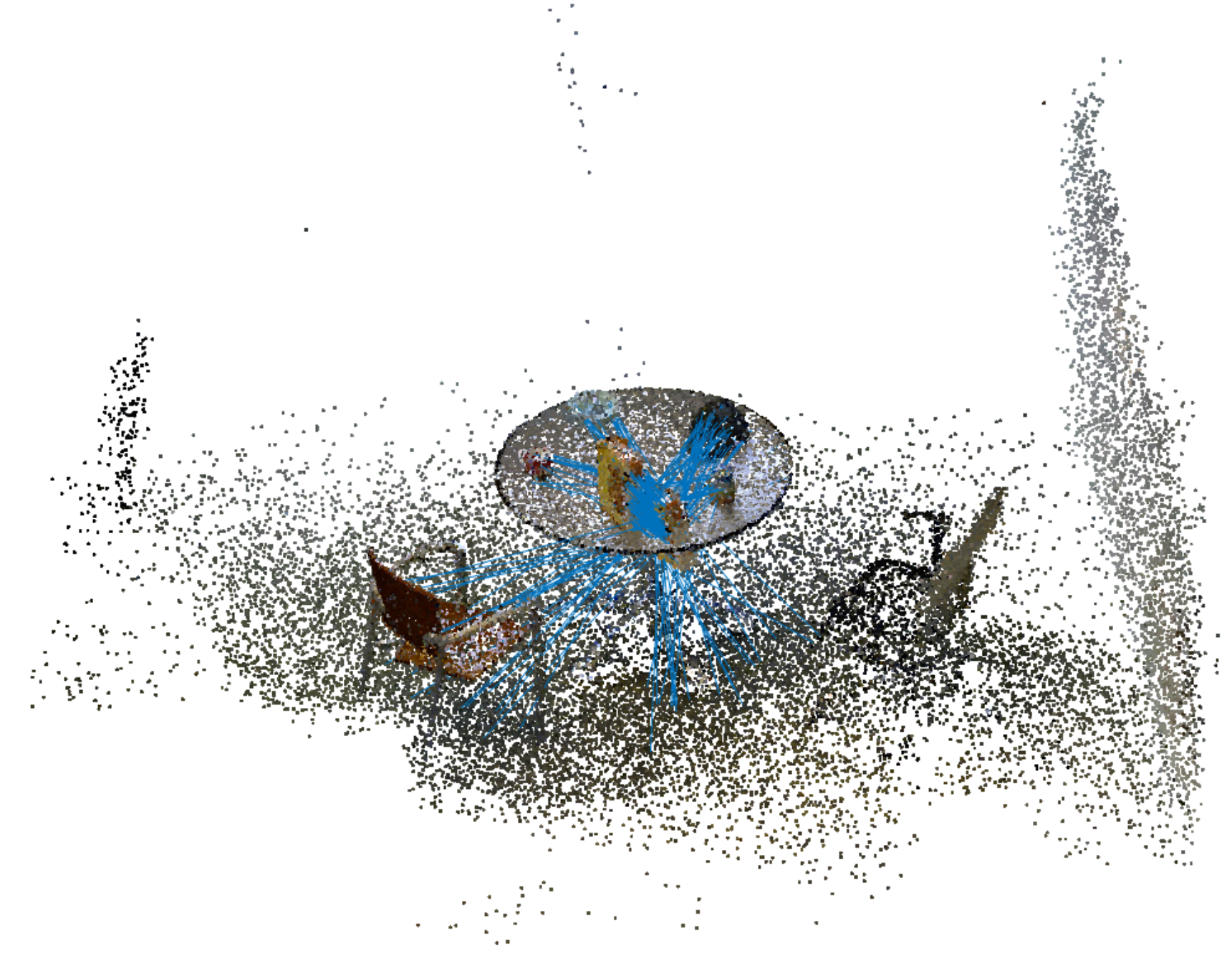} \\
			\end{minipage}
		& \myhspace
			\begin{minipage}{\mpw}%
			\centering%
			\includegraphics[width=\myRate\columnwidth]{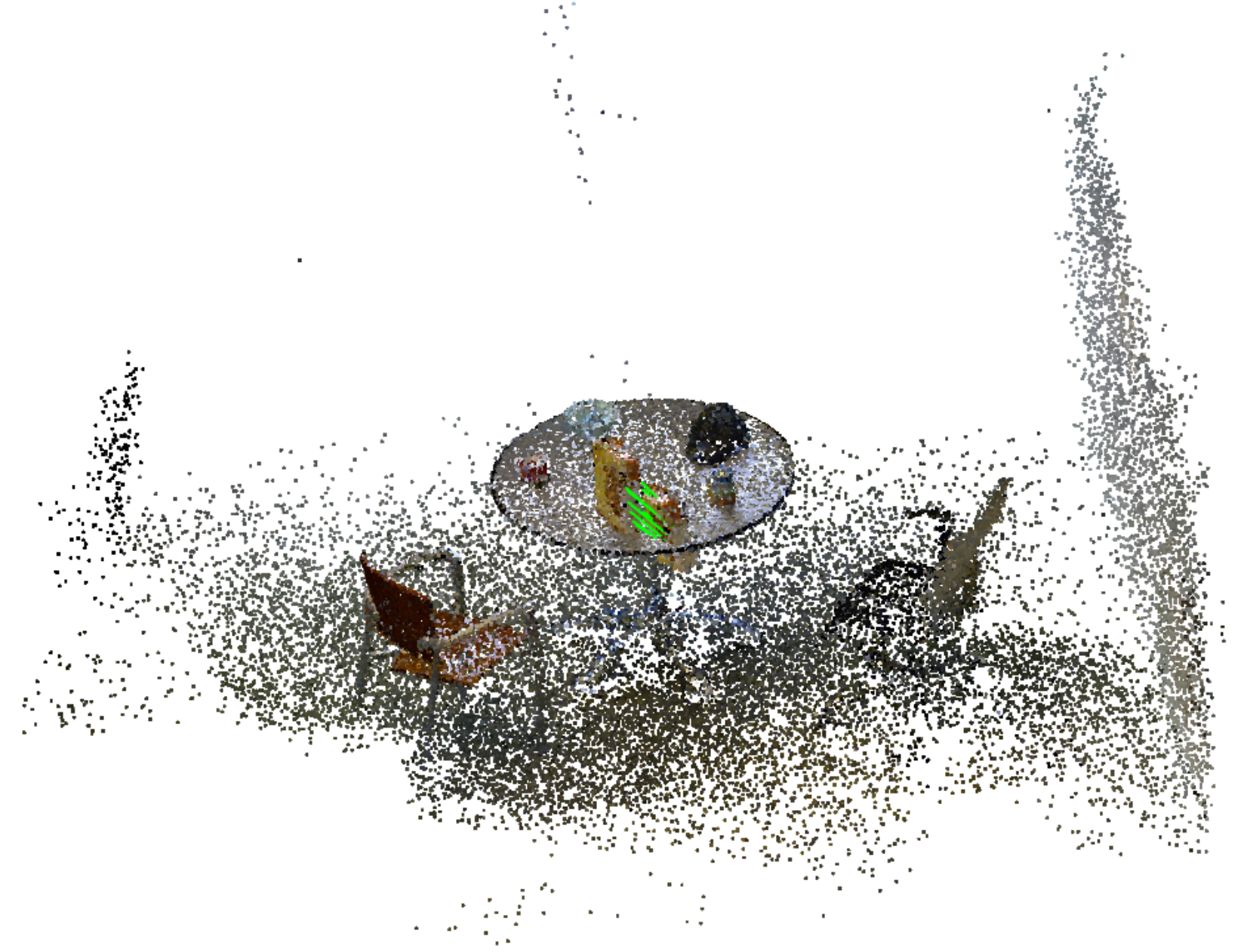} \\
			\end{minipage}
		& \myhspace
			\begin{minipage}{\mpw}%
			\centering%
			\includegraphics[width=\myRate\columnwidth]{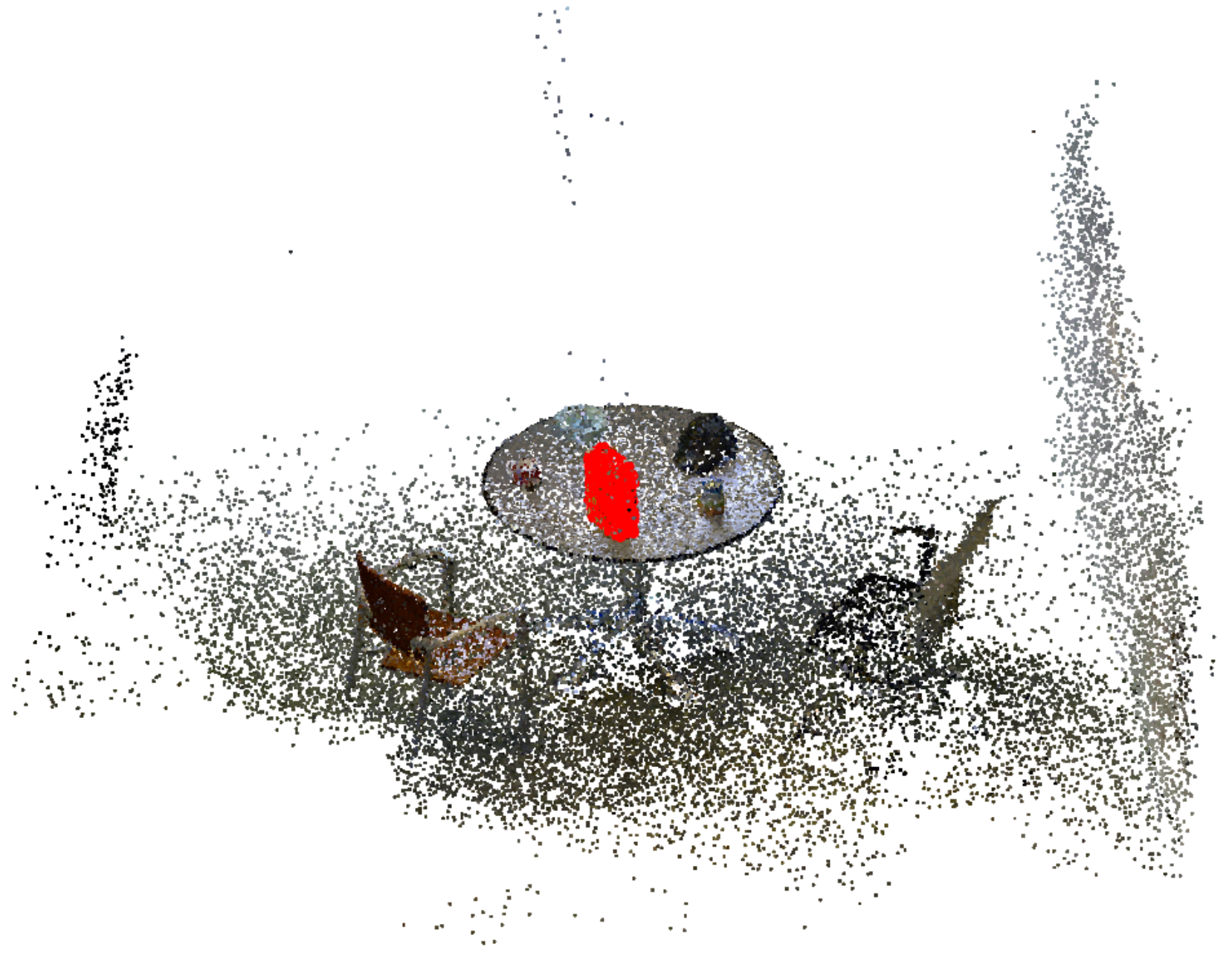} \\
			\end{minipage} \\
		\multicolumn{3}{c}{\emph{scene-7}, \# of FPFH correspondences: 416, Inlier ratio: 3.13\%, Rotation error: 0.058, Translation error: 0.097.} 
		
		\end{tabular}
	\end{minipage}
	\vspace{-3mm} 
	\caption{Object pose estimation on the large-scale RGB-D dataset~\cite{Lai11icra-largeRGBD}. First column: FPFH correspondences, second column: inlier correspondences after \name, third column: registration result with the registered object highlighted in red. Scene number and related registration information are listed below each scene.}
	\vspace{-8mm} 
	\end{center}
\end{figure*}

\clearpage

\begin{figure*}[h]\ContinuedFloat
	\begin{center}
	\begin{minipage}{\textwidth}
	\hspace{-0.2cm}
	\begin{tabular}{ccc}%
		\begin{minipage}{\mpw}%
			\centering%
			\includegraphics[width=1.0\columnwidth]{scene_5_a_9.pdf} \\
			\end{minipage}
		& \myhspace
			\begin{minipage}{\mpw}%
			\centering%
			\includegraphics[width=1.0\columnwidth]{scene_5_b_9.pdf} \\
			\end{minipage}
		& \myhspace
			\begin{minipage}{\mpw}%
			\centering%
			\includegraphics[width=1.0\columnwidth]{scene_5_c_9.pdf} \\
			\end{minipage} \\
	\multicolumn{3}{c}{\emph{scene-9}, \# of FPFH correspondences: 651, Inlier ratio: 8.29\%, Rotation error: 0.036, Translation error: 0.011.} \\
	
		\begin{minipage}{\mpw}%
			\centering%
			\includegraphics[width=1.0\columnwidth]{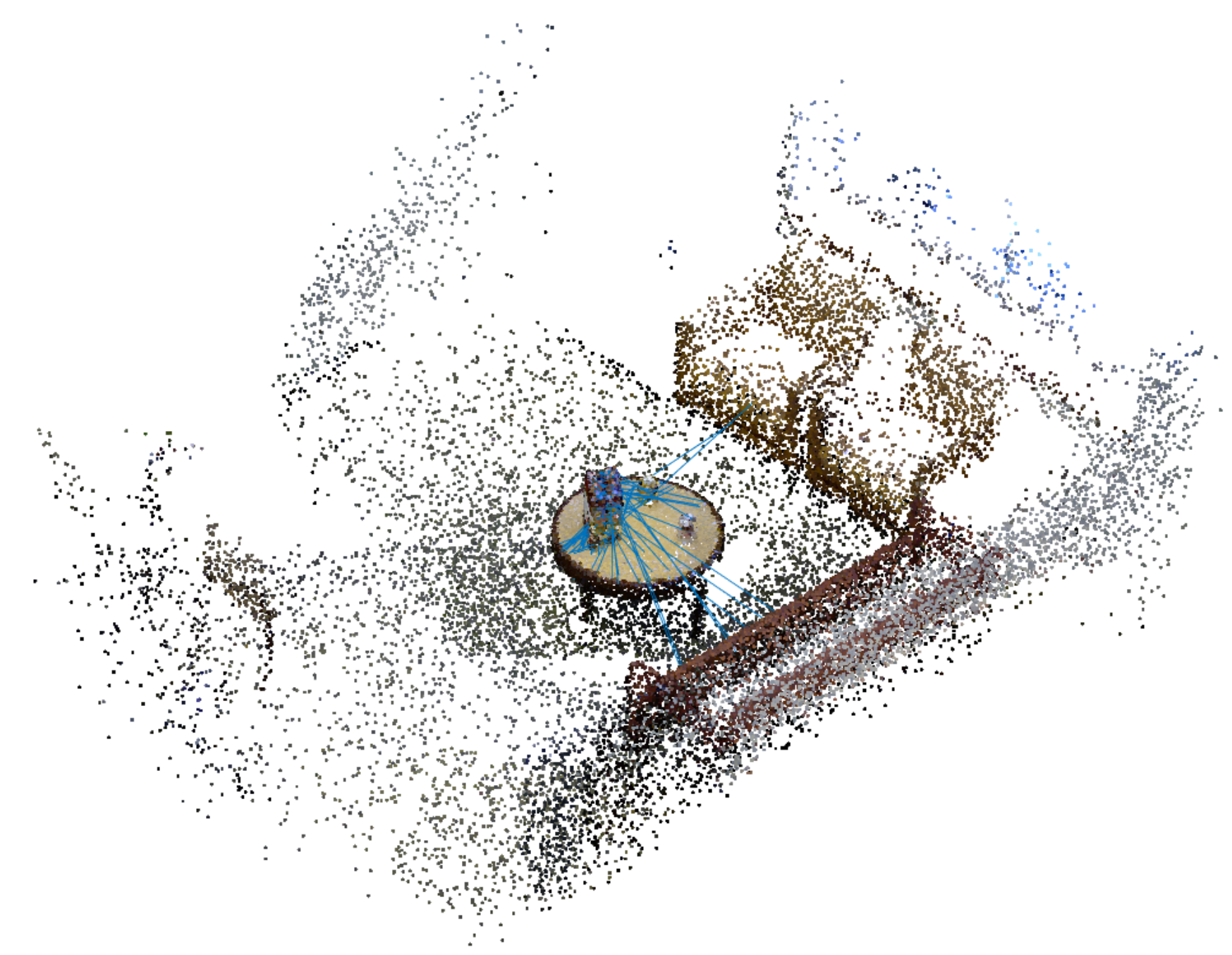} \\
			\end{minipage}
		& \myhspace
			\begin{minipage}{\mpw}%
			\centering%
			\includegraphics[width=1.0\columnwidth]{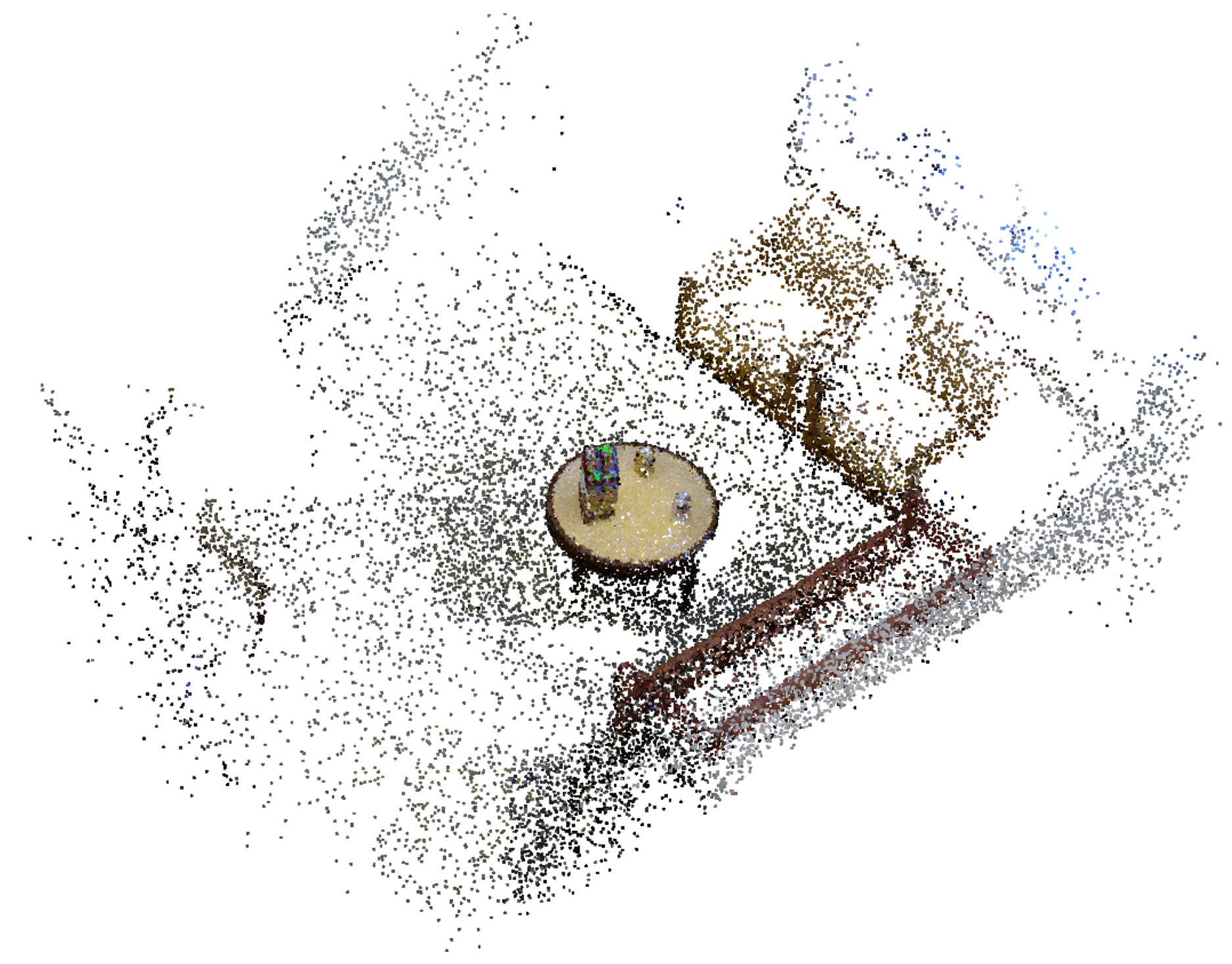} \\
			\end{minipage}
		& \myhspace
			\begin{minipage}{\mpw}%
			\centering%
			\includegraphics[width=1.0\columnwidth]{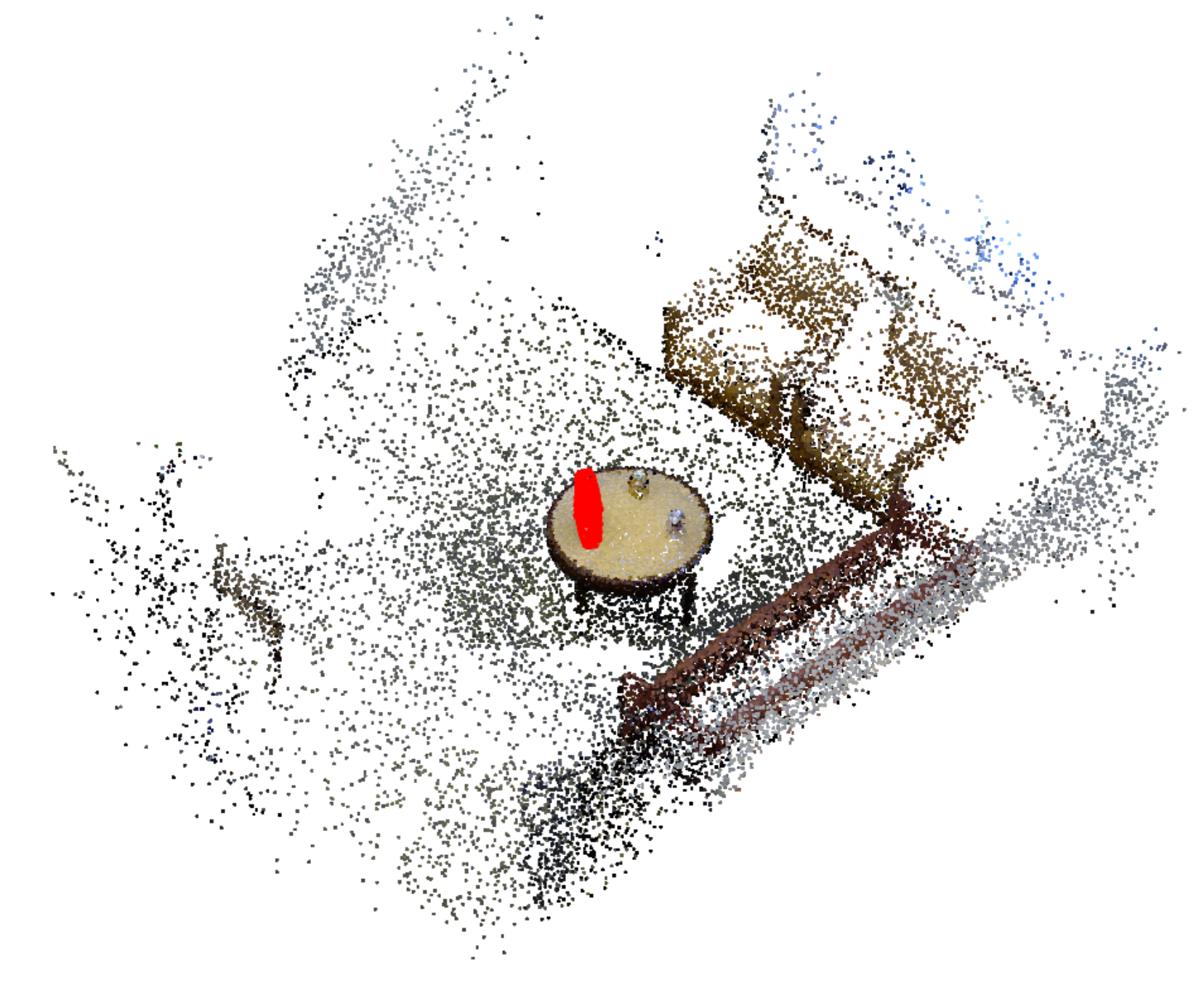} \\
			\end{minipage} \\
		\multicolumn{3}{c}{\emph{scene-11}, \# of FPFH correspondences: 445, Inlier ratio: 6.97\%, Rotation error: 0.028, Translation error: 0.016.} \\
		
		\begin{minipage}{\mpw}%
			\centering%
			\includegraphics[width=1.0\columnwidth]{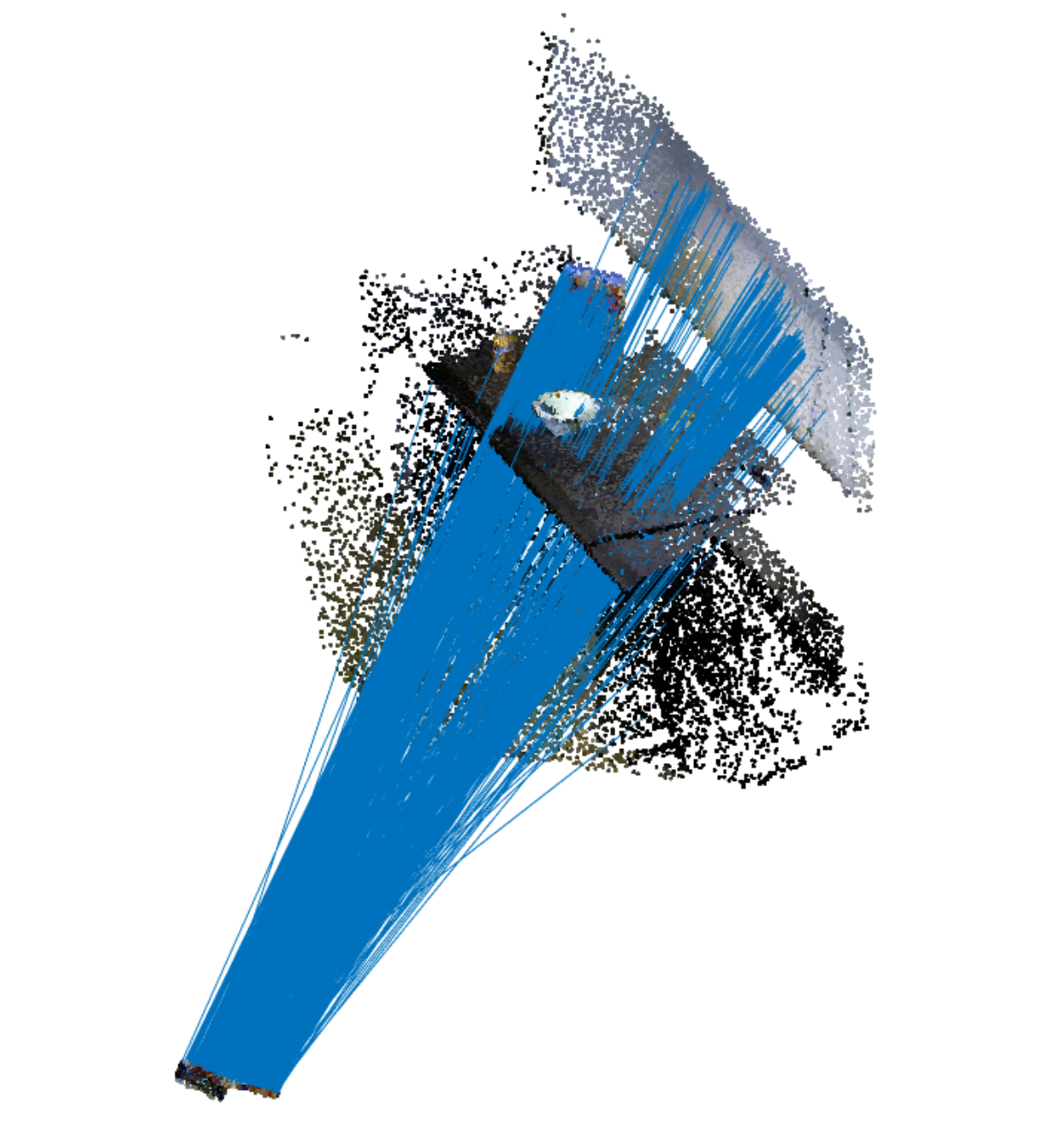} \\
			\end{minipage}
		& \myhspace
			\begin{minipage}{\mpw}%
			\centering%
			\includegraphics[width=1.0\columnwidth]{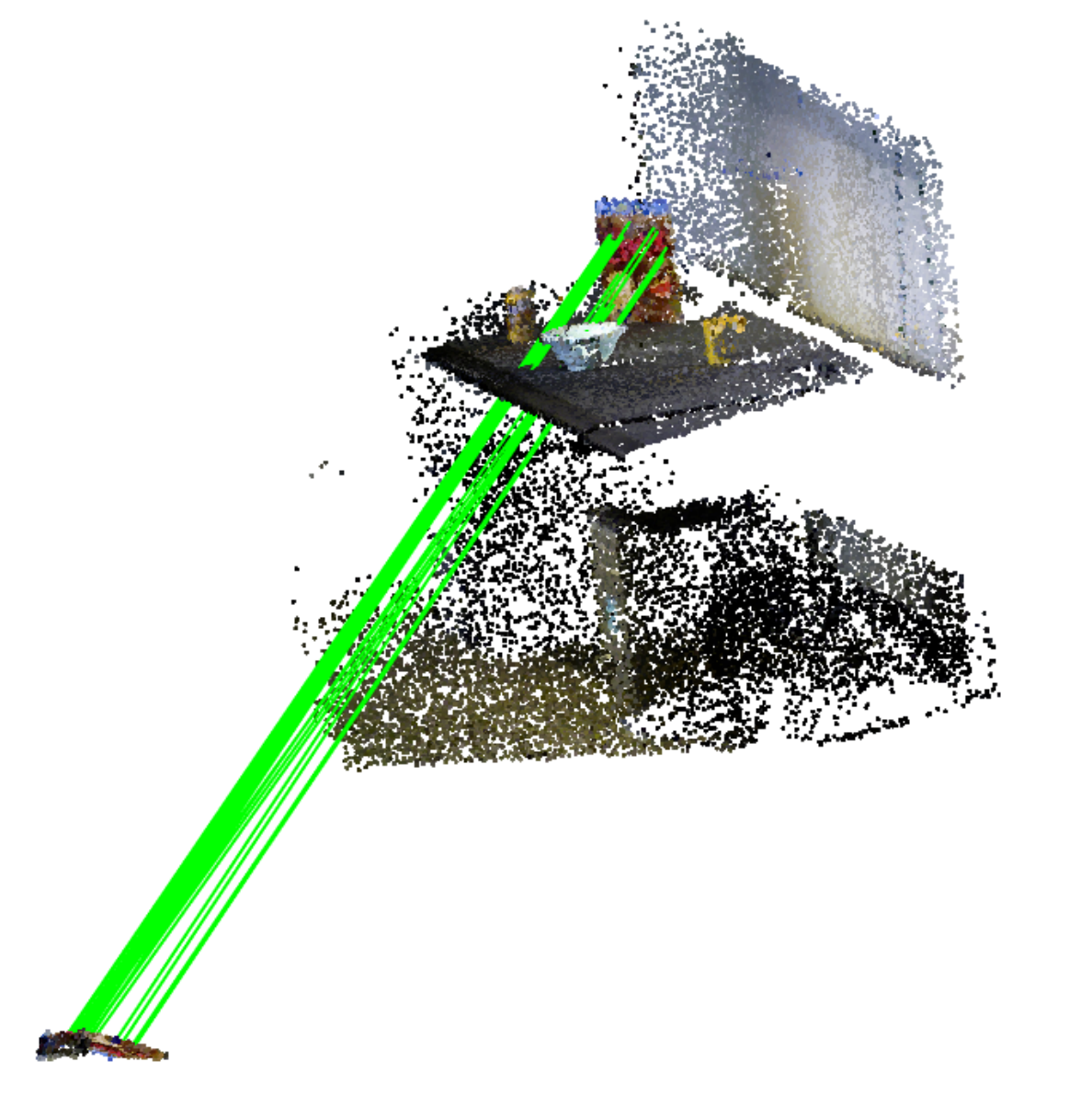} \\
			\end{minipage}
		& \myhspace
			\begin{minipage}{\mpw}%
			\centering%
			\includegraphics[width=1.0\columnwidth]{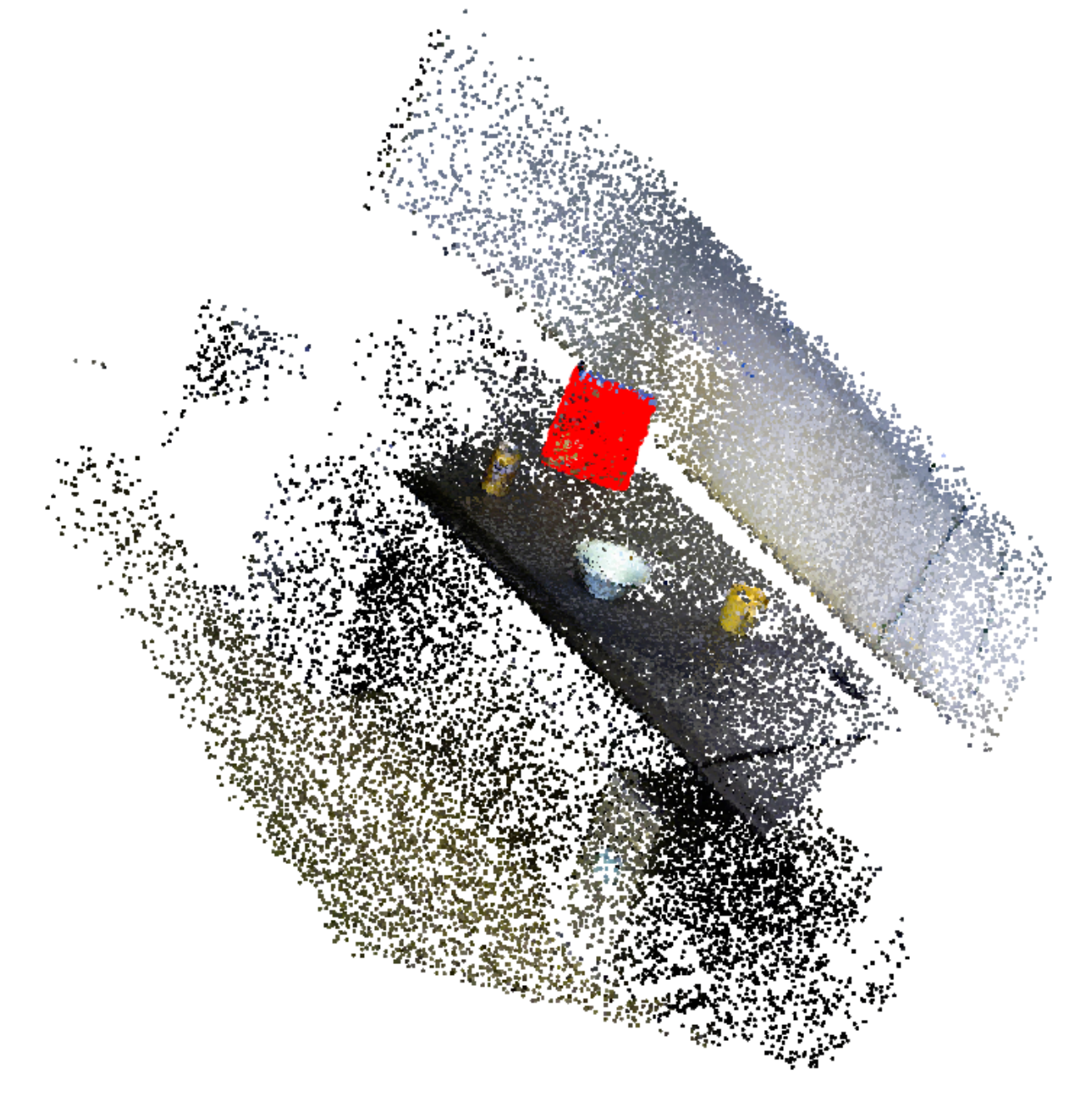} \\
			\end{minipage}\\
		\multicolumn{3}{c}{\emph{scene-13}, \# of FPFH correspondences: 612, Inlier ratio: 5.23\%, Rotation error: 0.036, Translation error: 0.064.} \\
		
		\begin{minipage}{\mpw}%
			\centering%
			\includegraphics[width=1.0\columnwidth]{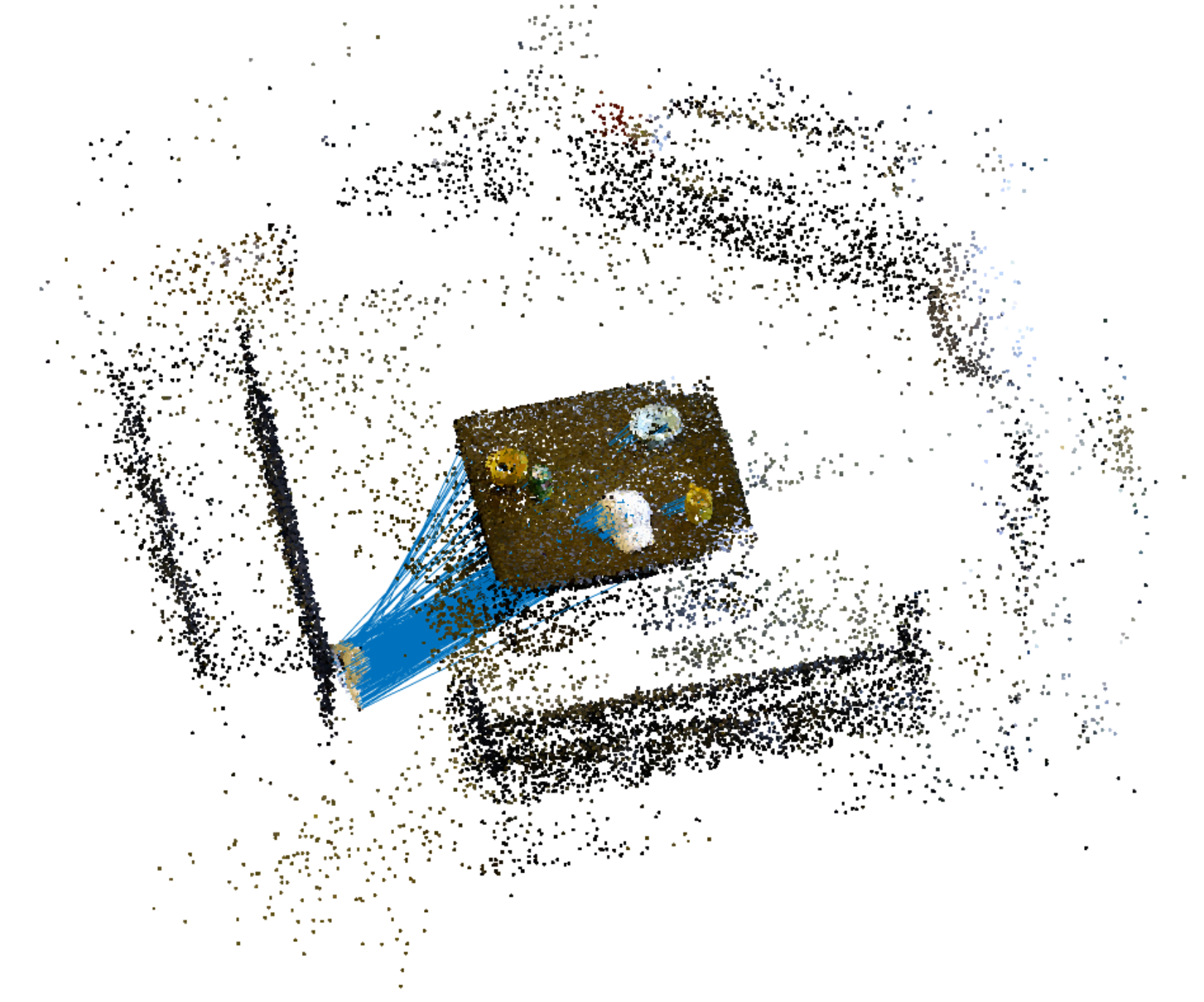} \\
			\end{minipage}
		& \myhspace
			\begin{minipage}{\mpw}%
			\centering%
			\includegraphics[width=1.0\columnwidth]{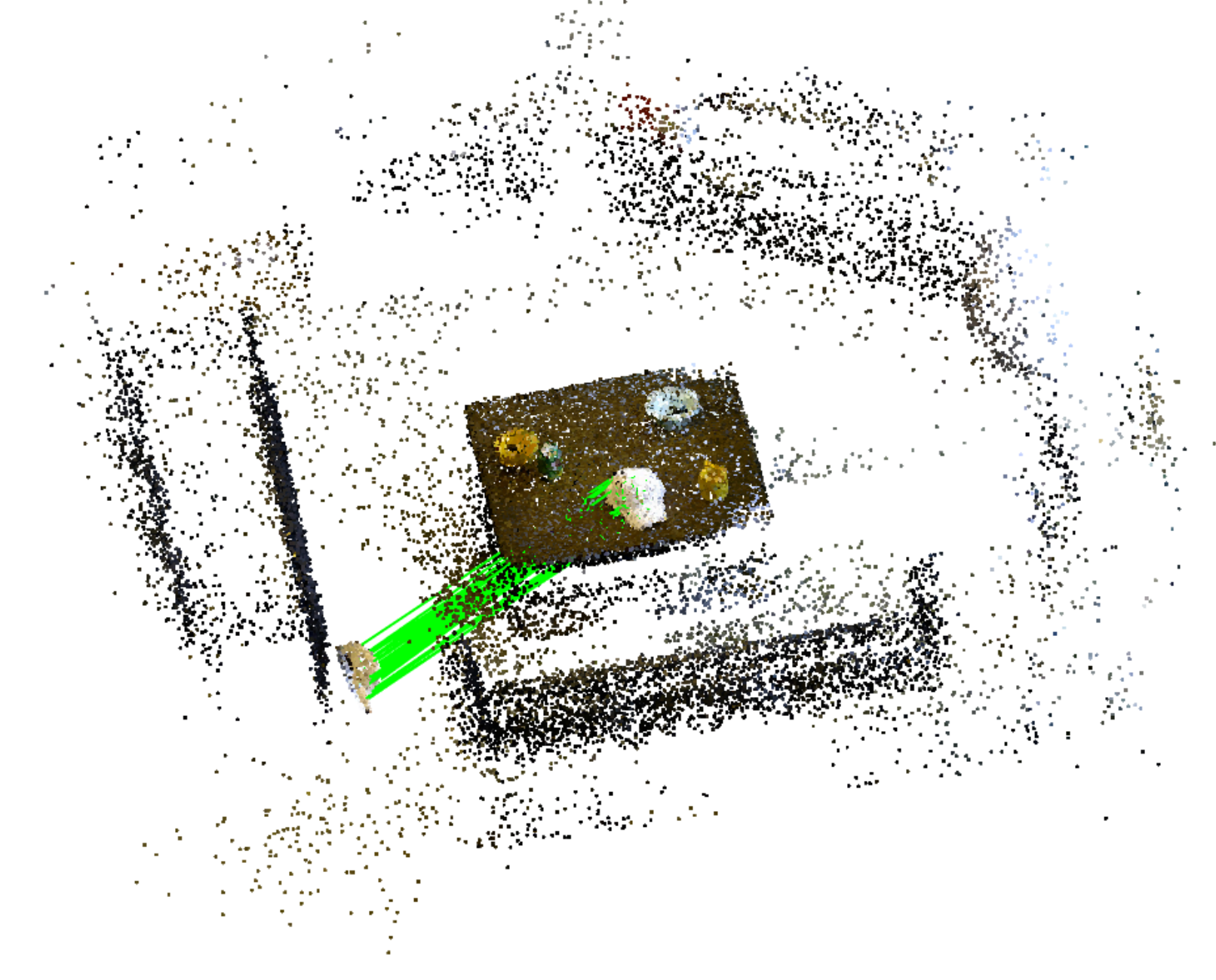} \\
			\end{minipage}
		& \myhspace
			\begin{minipage}{\mpw}%
			\centering%
			\includegraphics[width=1.0\columnwidth]{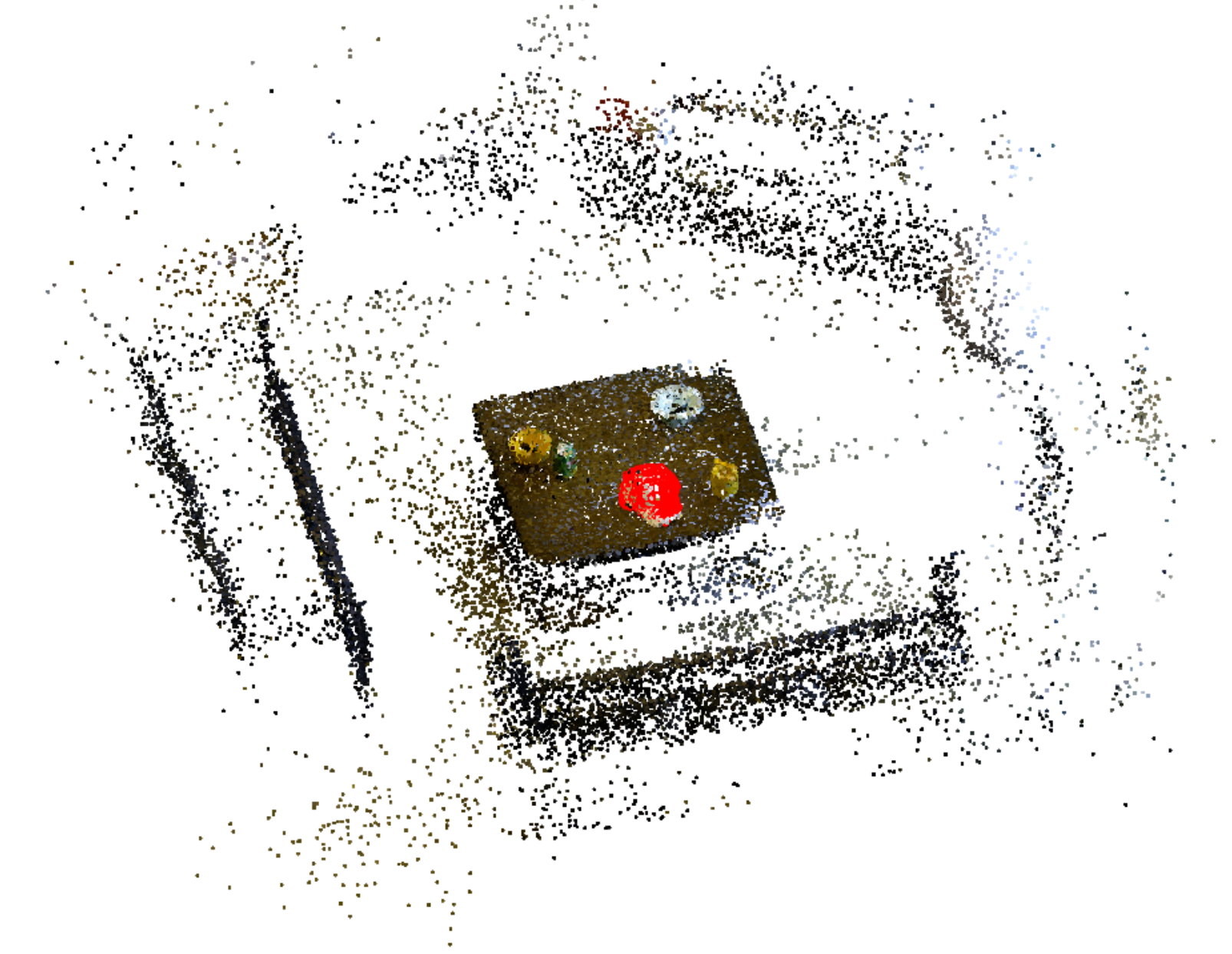} \\
			\end{minipage}\\
		\multicolumn{3}{c}{\emph{scene-1}, \# of FPFH correspondences: 207, Inlier ratio: 16.91\%, Rotation error: 0.066, Translation error: 0.090.} 
		
		\end{tabular}
	\end{minipage}
	\vspace{-3mm} 
	\caption{Object pose estimation on the large-scale RGB-D dataset~\cite{Lai11icra-largeRGBD} (cont.).
	 \label{fig:objectPoseEstimation}}
	\vspace{-8mm} 
	\end{center}
\end{figure*}

\clearpage

\subsection{Proof of Proposition~\ref{prop:proj2barcalM}: Projection onto $\affineSub$}
\label{sec:app-proof-prop-projectionAffineSpace}
\begin{proof}
We first analyze the constraints defined by the affine subspace $\affineSub$~\eqref{eq:defbarcalM}. 
For any matrix $\barMM \in \affineSub$, $\barMM$ can be written as $\barMM = \barMQ - \hatmu\MJ + \MDelta$, with $\MDelta \in \calH$ and $\barMM$ satisfies $\barMM\barvxx = \MZero$, we write this in matrix form:
\bea \label{eq:barcalMExpansion}
\bmat{cccccc}
-\hatmu\eye_4 - \sum_{k=1}^K [\MDelta]_{kk} & [\barMQ]_{01} + [\MDelta]_{01} & \cdots & [\barMQ]_{0k}+[\MDelta]_{0k} & \cdots & [\barMQ]_{0K} + [\MDelta]_{0K}  \\

[\barMQ]_{01} - [\MDelta]_{01} & [\barMQ]_{11} + [\MDelta]_{11} & \cdots & [\MDelta]_{1k} & \cdots & [\MDelta]_{1K} \\
\vdots & \vdots & \ddots & \vdots & \ddots & \vdots \\

[\barMQ]_{0k} - [\MDelta]_{0k} & -[\MDelta]_{1k} & \cdots & [\barMQ]_{kk} + [\MDelta]_{kk} & \cdots & [\MDelta]_{kK} \\
\vdots & \vdots & \ddots & \vdots & \ddots & \vdots \\

[\barMQ]_{0K} - [\MDelta]_{0K} & -[\MDelta]_{1K} & \cdots & -[\MDelta]_{kK} & \cdots & [\barMQ]_{KK} + [\MDelta]_{KK} 
\emat
\bmat{c}
\ve \\ \hattheta_1\ve \\ \vdots \\ \hattheta_k\ve \\ \vdots \\ \hattheta_K \ve 
\emat  =
\bmat{c}
\MZero \\ \MZero \\ \vdots \\ \MZero \\ \vdots \\ \MZero 
\emat,
\eea
where the expressions for $[\barMQ]_{0k}$ and $[\barMQ]_{kk}$ are given by the following:
\bea
\begin{cases}
[\barMQ]_{0k} = \bmat{cc} -\vSkew{\TIMa_k}^2 + \quarter (\| \hatxi_k \|^2 - \barcsq)\eye_3 + \half \TIMa_k\tran \hatxi_k \eye_3 -\half \vSkew{\hatxi_k}\vSkew{\TIMa_k} - \half \hatxi_k \TIMa_k\tran & \half \vSkew{\hatxi_k}\TIMa_k \\ \half (\vSkew{\hatxi_k}\TIMa_k)\tran & \quarter (\| \hatxi_k \|^2 - \barcsq) \emat \\
[\barMQ]_{kk} = \bmat{cc} -2\vSkew{\TIMa_k}^2 + \half (\| \hatxi_k \|^2 + \barcsq)\eye_3 + \TIMa_k\tran \hatxi_k \eye_3 - \vSkew{\hatxi_k}\vSkew{\TIMa_k} - \hatxi_k\TIMa_k\tran & \vSkew{\hatxi_k}\TIMa_k \\ (\vSkew{\hatxi_k}\TIMa_k)\tran & \half(\| \hatxi_k \|^2 + \barcsq)\emat
\end{cases}, \forall k=1,\dots,K.
\eea
Denote $\hattheta_0 = +1$, we write the $k$-th ($k=1,\dots,K$) block-row equality of~\eqref{eq:barcalMExpansion}:
\bea
\left( \hattheta_0 [\barMQ]_{0k} + \hattheta_k [\barMQ]_{kk} + \hattheta_k [\MDelta]_{kk} + \sum_{i=0,i\neq k}^K \hattheta_i [\MDelta]_{ki} \right) \ve = \MZero \\
\Leftrightarrow 
\bmat{cc}
\star  & \hattheta_0[\barMQ]_{0k}^v + \hattheta_k [\barMQ]_{kk}^v + \hattheta_k [\MDelta]_{kk}^v + \sum_{i=0,i\neq k}^K \hattheta_i [\MDelta]_{ki}^v \\
\star & \hattheta_0[\barMQ]_{0k}^s + \hattheta_k[\barMQ]_{kk}^s + \hattheta_k [\MDelta]_{kk}^s 
\emat \ve = \MZero \\
\Leftrightarrow 
\begin{cases}
[\MDelta]_{kk}^v = -\hattheta_0\hattheta_k[\barMQ]_{0k}^v - [\barMQ]_{kk}^v - \sum_{i=0,i\neq k}^K \hattheta_i\hattheta_k[\MDelta]_{ki}^v = - (\half\hattheta_0\hattheta_k + 1)\vSkew{\hatxi_k}\TIMa_k - \sum_{i=0,i\neq k}^K \hattheta_i\hattheta_k[\MDelta]_{ki}^v \\
[\MDelta]_{kk}^s = -\hattheta_0\hattheta_k[\barMQ]_{0k}^s - [\barMQ]_{kk}^s = -(\quarter\hattheta_0\hattheta_k + \half)\| \hatxi_k \|^2 + (\quarter\hattheta_0\hattheta_k - \half)\barcsq
\end{cases},\label{eq:blockroweqkfrom1toK}
\eea
and from eq.~\eqref{eq:blockroweqkfrom1toK}, we can get the a sum of all $[\MDelta]_{kk}^v$ to be:
\bea
\sum_{k=1}^K [\MDelta]_{kk}^v & = \displaystyle \sum_{k=1}^K - \left(\half\hattheta_0\hattheta_k + 1\right)\vSkew{\hatxi_k}\TIMa_k - \sum_{k=1}^K\sum_{i=0,i\neq k}\hattheta_i\hattheta_k [\MDelta]_{ki}^v \nonumber \\
& = \displaystyle \sum_{k=1}^K - \left(\half\hattheta_0\hattheta_k + 1\right)\vSkew{\hatxi_k}\TIMa_k + \sum_{k=1}^K \hattheta_0\hattheta_k [\MDelta]_{0k}^v, \label{eq:sumDeltakkv} 
\eea
and the sum of all $[\MDelta]_{kk}^s$ to be:
\bea \label{eq:sumDeltakks}
\sum_{k=1}^K [\MDelta]_{kk}^s = \sum_{k=1}^K - \left(\quarter\hattheta_0\hattheta_k + \half \right)\| \hatxi_k \|^2 + \left(\quarter\hattheta_0\hattheta_k - \half \right) \barcsq.
\eea
Now we write the $0$-th block row equality of~\eqref{eq:barcalMExpansion}:
\bea
\left(-\hattheta_0\hatmu\eye_4 - \sum_{k=1}^K \hattheta_0[\MDelta]_{kk} + \sum_{k=1}^K \hattheta_k[\barMQ]_{0k} + \sum_{k=1}^K \hattheta_k [\MDelta]_{0k} \right) \ve = \MZero \\
\Leftrightarrow
\bmat{cc}
\star & - \sum_{k=1}^K \hattheta_0[\MDelta]_{kk}^v + \sum_{k=1}^K \hattheta_k[\barMQ]_{0k}^v + \sum_{k=1}^K \hattheta_k[\MDelta]_{0k}^v \\
\star & -\hattheta_0\hatmu - \sum_{k=1}^K \hattheta_0[\MDelta]_{kk}^s + \sum_{k=1}^K \hattheta_k [\barMQ]_{0k}^s 
\emat \ve = \MZero \\
\Leftrightarrow 
\begin{cases}
\sum_{k=1}^K [\MDelta]_{kk}^v = \sum_{k=1}^K \hattheta_0\hattheta_k [\barMQ]_{0k}^v + \sum_{k=1}^K \hattheta_0\hattheta_k [\MDelta]_{0k}^v \\
\sum_{k=1}^K [\MDelta]_{kk}^s = -\hatmu + \sum_{k=1}^K \hattheta_0\hattheta_k [\barMQ]_{0k}^s 
\end{cases}\\
\Leftrightarrow 
\begin{cases}
\sum_{k=1}^K [\MDelta]_{kk}^v = \sum_{k=1}^K \half \hattheta_0\hattheta_k \vSkew{\hatxi_k}\TIMa_k + \sum_{k=1}^K \hattheta_0\hattheta_k [\MDelta]_{0k}^v \\
\sum_{k=1}^K [\MDelta]_{kk}^s = -\hatmu + \sum_{k=1}^K \quarter \hattheta_0\hattheta_k (\| \hatxi_k \|^2 - \barcsq)
\end{cases}, \label{eq:blockroweqk0}
\eea
and compare eq.~\eqref{eq:blockroweqk0} with eq.~\eqref{eq:sumDeltakkv}:
\bea
\sum_{k=1}^K -\left( \half\hattheta_0\hattheta_k + 1 \right) \vSkew{\hatxi_k} \TIMa_k + \sum_{k=1}^K \hattheta_0\hattheta_k [\MDelta]_{0k}^v = \sum_{k=1}^K \half\hattheta_0\hattheta_k \vSkew{\hatxi_k}\TIMa_k + \sum_{k=1}^K \hattheta_0\hattheta_k[\MDelta]_{0k}^v \\
\Leftrightarrow 
\sum_{k=1}^K -\left( \half\hattheta_0\hattheta_k + 1 \right) \vSkew{\hatxi_k} \TIMa_k = \sum_{k=1}^K \half\hattheta_0\hattheta_k \vSkew{\hatxi_k}\TIMa_k \\
\Leftrightarrow
\sum_{\hattheta_k = +1} - \frac{3}{2} \vSkew{\hatxi_k}\TIMa_k + \sum_{\hattheta_k = -1} -\half \vSkew{\hatxi_k}\TIMa_k = \sum_{\hattheta_k = +1}\half \vSkew{\hatxi_k}\TIMa_k + \sum_{\hattheta_k = -1} -\half \vSkew{\hatxi_k}\TIMa_k \\
\Leftrightarrow
\sum_{\hattheta_k = +1}\vSkew{\hatxi_k} \TIMa_k = \MZero, \label{eq:inlierresidualakzero}
\eea
which states that for inliers $\hattheta_k = +1$, the sum $\sum_{\hattheta_k = +1} \vSkew{\hatxi_k}\TIMa_k = \MZero$ must hold\footnote{This condition is always true because this is the first-order stationary condition for the \TLS optimization considering only the least squares part applied on inliers.}. We then compare eq.~\eqref{eq:blockroweqk0} with eq.~\eqref{eq:sumDeltakks}:
\bea
\sum_{k=1}^K -\left( \quarter \hattheta_0\hattheta_k + \half \right) \| \hatxi_k \|^2 + \left( \quarter \hattheta_0\hattheta_k - \half \right)\barcsq = -\hatmu + \sum_{k=1}^K \quarter\hattheta_0\hattheta_k \left( \| \hatxi_k \|^2 - \barcsq \right) \\
\Leftrightarrow 
\sum_{k=1}^K - \left( \quarter \hattheta_k + \half\right) \| \hatxi_k \|^2 + \left(\quarter\hattheta_k - \half \right)\barcsq = -\left( \half (\hattheta_k + 1)\| \hatxi_k \|^2 + \half( 1 - \hattheta_k) \barcsq \right) + \sum_{k=1}^K \quarter \hattheta_k \left( \| \hatxi_k \|^2 - \barcsq \right) \\
\Leftrightarrow
\sum_{k=1}^K - \left( \quarter \hattheta_k + \half\right) \| \hatxi_k \|^2 + \left(\quarter\hattheta_k - \half \right)\barcsq = \sum_{k=1}^K - \left( \quarter \hattheta_k + \half\right) \| \hatxi_k \|^2 + \left(\quarter\hattheta_k - \half \right)\barcsq,
\eea
which holds true without any condition. 
Therefore, we conclude that the equality constraint of the $0$-th block row adds no extra constraints on the matrix $\MDelta$. 
This leads to the following Theorem.
\begin{theorem}[Equivalent Projection to $\affineSub$]The projection of a matrix $\MM$ onto $\affineSub$ is equivalent to:
\bea
\Pi_{\affineSub}(\MM) = \Pi_{\barcalH}(\MM - \barMQ + \hatmu\MJ) + \barMQ - \hatmu\MJ,
\eea
where $\barcalH$ is defined by:
\bea \label{eq:barcalH}
\barcalH \doteq \{\MDelta: \MDelta \in \calH, \MDelta \text{ satisfies eq.}~\eqref{eq:blockroweqkfrom1toK} \}
\eea
\end{theorem}
Next we prove that the projection steps in Proposition~\ref{prop:proj2barcalM} define the correct projection $\Pi_{\barcalH}$. To this end, we use $\barMH$ to denote the projection of a matrix $\MH \in \calS^{4(K+1)}$ onto the set $\barcalH$. By the definition of projection, $\barMH$ is the minimizer of the following optimization:
\bea \label{eq:optProj2barcalH}
\barMH = \argmin_{\MDelta \in \barcalH} \| \MH - \MDelta \|_F^2.
\eea
Denoting the objective function of eq.~\eqref{eq:optProj2barcalH} as $f = \| \MH - \MDelta \|_F^2$, we separate the Frobenius norm of $\MH - \MDelta$ into a sum of block-wise Frobenius norms:
\bea
f = \| \MH - \MDelta \|_F^2 \\
= \underbrace{ \sum_{k=0}^{K}\| [\MDelta]_{kk}^m - [\MH]_{kk}^m \|_F^2 }_{f_{kk}^m: \text{ diagonal block matrix part}} + \underbrace{ \sum_{k=0}^K \| [\MDelta]_{kk}^s - [\MH]_{kk}^s \|_F^2 }_{f_{kk}^s: \text{ diagonal block scalar part}} + \underbrace{ 2\sum_{l=1}^L \| [\MDelta]_{\myBlocks(l)}^m - [\MH]_{\myBlocks(l)}^m \|_F^2 }_{f_{\myBlocks}^m: \text{ off-diagonal block matrix part}} + \underbrace{ 2\sum_{l=1}^L \| [\MDelta]_{\myBlocks(l)}^s - [\MH]_{\myBlocks(l)}^s \|_F^2 }_{f_{\myBlocks}^s: \text{ off-diagonal block scalar part}} \nonumber \\
+ 2\underbrace{ \left( \sum_{k=0}^K \| [\MDelta]_{kk}^v - [\MH]_{kk}^v \|_F^2 + \sum_{l=1}^L \| [\MDelta]_{\calW(l)}^v - [\MH]_{\calW(l)}^v \|_F^2 + \sum_{l=1}^L \| [\MDelta]_{\myBlocks(l)}^v - [\MH]_{\myBlocks(l)}^v \|_F^2 \right)}_{f^v: \text{ diagonal and off-diagonal vector part}}, \label{eq:separateFrobeniusNorm}
\eea
where $\calW$ is the following set of $L$ ordered indices that enumerate the lower-triangular blocks of any matrix $\MDelta \in \calS^{4(K+1)}$:
\bea \label{eq:setWloweridx}
\calW = \{(1,0),\dots,(K,0),(2,1),\dots,(K,1),\dots,(K,K-1)\},
\eea
and $\calW(l)$ equals to the flipped copy of $\myBlocks(l)$ (eq.~\eqref{eq:setUupperidx}). Using the separated Frobenius norm~\eqref{eq:separateFrobeniusNorm}, we can get the projection $\barMH$ block by block in the following steps.
\begin{enumerate}
	\item {\bf Diagonal block matrix part}. From the expression of $f^m_{kk}$ and the constraint that $\sum_{k=0}^K [\MDelta]_{kk}^m = \MZero$, we can get the diagonal block matrix part of $\barMH$ to be:
	\bea
	[\barMH]_{kk}^m = [\MH]_{kk}^m -  \frac{\displaystyle \sum_{k=0}^K [\MH]_{kk}^m }{K+1}, \quad \forall k = 0,\dots,K,
	\eea
	which first computes the mean of all diagonal block matrix parts, $\sum_{k=0}^K [\MH]_{kk}^m /(K+1)$, and then subtracts the mean from each individual diagonal block matrix part $[\MH]_{kk}^m$.

	\item {\bf Diagonal block scalar part}. From the expression of $f^s_{kk}$ and the constraint from eq.~\eqref{eq:blockroweqkfrom1toK} (which fixes the scalar parts to be known constants), we can get the diagonal block scalar part of $\barMH$ to be:
	\bea
	[\barMH]_{kk}^s = \begin{cases}
    -(\quarter \hattheta_k + \half) \| \hatxi_k \|^2 + (\quarter \hattheta_k - \half)\barcsq & k=1,\dots,K, \\
    \sum_{k=1}^K (\quarter \hattheta_k + \half ) \| \hatxi_k \|^2 - (\quarter\hattheta_k - \half)\barcsq & k=0.
	\end{cases},
	\eea
	which first computes the diagonal scalar parts for $k=1,\dots,K$ and then assigns the negative sum of all diagonal scalar parts to the scalar part of the top-left diagonal block ($[\barMH]_{00}^s$).

	\item {\bf Off-diagonal block matrix part}. From the expression of $f^m_{\myBlocks}$ and the constraint that $[\MDelta]_{\myBlocks(l)}$ is skew-symmetric, we can obtain the off-diagonal block matrix part of $\barMH$ to be:
	\bea
	[\barMH]_{\myBlocks(l)}^m = \frac{[\MH]_{\myBlocks(l)}^m - ([\MH]_{\myBlocks(l)}^m)\tran}{2},\quad \forall l=1,\dots,L,
	\eea
	which simply projects $[\MH]_{\myBlocks(l)}^m$ to its nearest skew-symmetric matrix. Note that it suffices to only project the upper-triangular off-diagonal blocks because both $\MH$ and $\barMH$ are symmetric matrices.

	\item {\bf Off-diagonal block scalar part}. The off-diagonal blocks are skew-symmetric matrices and the scalar parts must be zeros. Therefore, the off-diagonal block scalar parts of $\barMH$ are:
	\bea
	[\barMH]_{\myBlocks(l)}^s = 0, \quad \forall l=1,\dots,L.
	\eea

	\item {\bf Off-diagonal block vector part}. The projection of the vector parts of each block is non-trivial due to the coupling between the diagonal block vector parts $[\MDelta]_{kk}^v$ and the off-diagonal block vector parts $[\MDelta]_{\myBlocks(l)}^v, [\MDelta]_{\calW(l)}^v$. We first write the diagonal block vector parts as functions of the off-diagonal block vector parts by recalling eq.~\eqref{eq:blockroweqkfrom1toK} and~\eqref{eq:sumDeltakkv}:
	\bea
    \begin{cases}
    [\MDelta]_{kk}^v = \underbrace{- (\half\hattheta_0\hattheta_k + 1)\vSkew{\hatxi_k}\TIMa_k}_{:=\vphi_k} - \sum_{i=0,i\neq k}\hattheta_i\hattheta_k [\MDelta]_{ki}^v = - \sum_{i=0,i\neq k}\hattheta_i\hattheta_k [\MDelta]_{ki}^v + \vphi_k \\
    [\MDelta]_{00}^v = -\sum_{k=1}^K [\MDelta]_{kk}^v = \underbrace{\sum_{k=1}^K (\half\hattheta_0\hattheta_k + 1)\vSkew{\hatxi_k}\TIMa_k}_{:=\vphi_0} - \sum_{k=1}^K \hattheta_0\hattheta_k [\MDelta]_{0k}^v = -\sum_{i=0,i\neq 0} \hattheta_i\hattheta_k [\MDelta]_{0i} + \vphi_0
    \end{cases},
	\eea
	which shows that we can write $[\MDelta]_{kk}^v,k=0,\dots,K$ in a unified notation:
	\bea \label{eq:unifiedDeltakkv}
 	[\MDelta]_{kk}^v = -\sum_{i=0,i\neq k}^K \hattheta_k\hattheta_i [\MDelta]_{ki}^v + \vphi_k, \quad \forall k=0,\dots,K.
	\eea
	Inserting eq.~\eqref{eq:unifiedDeltakkv} into the expression of $f^v$ in eq.~\eqref{eq:separateFrobeniusNorm}, and using the fact that $[\MDelta]_{\calW(l)}^v = - [\MDelta]_{\myBlocks(l)}^v$, we obtain the following unconstrained optimization problem:
	\bea \label{eq:unconOptProjbarcalHvector}
    \min_{[\MDelta]_{\myBlocks(l)}^v} \sum_{l=1}^L \left\| [\MDelta]_{\myBlocks(l)}^v - [\MH]_{\myBlocks(l)}^v \right\|^2 + \left\| [\MDelta]_{\myBlocks(l)}^v + [\MH]_{\calW(l)}^v \right\|^2 + \sum_{k=0}^K \left\| \sum_{i=0,i\neq k}^K \hattheta_k\hattheta_i [\MDelta]_{ki}^v - \vphi_k + [\MH]_{kk}^v \right\|^2.
	\eea
	To solve problem~\eqref{eq:unconOptProjbarcalHvector}, we first derive the partial derivatives \wrt each variable $[\MDelta]_{\myBlocks(l)}^v$. For notation simplicity, we write $\myBlocks(l) = (r_l,c_l)$, where $r_l < c_l$ represent the block-row and block-column indices. In addition, we use $u(r,c)$ to denote the pointer that returns the index of $(r,c)$ in the set $\myBlocks$. Using this notation, the partial derivative of $f^v$ \wrt $[\MDelta]_{\myBlocks(l)}^v$ is:
	\bea
\frac{\partial f^v}{\partial [\MDelta]_{\myBlocks(l)}^v} = \frac{\partial \left( \substack{ \| [\MDelta]_{\myBlocks(l)}^v - [\MH]_{\myBlocks(l)}^v \|^2 + \| [\MDelta]_{\myBlocks(l)}^v + [\MH]_{\calW(l)}^v \|^2 + \\ \| \sum_{i=0,i\neq r_l}^K \hattheta_{r_l}\hattheta_i [\MDelta]_{r_l,i}^v - \vphi_{r_l} + [\MH]_{r_l,r_l}^v \|^2  +  \| \sum_{i=0,i\neq c_l}^K \hattheta_{c_l}\hattheta_i [\MDelta]_{c_l,i}^v - \vphi_{c_l} + [\MH]_{c_l,c_l}^v \|^2 } \right) }{\partial [\MDelta]_{\myBlocks(l)}^v} \\
= 2\left( \substack{ \displaystyle 2[\MDelta]_{\myBlocks(l)}^v - [\MH]_{\myBlocks(l)}^v + [\MH]_{\calW(l)}^v + \\ \displaystyle  
\hattheta_{r_l}\hattheta_{c_l} \left(\sum_{\substack{i=0\\i\neq r_l} }^K \hattheta_{r_l} \hattheta_i [\MDelta]_{r_l,i}^v \!-\! \vphi_{r_l} + [\MH]_{r_l,r_l}^v \right) - \hattheta_{c_l}\hattheta_{r_l} \left( \sum_{\substack{ i=0\\i\neq c_l }}^K \hattheta_{c_l}\hattheta_i [\MDelta]_{c_l,i}^v - \vphi_{c_l} + [\MH]_{c_l,c_l}^v \right) } \right) \\
\hspace{-12mm}= 2 \left( 4[\MDelta]_{\myBlocks(l)}^v \!-\! [\MH]_{\myBlocks(l)}^v + [\MH]_{\calW(l)}^u + \!\!\!\sum_{i\neq r_l,c_l}^K \hattheta_{c_l}\hattheta_i [\MDelta]_{r_l,i}^v \!-\!\!\!\sum_{i\neq r_l,c_l}^K \hattheta_{r_l}\hattheta_i [\MDelta]_{c_l,i}^v  \!+\! \hattheta_{r_l}\hattheta_{c_l}\left (\vphi_{c_l} - \vphi_{r_l} + [\MH]_{r_l,r_l}^v - [\MH]_{c_l,c_l}^v\right)  \right). \label{eq:halfwayderivative}
	\eea
	Notice that not all of the $[\MDelta]_{r_l,i}^v$ and $[\MDelta]_{c_l,i}^v$ in eq.~\eqref{eq:halfwayderivative} belong to the upper-triangular part of $\MDelta$, therefore, we use the fact that $[\MDelta]_{ij}^v = -[\MDelta]_{ji}^v$ for any $i\neq j$ to convert all lower-triangular vectors to upper-triangular. In addition, we set the partial derivative to zero to obtain sufficient and necessary conditions to solve the convex unconstrained optimization~\eqref{eq:unconOptProjbarcalHvector}:
\bea
\hspace{-6mm}4[\MDelta]_{\myBlocks(l)}^v + \sum_{i\neq r_l,c_l}^K \hattheta_{c_l}\hattheta_i [\MDelta]_{r_l,i}^v - \sum_{i\neq r_l,c_l}^K \hattheta_{r_l}\hattheta_i [\MDelta]_{c_l,i}^v = \underbrace{ [\MH]_{\myBlocks(l)}^v - [\MH]_{\calW(l)}^v + \hattheta_{r_l}\hattheta_{c_l}\left( \vphi_{r_l} - \vphi_{c_l} - [\MH]_{r_l,r_l}^v + [\MH]_{c_l,c_l}^v \right) }_{:= \calF_l(\MH)} \label{eq:defcalFl}\\
\Leftrightarrow 
4[\MDelta]_{\myBlocks(l)}^v + \sum_{i < r_l} \hattheta_{c_l}\hattheta_i [\MDelta]_{r_l,i}^v + \sum_{i > r_l, i\neq c_l} \hattheta_{c_l}\hattheta_i [\MDelta]_{r_l,i}^v - \sum_{i<c_l, i\neq r_l} \hattheta_{r_l}\hattheta_i [\MDelta]_{c_l,i}^v - \sum_{i>c_l} \hattheta_{r_l}\hattheta_i [\MDelta]_{c_l,i}^v = \calF_l(\MH) \\
\hspace{-6mm}\Leftrightarrow 
4[\MDelta]_{\myBlocks(l)}^v - \sum_{i < r_l} \hattheta_{c_l}\hattheta_i [\MDelta]^v_{u(i,r_l)} + \sum_{i>r_l, i \neq c_l} \hattheta_{c_l}\hattheta_i [\MDelta]^v_{u(r_l,i)} + \sum_{i < c_l,i\neq r_l} \hattheta_{r_l}\hattheta_i [\MDelta]^v_{u(i,c_l)} - \sum_{i > c_l}\hattheta_{r_l}\hattheta_i[\MDelta]^v_{u(c_l,i)} = \calF_l(\MH) \\
\Leftrightarrow \MA_l\tran \calC(\MDelta) = \calF_l\tran(\MH), \label{eq:singleLinearEquationPartialDerivative}
\eea
where the linear map $\calC: \calS^{4(K+1)} \rightarrow \Real{L \times 3}$ is as defined in eq.~\eqref{eq:defLinearMapC}, and $\calF_l(\MH) \in \Real{3}$ is defined as in eq.~\eqref{eq:defcalFl}, and $\MA_l \in \Real{L}$ is the following vector:
\bea \label{eq:eachRowmatrixA}
\MA_l(u) = \begin{cases}
4 & \text{if } u=l\\
-\hattheta_{c_l}\hattheta_i & \text{if } u = u(i,r_l),\forall 0 \leq i < r_l \\
\hattheta_{c_l}\hattheta_i & \text{if } u = u(r_l,i),\forall r_l <i \leq K, i\neq c_l \\
\hattheta_{r_l}\hattheta_i & \text{if } u = u(i,c_l),\forall 0 \leq i<c_l,i\neq r_l\\
-\hattheta_{r_l}\hattheta_i & \text{if } u = u(c_l,i),\forall c_l < i \leq K \\
0 & \text{otherwise}
\end{cases},
\eea
where $\MA_l(u)$ denotes the $u$-th entry of vector $\MA_l$.
Because the partial derivative \wrt each variable $[\MDelta]_{\myBlocks(l)}^v, l=1,\dots,L$, must be zero at the global minimizer, $\barMH$, of problem~\eqref{eq:unconOptProjbarcalHvector}, we have $L$ linear equations of the form eq.~\eqref{eq:singleLinearEquationPartialDerivative}, which allows us to write the following matrix form:
\bea
\underbrace{[\MA_1,\dots,\MA_l,\dots,\MA_L]\tran}_{\MA \in \calS^L} \underbrace{\calC(\barMH)}_{\in \Real{L \times 3}} = \underbrace{[\calF_1(\MH),\dots,\calF_l(\MH),\dots,\calF_L(\MH)]\tran}_{\calF(\MH) \in \Real{L \times 3}}
\eea
and we can compute $\calC(\barMH)$ by:
\bea \label{eq:proj2barcalHoffdiagvector}
\calC(\barMH) = \MA\inv \calF(\MH).
\eea
Surprisingly, due to the structure of $\MA$, its inverse can be computed in closed form.
\begin{theorem}[Closed-form Matrix Inverse]\label{thm:closedformmatrixInverse}
The inverse of matrix $\MA$~\eqref{eq:eachRowmatrixA}, denoted $\MP$, can be computed in closed form, with each column of $\MP$ being:
\bea \label{eq:eachRowmatrixP}
\MP_l(u) = \begin{cases}
p_1 & \text{if } u=l\\
\hattheta_{c_l}\hattheta_i p_2 & \text{if } u = u(i,r_l),\forall 0 \leq i < r_l \\
- \hattheta_{c_l}\hattheta_i p_2 & \text{if } u = u(r_l,i),\forall r_l <i \leq K, i\neq c_l \\
- \hattheta_{r_l}\hattheta_i p_2 & \text{if } u = u(i,c_l),\forall 0 \leq i<c_l,i\neq r_l\\
 \hattheta_{r_l}\hattheta_i p_2 & \text{if } u = u(c_l,i),\forall c_l < i \leq K \\
0 & \text{otherwise}
\end{cases},
\eea
where $p_1$ and $p_2$ are the following constants:
\bea \label{eq:inverseApconstants}
p_1 = \frac{K+1}{2K + 6},\quad p_2 = \frac{1}{2K + 6}.
\eea
\end{theorem}
\begin{proof}
To show $\MP$ is the inverse of $\MA$, we need to show $\MA \MP = \eye_L$. To do so, we first show that $\MA_l\tran \MP_l = 1$, which means the diagonal entries of $\MA\MP$ are all equal to 1. Note that from the definition of $\MA_l$ in~\eqref{eq:eachRowmatrixA} and $\MP_l$ in~\eqref{eq:eachRowmatrixP}, one can see that there are $2(K-1) + 1$ non-zero entries in $\MA_l$ and $\MP_l$ and $\MP_l(u) = -p_2 \MA_l(u)$ when $u\neq l$, and $\MP_l(u) = p_1,\MA_l(u) = 4$ when $u=l$, therefore, we have:
\bea
\MA_l\tran\MP_l = 4p_1 - 2(K-1)p_2 = \frac{4K+4}{2K+6} - \frac{2K-2}{2K+6} = \frac{2K+6}{2K+6} = 1.
\eea
We then show that $\MA_l\tran\MP_m = 0$ for any $l \neq m$. The non-zero indices of $\MA_l$, $\MP_m$ and their corresponding entries are:
\bea
\begin{array}{c|c}
u & \MA_l(u) \\
\hline
u(r_l,c_l) & 4 \\
u(i,r_l),\forall 0 \leq i < r_l & -\hattheta_{c_l}\hattheta_i \\
u(r_l,i),\forall r_l < i \leq K, i\neq c_l & \hattheta_{c_l} \hattheta_i \\
u(i,c_l),\forall 0 \leq i < c_l, i \neq r_l & \hattheta_{r_l}\hattheta_i \\
u(c_l,i),\forall c_l < i \leq K, & -\hattheta_{r_l}\hattheta_i 
\end{array} \quad 
\begin{array}{c|c}
u & \MP_m(u) \\
\hline
u(r_m,c_m) & p_1 \\
u(i,r_m),\forall 0 \leq i < r_m & \hattheta_{c_m}\hattheta_i p_2 \\
u(r_m,i),\forall r_m < i \leq K, i\neq c_m & -\hattheta_{c_m} \hattheta_i p_2\\
u(i,c_m),\forall 0 \leq i < c_m, i \neq r_m & - \hattheta_{r_m}\hattheta_i p_2 \\
u(c_m,i),\forall c_m < i \leq K, & \hattheta_{r_m}\hattheta_i p_2
\end{array},
\eea
where we have used $\myBlocks(l) = (r_l,c_l)$ and $\myBlocks(m) = (r_m,c_m)$. Now depending on the 12 different relations between $(r_l,c_l)$ and $(r_m,c_m)$, $\MA_l$ and $\MP_m$ can have different ways of sharing common indices whose corresponding entries in both $\MA_l$ and $\MP_m$ are non-zero. We next show that in all these 12 different cases, $\MA_l\tran\MP_m =0 $ holds true. 
\begin{enumerate}
\item $r_l < c_l < r_m < c_m$: $\MA_l$ and $\MP_m$ share 4 non-zero entries with the same indices:
\bea
\begin{array}{c|c|c}
u & \MA_l(u) & \MP_m(u)\\
\hline
u(r_l,r_m) & \hattheta_{c_l}\hattheta_{r_m} & \hattheta_{c_m}\hattheta_{r_l} p_2 \\
u(r_l,c_m) & \hattheta_{c_l}\hattheta_{c_m} & -\hattheta_{r_m}\hattheta_{r_l} p_2 \\
u(c_l,r_m) & -\hattheta_{r_l}\hattheta_{r_m} & \hattheta_{c_m}\hattheta_{c_l} p_2 \\
u(c_l,c_m) & -\hattheta_{r_l}\hattheta_{c_m} & -\hattheta_{r_m}\hattheta_{c_l} p_2
\end{array}.
\eea
Using the above table, we easily see that $\MA_l\tran\MP_m = 0$. 

\item $r_l < c_l = r_m < c_m$: $\MA_l$ and $\MP_m$ share the following common indices whose corresponding entries are non-zero:
\bea
\begin{array}{c|c|c}
u & \MA_l(u) & \MP_m(u)\\
\hline
u(r_l,c_l) & 4 & \hattheta_{c_m}\hattheta_{r_l}p_2 \\
u(r_l,c_m) & \hattheta_{c_l}\hattheta_{c_m} & -\hattheta_{c_l}\hattheta_{r_l}p_2 \\
u(i,c_l),\forall 0 \leq i < c_l, i \neq r_l & \hattheta_{r_l}\hattheta_i & \hattheta_{c_m} \hattheta_i p_2 \\
u(c_l,i),\forall c_l < i \leq K, i \neq c_m & -\hattheta_{r_l}\hattheta_i & -\hattheta_{c_m}\hattheta_i p_2 \\
u(c_l,c_m) & -\hattheta_{r_l}\hattheta_{c_m} & p_1
\end{array}
\eea
Using the table above, we can calculate $\MA_l\tran\MP_m$:
\bea
\MA_l\tran\MP_m = 4\hattheta_{c_m}\hattheta_{r_l}p_2 - \hattheta_{r_l}\hattheta_{c_m}p_2 + (c_l - 1)\hattheta_{r_l}\hattheta_{c_m}p_2 + (K-c_l-1)\hattheta_{r_l}\hattheta_{c_m}p_2 - \hattheta_{r_l}\hattheta_{c_m}p_1 \nonumber \\
= ((K+1)p_2 - p_1)\hattheta_{r_l}\hattheta_{c_m} = 0,
\eea
which ends up being $0$ due to $p_1 = (K+1) p_2$ (\cf~eq.~\eqref{eq:inverseApconstants}).

\item $r_l < r_m < c_l < c_m$: $\MA_l$ and $\MP_m$ share 4 common indices with non-zero entries:
\bea
\begin{array}{c|c|c}
u & \MA_l(u) & \MP_m(u)\\
\hline
u(r_l,r_m) & \hattheta_{c_l}\hattheta_{r_m} & \hattheta_{c_m}\hattheta_{r_l}p_2 \\
u(r_l,c_m) & \hattheta_{c_l}\hattheta_{c_m} & -\hattheta_{r_m}\hattheta_{r_l}p_2 \\
u(r_m,c_l) & \hattheta_{r_l}\hattheta_{r_m} & -\hattheta_{c_m}\hattheta_{c_l}p_2 \\
u(c_l,c_m) & -\hattheta_{r_l}\hattheta_{c_m} & -\hattheta_{r_m}\hattheta_{c_l}p_2
\end{array}
\eea
from which we can compute that $\MA_l\tran\MP_m = 0$.

\item $r_l = r_m < c_l < c_m$: $\MA_l$ and $\MP_m$ share the following common indices with non-zero entries:
\bea
\begin{array}{c|c|c}
u & \MA_l(u) & \MP_m(u)\\
\hline
u(r_l,c_l) & 4 & -\hattheta_{c_m}\hattheta_{c_l}p_2 \\
u(i,r_l), \forall 0 \leq i < r_l & -\hattheta_{c_l}\hattheta_i & \hattheta_{c_m}\hattheta_i p_2 \\
u(r_l,i), \forall r_l < i \leq K,i\neq c_l,c_m & \hattheta_{c_l}\hattheta_i & -\hattheta_{c_m}\hattheta_i p_2 \\
u(r_l,c_m) & \hattheta_{c_l}\hattheta_{c_m} & p_1 \\
u(c_l,c_m) & -\hattheta_{r_l}\hattheta_{c_m} & -\hattheta_{r_k}\hattheta_{c_l}p_2
\end{array}
\eea
from which we can compute $\MA_l\tran\MP_m$:
\bea
\MA_l\tran\MP_m = (-4p_2 - r_lp_2 - (K-r_l-2)p_2 + p_1 + p_2)\hattheta_{c_l}\hattheta_{c_m} = (p_1 - (K+1)p_2)\hattheta_{c_l}\hattheta_{c_m} = 0
\eea

\item $r_l < r_m < c_l = c_m$: $\MA_l$ and $\MP_m$ share the following common indices with non-zero entries:
\bea
\begin{array}{c|c|c}
u & \MA_l(u) & \MP_m(u)\\
\hline
u(r_l,c_l) & 4 & -\hattheta_{r_m}\hattheta_{r_l}p_2\\
u(r_l,r_m) & \hattheta_{c_l}\hattheta_{r_m} & \hattheta_{c_l}\hattheta_{r_l}p_2 \\
u(i,c_l),\forall 0 \leq i < c_l,i \neq r_l,r_m & \hattheta_{r_l}\hattheta_i & -\hattheta_{r_m}\hattheta_i p_2 \\
u(r_m,c_l) & \hattheta_{r_l}\hattheta_{r_m} & p_1 \\
u(c_l,i) \forall c_l < i \leq K & -\hattheta_{r_l}\hattheta_i & \hattheta_{r_m}\hattheta_i p_2 
\end{array}
\eea
from which we compute $\MA_l\tran\MP_m$:
\bea
\MA_l\tran\MP_m = (-4p_2 + p_2 - (c_l-2)p_2 + p_1 - (K-c_l)p_2)\hattheta_{r_l}\hattheta_{r_m} = (p_1 - (K+1)p_2)\hattheta_{r_l}\hattheta_{r_m} = 0.
\eea

\item $r_m < r_l < c_l < c_m$: $\MA_l$ and $\MP_m$ share 4 common indices with non-zero entries:
\bea
\begin{array}{c|c|c}
u & \MA_l(u) & \MP_m(u)\\
\hline
u(r_m,r_l) & -\hattheta_{c_l}\hattheta_{r_m} & -\hattheta_{c_m}\hattheta_{c_m}\hattheta_{r_l}p_2 \\
u(r_l,c_m) & \hattheta_{c_l}\hattheta_{c_m} & -\hattheta_{r_m}\hattheta_{r_l} p_2 \\
u(r_m,c_l) & \hattheta_{r_l}\hattheta_{r_m} & -\hattheta_{c_m}\hattheta_{c_l} p_2 \\
U(c_l,c_m) & -\hattheta_{r_l}\hattheta_{c_m} & -\hattheta_{r_m}\hattheta_{c_l}p_2
\end{array}
\eea
from which we easily see that $\MA_l\tran \MP_m = 0$.

\item $r_m < r_l < c_l = c_m$: same as $r_l < r_m < c_m = c_l$ by symmetry (switch $m$ and $l$).

\item $r_m < r_l < c_m < c_l$: same as $r_l < r_m < c_l < c_m$ by symmetry.

\item $r_m < r_l = c_m < c_l$: same as $r_l < r_m = c_l < c_m$ by symmetry.

\item $r_m < c_m < r_l < c_l$: same as $r_l < c_l < r_m < c_m$ by symmetry.

\item $r_l < r_m < c_m < c_l$: same as $r_m < r_l < c_l < c_m$ by symmetry.

\item $r_m = r_l < c_m < c_l$: same as $r_l = r_m < c_l < c_m$ by symmetry.
\end{enumerate}
As we show above, in all cases, $\MA_l\MP_m = 0$, when $l\neq m$. Therefore, we prove that $\MA\MP = \eye_L$ and $\MP$ is the closed-form inverse of $\MA$.
\end{proof}
Therefore, using Theorem~\ref{thm:closedformmatrixInverse}, we can first generate $\MP$ directly (using eq.~\eqref{eq:eachRowmatrixP}) and then apply eq.~\eqref{eq:proj2barcalHoffdiagvector} to obtain the off-diagonal block vector parts of $\barMH$.

\item {\bf Diagonal block vector part}. After obtaining off-diagonal block vector parts, we can insert them back into eq.~\eqref{eq:unifiedDeltakkv} to obtain the diagonal block vector parts:
\bea
[\barMH]_{kk}^v = -\sum_{i=0,i\neq k}^K \hattheta_k\hattheta_i [\barMH]_{ki}^v + \vphi_k, \\
\vphi_k = \begin{cases}
-\left( \half \hattheta_k + 1\right)\vSkew{\hatxi_k}\TIMa_k & k=1,\dots,K\\
\sum_{k=1}^K \left( \half \hattheta_k + 1\right)\vSkew{\hatxi_k}\TIMa_k  & k = 0
\end{cases}.
\eea
\end{enumerate}
\end{proof}

\subsection{Initial Guess for Algorithm~\ref{alg:optCertification}}
\label{sec:app-initialGuess}
We describe the initial guess $\barMM_0$ as:
\bea
\barMM_0 = \barMQ - \hatmu\MJ + \MDelta_0,
\eea
where $\MDelta_0 \in \barcalH$ is the following:
\bea
\text{Off-diagonal blocks:} & [\MDelta]_{\calU(l)} = \MZero, \forall l=1,\dots,L \\
\text{Diagonal block scalar part:} & [\MDelta]_{kk}^s = \begin{cases}
-(\quarter \hattheta_k + \half) \| \hatxi_k \|^2 + (\quarter \hattheta_k - \half)\barcsq & k=1,\dots,K, \\
\sum_{k=1}^K (\quarter \hattheta_k + \half ) \| \hatxi_k \|^2 - (\quarter\hattheta_k - \half)\barcsq & k=0.
\end{cases} \\
\text{Diagonal block vector part:} & [\MDelta]_{kk}^v = \begin{cases}
-\left(\half \hattheta_k + 1 \right)\vSkew{\hatxi_k}\TIMa_k  & k = 1,\dots,K, \\
\sum_{k=1}^K \left( \half\hattheta_k + 1 \right)\vSkew{\hatxi_k}\TIMa_k & k = 0
\end{cases} \\
\text{Diagonal block matrix part:} & [\MDelta]_{kk}^m = \begin{cases}
-[\barMQ]_{0k} - \left( \quarter\hattheta_k + \quarter \right) \| \hatxi_k \|^2 \eye_3 - \half\barcsq \eye_3  & k = 1,\dots,K \\
\sum_{k=1}^K \left( [\barMQ]_{0k} + \left( \quarter\hattheta_k + \quarter \right) \| \hatxi_k \|^2 \eye_3 + \half\barcsq \eye_3 \right) & k = 0
\end{cases}.
\eea
One can verify that $\MDelta_0$ satisfies all the constraints defined by set $\barcalH$~\eqref{eq:barcalH}.

\subsection{Tighter Bounds for Theorem~\ref{thm:certifiableRegNoisy}}
\label{sec:tighterBounds}

In this section, we improve over the bounds presented in Theorem~\ref{thm:certifiableRegNoisy} and derive tighter but more complex bounds for scale, rotation, and translation estimation.

\begin{theorem}[Estimation Contract with Noisy Inliers and Adversarial Outliers]
\label{thm:certifiableRegNoisyTighter}
Assume (i) the set of correspondences contains at least 3 distinct and non-collinear inliers, where
 for any inlier $i$, $\| \vb_i - \sgt \MRgt \va_i - \vtgt \| \leq \beta_i$, and 
 $(\sgt, \MRgt, \vtgt)$ denotes the ground truth transformation. 
 Assume (ii) the inliers belong to the maximum consensus set in each 
 subproblem\footnote{In other words: the inliers belong to the largest set $\calS$  
 such that for any $i,j \in \calS$, 
 $| s_{ij} - s | \leq \alpha_{ij}$, 
 $\| \TIMb_{ij} - s R \TIMa_{ij}\| \leq \TIMNoiseBound_{ij}$, 
 and 
$\left| [\vb_{i} - s R \va_{i} - \vt]_l \right| \leq \beta_{i},\forall l=1,2,3$ 
for any transformation $(s,\MR,\vt)$. Note that for translation the assumption is applied component-wise. }, and 
 (iii) the second largest consensus set is ``sufficiently smaller'' than the maximum consensus set 
(as formalized in \isExtended{Lemma~\ref{thm:TLSandMC}}{Lemma~\ref{thm:TLSandMC}\arxivCite}).
 If
 (iv) the rotation subproblem produces a valid certificate 
 (as per Theorems~\ref{thm:quasarOptimalityGuarantee}-\ref{thm:optimalityCertification}), 
then the output $(\hats, \hatMR, \hatvt)$ of \name satisfies:
\bea
& \displaystyle |\hats - \sgt| \leq \min\{\bounds_i, i\in \setTrueIn \}, \label{eq:bounds_scale_tls} \\
& \displaystyle
\sgt(1-\cos(\theta_R)) \leq \frac{  \sum_{i\in\setTrueIn} (\boundR_i)^2 - (\hats-\sgt)^2 \| \MA_{\setTrueIn} \|_F^2}{2\hats \left( \sigma_1^2(\MA_{\setTrueIn}) + \sigma_2^2(\MA_{\setTrueIn})  \right)},
\label{eq:bounds_rot_tls} \\ 
&
\displaystyle |\hatt_l - \tgt_l| \leq \min \left\{ \left( (\hats - \sgt)^2  + 2\hats\sgt(1-\cos(\theta_R))  \right)\| \va_i \|^2 + [\boundt_i]_l:  i\in\setTrueIn\right\}, \forall l=1,2,3.
 \label{eq:bounds_tran_tls}
\eea
where $\setTrueIn$ is the set of ground-truth inliers with size $\sizeTrueIn$,~\ie~ground-truth \TRIMs in the scale bound, ground-truth \TIMs in the rotation bound, and ground-truth correspondences in the translation bound; $\hatt_l,\tgt_l,l=1,2,3,$ denotes the $l$-th entry of the translation vector; $\bounds_i$, $\boundR_i$ and $[\boundt_i]_l$ are defined as follows:
\bea
\begin{cases}
\bounds_i = |s_i - \hats| + \alpha_i \\
\boundR_i = \frac{\| \TIMb_i - \hats\hatMR\TIMa_i \|}{\| \TIMa_i \|} + \alpha_i \\
[\boundt_i]_l = | [\vb_i]_l - [\hats\hatMR\va_i]_l - \hatt_l | + \beta_i \\
\end{cases},\forall i\in\setTrueIn, \forall l=1,2,3;
\eea 
$\theta_R$ is the angular (geodesic) distance between $\hatMR$ and $\MRgt$; $\MA_{\calI} = \left[ \frac{\TIMa_{i_1}}{\|\TIMa_i \|},\dots,\frac{\TIMa_{i_{\sizeTrueIn}}}{\|\TIMa_{\sizeTrueIn} \|}  \right] \in \Real{3\times \sizeTrueIn}$ is the matrix that assembles all normalized \TIMs in the set $\setTrueIn$, and $\sigma_1(\MA_{\setTrueIn})$, $\sigma_2(\MA_{\setTrueIn})$ denotes the smallest and second smallest singular values of matrix $\MA_{\setTrueIn}$. Note that since there are at least 3 non-collinear ground-truth inliers, $\rank{\MA_{\setTrueIn}} \geq 2$ and $\sigma_1(\MA_{\setTrueIn}) + \sigma_2(\MA_{\setTrueIn}) > 0$, meaning the rotation bound is well defined. 
Since the bounds rely on the knowledge of the inliers $\setTrueIn$,
to compute the worst-case bounds, one has to enumerate over all possible subsets of $\sizeTrueIn = 3$ inliers to compute the largest bounds for scale, rotation and translation.
\end{theorem}
\begin{proof} We will prove the bounds for scale, rotation, and translation separately.


{\bf Scale Error Bound}. Let us denote with $\setCon$ the consensus set associated with the optimal solution of the \TLS scale estimation problem, with $|\setCon| = \sizeCon$. 
To establish a connection with the ground-truth scale $\sgt$, we partition $\setCon$ into two sets: $\setTrueIn$ that contains all the indices of the \emph{true inliers}, and $\setFakeIn$ that contains all the indices of the \emph{false inliers}, which are not generated from the true inlier model~\eqref{eq:TRIM} where classified as inliers by \TLS and ended up in $\setCon$. Let $\sizeTrueIn = |\setTrueIn|$ and $\sizeFakeIn = |\setFakeIn|$, and we have $\sizeTrueIn + \sizeFakeIn = \sizeCon$. For scale measurements that are in $\setTrueIn$ and $\setFakeIn$, we have the following relations:
\bea \label{eq:scaleTrueFakeModel}
\begin{cases}
s_i = \sgt + \epsilon_i = \hats + \phi_i & \forall i \in \setTrueIn \\
s_j = \hats + \phi_j & \forall j \in \setFakeIn
\end{cases},
\eea
where $\phi_i$'s and $\phi_j$'s are \emph{residuals} that can be computed from $s_i,s_j,$ and $\hats$, and $\epsilon_i$'s are  \emph{unknown} \emph{noise} terms. Note that we have $|\epsilon_i|\leq\alpha_i,|\phi_i|\leq\alpha_i,|\phi_j|\leq \alpha_j$.
From the first equation in eq.~\eqref{eq:scaleTrueFakeModel}, we have:
\bea
\hats - \sgt = \epsilon_i - \phi_i \Rightarrow |\hats - \sgt| \leq |\phi_i| + \alpha_i \doteq \bounds_i, \forall i \in \setTrueIn, \label{eq:scaleBoundEachInlier} \\
\Rightarrow |\hats - \sgt| \leq \min\{\bounds_i, i\in \setTrueIn \} \label{eq:scaleBoundTight}
\eea
where we have defined a new quantity $\bounds_i = |\phi_i| + \alpha_i$, that can be computed for every \TRIM in the consensus set $\setCon$. Eq.~\eqref{eq:scaleBoundEachInlier} states that the scale error is bounded by the minimum $\zeta_i$ in the set of true inliers. Therefore, if we know what the set of true inliers $\setTrueIn$ is, then we can simply take the minimum of the set $\{ \zeta_i, i\in \setTrueIn \}$. However, if we do not know the true inlier set, assuming $\sizeTrueIn \geq 3$, we get the worst-case bound:
\bea
|\hats - \sgt| \leq \bounds_{m_3},
\eea
where $\zeta_{m_3}$ is the third-largest $\bounds_k$ in the set $\{\bounds_k = |s_k - \hats| + \alpha_k : k \in \setCon \}$.


{\bf Rotation Error Bound}. Similar to the previous case, we use $\setCon$ to denote the 
consensus set of \TIM at the optimal \TLS estimate. We also use $\setTrueIn$ to denote the set of \emph{true inliers} and use $\setFakeIn$ to denote the set \emph{false inliers}, \ie outliers that \TLS selected as inliers. We have the following equations for the true inliers and false inliers:
\bea
\label{eq:rotTrueFakeModel}
\begin{cases}
\TIMb_i = \sgt\MRgt\TIMa_i + \vepsilon_i = \hats\hatMR\TIMa_i + \vphi_i & \forall i \in \setTrueIn \\
\TIMb_j = \hats\hatMR\TIMa_j + \vphi_j & \forall j \in \setFakeIn
\end{cases},
\eea
where $\vphi_i,\vphi_j$ are the residuals that can be computed and $\vepsilon_i$ is the unknown noise term. From the first equation of eq.~\eqref{eq:rotTrueFakeModel}, we have:
\bea
(\hats\hatMR - \sgt\MRgt)\TIMa_i = \vepsilon_i - \vphi_i \Leftrightarrow (\hats\hatMR - \sgt\MRgt)\frac{\TIMa_i}{ \| \TIMa_i \|} = \frac{\vepsilon_i - \vphi_i}{ \| \TIMa_i \|} \\ 
\Rightarrow \| (\hats\hatMR - \sgt\MRgt) \vu_i \| \leq \frac{\| \vphi_i \|}{\| \TIMa_i \|} + \alpha_i \doteq \boundR_i, \forall i \in \setTrueIn, \label{eq:rotationBoundEachTIM}
\eea
where $\vu_i \doteq \frac{\TIMa_i}{ \| \TIMa_i \|}$ and we have defined a new quantity $\boundR_i \doteq \| \vphi_i \|/\| \TIMa_i \| + \alpha_i$ that can be computed from the \TLS solution. Since eq.~\eqref{eq:rotationBoundEachTIM} holds for every \TIM in $\setTrueIn$, we can aggregate all inequalities:
\bea \label{eq:upperboundFrobenius}
\| (\hats\hatMR - \sgt\MRgt) \MA_{\setTrueIn} \|_F^2 \leq \sum_{i\in\setTrueIn} (\boundR_i)^2,
\eea
where $\MA_{\setTrueIn} = [\vu_{i_1},\dots,\vu_{i_{\sizeTrueIn}}] \in \Real{3\times\sizeTrueIn},i_1,\dots,i_{\sizeTrueIn} \in \setTrueIn,$ assembles all normalized \TIMs $\vu_i$ in the set $\setTrueIn$.
Now we want to get a lower bound on $\| (\hats\hatMR - \sgt\MRgt) \MA_{\setTrueIn} \|_F^2$ so that we can upper-bound the rotation error. 

Towards this goal, we first analyze the singular values of $(\hats\hatMR - \sgt\MRgt)$.
Let the three singular values of $(\hats\hatMR - \sgt\MRgt)$ be $\sigma_l,l=1,2,3$, then we know that $\sigma_l^2,l=1,2,3,$ are equal to the eigenvalues of the matrix $\MB = (\hats\hatMR - \sgt\MRgt)\tran(\hats\hatMR - \sgt\MRgt)$. Expanding $\MB$, we get:
\bea \label{eq:eigenB}
\MB = (\hats\hatMR\tran - \sgt(\MRgt)\tran)(\hats\hatMR - \sgt\MRgt) = ( \hats^2 + (\sgt)^2 ) \eye_3 - \hats\sgt\left( \hatMR\tran\MRgt + (\MRgt)\tran\hatMR \right) = ( \hats^2 + (\sgt)^2 ) \eye_3 - \hats\sgt \left( \MR_s + \MR_s\tran \right),
\eea
where we have denoted $\MR_s = \hatMR\tran\MRgt \in \SOthree$, which is the composition of two rotation matrices. Let the axis-angle representation of $\MR_s$ be $(\vPPhi_s, \theta_R)$, then the axis-angle representation for $\MR_s$ is simply $(\vPPhi_s, -\theta_R)$. According to the Rodrigues’ rotation formula, we can write $\MR_s$ and $\MR_s\tran$ as:
\bea \label{eq:RodriguesFormula}
\begin{cases} 
\MR_s  = \cos(\theta_R)\eye_3 + \sin(\theta_R)\vSkew{\vPPhi_s} + (1-\cos(\theta_R))\vPPhi_s\vPPhi_s\tran \\
\MR_s\tran = \cos(-\theta_R)\eye_3 + \sin(-\theta_R)\vSkew{\vPPhi_s} + (1-\cos(-\theta_R))\vPPhi_s\vPPhi_s\tran = \cos(\theta_R)\eye_3  - \sin(\theta_R)\vSkew{\vPPhi_s} + (1-\cos(\theta_R))\vPPhi_s\vPPhi_s\tran
\end{cases}.
\eea
Inserting eq.~\eqref{eq:RodriguesFormula} back into eq.~\eqref{eq:eigenB}, we have:
\bea
\MB = ( \hats^2 + (\sgt)^2 ) \eye_3 - \hats\sgt \left( 2\cos(\theta_R) \eye_3 + 2(1-\cos(\theta_R)) \vPPhi_s\vPPhi_s\tran \right),
\eea
from which we can see that the three eigenvectors of $\MB$ are $\vPPhi_s$ and $\vPPhi_1^\perp$ and $\vPPhi_2^\perp$, where we pick $\vPPhi_1^\perp$ and $\vPPhi_2^\perp$ to be 2 unit orthogonal vectors in the plane that is orthogonal to $\vPPhi_s$:
\bea
\begin{cases}
\MB \vPPhi_s = \left( (\hats^2 + (\sgt)^2) - 2\hats\sgt \right) \vPPhi_s = (\hats - \sgt)^2 \vPPhi_s \\
\MB \vPPhi_l^\perp = \left( \hats^2 + (\sgt)^2 - 2\hats\sgt\cos(\theta_R) \right) \vPPhi_l^\perp, \forall l=1,2
\end{cases}.
\eea
Therefore, we have $\sigma_1=\sigma_2 = \hats^2 + (\sgt)^2 - 2\hats\sgt\cos(\theta_R) = (\hats-\sgt)^2 + 2\hats\sgt(1-\cos(\theta_R))$ and $\sigma_3 = (\hats-\sgt)^2$. Let the SVD of $(\hats\hatMR - \sgt\MRgt)$ be $(\hats\hatMR - \sgt\MRgt) = \MU_{\Delta}\MS_{\Delta}\MV_{\Delta}\tran$, we lower-bound $\| (\hats\hatMR - \sgt\MRgt) \MA_{\setTrueIn} \|_F^2$:
\bea
\| (\hats\hatMR - \sgt\MRgt) \MA_{\setTrueIn} \|_F^2 = \| \MU_{\Delta}\MS_{\Delta}\MV_{\Delta}\tran \MA_{\setTrueIn} \|_F^2 = \| \MS_{\Delta}\MV_{\Delta}\tran \MA_{\setTrueIn} \|_F^2 = \left\Vert \bmat{ccc} \sigma_1 & 0 & 0 \\ 0 & \sigma_2 & 0 \\ 0 & 0 & \sigma_3 \emat \bmat{c} (\MV_{\Delta}\tran\MA_{\setTrueIn})_1  \\ (\MV_{\Delta}\tran\MA_{\setTrueIn})_2 \\ (\MV_{\Delta}\tran\MA_{\setTrueIn})_3 \emat \right\Vert_F^2 \\
= (\hats-\sgt)^2 \left( \| (\MV_{\Delta}\tran\MA_{\setTrueIn})_1 \|^2 + \| (\MV_{\Delta}\tran\MA_{\setTrueIn})_2 \|^2 + \| (\MV_{\Delta}\tran\MA_{\setTrueIn})_3 \|^2 \right) + 2\hats\sgt(1-\cos(\theta_R))\left( \| (\MV_{\Delta}\tran\MA_{\setTrueIn})_1 \|^2 + \| (\MV_{\Delta}\tran\MA_{\setTrueIn})_2 \|^2 \right) \\
= (\hats-\sgt)^2 \| \MA_{\setTrueIn} \|_F^2 + 2\hats\sgt(1-\cos(\theta_R)) \left( \| \MA_{\setTrueIn} \|_F^2 -  \| (\MV_{\Delta}\tran\MA_{\setTrueIn})_3 \|^2  \right) \\
\geq (\hats-\sgt)^2 \| \MA_{\setTrueIn} \|_F^2 + 2\hats\sgt(1-\cos(\theta_R)) \left( \sigma_1^2(\MA_{\setTrueIn}) + \sigma_2^2(\MA_{\setTrueIn}) \right), \label{eq:lowerboundFrobenius}
\eea 
where we have used $\|\MV_{\Delta}\MA_{\setTrueIn} \|_F^2 = \| \MA_{\setTrueIn} \|_F^2 = \sigma_1^2(\MA_{\setTrueIn}) + \sigma_2^2(\MA_{\setTrueIn}) + \sigma_3^2(\MA_{\setTrueIn})$, with $\sigma_1(\MA_{\setTrueIn}) \leq \sigma_2(\MA_{\setTrueIn}) \leq \sigma_3(\MA_{\setTrueIn})$ being the three singular values of $\MA_{\setTrueIn}$ in ascending order. The last inequality in eq.~\eqref{eq:lowerboundFrobenius} follows from:
\bea
\| (\MV_{\Delta}\tran\MA_{\setTrueIn})_3 \|^2  = \left\|\bmat{ccc} 0 & 0 & 0\\0 & 0 & 0 \\ 0 & 0 & 1\emat \MV_{\Delta}\tran \MA_{\setTrueIn} \right\|_F^2 \leq \sigma_3^2(\MA_{\setTrueIn}) \left\|\bmat{ccc} 0 & 0 & 0\\0 & 0 & 0 \\ 0 & 0 & 1\emat \MV_{\Delta}\tran \right\|_F^2 = \sigma_3^2(\MA_{\setTrueIn}).
\eea
Now we combine eq.~\eqref{eq:lowerboundFrobenius} and~\eqref{eq:upperboundFrobenius} to bound $1-\cos(\theta_R)$:
\bea
(\hats-\sgt)^2 \| \MA_{\setTrueIn} \|_F^2 + 2\hats\sgt(1-\cos(\theta_R)) \left( \sigma_1^2(\MA_{\setTrueIn}) + \sigma_2^2(\MA_{\setTrueIn}) \right) \leq \sum_{i\in\setTrueIn} (\boundR_i)^2 \\
\Leftrightarrow \sgt(1-\cos(\theta_R)) \leq \frac{ \sum_{i\in\setTrueIn} (\boundR_i)^2 - (\hats-\sgt)^2 \| \MA_{\setTrueIn} \|_F^2}{2\hats \left( \sigma_1^2(\MA_{\setTrueIn}) + \sigma_2^2(\MA_{\setTrueIn})  \right)}, \label{eq:RotationBoundTight}
\eea
where we note that (i) the existence of at least 3 non-collinear inliers guarantees that $\rank{\MA_{\setTrueIn}} \geq 2$, which means $\sigma_1^2(\MA_{\setTrueIn}) + \sigma_2^2(\MA_{\setTrueIn}) > 0$ and the upper bound is well-defined; (ii) since $(\vPPhi_s,\theta_R)$ is the axis-angle representation for $\hatMR\tran\MRgt$, $\theta_R$ is the angular (geodesic) distance between $\hatMR$ and $\MRgt$; (iii) the rotation bound is dependent on the scale bound $|\hats-\sgt|$~\eqref{eq:scaleBoundTight}.

Again, when the true inlier set $\setTrueIn$ is known, one can directly compute the rotation bound using eq.~\eqref{eq:RotationBoundTight}; when the true inlier set $\setTrueIn$ is unknown, one has to enumerate over all possible subsets of 3 \TIMs (formed by a triplet of $(i,j,k)$ correspondences) in $\setCon$ to compute the worst-case bound.


{\bf Translation Error Bound}. 
Let $\vtgt = [\tgt_1,\tgt_2,\tgt_3]\tran$ and $\hatvt = [\hatt_1,\hatt_2,\hatt_3]\tran$ be the ground-truth and estimated translation. Let $\setCon_l,l=1,2,3,$ be the set of measurements $(\va_k,\vb_k)$ that pass the $l$-th component-wise adaptive voting, \ie~$|[\vb_k - \hats\hatMR\va_k - \hatvt]_l| \leq \beta_k, \forall k \in \setCon_l,l=1,2,3$. Again, we let $\setTrueIn_l$ and $\setFakeIn_l$ be the set of true and false inliers within $\setCon_l,l=1,2,3$. 
From condition (ii), we have that $\setTrueIn_1 = \setTrueIn_2 = \setTrueIn_3 = \setTrueIn$, which says that the true inliers pass all three component-wise adaptive voting. However, $\setFakeIn_l,l=1,2,3,$ could potentially be different. Then we write the following relations:
\bea
\label{eq:tranTrueFakeModel}
\begin{cases}
[\vb_i]_l  = [\hats\hatMR\va_i]_l + \hatt_l + [\vphi_i]_l = [\sgt\MRgt\va_i]_l + \tgt_l + [\vepsilon_i]_l & \forall i \in \setTrueIn, \;\; l = 1,2,3 \\
[\vb_j]_l = [\hats\hatMR\va_j]_l + \hatt_l + [\vphi_j]_l & \forall j \in \setFakeIn_l,l=1,2,3 
\end{cases},
\eea
where $\vepsilon_i$ is the unknown noise that satisfies $\| \vepsilon_i \|\leq \beta_i$ and $[\vphi_i]_l$ and $[\vphi_j]_l$ are the residuals that can be computed from the \TLS estimate and satisfy $|[\vphi_i]_l| \leq \beta_i$ and $|[\vphi_j]_l|\leq\beta_j$. From the first equation in~\eqref{eq:tranTrueFakeModel}, we have that:
\bea
\hatt_l - \tgt_l = [\sgt\MR\gt\va_i]_l - [\hats\hatMR\va_i]_l + [\vepsilon_i]_l - [\vphi_i]_l \\
\Rightarrow | \hatt_l - \tgt_l | \leq |[\sgt\MR\gt\va_i]_l - [\hats\hatMR\va_i]_l| + [\vphi_i]_l + \beta_i = |[\sgt\MR\gt\va_i]_l - [\hats\hatMR\va_i]_l| + [\boundt_i]_l, \label{eq:translationBoundStep1}
\eea
where we have defined a new quantity $[\boundt_i]_l = [\vphi_i]_l + \beta_i$ that can be computed from the solution. Now we want to upper-bound $|[\sgt\MRgt\va_i]_l - [\hats\hatMR\va_i]_l|$, using a similar SVD-based approach as in eq.~\eqref{eq:lowerboundFrobenius} (but this time we upper-bound):
\bea
([\sgt\MRgt\va_i]_l - [\hats\hatMR\va_i]_l)^2 \leq \| (\hats\hatMR - \sgt\MRgt) \va_i \|^2 = \| \MU_{\Delta} \MS_{\Delta} \MV_{\Delta}\tran \va_i \|^2 = \| \MS_{\Delta} \MV_{\Delta}\tran \va_i \|^2 \\
= (\hats - \sgt)^2  \|\va_i \|^2 + 2\hats\sgt(1-\cos(\theta_R)) ( (\MV_{\Delta}\tran \va_i)_1^2 + (\MV_{\Delta}\tran \va_i)_2^2   ) \leq \left( (\hats - \sgt)^2  + 2\hats\sgt(1-\cos(\theta_R))  \right)\| \va_i \|^2  \label{eq:upperboundTranslationScaleRotationPart}
\eea
Inserting the inequality~\eqref{eq:upperboundTranslationScaleRotationPart} back into eq.~\eqref{eq:translationBoundStep1}, we get the following translation bound for each $i\in\setTrueIn$
\bea \label{eq:translationBoundTight}
|\hatt_l - \tgt_l| \leq \left( (\hats - \sgt)^2  + 2\hats\sgt(1-\cos(\theta_R))  \right)\| \va_i \|^2 + [\boundt_i]_l, \forall i\in\setTrueIn,\forall l=1,2,3,
\eea
there the final translation bound is:
\bea
|\hatt_l - \tgt_l| \leq \min \left\{ \left( (\hats - \sgt)^2  + 2\hats\sgt(1-\cos(\theta_R))  \right)\| \va_i \|^2 + [\boundt_i]_l:  i\in\setTrueIn\right\}, \forall l=1,2,3.
\eea
We note that the translation bound depends on both the scale bound~\eqref{eq:scaleBoundTight} and the rotation bound~\eqref{eq:RotationBoundTight}. When the true inlier set $\setTrueIn$ is known, one can directly compute the bound using eq.~\eqref{eq:translationBoundTight}; when the true inlier set $\setTrueIn$ is unknown, one has to enumerate over all possible subsets of 3 inliers in $\setCon$ to compute the worst-case translation bound.

\end{proof}

\end{document}